\documentclass[preprint,3p,times]{elsarticle}
\usepackage{amssymb}
\usepackage{amsthm}
\usepackage{amsmath}
\usepackage{amsfonts}
\RequirePackage{algorithm,algpseudocode}
\usepackage{enumerate}
\RequirePackage{mathrsfs,multirow,booktabs,subfigure,color,array,hhline,makecell}
\usepackage{url}
\RequirePackage[colorlinks,citecolor=blue,urlcolor=blue]{hyperref}
\usepackage{changes}

\newtheorem{thm}{Theorem}
\newtheorem{defin}{Definition}
\newtheorem{lem}{Lemma}
\newtheorem{assum}{Assumption}
\newtheorem{rem}{Remark}

\newtheorem{Ex}{Example}
\allowdisplaybreaks

\journal{}

\begin{document}
	
	\begin{frontmatter}

		\title{Community detection for weighted  bipartite networks}
		\author[label1]{Huan Qing}
        \ead{qinghuan@u.nus.edu & qinghuan@cumt.edu.cn & qinghuan07131995@163.com}
		\affiliation[label1]{organization={School of Mathematics,  China University of Mining and Technology},%Department and Organization
				city={Xuzhou},
				postcode={221116},
				state={Jiangsu},
				country={China}}		
			\author[label2]{Jingli Wang\corref{cor1}}
		\cortext[cor1]{Corresponding author}
			\ead{jlwang@nankai.edu.cn}
			\affiliation[label2]{organization={School of Statistics and Data Science,  KLMDASR, LEBPS, and LPMC, Nankai University},
					city={Tianjin},
					postcode={300071},
					state={Tianjin},
					country={China}}

\begin{abstract}
The bipartite network appears in various areas, such as biology, sociology, physiology, and computer science. \cite{rohe2016co} proposed Stochastic co-Blockmodel (ScBM) as a tool for detecting community structure of binary bipartite graph data in network studies. However, ScBM completely ignores edge weight and is unable to explain the block structure of a weighted bipartite network. Here, to model a weighted bipartite network, we introduce a Bipartite Distribution-Free model by releasing ScBM's distribution restriction. We also build an extension of the proposed model by considering the variation of node degree. Our models do not require a specific distribution on generating elements of the adjacency matrix but only a block structure on the expected adjacency matrix. Spectral algorithms with theoretical guarantees on the consistent estimation of node labels are presented to identify communities. Our proposed methods are illustrated by simulated and empirical examples.
\end{abstract}

\begin{keyword}
Community detection, weighted bipartite network, spectral clustering, distribution-free model.
\end{keyword}

\end{frontmatter}
%\noindent
%{\it MSC:} 62H30; 91C20; 62P25

\section{Introduction}
In recent years, network science has received a lot of attention \citep{SBM,ng2001spectral, newman2003structure, Goldenberg2009, bu2013fast, zhu2020community, su2022comprehensive}. In real life, there exist many directed weighted networks such as rating networks, trust/distrust networks, dominance networks, food chain networks, telephone networks, and email networks. Because directed network naturally generalizes to
bipartite network \citep{rohe2016co}, this paper focuses on the bipartite network, a kind of popular complex network in which nodes are taken as two classes: one class is for sending nodes; another is for receiving nodes \citep{zhou2007bipartite}. The numbers of nodes for these two classes are not always equal. And edges only between the two classes instead of inside the class, i.e., there are edges from sending nodes to receiving nodes, but no edges from sending nodes to sending nodes or receiving nodes to receiving nodes. Generally, if there is a directed edge from sending node $a$ to sending node $b$, we can add the sending node $b$ to the receiving class and ignore the connection between sending nodes. Thus in a bipartite network, the edges have direction naturally, and sending nodes and receiving nodes are allowed to be the same in this paper. Since edge weight contains meaningful information and can improve the understanding of network structure \citep{newman2004analysis,barrat2004the}, this paper focuses on weighted bipartite networks.

Detecting communities plays a key role in understanding the structure of networks in various areas, including but not limited to computer science, social science, physics, biology, and statistics \citep{mcpherson2001birds,duch2005community,fortunato2010community, papadopoulos2012community}.
A large number of methods have been developed for community detection in recent years.  These methods can be categorized as heuristic methods \citep{girvan2002community, pons2005}, modularity based methods \citep{newman2006modularity,bu2013fast,zhu2020community,yuan2021community,boroujeni2022role},  model-based methods \citep{SBM,rohe2011spectral,RSC,joseph2016impact,gao2017achieving,CMM}, deep learning methods \citep{zhao2021community,su2022comprehensive}, and others \citep{2020Community,huang2021higher,reihanian2023enhanced,zhou2023community,duong2023deep}. In this paper, we focus on model-based methods. For undirected unweighted networks, the stochastic block model (SBM) \citep{SBM} is mathematically simple and relatively easy to analyze, and it is deemed as a standard framework \citep{PJBAC}. SBM assumes nodes within each community have the same expected degrees. However, this assumption is too restrictive to implement, and in many natural networks, the degrees follow approximately a power-law distribution \citep{Kolaczyk2009, Goldenberg2009, SCORE}. Therefore, the degree-corrected stochastic block model (DCSBM) \citep{DCSBM} is developed and it allows the degree of nodes to vary among different communities. By considering the degree heterogeneity of nodes, many methods are developed based on DCSBM \citep{SCORE, RSC,gao2018community, OCCAM, SLIM}. Thorough surveys of SBM's recent developments literature
have been provided by \cite{abbe2017community}, among others. Though SBM and DCSBM are popular models in network science, they can not model weighted networks, while edge weight can improve community detection \cite{newman2004analysis,barrat2004the}. For undirected weighted networks, some Weighted Stochastic Blockmodels (WSBM) have been developed, to name a few, \cite{aicher2015learning,ahn2018hypergraph, palowitch2017significance,peixoto2018nonparametric,xu2020optimal,ng2021weighted}.

However, the above-mentioned models and approaches ignore edge directionality. While if these methods are directly applied to weighted bipartite networks, they may fail to capture the asymmetric relationships. Then a natural question arises if we could extend these methods to directed networks, what shall we do, and do their theoretical properties still hold? There are a few works that are proposed for directed unweighted networks. For example, \cite{rohe2016co}  identified the communities in bipartite unweighted networks by extending SBM to Stochastic co-Blockmodel (ScBM for short) and DCSBM to DCScBM, and they also studied the co-clustering frameworks based on regularized Laplacian matrix. \cite{wang2020spectral} constructed the D-SCORE algorithm which is an extension of SCORE algorithm \citep{SCORE} to detect directed unweighted networks under a sub-model of DCScBM. \cite{zhou2019analysis} studied the spectral clustering algorithms designed by a data-driven regularization of the adjacency matrix under ScBM for unweighted bipartite networks. It should be noted that ScBM and DCScBM require that elements of the adjacency matrix are generated from Bernoulli distribution, and this is the reason that these two models are applied for unweighted bipartite networks.

However, all aforementioned works ignore either edge direction or edge weight. In this paper, we mainly focus on building models for weighted bipartite networks and designing spectral clustering methods with theoretical guarantees of estimation consistency to fit our models.  This paper makes contributions to the following four aspects:

(1) Model. (i) We propose a Bipartite Distribution-Free model (BiDFM) for weighted bipartite networks. BiDFM is built by realizing ScBM's distribution restriction such that elements of the adjacency matrix can be generated from any distribution as long as the expected adjacency matrix satisfies a block structure.  In detail, ScBM requires that elements of the adjacency matrix should be generated from the Bernoulli distribution, for comparison, our BiDFM allows elements of the adjacency matrix to be generated from any distribution, such as Bernoulli, Normal, Binomial, Poisson, Uniform, and Gamma. Especially, when it follows a Bernoulli distribution, our BiDFM degenerates to ScBM. (ii) We also propose a degree-corrected version of BiDFM by considering the node heterogeneity, called the Bipartite Degree-Corrected Distribution-Free model (BiDCDFM).

(2) Algorithm. We develop two community detection methods to fit BiDFM and BiDCDFM. Under BiDFM, we apply the singular value decomposition (SVD) on the adjacency matrix to obtain the leading left and right singular vectors with unit-norm, then form such vectors into two matrices, and finally, the K-means algorithm is used on these two matrices to get clusters. Naturally, by normalizing singular vector matrices, we construct the degree-corrected version of the proposed method under BiDCDFM.

 (3) Theoretical property. We build theoretical frameworks for the two proposed spectral clustering methods. With some mild assumptions, we obtain the theoretical bounds of estimation errors, which guarantee the proposed spectral clustering methods can recover both row node and column node memberships of all but a vanishing fraction of the nodes. Furthermore, several examples are given to illustrate the theoretical bounds for error rates when the adjacency matrix is generated from some specific distributions.  These examples guarantee the generality of our models and methods.

(4) Empirical validation. In Section \ref{NumericalStudy}, we first investigate the performances of the proposed methods for some specific distributions, and results on simulated weighted bipartite networks support our theoretical results. We then apply our methods to some real-world networks  and detect latent community information for these networks.

The following notations will be used throughout the paper: $\|\cdot\|_{F}$ denotes the Frobenius norm. $\|\cdot\|$ for a matrix denotes the spectral norm. For convenience, when we say ``leading eigenvectors'' or ``leading eigenvalues'', we are comparing the \emph{magnitudes} of the eigenvalues. For any matrix or vector $x$, $x'$ denotes the transpose of $x$. Let $\mathbb{M}_{m, K}$ be the collection of all $m\times K$ matrices where each row has only one $1$ and all others $0 $. Let $\lambda_{k}(X)$ denote the $k$-th leading eigenvalue of matrix $X$, and $\sigma_{k}(X)$ denote the $k$-th largest singular value of matrix $X$. Let $I_{m}$ be the $m\times m$ identity matrix.
\section{Bipartite Distribution-Free model}\label{sec2}
In this section, we first propose a Bipartite Distribution-Free model to generate weighted bipartite networks by releasing the distribution assumption on the adjacency matrix. Then we build a spectral clustering method to detect communities for weighted bipartite networks generated from this model. Finally, we show the theoretical consistency of the proposed method under the model. Furthermore, several examples are given to illustrate the theoretical results when adjacency matrices are generated from different distributions.
\subsection{Model}
Consider a weighted bipartite network $\mathcal{N}$ with $n_{r}$ row nodes and $n_{c}$ column nodes such that all row nodes belong to $K_{r}$ row clusters and all column nodes belong to $K_{c}$ column clusters. To emphasize that row nodes may differ from column nodes, let $i_{r}$ be row node $i$, where $i_{r}=1,2,\ldots, n_{r}$, and $j_{c}$ be column node $j$, where $j_{c}=1,2,\ldots, n_{c}$. Note that when row nodes are the same as column nodes, $\mathcal{N}$ reduces to a weighted directed network. Let $\{\mathcal{C}_{r,1}, \mathcal{C}_{r,2}, \ldots, \mathcal{C}_{r,K_{r}}\}$ be the $K_{r}$ row clusters, and $\{\mathcal{C}_{c,1}, \mathcal{C}_{c,2}, \ldots, \mathcal{C}_{c,K_{c}}\}$ be the $K_{c}$ column clusters. Let $A\in\mathbb{R}^{n_{r}\times n_{c}}$ be an asymmetric adjacency matrix of $\mathcal{N}$ such that $A(i_{r},j_{c})$ denotes the weight from row node $i_{r}$ to column node $j_{c}$ for $i_{r}=1,2,\ldots,n_{r}, j_{c}=1,2,\ldots, n_{c}$. The element $A(i_{r},j_{c})$ can be positive or negative to show the strength of weight. For example, in a man$\&$woman friendship network, let row nodes denote men, column nodes represent women, and $A(i_{r},j_{c})$ be the weight of how man $i_{r}$ rates his friendship with woman $j_{c}$. We see that a positive value of $A(i_{r},j_{c})$ means that man $i_{r}$ treats woman $j_{c}$ as his friend while a negative value of $A(i_{r},j_{c})$ means a hostile relation; in a protein function network, a positive value of $A(i_{r}, j_{c})$ means that protein $i_{r}$ promotes the synthesis of protein $j_{c}$ while a negative value of $A(i_{r}, j_{c})$ means inhibition. For a bipartite weighted network, all entries of $A$ can be any finite real numbers, instead of simply 0 and 1 as adjacency matrix for unweighted networks. To model such bipartite weighted network $\mathcal{N}$, we propose a Bipartite Distribution-Free model and its degree-corrected version in this paper. A Bipartite Distribution-Free model is parameterized by a triple of matrices $(Z_{r}, Z_{c}, P)$ and a parameter $\rho$, where $Z_{r}\in \mathbb{M}_{n_{r}, K_{r}}$ is the row membership matrix, $Z_{c}\in \mathbb{M}_{n_{c},K_{c}}$ is the column membership matrix, $P\in \mathbb{R}^{K_{r}\times K_{c}}$ is an asymmetric matrix with full rank, and $\rho>0$. For convenience, let $g_{i_{r}}$ ($g_{i_{r}}\in\{1,2, \ldots, K_{r}\}$) be community label for row node $i_{r}$. Then $Z_{r}$ can be written as
\begin{align*}
Z_{r}(i_{r},k)=\begin{cases}
1, & \mbox{when~} k=g_{i_{r}}\\
0, & \mbox{otherwise}
\end{cases},
\end{align*}
i.e., the $i_{r}$-th row and $g_{i_{r}}$-th column of $Z_{r}$ is 1  and $0$ otherwise.
Similar notations hold for column nodes. Throughout the paper, each row (column) node belongs to exactly only one row (column) cluster, and $K_{r}, K_{c}$ are known. And we assume that each row (column) cluster has at least one node, i.e., $\mathrm{rank}(Z_{r})=K_{r}, \mathrm{rank}(Z_{c})=K_{c}$. Now, we are ready to present the Bipartite Distribution-Free model.
\begin{defin}\label{BiDFM}
Let $A\in\mathbb{R}^{n_{r}\times n_{c}}$ be the adjacency matrix of bipartite weighted network $\mathcal{N}$. Let $Z_{r}\in\mathbb{M}_{n_{r}, K_{r}}, Z_{c}\in\mathbb{M}_{n_{c}, K_{c}}$ and $P\in\mathbb{R}^{K_{r}\times K_{c}}$ where $\mathrm{max}_{k, l}|P(k,l)|=1$. $P$ is full rank, i.e., $\mathrm{rank}(P)=\mathrm{min}(K_{r}, K_{c})$. Let $\rho>0$ and call it the sparsity parameter.  For $i_{r}=1,2,\ldots,n_{r}, j_{c}=1,2,\ldots,n_{c}$, Bipartite Distribution-Free model (BiDFM) assumes that for arbitrary distribution $\mathcal{F}$, $A(i_{r},j_{c})$ are independent random variables generated according to the distribution $\mathcal{F}$  and satisfy
\begin{align}\label{AFOmega}
\mathbb{E}[A(i_{r},j_{c})]=\Omega(i_{r},j_{c}), \mathrm{where~}\Omega:=\rho Z_{r}PZ'_{c}.
\end{align}
\end{defin}
By the above definition, we know that BiDFM is formed by the four model parameters $Z_{r}, Z_{c}, P,\rho$, so we denote it by $BiDFM(Z_{r}, Z_{c}, P,\rho)$. By the fact that we do not constrain $\mathcal{F}$ to be any specific distribution, Eq (\ref{AFOmega}) means that BiDFM only has a restriction on the expectation of adjacency matrix but no requirement on the specific distribution of $A$. When $\mathcal{F}$ is Bernoulli distribution, our BiDFM degenerates to ScBM \citep{rohe2016co} for unweighted directed networks, i.e., ScBM is a sub-model of our BiDFM. Meanwhile, for different distributions $\mathcal{F}$, the range of $\rho$ can be different. For detail, see Examples \ref{Bernoulli}-\ref{Signed}.

The goal of community recovery for  weighted bipartite network is to recover the row and column membership matrices $Z_{r}$ and $Z_{c}$ up to permutations. For any estimates $\hat{Z}_{r}$ and $\hat{Z}_{c}$, we use the performance criterion defined in \cite{joseph2016impact} to measure estimation error. It is introduced as follows:

Let $\mathcal{\hat{\mathcal{C}}}_{r}=\{\mathcal{\hat{\mathcal{C}}}_{r,1}, \mathcal{\hat{\mathcal{C}}}_{r,2}, \ldots, \mathcal{\hat{\mathcal{C}}}_{r,K_{r}}\}$ be the partition of row nodes $\{1,2,\ldots, n_{r}\}$ obtained from $\hat{Z}_{r}$ (i.e., if $\hat{Z}_{r}(i_{r},:)=\hat{Z}_{r}(\bar{i}_{r},:)$ for $1\leq i_{r}\neq \bar{i}_{r}\leq n_{r}$, then $i_{r}$ and $\bar{i}_{r}$ are in the same partition.). The criterion for row clustering error associated with $\mathcal{\hat{\mathcal{C}}}_{r,1}, \mathcal{\hat{\mathcal{C}}}_{r,2}, \ldots, \mathcal{\hat{\mathcal{C}}}_{r,K_{r}}$ is defined as
\begin{align*}
\hat{f}_{r}=\mathrm{min}_{\pi\in S_{K_{r}}}\mathrm{max}_{1\leq k\leq K_{r}}(|\mathcal{C}_{r,k}\cap \mathcal{\hat{\mathcal{C}}}^{c}_{r,\pi(k)}|+|\mathcal{C}^{c}_{r,k}\cap \mathcal{\hat{\mathcal{C}}}_{r,\pi(k)}|)/n_{r,k},
\end{align*}
where $n_{r,k}$ is the size of the $k$-th row community,  $S_{K_{r}}$ is the set of all permutations of $\{1,2,\ldots, K_{r}\}$, $\mathcal{\hat{\mathcal{C}}}^{c}_{r,\pi(k)}$ and $\mathcal{C}^{c}_{r,k}$ are the complementary sets. As stated in \cite{joseph2016impact}, $\hat{f}_{r}$ measures the maximum proportion of row nodes in the symmetric difference of $\mathcal{C}_{r,k}$ and $\mathcal{\hat{\mathcal{C}}}_{r,\pi(k)}$. Similarly, we can define $\mathcal{\hat{\mathcal{C}}}_{c,\pi(k)}$ and $\hat{f}_{c}$ for column nodes.
%%%------------------------------------------------------------
\subsection{Algorithm}\label{alg1}
Spectral clustering is a classical method for community recovery \citep{SCORE,lei2015consistency,joseph2016impact,wang2020spectral}, while it is not trivial to apply it to bipartite networks. Under BiDFM, the heuristic of spectral clustering is to group the compact singular value decomposition of the expectation matrix of $A$. Without loss of generality, we assume $K_{r}\leq K_{c}$ in this paper. Thus, we have $\mathrm{rank}(P)=\mathrm{min}(K_{r},K_{c})=K_{r}$, and $\mathrm{rank}(\Omega)=\mathrm{min}(K_{r},K_{c})=K_{r}$, i.e., $\Omega$ has $K_{r}$ nonzero singular values. Let $\Omega=U_{r}\Lambda U'_{c}$ be the compact singular value decomposition of $\Omega$ such that $U_{r}\in\mathbb{R}^{n_{r}\times K_{r}}, U_{c}\in \mathbb{R}^{n_{c}\times K_{r}}$, diagonal matrix $\Lambda\in\mathbb{R}^{K_{r}\times K_{r}}$, $U'_{r}U_{r}=I_{K_{r}}$ and $U'_{c}U_{c}=I_{K_{r}}$. Then we can have the following lemma which shows that $U_{r}$ has $K_{r}$ distinct rows and $U_{c}$ has $K_{c}$ distinct rows, and two row nodes are in the same row cluster if and only if their corresponding rows in $U_{r}$ are the same, and so are the column nodes.
\begin{lem}\label{UrUcBiDFM}
Under $BiDFM(Z_{r},Z_{c},P,\rho)$, let $\Omega=U_{r}\Lambda U'_{c}$ be the compact singular value decomposition of $\Omega$. Then for row nodes, we have $U_{r}=Z_{r}X_{r}$ where $X_{r}\in \mathbb{R}^{K_{r}\times K_{r}}$ and $\|X_{r}(k,:)-X_{r}(l,:)\|_{F}=(n_{r,k}^{-1}+n_{r,l}^{-1})^{1/2}$ for all $1\leq k<l\leq K_{r}$, where $n_{r,k}$ is the number of nodes for $k$-th row cluster. For column nodes, we have $U_{c}=Z_{c}X_{c}$ where $X_{c}\in \mathbb{R}^{K_{c}\times K_{r}}$. Meanwhile,  $U_{r}(i_{r},:)=U_{r}(\bar{i}_{r},:)$ when $g_{i_{r}}=g_{\bar{i}_{r}}$ for any two distinct row nodes $i_{r}, \bar{i}_{r}$, and $U_{c}(j_{c},:)=U_{c}(\bar{j}_{c},:)$ when $g_{j_{c}}=g_{\bar{j}_{c}}$ for any two distinct column nodes $j_{c}, \bar{j}_{c}$. Furthermore, when $K_{c}=K_{r}$, we have $\|X_{c}(k,:)-X_{c}(l,:)\|_{F}=(n_{c,k}^{-1}+n_{c,l}^{-1})^{1/2}$ for all $1\leq k<l\leq K_{c}$.
\end{lem}
This lemma is consistent with the Lemma 2 in \cite{guo2020randomized} which is built based on the ScBM model.
Based on Lemma \ref{UrUcBiDFM}, once we estimate $U_{r}$ and $U_{c}$, we can find the estimation of membership matrices $Z_r$ and $Z_c$. Let $\tilde{A}=\hat{U}_{r}\hat{\Lambda}\hat{U}'_{c}$ be the $K_{r}$-dimensional singular value decomposition of $A$ corresponding to the $K_{r}$ largest singular values such that $\hat{U}_{r}\in\mathbb{R}^{n_{r}\times K_{r}}, \hat{U}_{c}\in \mathbb{R}^{n_{c}\times K_{r}}$, diagonal matrix $\hat{\Lambda}\in\mathbb{R}^{K_{r}\times K_{r}}$, $\hat{U}'_{r}\hat{U}_{r}=I_{K_{r}} $ and $\hat{U}'_{c}\hat{U}_{c}=I_{K_{r}}$. We see that $\hat{U}_{r}$ should have roughly $K_{r}$ distinct rows because they are slightly perturbed versions of the rows of $U_{r}$, and $\hat{U}_{c}$ should have roughly $K_{c}$ distinct rows. Therefore, applying a clustering algorithm on the rows of $\hat{U}_{r}$ (or $\hat{U}_{c}$) can return a good community partition. In this paper, we consider the K-means clustering, defined as
\begin{align}\label{Kmeans}
&(\hat{Z}_{r},\hat{X}_{r})=\mathrm{arg~}\mathrm{min}_{\bar{Z}_{r}\in \mathbb{M}_{n_{r},K_{r}}, \bar{X}_{r}\in \mathbb{R}^{K_{r}\times K_{r}}}\|\bar{Z}_{r}\bar{X}_{r}-\hat{U}_{r}\|^{2}_{F},\\\nonumber &(\hat{Z}_{c},\hat{X}_{c})=\mathrm{arg~}\mathrm{min}_{\bar{Z}_{c}\in \mathbb{M}_{n_{c},K_{c}}, \bar{X}_{c}\in \mathbb{R}^{K_{c}\times K_{r}}}\|\bar{Z}_{c}\bar{X}_{c}-\hat{U}_{c}\|^{2}_{F}.
\end{align}
The proposed algorithm is summarized in Algorithm \ref{algBiSC}, where BiSC is short for Bipartite Spectral Clustering. %Step 2 in BiSC is, applying K-means method to $\hat{U}_{r}$ with $K_{r}$ row clusters to obtain membership matrix $\hat{Z}_{r}$, and to $\hat{U}_{c}$ with $K_{c}$ column clusters to obtain membership matrix $\hat{Z}_{c}$.
Note that in the BiSC algorithm, if the input matrix is the population adjacency matrix $\Omega$ instead of $A$, by Lemma \ref{UrUcBiDFM}, we can exactly recover $Z_{r}$ and $Z_{c}$ up to permutation of cluster labels, and this guarantees the identifiability of BiDFM. %\deleted{Meanwhile, if $K_{r}\geq K_{c}$ in practice, just let $\hat{U}_{r}$ and $\hat{U}_{c}$ be the top $K_{c}$ left and right singular vectors of $A$ in BiSC algorithm.}
\begin{algorithm}
	\caption{BiSC}
	\label{algBiSC}
	\begin{algorithmic}[1]
		\Require Adjacency matrix $A\in\mathbb{R}^{n_{r}\times n_{c}}$, number of row clusters $K_{r}$, number of column clusters $K_{c}$, with $K_{r}\leq K_{c}$.
		\Ensure Row nodes membership matrix $\hat{Z}_{r}\in \mathbb{R}^{n_{r}\times K_{r}}$, and column nodes membership matrix $\hat{Z}_{c}\in \mathbb{R}^{n_{c}\times K_{c}}$.
		\State Calculate $\hat{U}_{r}\in\mathbb{R}^{n_{r}\times K_{r}}, \hat{U}_{c}\in\mathbb{R}^{n_{c}\times K_{r}}$ of $A$.
		\State Obtain $(\hat{Z}_{r}, \hat{X}_{r})$ and $(\hat{Z}_{c}, \hat{X}_{c})$ by K-means method, i.e., by Eq (\ref{Kmeans}).
		\State Output $\hat{Z}_{r}$ and $\hat{Z}_{c}$.
	\end{algorithmic}
\end{algorithm}
\subsection{The consistency of BiSC under BiDFM}\label{sec3}
In this section, we establish the performance guarantee for BiSC under BiDFM. For convenience, set $\tau=\mathrm{max}_{ i_{r}, j_{c}}|A(i_{r},j_{c})-\Omega(i_{r},j_{c})|$ and  $\gamma=\mathrm{max}_{ i_{r}, j_{c}}\frac{\mathrm{Var}(A(i_{r},j_{c}))}{\rho}$, where $\mathrm{Var}(A(i_{r},j_{c}))$ denotes the variance of $A(i_{r},j_{c})$ under distribution $\mathcal{F}$. We make the following assumption.
\begin{assum}\label{asump}
Assume	$\gamma\rho\geq \tau^{2}\mathrm{log}(n_{r}+n_{c})/ \mathrm{~max}(n_{r},n_{c})$.
\end{assum}
This assumption is common when showing the consistency for community detection, such as \cite{lei2015consistency,SCORE, mao2020estimating,guo2020randomized}. Based on the rectangular version of Bernstein inequality in \cite{tropp2012user}, we can obtain the bound for $\|A-\Omega\|$ under $BiDFM(Z_{r},Z_{c},P,\rho)$.
\begin{lem}\label{boundABiDFM}
	Under $BiDFM(Z_{r}, Z_{c}, P,\rho)$, suppose Assumption \ref{asump} holds, with probability at least $1-o((n_{r}+n_{c})^{-\alpha})$ for any $\alpha>0$, we have
	\begin{align*}
	\|A-\Omega\|\leq C_{\alpha}(\gamma\rho\mathrm{~max}(n_{r},n_{c})\mathrm{log}(n_{r}+n_{c}))^{1/2},
	\end{align*}
	where $C_{\alpha}$ is a positive constant and proportional to  $\alpha$.
\end{lem}
In fact, when $n_{r}=n_{c}=n$, the upper bound of $\|A-\Omega\|$ in Lemma \ref{boundABiDFM} is consistent with Corollary 6.5 in \cite{cai2015robust}.

Then we can get our main result, Theorem \ref{mainBiDFM}, which provides an upper bound on clustering errors of row nodes and column nodes in terms of several model parameters.
\begin{thm}\label{mainBiDFM}
	Under $BiDFM(Z_{r}, Z_{c}, P,\rho)$, when Assumption \ref{asump} holds, with probability at least $1-o((n_{r}+n_{c})^{-\alpha})$, we have
\begin{align*}
&\hat{f}_{r}=O(\gamma\frac{K^{2}_{r}n_{r,\mathrm{max}}}{n_{r,\mathrm{min}}}\frac{\mathrm{~max}(n_{r},n_{c})\mathrm{log}(n_{r}+n_{c})}{\sigma^{2}_{K_{r}}(P)\rho
n_{r,\mathrm{min}}n_{c,\mathrm{min}}}),
\hat{f}_{c}=O(\gamma\frac{K_{r}K_{c}}{\delta^{2}_{c}n_{c,\mathrm{min}}}\frac{\mathrm{~max}(n_{r},n_{c})\mathrm{log}(n_{r}+n_{c})}{\sigma^{2}_{K_{r}}(P)\rho
n_{r,\mathrm{min}}n_{c,\mathrm{min}}}),
\end{align*}
where $\delta_{c}=\mathrm{min}_{k\neq l}\|X_{c}(k,:)-X_{c}(l,:)\|_{F}$, $n_{r,\mathrm{max}}=\mathrm{max}_{1\leq k\leq K_r}\{n_{r,k}\}$, $n_{r,\mathrm{min}}, n_{c,\mathrm{max}}$ and $ n_{c,\mathrm{min}}$ are defined similarly.
\end{thm}
When $K_{r}=K_{c}$, Theorem \ref{mainBiDFM} can be simplified because $\delta_{c}\geq\sqrt{\frac{2}{n_{c,\mathrm{max}}}}$ by Lemma \ref{UrUcBiDFM}. From Theorem \ref{mainBiDFM}, we see that $n_{r,\mathrm{min}}, n_{c,\mathrm{min}}$, and $\sigma_{K_{r}}(P)$ have a negative influence on the detecting of both row and column communities. For $\rho$, its influence on BiSC's performance depends on distribution $\mathcal{F}$ since $\gamma$ may be related to $\rho$. For detail, see Examples \ref{Bernoulli}-\ref{Signed}.

\begin{Ex}\label{Bernoulli}
When $\mathcal{F}$ follows a \textbf{Bernoulli distribution}, then BiDFM reduces to ScBM \citep{rohe2016co} for bipartite unweighted networks in which all entries of $A$ are either 1 or 0, i.e.,  $A(i_{r},j_{c})\sim \mathrm{Bernoulli}(\Omega(i_{r},j_{c}))$. Bernoulli distribution requires that all entries of $P$ should be nonnegative. Then we have $\mathbb{E}[A(i_{r},j_{c})]=\Omega(i_{r},j_{c})$ satisfying Eq (\ref{AFOmega}), $\rho P$ denotes the probability matrix for this case, and $\mathbb{P}(A(i_{r},j_{c})=1)=\Omega(i_{r},j_{c})$. For Bernoulli distribution, $\rho$'s range is $(0,1]$. For this case, since $A(i_{r},j_{c})\in\{0,1\}$, $\Omega(i_{r},j_{c})=\rho P(g_{i_{r}},g_{j_{c}})\in[0,1]$, and $\frac{\mathrm{Var}(A(i_{r},j_{c}))}{\rho}=\frac{\Omega(i_{r},j_{c})(1-\Omega(i_{r},j_{c}))}{\rho}\leq\frac{\Omega(i_{r},j_{c})}{\rho}\leq 1$, $\tau\leq 1$ and $\gamma\leq 1$. Then, Assumption \ref{asump} equals to that $\rho\geq\frac{\mathrm{log}(n_{r}+n_{c})}{\mathrm{max}(n_{r},n_{c})}$. Especially, when $Z_{r}=Z_{c}$ such that BiDFM reduces to SBM for undirected unweighted network, Assumption \ref{asump} requires that $\rho$ should shrink slower than $\frac{\mathrm{log}(n)}{n}$, which is consistent with the sparsity requirement in Theorem 3.1 in \cite{lei2015consistency}, where we set $n_{r}=n_{c}=n$ for this case. Setting $\gamma$ as $1$ in Theorem \ref{mainBiDFM} obtains upper bounds of error rates of BiSC when $\mathcal{F}$ is Bernoulli distribution. Note that when setting $\gamma$ as 1 in error bounds, increasing $\rho$ decreases error rates.
\end{Ex}
\begin{Ex}\label{Normal}
When $\mathcal{F}$ is a \textbf{Normal distribution} such that $A(i_{r},j_{c})\sim \mathrm{Normal}(\Omega(i_{r},j_{c}),\sigma^{2}_{A})$ for bipartite weighted network in which all entries of $A$ are real values, where $\sigma^{2}_{A}$ is the variance of Normal distribution.  $P$'s elements are real values, i.e., $P$'s elements can be negative and $\rho$'s range is $(0,+\infty)$. Normal distribution's property gives that $\mathbb{E}[A(i_{r},j_{c})]=\Omega(i_{r},j_{c})$ satisfying Eq (\ref{AFOmega}), and $\mathrm{Var}(A(i_{r},j_{c}))=\sigma^{2}_{A}$. Therefore, $\tau$ is a finite number and $\gamma$ is $\frac{\sigma^{2}_{A}}{\rho}$. Setting  $\gamma$ as $\frac{\sigma^{2}_{A}}{\rho}$ in Theorem \ref{mainBiDFM} obtains error rates of BiSC when $\mathcal{F}$ is Normal distribution. Especially, decreasing $\sigma^{2}_{A}$ (or increasing $\rho$) decreases error rates. For a special case that $\sigma^{2}_{A}=0$, error rates are zeros since $A=\Omega$ when $\sigma^{2}_{A}=0$.
\end{Ex}
\begin{Ex}\label{Signed}
BiDFM can also generate a \textbf{bipartite signed network} in which all entries of $A$ are either $1$ or $-1$ by setting $\mathbb{P}(A(i_{r},j_{c})=1)=\frac{1+\Omega(i_{r},j_{c})}{2}$ and $\mathbb{P}(A(i_{r},j_{c})=-1)=\frac{1-\Omega(i_{r},j_{c})}{2}$. For this case, all entries of $P$ are real values and $\rho$ should be set in the interval $(0,1)$ because $\frac{1+\Omega(i_{r},j_{c})}{2}$ is a probability and we let $\mathrm{max}_{1\leq k\leq K_{r}, 1\leq l\leq K_{c}}|P(k,l)|=1$ in Definition \ref{BiDFM}. Thus, we have $\mathbb{E}[A(i_{r},j_{c})]=\Omega(i_{r},j_{c})$ satisfying Eq (\ref{AFOmega}), $\tau$ is no larger than 1, and $\mathrm{Var}(A(i_{r},j_{c}))=1-\Omega^{2}(i_{r},j_{c})\leq1$, i.e., $\gamma\leq \frac{1}{\rho}$ is finite. Setting $\gamma$ as $\frac{1}{\rho}$ in Theorem \ref{mainBiDFM}, we see that increasing $\rho$ decreases error rates.
\end{Ex}

Other choices of $\mathcal{F}$ are also possible as long as Eq (\ref{AFOmega}) holds for distribution $\mathcal{F}$. For example, $\mathcal{F}$ can be Binomial, Poisson, Uniform, Double exponential, Gamma, Laplace, and Geometric distributions in \url{http://www.stat.rice.edu/~dobelman/courses/texts/distributions.c&b.pdf}.

\section{Bipartite Degree-Corrected Distribution-Free model}\label{sec4}
In this section, we extend BiDFM by introducing degree heterogeneity, that is, allowing the degree changes within a community. And the extended model is named as Bipartite Degree-Corrected Distribution-Free model (BiDCDFM). We also build a spectral clustering method with a theoretical guarantee of estimation consistency to fit BiDCDFM.

\subsection{Model and algorithm}
To build a degree-corrected model, a set of tuning parameters are used to control the node degrees. Let $\theta_{r}$ be an $n_{r}\times 1$ degree vector for row nodes and $\theta_{r}(i_{r})>0$ be the $i_{r}$-th element of $\theta_{r}$. So are for $\theta_{c}$ and $\theta_{c}(j_{c})>0$.
Let $\Theta_{r}\in \mathbb{R}^{n_{r}\times n_{r}}$ be a diagonal matrix whose $i_{r}$-th diagonal entry is $\theta_{r}(i_{r})$, and let $\Theta_{c}\in \mathbb{R}^{n_{c}\times n_{c}}$ be a diagonal matrix whose $j_{c}$-th diagonal entry is $\theta_{c}(j_{c})$. Now we are ready to present the model.

\begin{defin}\label{BiDCDFM}
Let $Z_{r}, Z_{c}$ and $P$ satisfy conditions in Definition \ref{BiDFM}.  For $i_{r}=1,2,\ldots,n_{r}, j_{c}=1,2,\ldots,n_{c}$, Bipartite Degree-Corrected Distribution-Free model (BiDCDFM) assumes that for arbitrary distribution $\mathcal{F}$, $A(i_{r},j_{c})$ are independent random variables generated from the distribution $\mathcal{F}$ satisfying
\begin{align}\label{AFOmegaBiDCDFM}
\mathrm{E}[A(i_{r},j_{c})]=\Omega(i_{r},j_{c}), \mathrm{where~}\Omega:=\Theta_{r}Z_{r}PZ'_{c}\Theta_{c}.
\end{align}
\end{defin}

Let $\theta_{r,\mathrm{min}}=\mathrm{min}_{ i_{r}}\theta_{r}(i_{r}),\theta_{r,\mathrm{max}}=\mathrm{max}_{i_{r}}\theta_{r}(i_{r}), \theta_{c,\mathrm{min}}=\mathrm{min}_{ j_{c}}\theta_{c}(j_{c})$ and $\theta_{c,\mathrm{max}}=\mathrm{max}_{ j_{c}}\theta_{c}(j_{c})$. Since we consider bipartite weighted networks in this paper, $\theta_{r,\mathrm{max}}$ and $\theta_{c,\mathrm{max}}$ can be larger than 1. By setting $\Theta_{r}=\sqrt{\rho}I_{n_{r}}$ and $\Theta_{c}=\sqrt{\rho}I_{n_{c}}$, BiDCDFM exactly reduces to $BiDFM(Z_{r},Z_{c},P,\rho)$. When $\mathcal{F}$ is Bernoulli distribution, our BiDCDFM degenerates to DCScBM \citep{rohe2016co} for unweighted directed networks, i.e., DCScBM is a sub-model of our BiDCDFM.

By basic algebra, we see the rank of $\Omega$ is $K_{r}$ under BiDCDFM since $K_{r}\leq K_{c}$ is assumed. Without confusion, using the same notations as in Section \ref{sec2},  let $\Omega=U_{r}\Lambda U'_{c}$ be the compact singular value decomposition of $\Omega$. Let $U_{r,*}\in \mathbb{R}^{n_{r,K_{r}}}$ be the row-normalization version of $U_{r}$ such that $U_{r,*}(i_{r},:)={U_{r}(i_{r},:)}/{\|U_{r}(i_{r},:)\|_{F}}$ for $1\leq i_{r}\leq n_{r}$, and $U_{c,*}$ is defined similarly. Then applying the K-means algorithm on the rows of $U_{r,*}$ ($U_{c,*}$) can obtain clusters for row (column) nodes, which is guaranteed by the next lemma.

\begin{lem}\label{populationBiDCDFM}
Under $BiDCDFM(Z_{r}, Z_{c}, P, \Theta_{r}, \Theta_{c})$, $U_{r,*}=Z_{r}V_{r}$ where $V_{r}\in \mathbb{R}^{K_{r}\times K_{r}}$ and $\|V_{r}(k,:)-V_{r}(l,:)\|_{F}=2^{1/2}$ for all $1\leq k<l\leq K_{r}$. For column nodes,  $U_{c,*}=Z_{c}V_{c}$ where $V_{c}\in \mathbb{R}^{K_{c}\times K_{r}}$. Thus,  $U_{r,*}(i_{r},:)=U_{r,*}(\bar{i}_{r},:)$ when $g_{i_{r}}=g_{\bar{i}_{r}}$ for any two distinct row nodes $i_{r}, \bar{i}_{r}$, and $U_{c,*}(j_{c},:)=U_{c,*}(\bar{j}_{c},:)$ when $g_{j_{c}}=g_{\bar{j}_{c}}$ for any two distinct column nodes $j_{c}, \bar{j}_{c}$. Furthermore, when $K_{c}=K_{r}$, we have $\|V_{c}(k,:)-V_{c}(l,:)\|_{F}=2^{1/2}$ for all $1\leq k<l\leq K_{c}$.
\end{lem}
Let $\hat{U}_{r,*}$ ($\hat{U}_{c,*}$) be row-normalized version of $\hat{U}_{r}$ ($\hat{U}_{c}$). The row and column nodes membership matrices can be estimated by
\begin{align}\label{KmeansDC}
&(\hat{Z}_{r},\hat{V}_{r})=\mathrm{arg~}\mathrm{min}_{\bar{Z}_{r}, \bar{V}_{r}}\|\bar{Z}_{r}\bar{V}_{r}-\hat{U}_{r,*}\|^{2}_{F},\\\nonumber &(\hat{Z}_{c},\hat{V}_{c})=\mathrm{arg~}\mathrm{min}_{\bar{Z}_{c}, \bar{V}_{c}}\|\bar{Z}_{c}\bar{V}_{c}-\hat{U}_{c,*}\|^{2}_{F}.
\end{align}

The practical procedure is summarized in Algorithm \ref{algnBiSC}. This algorithm is called `nBiSC', where the `n' denotes normalized.
\begin{algorithm}
	\caption{nBiSC}
	\label{algnBiSC}
	\begin{algorithmic}[1]
		\Require Adjacency matrix $A\in\mathbb{R}^{n_{r}\times n_{c}}$, number of row clusters $K_{r}$, number of column clusters $K_{c}$, with $K_{r}\leq K_{c}$.
		\Ensure Row nodes membership matrix $\hat{Z}_{r}\in \mathbb{R}^{n_{r}\times K_{r}}$, and column nodes membership matrix $\hat{Z}_{c}\in \mathbb{R}^{n_{c}\times K_{c}}$.
		\State Calculate $\hat{U}_{r}$ and $\hat{U}_{c}$ as Algorithm \ref{algBiSC}. Then calculate $\hat{U}_{r,*}$ and $\hat{U}_{c,*}$.
		\State Obtain $(\hat{Z}_{r}, \hat{V}_{r})$ and $(\hat{Z}_{c}, \hat{V}_{c})$ by K-means method, i.e., by Eq (\ref{KmeansDC}).
		\State Output $\hat{Z}_{r}$ and $\hat{Z}_{c}$.
	\end{algorithmic}
\end{algorithm}

\subsection{The consistency of nBiSC under BiDCDFM}
In this section, we establish the performance guarantee for nBiSC. For convenience, set $\gamma_{*}=\mathrm{max}_{ i_{r}, j_{c}}\frac{\mathrm{Var}(A(i_{r},j_{c}))}{\theta_{r}(i_{r})\theta_{c}(j_{c})}$ under BiDCDFM. We make the following assumption
\begin{assum}\label{assump2}
Assume $\gamma_{*}\mathrm{max}(\theta_{r,\mathrm{max}}\|\theta_{c}\|_{1},\theta_{c,\mathrm{max}}\|\theta_{r}\|_{1})\geq\tau^{2}\mathrm{log}(n_{r}+n_{c})$.
\end{assum}
Assumption \ref{assump2} functions similar to Assumption \ref{asump}, and Assumption \ref{assump2} is consistent with Assumption \ref{asump} when BiDCDFM reduces to BiDFM by setting $\Theta_{r}=\sqrt{\rho}I_{n_{r}}$ and $\Theta_{c}=\sqrt{\rho}I_{n_{c}}$. Similar to                                                                     Lemma \ref{boundABiDFM}, next lemma bounds $\|A-\Omega\|$ under BiDCDFM.
\begin{lem}\label{boundABiDCDFM}
	Under $BiDCDFM(Z_{r},Z_{c},P,\Theta_{r},\Theta_{c})$, suppose Assumption \ref{assump2} holds, with probability at least $1-o((n_{r}+n_{c})^{-\alpha})$ for any $\alpha>0$, we have
	\begin{align*}
	\|A-\Omega\|\leq C_{\alpha}(\gamma_{*}{\mathrm{max}(\theta_{r,\mathrm{max}}\|\theta_{c}\|_{1},\theta_{c,\mathrm{max}}\|\theta_{r}\|_{1})\mathrm{log}(n_{r}+n_{c})})^{1/2}.
	\end{align*}
\end{lem}
When BiDCDFM degenerates to BiDFM, Lemma \ref{boundABiDCDFM} is consistent with Lemma \ref{boundABiDFM}. Now we are ready to bound $\hat{f}_{r}$ and $\hat{f}_{c}$ of the nBiSC algorithm.
\begin{thm}\label{mainBiDCDFM}
	Under $BiDCDFM(Z_{r},Z_{c},P,\Theta_{r},\Theta_{c})$, when Assumption \ref{assump2} holds, for any $\alpha>0$, with probability at least $1-o((n_{r}+n_{c})^{-\alpha})$, we have
\begin{align*}
&\hat{f}_{r}=O(\gamma_{*}\frac{\theta^{2}_{r,\mathrm{max}}K^{2}_{r}n_{r,\mathrm{max}}\mathrm{~max}(\theta_{r,\mathrm{max}}\|\theta_{c}\|_{1},\theta_{c,\mathrm{max}}\|\theta_{r}\|_{1})\mathrm{log}(n_{r}+n_{c})}{\theta^{4}_{r,\mathrm{min}}\theta^{2}_{c,\mathrm{min}}\sigma^{2}_{K_{r}}(P)n^{2}_{r,\mathrm{min}}n_{c,\mathrm{min}}}),\\
&\hat{f}_{c}=O(\gamma_{*}\frac{\theta^{2}_{c,\mathrm{max}}K_{r}K_{c}n_{c,\mathrm{max}}\mathrm{~max}(\theta_{r,\mathrm{max}}\|\theta_{c}\|_{1},\theta_{c,\mathrm{max}}\|\theta_{r}\|_{1})\mathrm{log}(n_{r}+n_{c})}{\theta^{2}_{r,\mathrm{min}}\theta^{4}_{c,\mathrm{min}}\sigma^{2}_{K_{r}}(P)\delta^{2}_{c,*}m^{2}_{V_{c}}n_{r,\mathrm{min}}n^{2}_{c,\mathrm{min}}}),
\end{align*}
where $\delta_{c,*}=\mathrm{min}_{k\neq l}\|V_{c}(k,:)-V_{c}(l,:)\|_{F}$ and $m_{V_{c}}=\mathrm{min}_{k}\|V_{c}(k,:)\|_{F}$.
\end{thm}
When $K_{r}=K_{c}$, Theorem \ref{mainBiDCDFM} can be simplified because $\delta_{c,*}=2^{1/2}$ by Lemma \ref{populationBiDCDFM} and $m_{V_{c}}=1$ by the proof of Lemma \ref{populationBiDCDFM}. By Theorem \ref{mainBiDCDFM}, we see that a smaller minimum column (row) degree heterogeneity $\theta_{c,\mathrm{min}}$ ($\theta_{r,\mathrm{min}}$) increases the difficulty of detecting both row and column communities.

Following similar analysis of Examples \ref{Bernoulli}-\ref{Signed}, here we also provide nBiSC's error rates under the same distributions based on Theorem \ref{mainBiDCDFM} by applying $0<\theta_{r,\mathrm{min}}\theta_{c,\mathrm{min}}\leq\theta_{r}(i_{r})\theta_{c}(j_{c})\leq \theta_{r,\mathrm{max}}\theta_{c,\mathrm{max}}$. The below analysis is similar to that of the examples under BiDFM, hence we omit most of the details and only show the fineness of $\gamma_{*}$ under different distribution $\mathcal{F}$. Meanwhile, nBiSC's error rates can be obtained immediately by setting $\gamma_{*}$ in Theorem \ref{mainBiDCDFM} as the upper bound of $\gamma_{*}$ given below under different distribution $\mathcal{F}$.
\begin{Ex}\label{BernoulliDC}
When $A(i_{r},j_{c})\sim \mathrm{Bernoulli}(\Omega(i_{r},j_{c}))$, $\frac{\mathrm{Var}(A(i_{r},j_{c}))}{\theta_{r}(i_{r})\theta_{c}(j_{c})}=\frac{\Omega(i_{r},j_{c})(1-\Omega(i_{r},j_{c}))}{\theta_{r}(i_{r})\theta_{c}(j_{c})}\leq\frac{\Omega(i_{r},j_{c})}{\theta_{r}(i_{r})\theta_{c}(j_{c})}=\frac{\theta_{r}(i_{r})\theta_{c}(j_{c})P(g_{i_{r}},g_{j_{c}})}{\theta_{r}(i_{r})\theta_{c}(j_{c})}\leq 1$, i.e., $\gamma_{*}\leq 1$.
\end{Ex}
\begin{Ex}\label{NormalDC}
When $A(i_{r},j_{c})\sim \mathrm{Normal}(\Omega(i_{r},j_{c}),\sigma^{2}_{A})$, $\mathrm{Var}(A(i_{r},j_{c}))=\sigma^{2}_{A}$, i.e., $\gamma_{*}\leq \frac{\sigma^{2}_{A}}{\theta_{r,\mathrm{min}}\theta_{c,\mathrm{min}}}$.
\end{Ex}
\begin{Ex}\label{SignedDC}
When $\mathbb{P}(A(i_{r},j_{c})=1)=\frac{1+\Omega(i_{r},j_{c})}{2}$ and $\mathbb{P}(A(i_{r},j_{c})=-1)=\frac{1-\Omega(i_{r},j_{c})}{2}$ for directed signed network, $\mathrm{Var}(A(i_{r},j_{c}))=1-\Omega^{2}(i_{r},j_{c})\leq1$, i.e., $\gamma_{*}\leq\frac{1}{\theta_{r,\mathrm{min}}\theta_{c,\mathrm{min}}}$.
\end{Ex}
Note that, if setting $\Theta_{r}=\sqrt{\rho}I_{n_{r}}$ and $\Theta_{c}=\sqrt{\rho}I_{n_{c}}$ such that BiDCDFM reduces to BiDFM, upper bounds of $\gamma_{*}$ in Examples \ref{BernoulliDC}-\ref{SignedDC} are the same as $\gamma$'s upper bounds given in Examples \ref{Bernoulli}-\ref{Signed}, respectively.

\section{Numerical study}\label{NumericalStudy}
We present both simulated and empirical experiments to investigate performances of BiSC and nBiSC for community detection on bipartite weighted networks. In addition to BiSC and nBiSC, three other spectral clustering algorithms capable of detecting communities for networks generated from the DCScBM model introduced by \cite{rohe2016co} are applied to our numerical study. These methods are DI-SIM \citep{rohe2016co}, D-SCORE \citep{wang2020spectral}, and rD-SCORE \citep{wang2020spectral}, where rD-SCORE applies a regularized Laplacian matrix to replace the adjacency matrix in D-SCORE. Note that the original D-SCORE and rD-SCORE algorithms proposed by \cite{wang2020spectral} are designed for directed networks, we have modified them to work for bipartite networks by applying K-means on $R_{\hat{U}}$ (and $R_{\hat{V}}$) in Algorithm 1 of \cite{wang2020spectral} with $K_{r}$ row clusters (and $K_{c}$ column clusters) to estimate node labels for row (and column) nodes.

\subsection{Evaluation metrics}
In this paper, when the ground truth of the node label is known, we use three widely used evaluation metrics to measure the quality of community partition including Hamming error \citep{SCORE}, normalized mutual
information (NMI) \citep{strehl2002cluster,danon2005comparing,bagrow2008evaluating,luo2017community}, and adjusted rand index (ARI) \citep{hubert1985comparing,vinh2009information,luo2017community}. Before presenting our numerical study, we briefly introduce these indicators.
\begin{itemize}
  \item Hamming error rates for row nodes is defined as
\begin{align*}
&\mathrm{ErrorRate}_{r}=n^{-1}_{r}\mathrm{min}_{J_{r}\in\mathcal{P}_{K_{r}}}\|\hat{Z}_{r}J_{r}-Z_{r}\|_{0},%~\mathrm{ErrorRate}_{c}=n^{-1}_{c}\mathrm{min}_{J_{c}\in\mathcal{P}_{K_{c}}}\|\hat{Z}_{c}J_{c}-Z_{c}\|_{0},
\end{align*}
where $\mathcal{P}_{K_{r}}$ is a set of all $K_{r}\times K_{r}$ permutation matrices. Similarly, we can define $\mathrm{ErrorRate}_{c}$ for column nodes. Instead of showing error rates for both row and column nodes, we report $\mathrm{max}(\mathrm{ErrorRate}_{r},\mathrm{ErrorRate}_{c})$, and denote it as $\mathrm{ErrorRate}$ for convenience. ErrorRate ranges in $[0,1]$, and a smaller ErrorRate means a better performance of community detection for both row and column nodes.
  \item For row nodes, recall that $\hat{\mathcal{C}}_{r}$ is the estimated community partition from $\hat{Z}_{r}$ and $\mathcal{C}_{r}=\{\mathcal{C}_{r,1}, \mathcal{C}_{r,2},\ldots, \mathcal{C}_{r,K_{r}}\}$ is the true community partition. Let $C$ be the confusion matrix whose elements $C(k,l)$ is the number of common nodes between ground-truth community $\mathcal{C}_{r,k}$ and estimated community $\hat{\mathcal{C}}_{r,l}$. The $\mathrm{NMI}(\hat{\mathcal{C}}_{r},\mathcal{C}_{r})$ is defined as follows
\begin{align}\label{NMI}
\mathrm{NMI}(\hat{\mathcal{C}}_{r},\mathcal{C}_{r})=\frac{-2\sum_{k,l}C(k,l)\mathrm{log}(\frac{C(k,l)n_{r}}{C_{k.}C_{.l}})}{\sum_{k}C_{k.}\mathrm{log}(\frac{C_{k.}}{n_{r}})+\sum_{l}C_{.l}\mathrm{log}(\frac{C_{.l}}{n_{r}})},
\end{align}
where $C_{k.} (\mathrm{and~}C_{.l})$ is the sum of the entries of $C$ in row $k$ (and column $l$). $\mathrm{NMI}(\hat{\mathcal{C}}_{r},\mathcal{C}_{r})$ ranges in $[0,1]$, and it gets the maximum value 1 when $\hat{\mathcal{C}}_{r}$ and $\mathcal{C}_{r}$ are exactly the same. Similarly, let $\mathcal{C}_{c}=\{\mathcal{C}_{c,1}, \mathcal{C}_{c,2},\ldots, \mathcal{C}_{c,K_{c}}\}$, we compute $\mathrm{NMI}(\hat{\mathcal{C}}_{c},\mathcal{C}_{c})$ for column nodes. For convenience, we let $\mathrm{NMI}=\mathrm{min}(\mathrm{NMI}(\hat{\mathcal{C}}_{r},\mathcal{C}_{r}), \mathrm{NMI}(\hat{\mathcal{C}}_{c},\mathcal{C}_{c}))$. Thus, a larger NMI indicates a better performance of community detection for both row and column nodes.
  \item For row nodes, the $\mathrm{ARI}(\hat{\mathcal{C}}_{r},\mathcal{C}_{r})$ is defined as follows
  \begin{align}\label{ARI}
\mathrm{ARI}(\hat{\mathcal{C}}_{r},\mathcal{C}_{r})=\frac{\sum_{k,l}\binom{C(k,l)}{2}-\frac{\sum_{k}\binom{C_{k.}}{2}\sum_{l}\binom{C_{.l}}{2}}{\binom{n_{r}}{2}}}{\frac{1}{2}[\sum_{k}\binom{C_{k.}}{2}+\sum_{l}\binom{C_{.l}}{2}]-\frac{\sum_{k}\binom{C_{k.}}{2}\sum_{l}\binom{C_{.l}}{2}}{\binom{n_{r}}{2}}},
  \end{align}
where $\binom{.}{.}$ is a binomial coefficient. $\mathrm{ARI}(\hat{\mathcal{C}}_{r},\mathcal{C}_{r})$ ranges from -1 to 1. The better the algorithm performs for row nodes, the higher $\mathrm{ARI}(\hat{\mathcal{C}}_{r},\mathcal{C}_{r})$ value. Similarly, we compute $\mathrm{ARI}(\hat{\mathcal{C}}_{c},\mathcal{C}_{c})$ for column nodes and set $\mathrm{ARI}=\mathrm{min}(\mathrm{ARI}(\hat{\mathcal{C}}_{r},\mathcal{C}_{r}), \mathrm{ARI}(\hat{\mathcal{C}}_{c},\mathcal{C}_{c}))$. Thus, a larger ARI means a better performance.
\end{itemize}

\subsection{Simulation}
In this section, we conduct simulated studies for performances of the 5 aforementioned algorithms on synthetic networks when $\mathcal{F}$ are distributions as Examples \ref{Bernoulli}-\ref{SignedDC} under BiDFM and BiDCDFM.

In all synthetic networks, set $K_{r}=2,K_{c}=3$ and generate $Z_{r}$ (and $Z_{c}$) such that each row (and column) node belongs to one of the row (and column) community with equal probability. For BiDCDFM, the node heterogeneity parameters $\theta_{r}$ and $\theta_{c}$ are generated as $\theta_{r}(i_{r})=\sqrt{\rho}a_{r}(i_{r}), \theta_{c}(j_{c})=\sqrt{\rho}b_{c}(j_{c})$ for $i_{r}=1,2,\ldots,n_{r}, j_{c}=1,2,\ldots, n_{c}$, where $a_{r}(i_{r})$ and $b_{c}(j_{c})$ are random values in $(0,1)$. For the $K_{r}\times K_{c}$ matrix $P$, there is no critical criterion on choosing its elements as long as $\mathrm{rank}(P)=\mathrm{min}(K_{r},K_{c}), \mathrm{max}_{k,l}|P(k,l)|=1$ as provided in Definition \ref{BiDFM} and $P$'s elements should be set positive or real numbers depending on distribution $\mathcal{F}$ as analyzed in Examples \ref{Bernoulli}-\ref{Signed}, we simply set $P$ as two cases given below. For distribution which needs all entries of $P$ be nonnegative, we set $P$ as $P_{1}$ given below:
\[P_{1}=\begin{bmatrix}
    1&0.2&0.3\\
    0.3&0.8&0.2\\
\end{bmatrix}.\]
For distribution which allows $P$ to have negative entries, we set $P$ as $P_{2}$ given below:
\[P_{2}=\begin{bmatrix}
    -1&0.3&-0.5\\
    -0.4&0.8&0.2\\
\end{bmatrix}.\]
Meanwhile, when we consider the case $n_{r}=n_{c}$, we set $n=n_{r}=n_{c}$ for convenience. Each simulation experiment contains the following steps:

Step 1: set $\Omega=\rho Z_{r}PZ'_{c}$ under $BiDFM(Z_{r},Z_{c},P,\rho)$ (or set $\Omega=\Theta_{r}Z_{r}PZ'_{c}\Theta_{c}$ under $BiDCDFM(Z_{r},Z_{c},P,\Theta_{r},\Theta_{c})$).

Step 2: generate the $n_{r}\times n_{c}$ asymmetric adjacency matrix $A$ by letting $A(i_{r},j_{c})$ generated from a distribution $\mathcal{F}$ with expectation $\Omega(i_{r},j_{c})$ for $i_{r}=1,2,\ldots,n_{r}, j_{c}=1,2,\ldots, n_{c}$.

Step 3: apply a community detection algorithm to $A$. Record ErrorRate, NMI and ARI.

Step 4: repeat steps 2-3 for 50 times, and report the averaged ErrorRate, NMI, and ARI.

We consider the following simulation setups.
\subsubsection{Bernoulli distribution}
When $\mathcal{F}$ is a Bernoulli distribution such that $A(i_{r},j_{c})\sim \mathrm{Bernoulli}(\Omega(i_{r},j_{c}))$, as analyzed in Examples \ref{Bernoulli} and \ref{BernoulliDC}, all entries of $P$ should be nonnegative. Hence, $P$ is set as $P_{1}$. Since $\mathbb{P}(A(i_{r},j_{c})=1)=\rho P(g_{i_{r}},g_{j_{c}})\in[0,1]$  and maximum entry of $P$ is 1, $\rho$ should be set no larger than 1.

\textbf{Simulation 1(a): changing $\rho$ under BiDFM}. Let $n_{r}=200,n_{c}=300$. Let $\rho$ range in $\{0.1,0.2,0.3,\ldots,1\}$.

\textbf{Simulation 1(b): changing $\rho$ under BiDCDFM}. Let $n_{r}=600,n_{c}=900$. Let $\rho$ range in $\{0.1,0.2,0.3,\ldots,1\}$.

\textbf{Simulation 1(c): changing $n$ under BiDFM}. Let $\rho=0.5$. Let $n$ range in $\{50,100,150,\ldots,500\}$.

\textbf{Simulation 1(d): changing $n$ under BiDCDFM}. Let $\rho=0.5$. Let $n$ range in $\{500,1000,1500,2000,2500,3000\}$.
\begin{figure}[!h]
\centering

\subfigure[SIM 1(a)]{\includegraphics[width=0.24\textwidth]{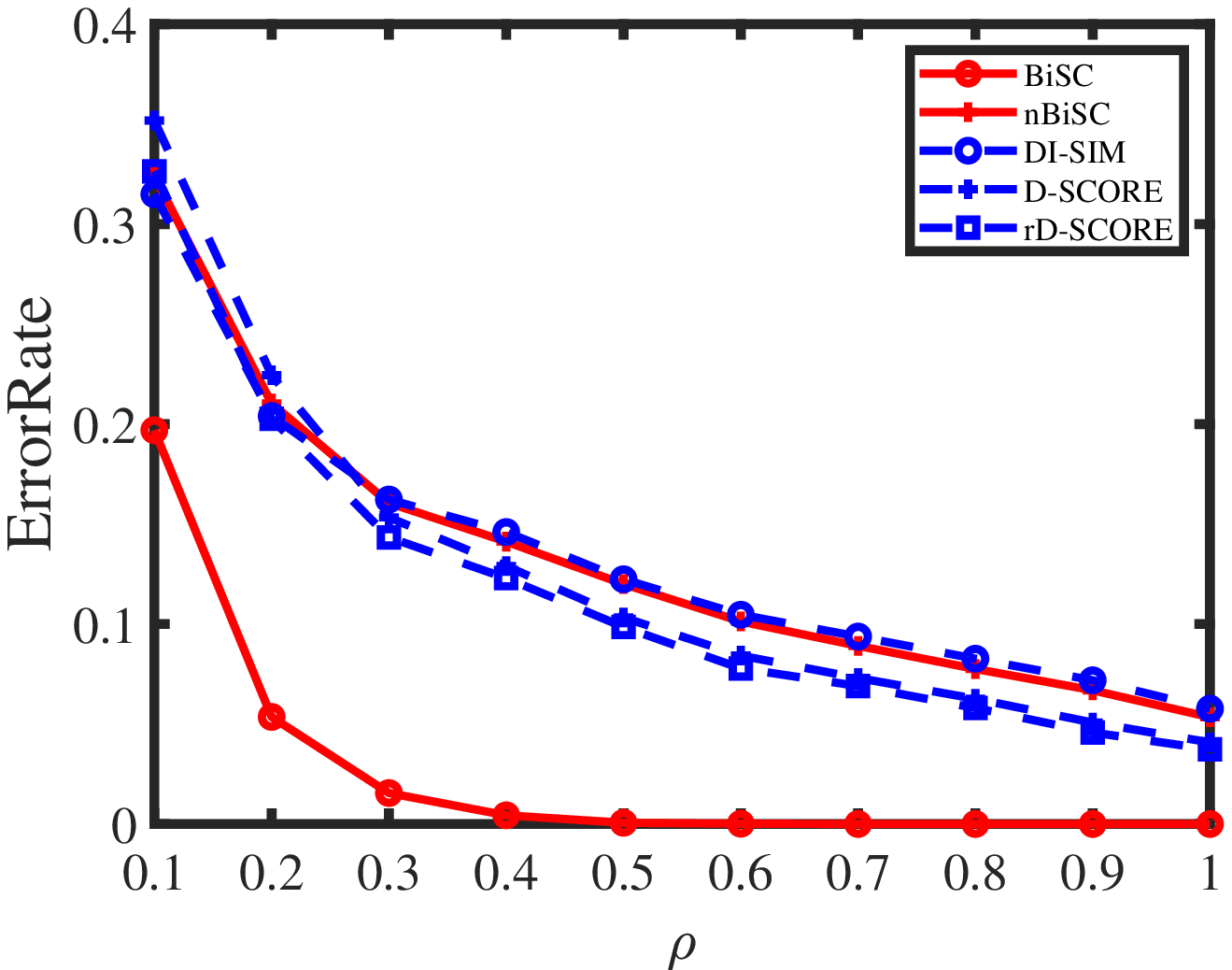}}
\subfigure[SIM 1(b)]{\includegraphics[width=0.24\textwidth]{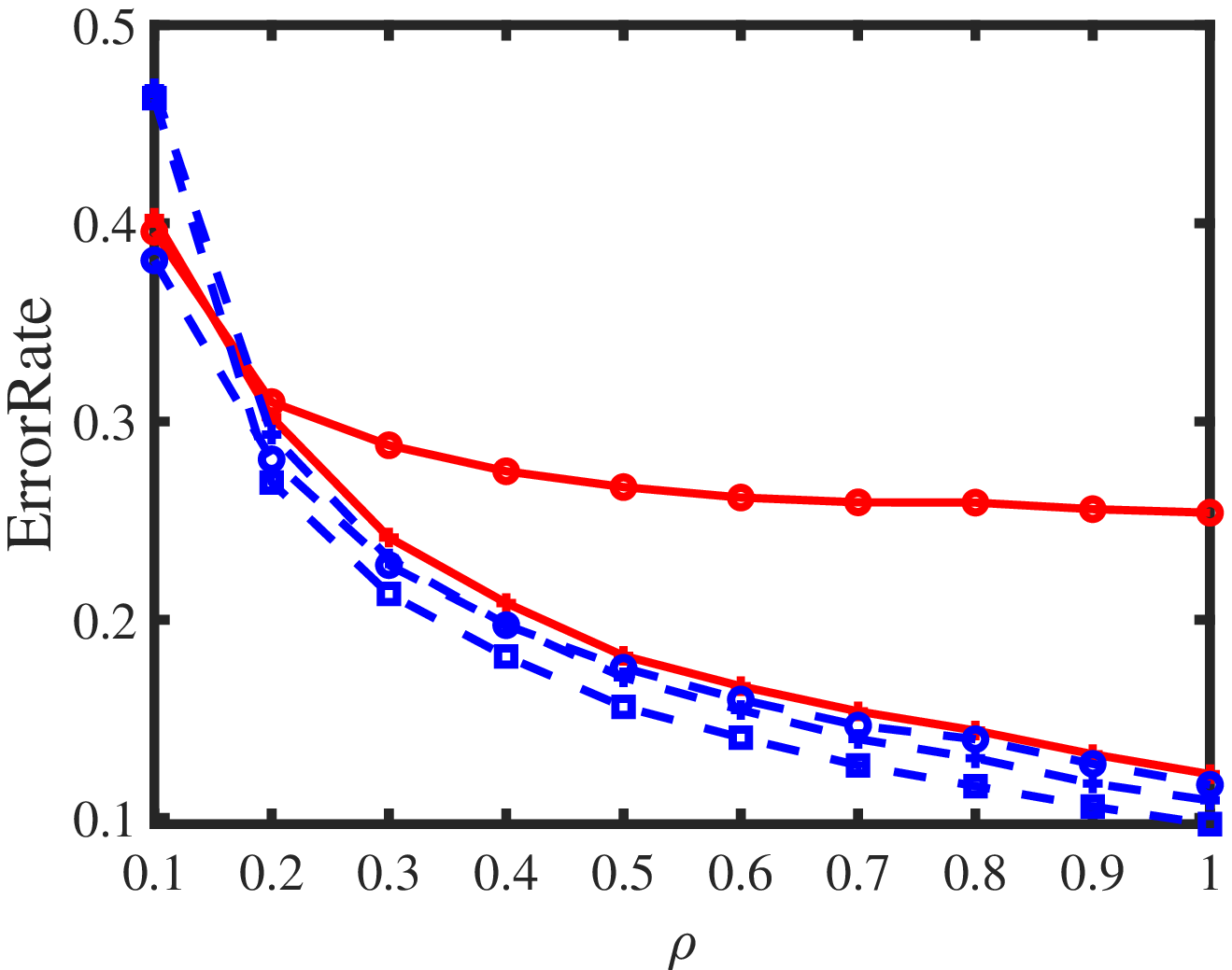}}
\subfigure[SIM 1(c)]{\includegraphics[width=0.24\textwidth]{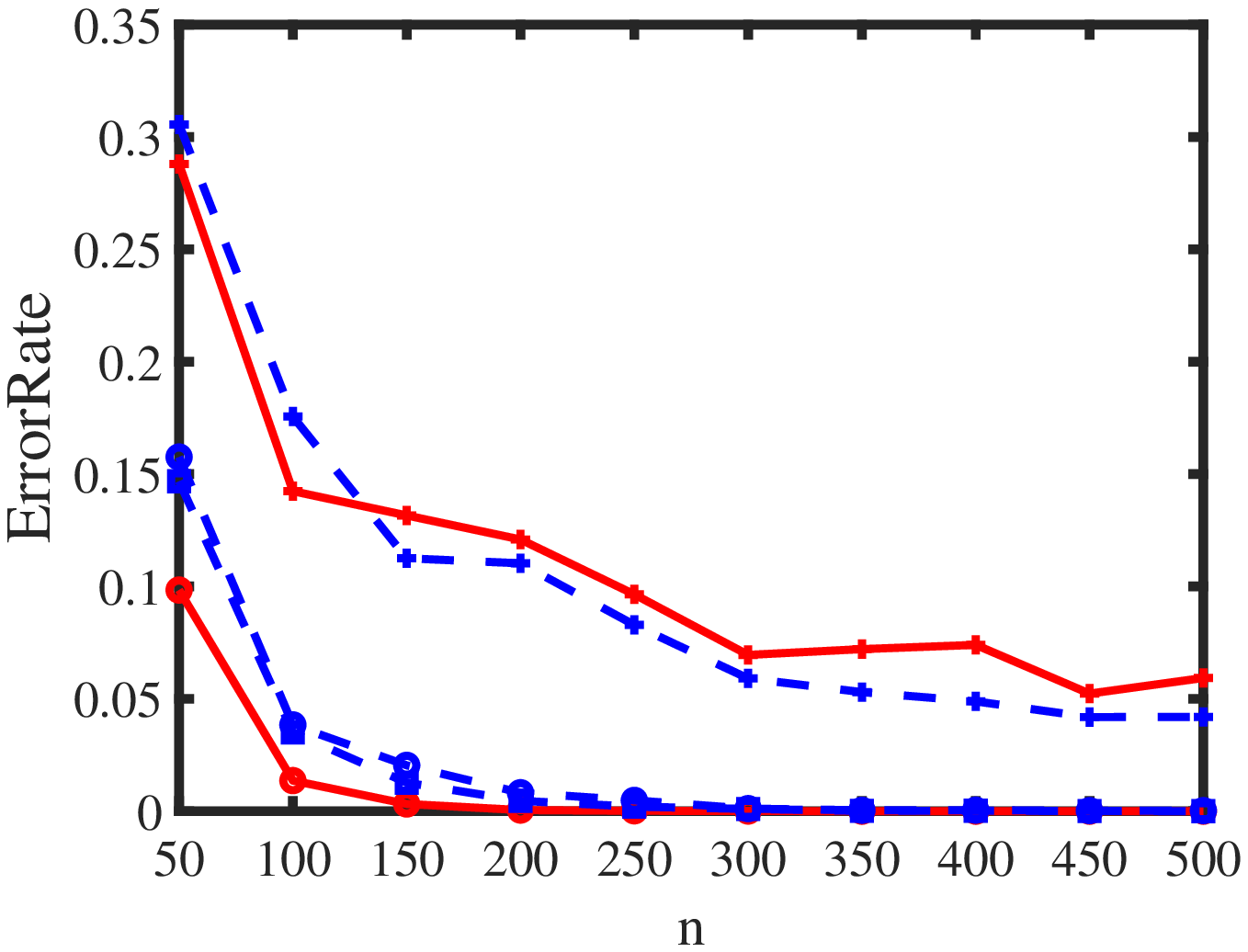}}
\subfigure[SIM 1(d)]{\includegraphics[width=0.24\textwidth]{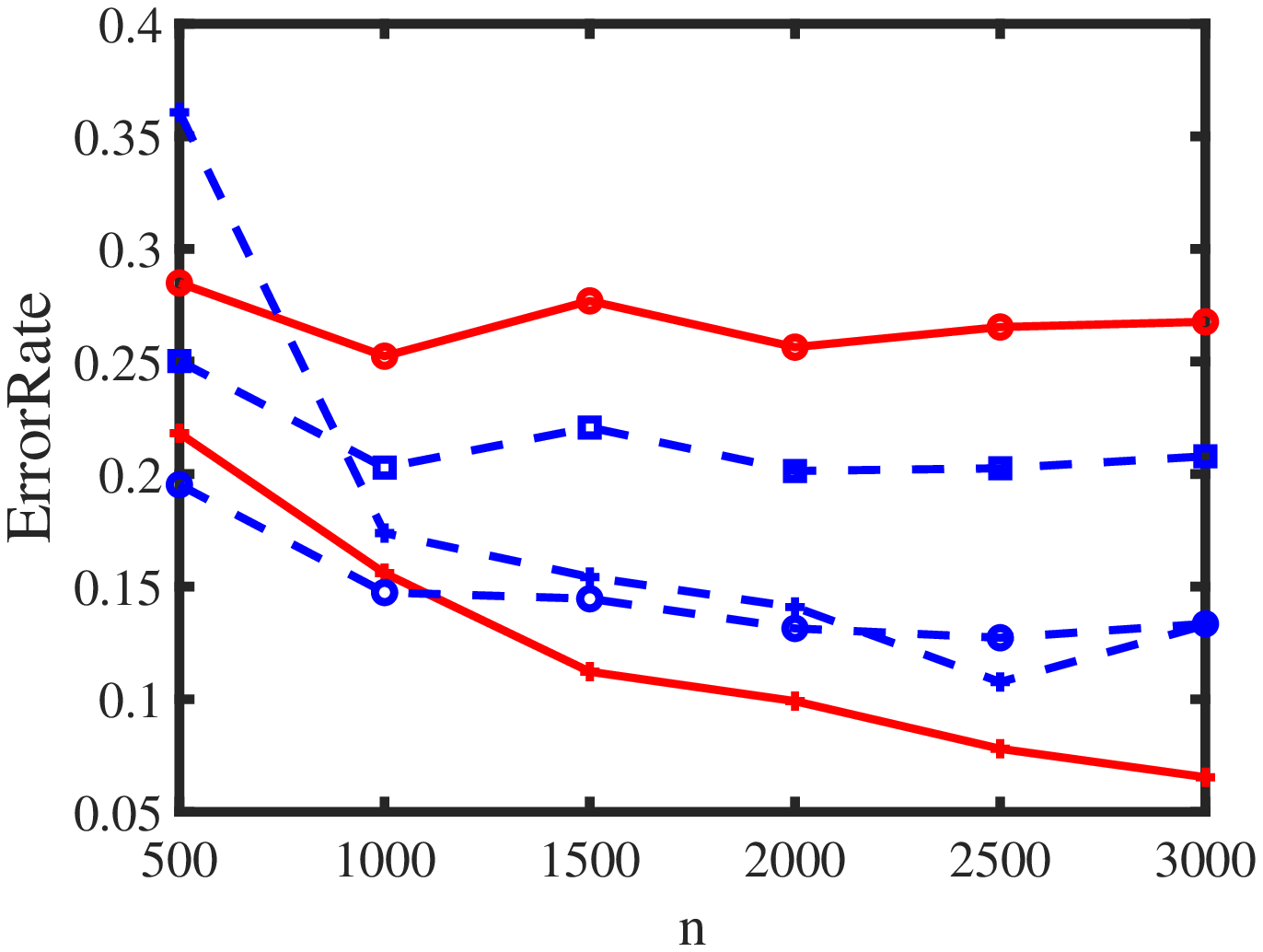}}
\subfigure[SIM 1(a)]{\includegraphics[width=0.24\textwidth]{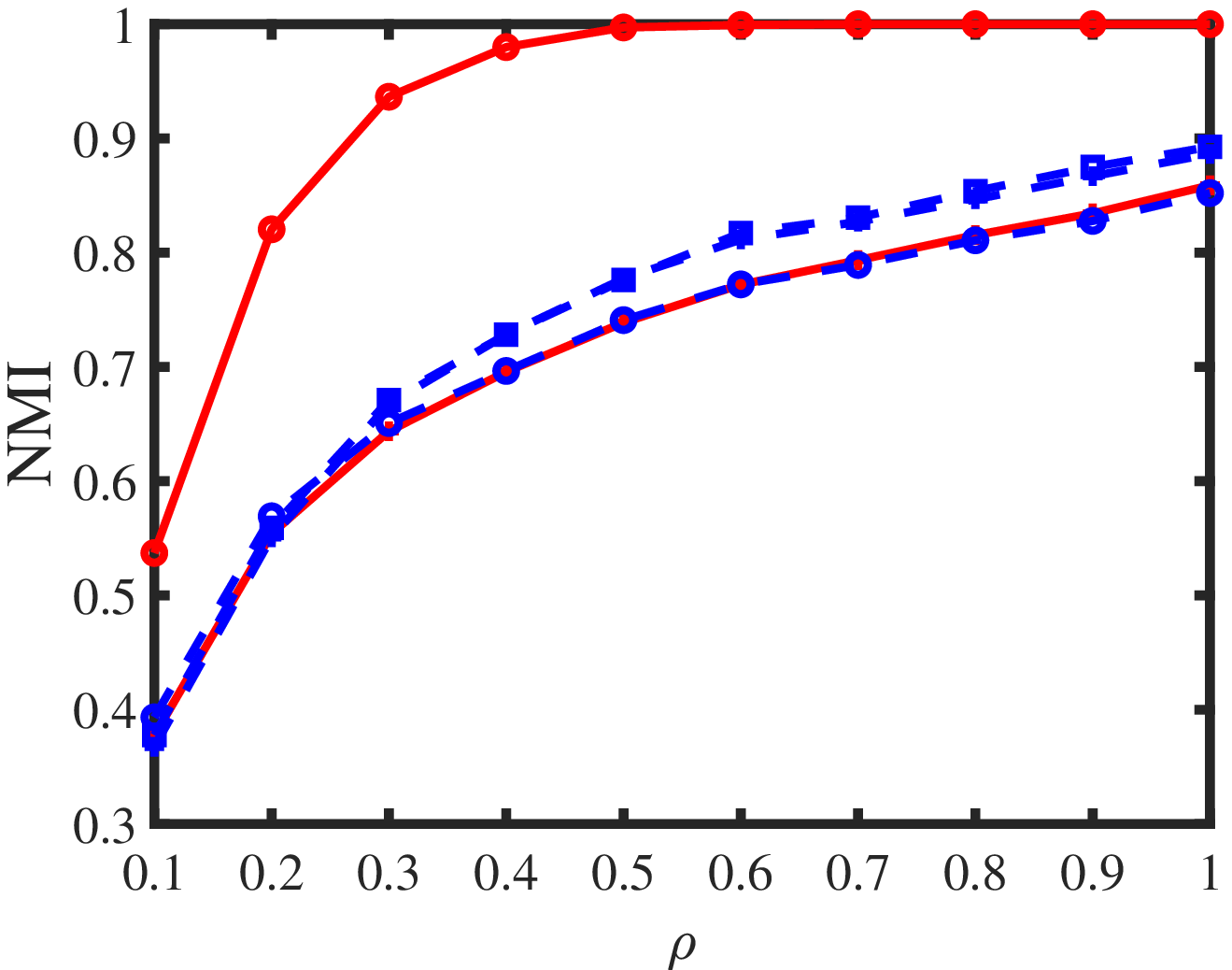}}
\subfigure[SIM 1(b)]{\includegraphics[width=0.24\textwidth]{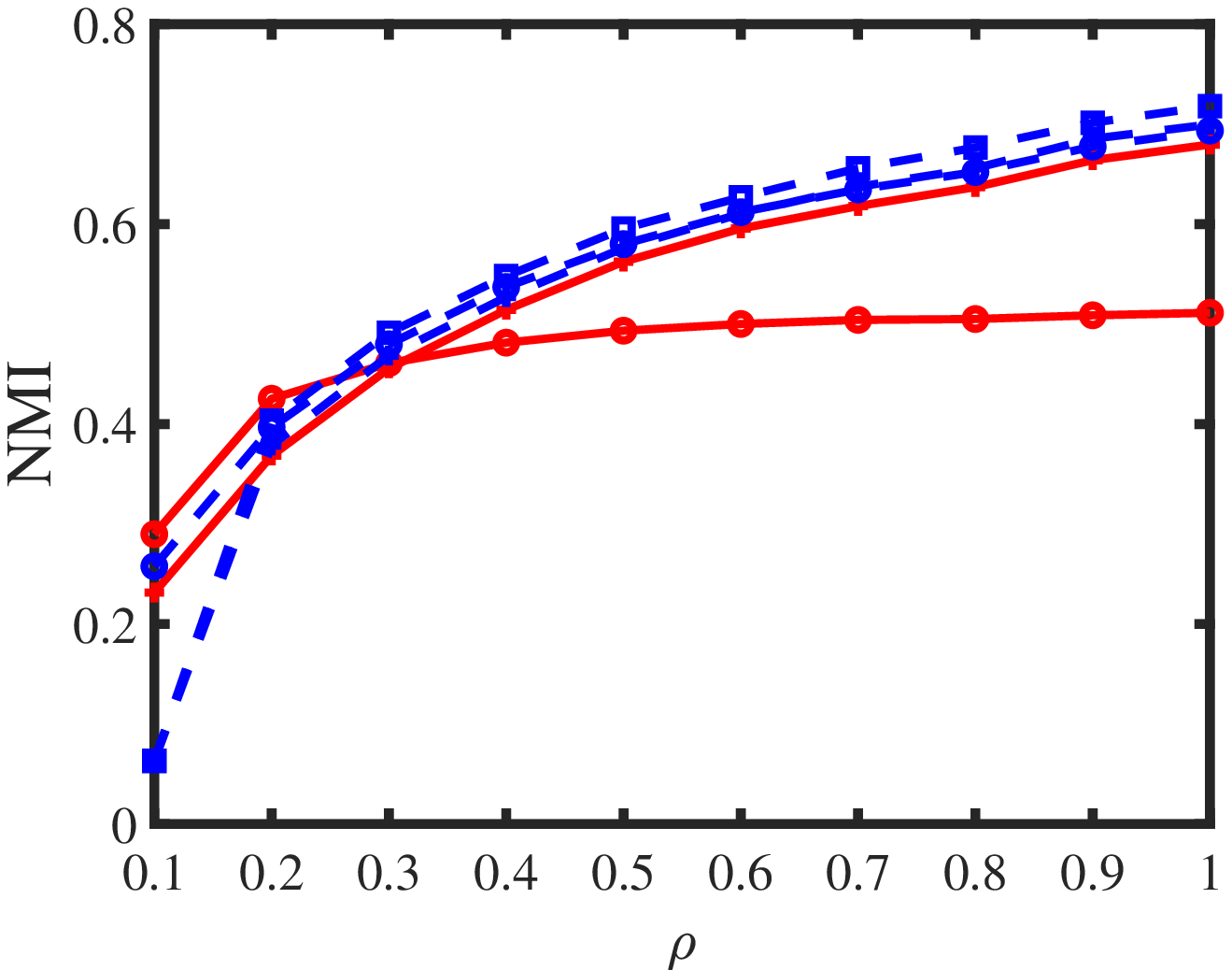}}
\subfigure[SIM 1(c)]{\includegraphics[width=0.24\textwidth]{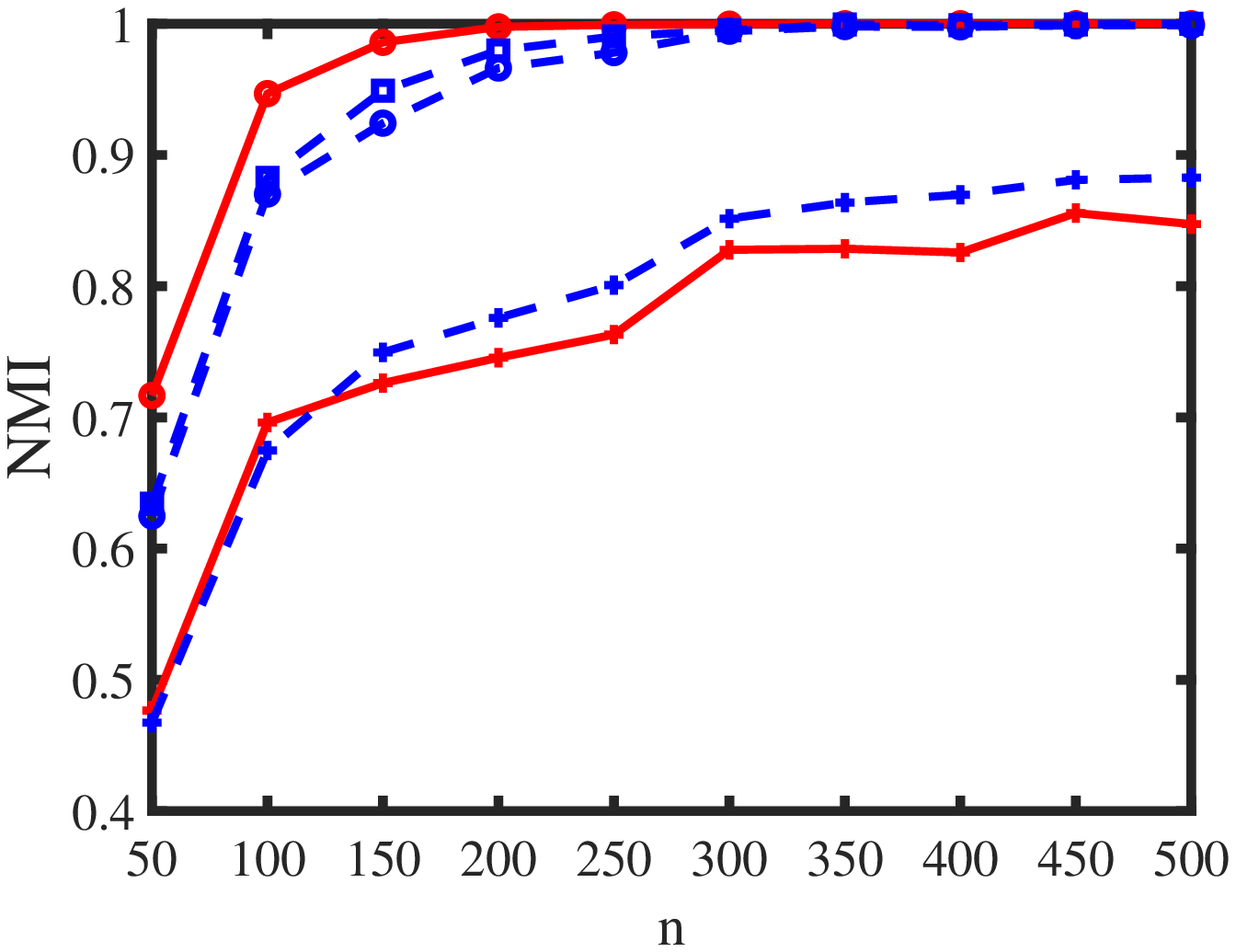}}
\subfigure[SIM 1(d)]{\includegraphics[width=0.24\textwidth]{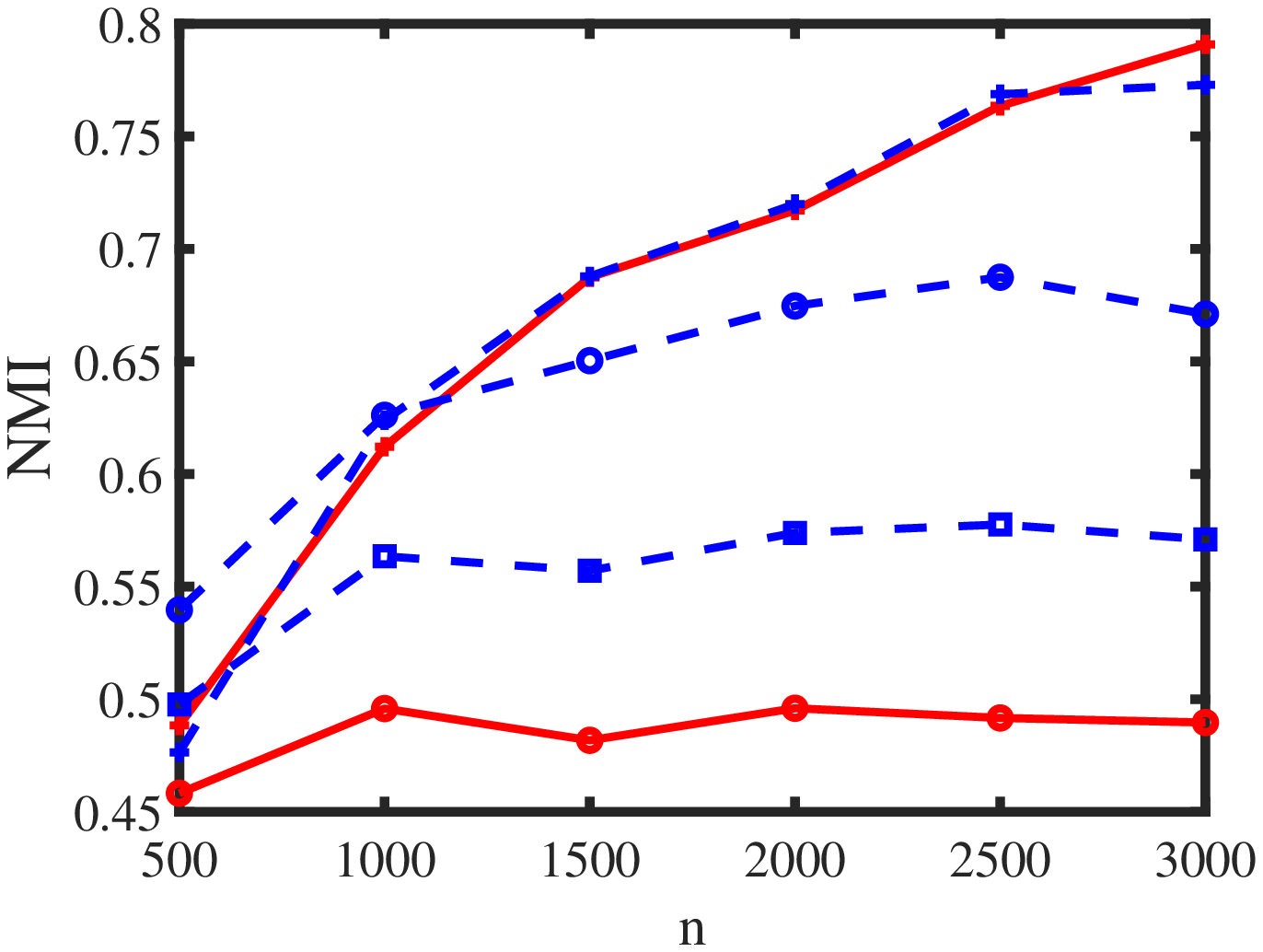}}
\subfigure[SIM 1(a)]{\includegraphics[width=0.24\textwidth]{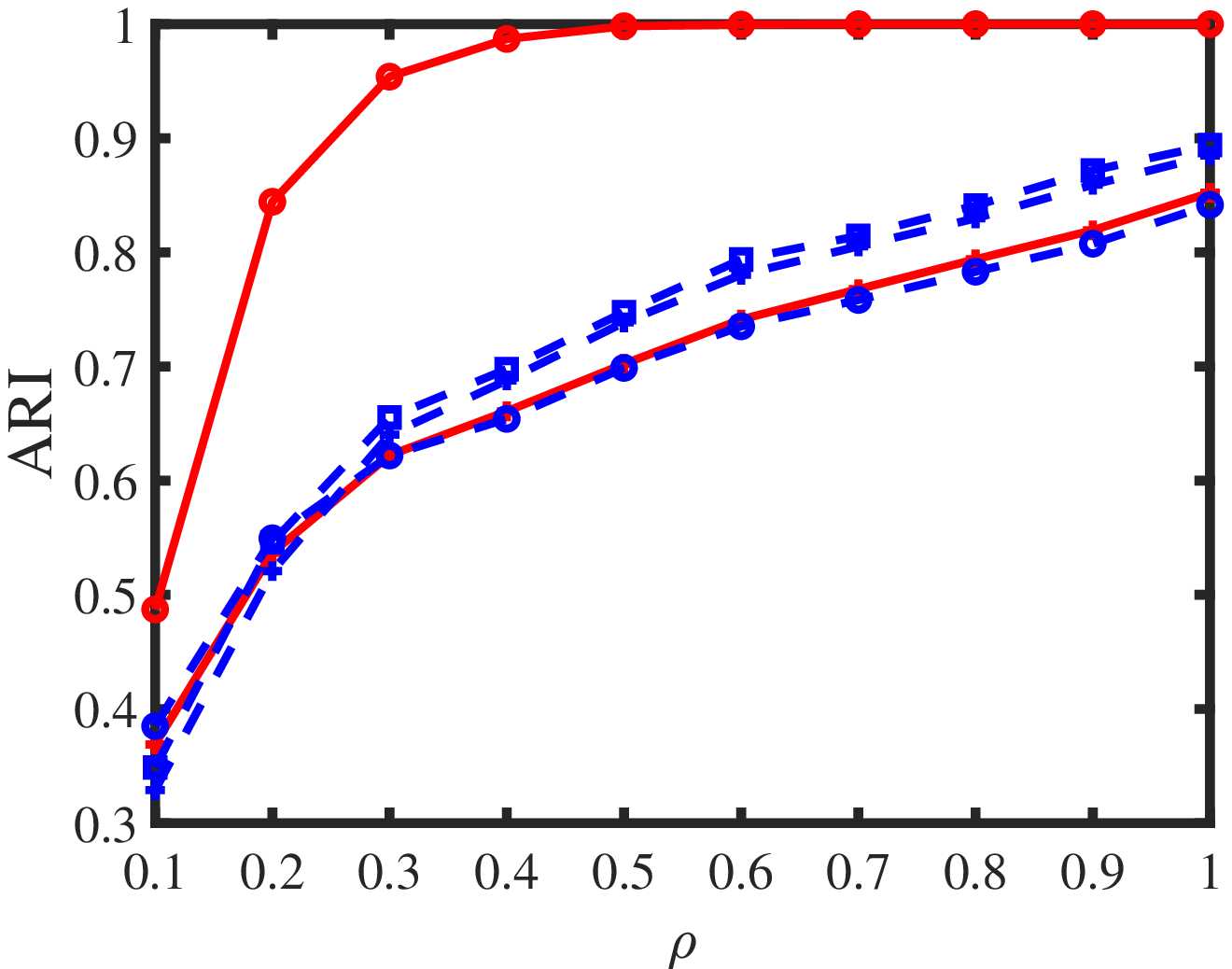}}
\subfigure[SIM 1(b)]{\includegraphics[width=0.24\textwidth]{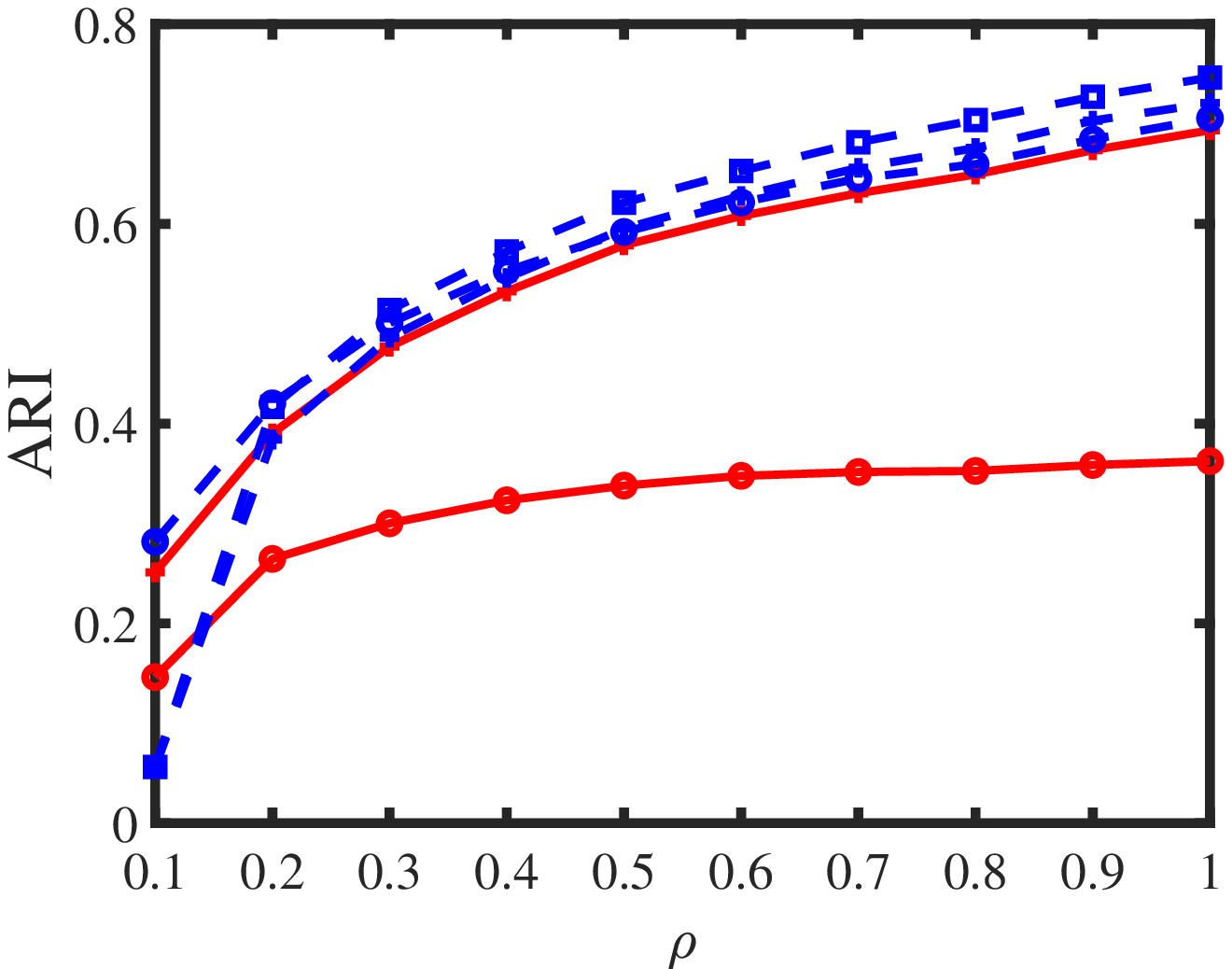}}
\subfigure[SIM 1(c)]{\includegraphics[width=0.24\textwidth]{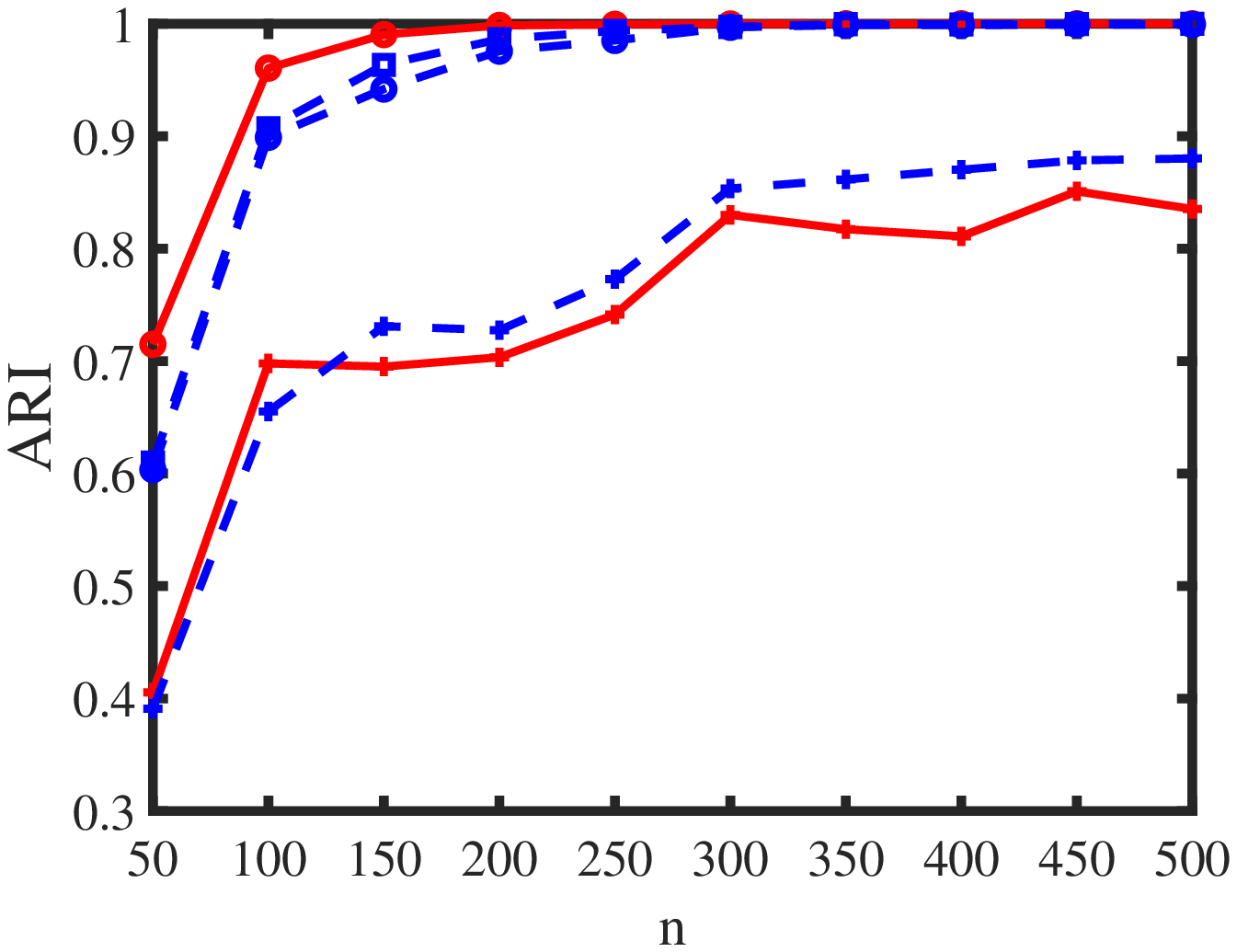}}
\subfigure[SIM 1(d)]{\includegraphics[width=0.24\textwidth]{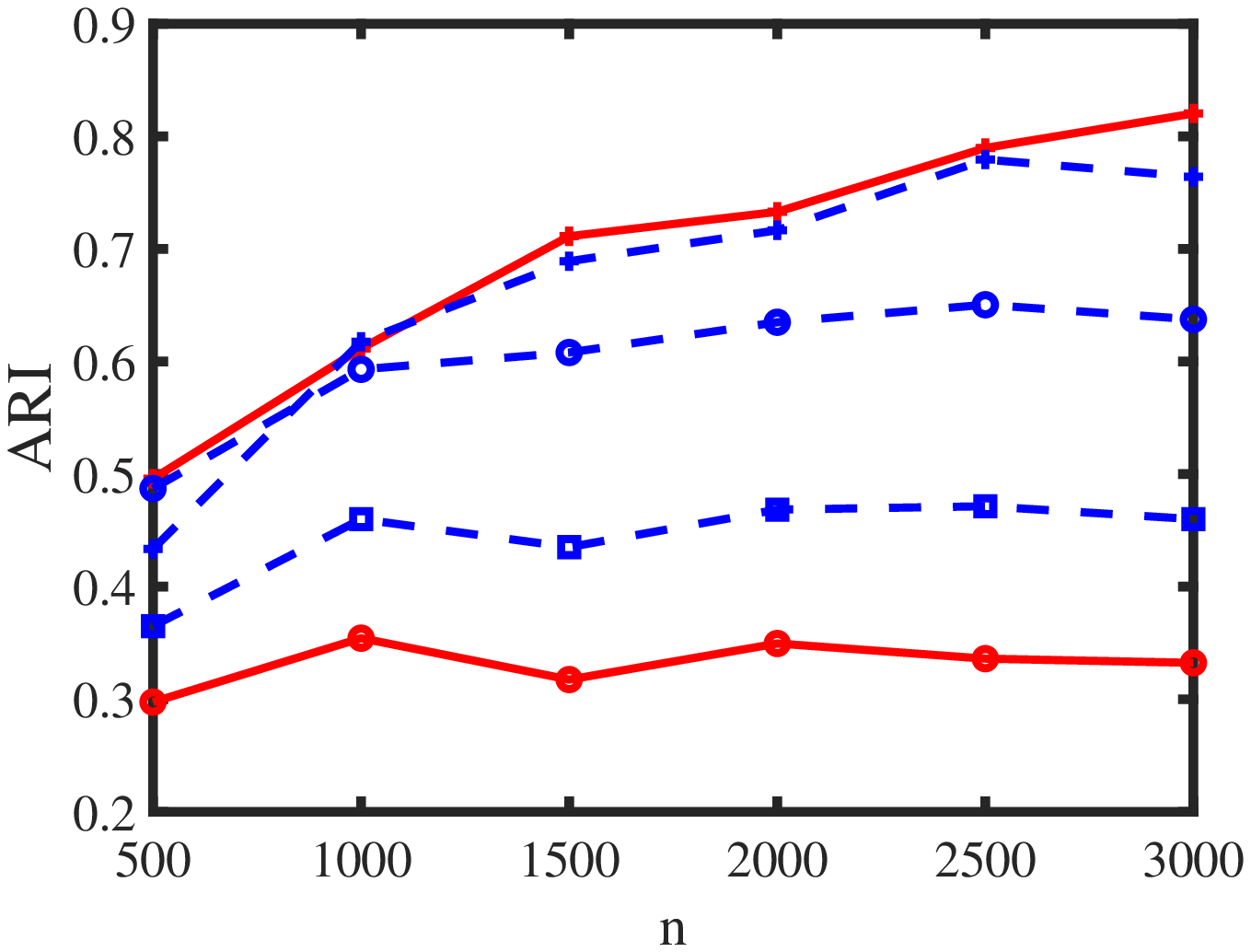}}
\caption{Numerical results of Simulation 1.}
\label{S1} %% label for entire figure
\end{figure}

The results of Simulation 1 are summarized in Figure \ref{S1}, where Simulation is abbreviated as SIM. In Figure \ref{S1} (a) and (b), we see that all methods perform better as $\rho$ increases in terms of ErrorRate, and this is consistent with the analysis in Examples \ref{Bernoulli} and \ref{BernoulliDC}. And we also can find that when $\rho\leq 0.3$, the error rates for BiSC under two models decrease rapidly, while when $\rho$ increases further, the error rates keep stable. The error rates for the other four methods also decrease fast when $\rho\leq 0.3$, and decrease slowly when $\rho> 0.3$. Such results are natural since $\rho$ is a sparse parameter and a smaller $\rho$ indicates a sparser network, thus a larger estimation error. From Figure \ref{S1} (c)-(d), we can find that these five methods tend to perform better as $n$ increases which is consistent with the conclusion in Theorems \ref{mainBiDFM} and \ref{mainBiDCDFM}. In terms of NMI and ARI, the analysis of the results shown in panels (e)-(l) of Figure \ref{S1} is similar to that of ErrorRate. Naturally, for all cases, when the data is generated from BiDFM, BiSC outperforms nBiSC, and when the data is generated from BiDCDFM, nBiSC is better. We also find that nBiSC performs similarly to DI-SIM, D-SCORE, and rD-SCORE for Simulation 1. This is reasonable because Simulation 1 is designed under DCScBM when $\mathcal{F}$ is Bernoulli while nBiSC, DI-SIM, D-SCORE, and rD-SCORE all can fit DCScBM.
\subsubsection{Normal distribution}
When $\mathcal{F}$ is Normal distribution such that $A(i_{r},j_{c})\sim\mathrm{Normal}(\Omega(i_{r},j_{c}),\sigma^{2}_{A})$ for some $\sigma^{2}_{A}>0$, by Examples \ref{Normal} and \ref{NormalDC}, $P$ is set as $P_{2}$ and $\rho$ can be set larger than 1.

\textbf{Simulation 2(a): changing $\rho$ under BiDFM}. Let $n_{r}=200,n_{c}=300$ and $\sigma^{2}_{A}=1$. Let $\rho$ range in $\{0.1,0.2,0.3,\ldots,2\}$.

\textbf{Simulation 2(b): changing $\rho$ under BiDCDFM}. Let $n_{r}=600,n_{c}=900$ and $\sigma^{2}_{A}=1$. Let $\rho$ range in $\{0.1,0.2,0.3,\ldots,2\}$.

\textbf{Simulation 2(c): changing $\sigma^{2}_{A}$ under BiDFM}. Let $n_{r}=200, n_{c}=300$ and $\rho=0.5$. Let $\sigma^{2}_{A}$ range in $\{0.2,0.4,0.6,\ldots,2\}$.

\textbf{Simulation 2(d): changing $\sigma^{2}_{A}$ under BiDCDFM}. Let $n_{r}=600, n_{c}=900$ and $\rho=3$. Let $\sigma^{2}_{A}$ range in $\{0.2,0.4,0.6,\ldots,2\}$.

\textbf{Simulation 2(e): changing $n$ under BiDFM}. Let $\rho=0.5$ and $\sigma^{2}_{A}=1$. Let $n$ range in $\{50,100,150,\ldots,500\}$.

\textbf{Simulation 2(f): changing $n$ under BiDCDFM}. Let $\rho=1$ and $\sigma^{2}_{A}=1$. Let $n$ range in $\{500,1000,1500,2000,2500,3000\}$.
\begin{figure}
\centering
\subfigure[SIM 2(a)]{\includegraphics[width=0.24\textwidth]{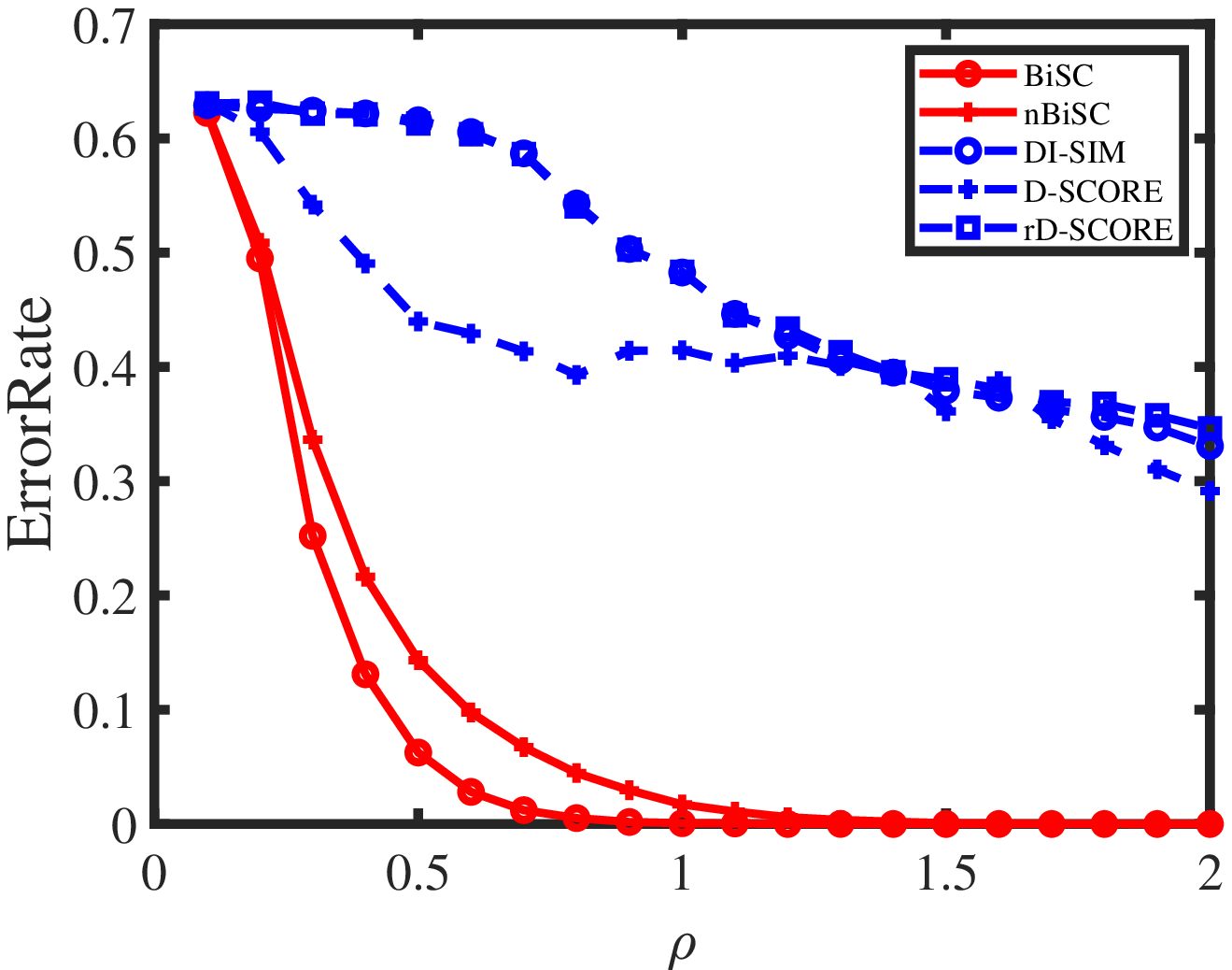}}
\subfigure[SIM 2(b)]{\includegraphics[width=0.24\textwidth]{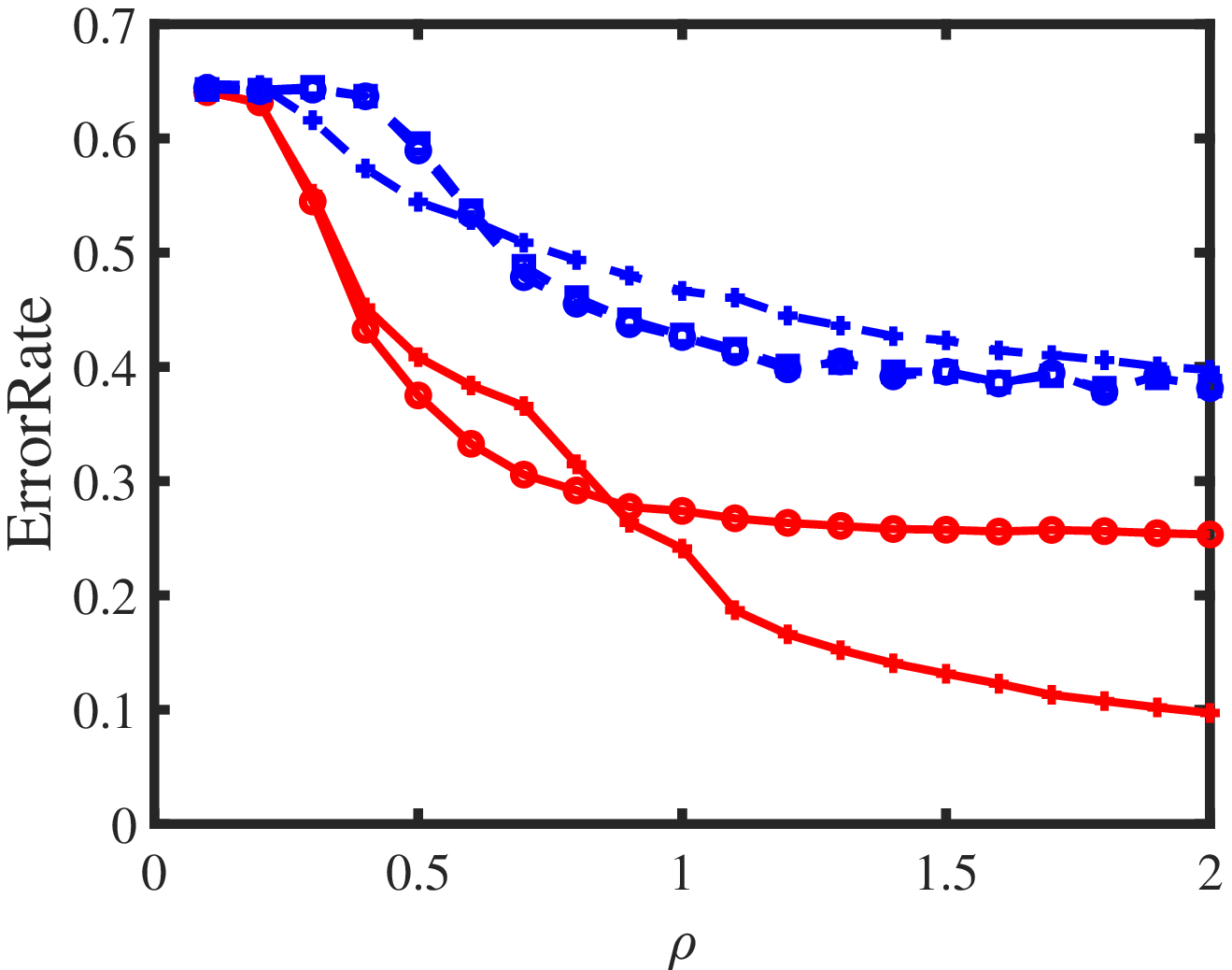}}
\subfigure[SIM 2(c)]{\includegraphics[width=0.24\textwidth]{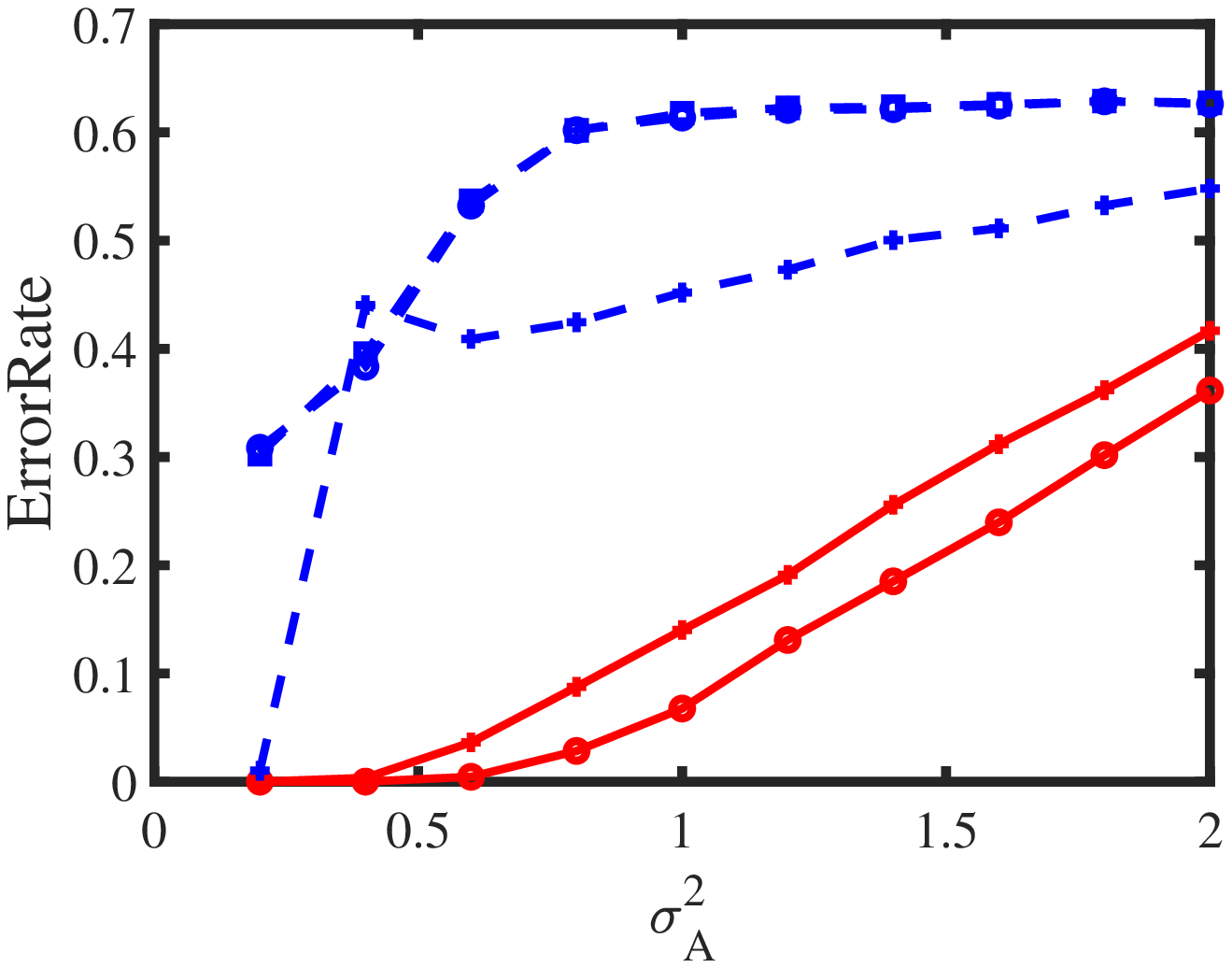}}
\subfigure[SIM 2(d)]{\includegraphics[width=0.24\textwidth]{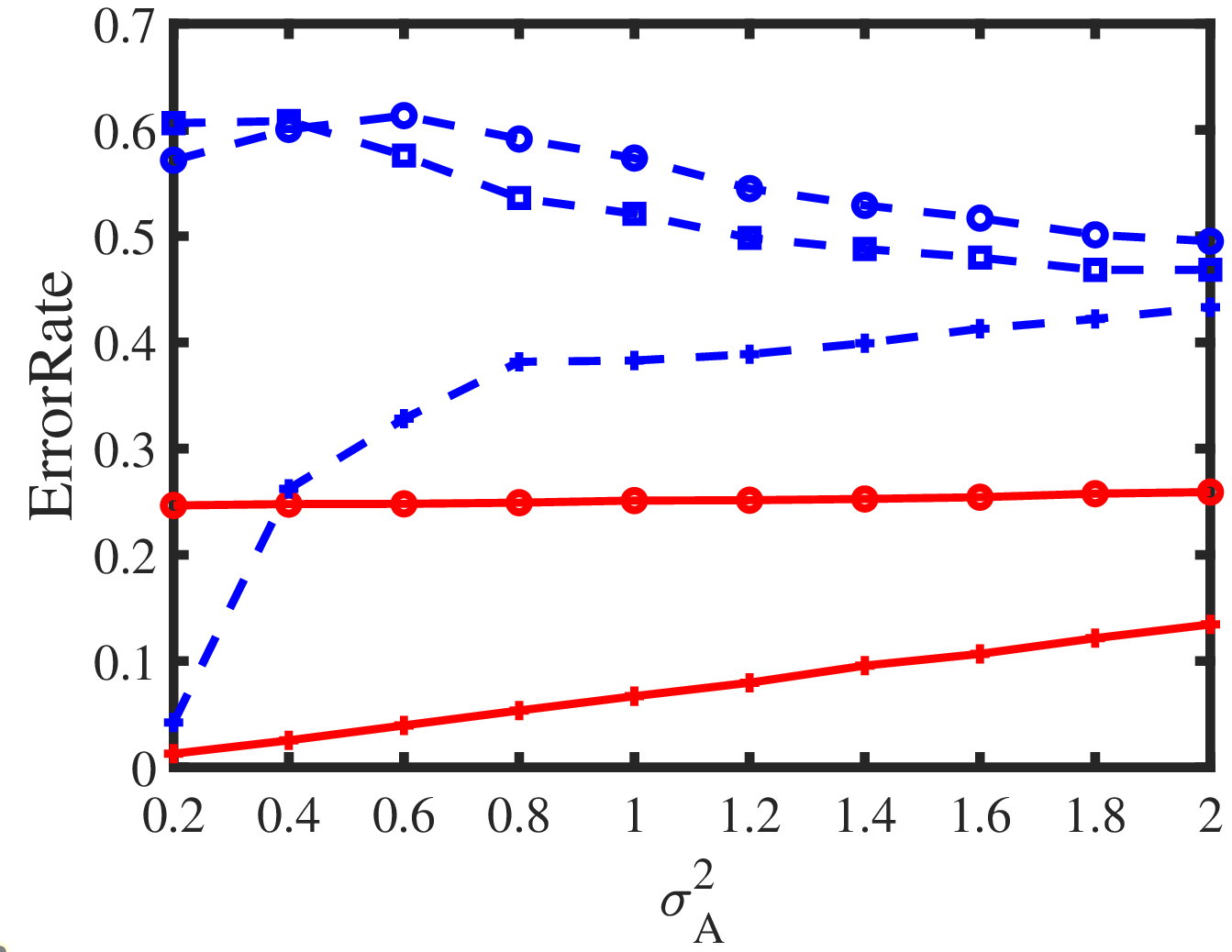}}
\subfigure[SIM 2(e)]{\includegraphics[width=0.24\textwidth]{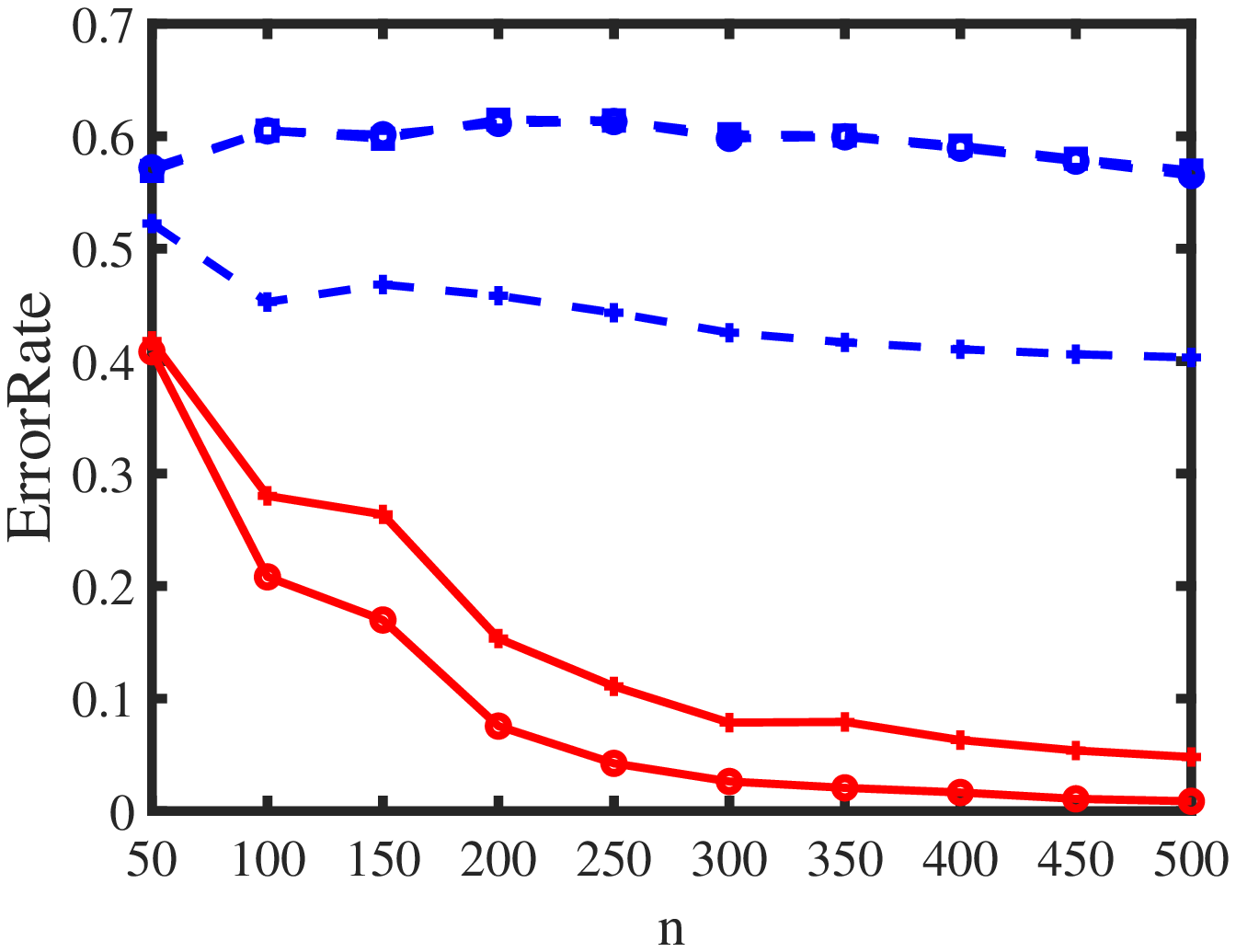}}
\subfigure[SIM 2(f)]{\includegraphics[width=0.24\textwidth]{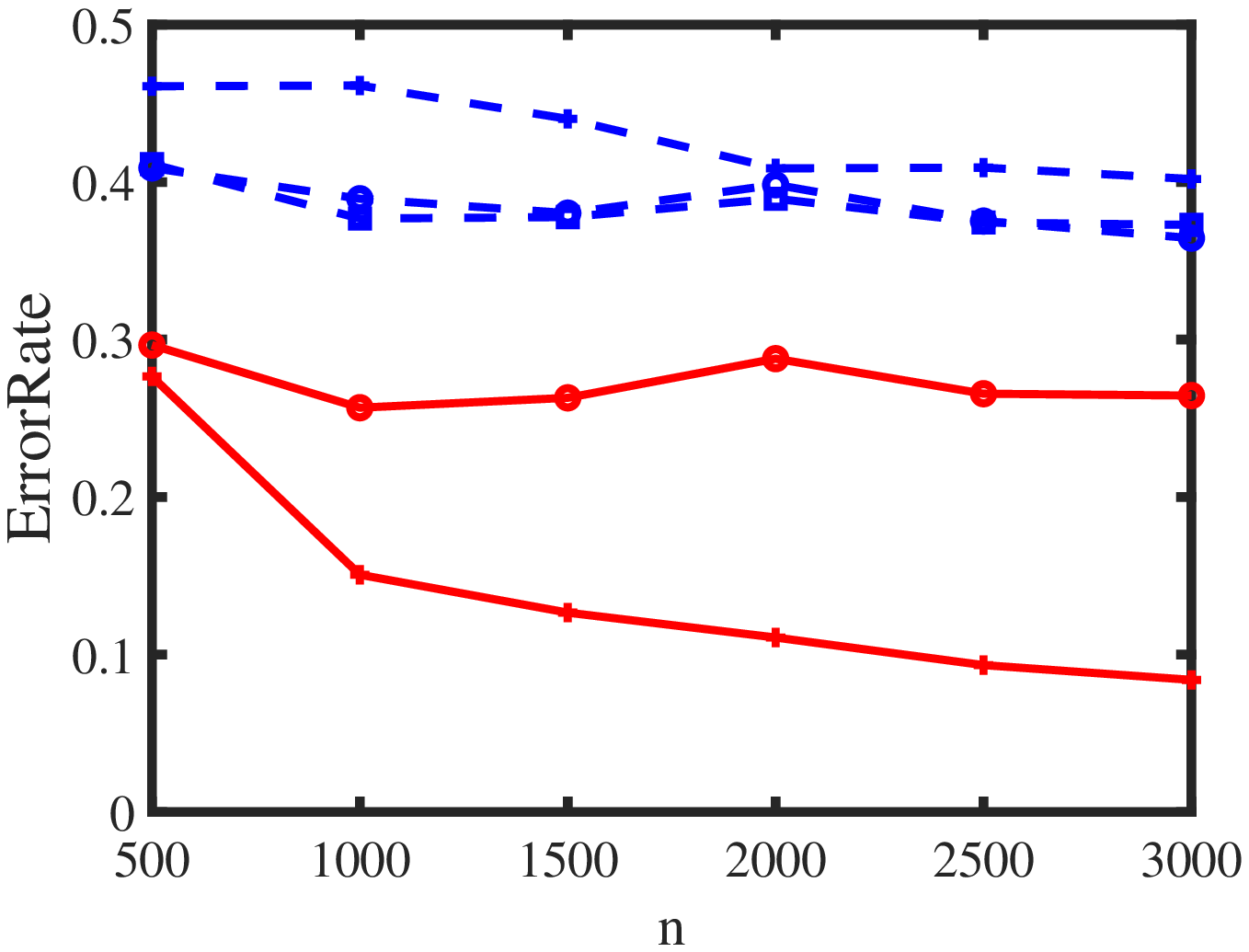}}
\subfigure[SIM 2(a)]{\includegraphics[width=0.24\textwidth]{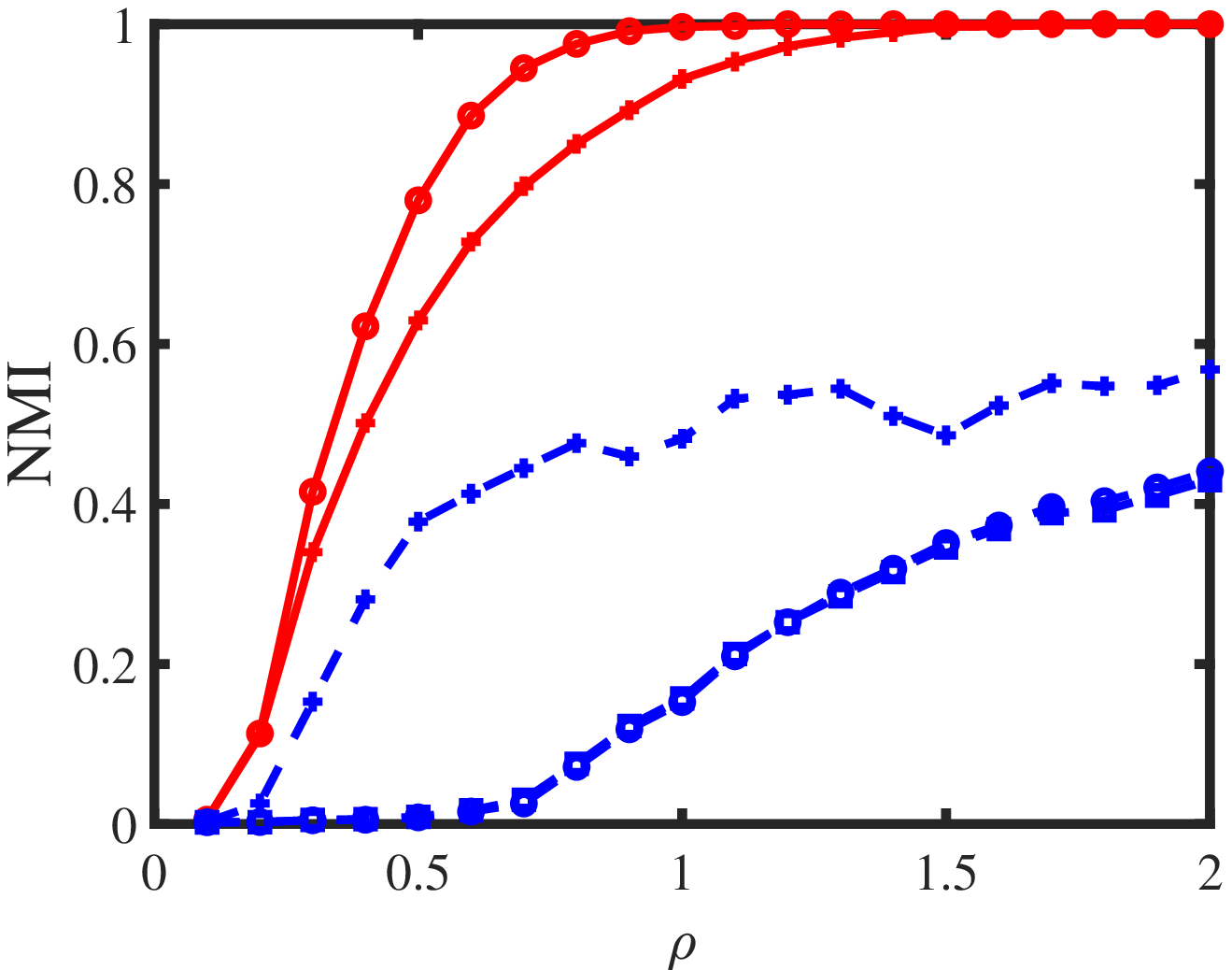}}
\subfigure[SIM 2(b)]{\includegraphics[width=0.24\textwidth]{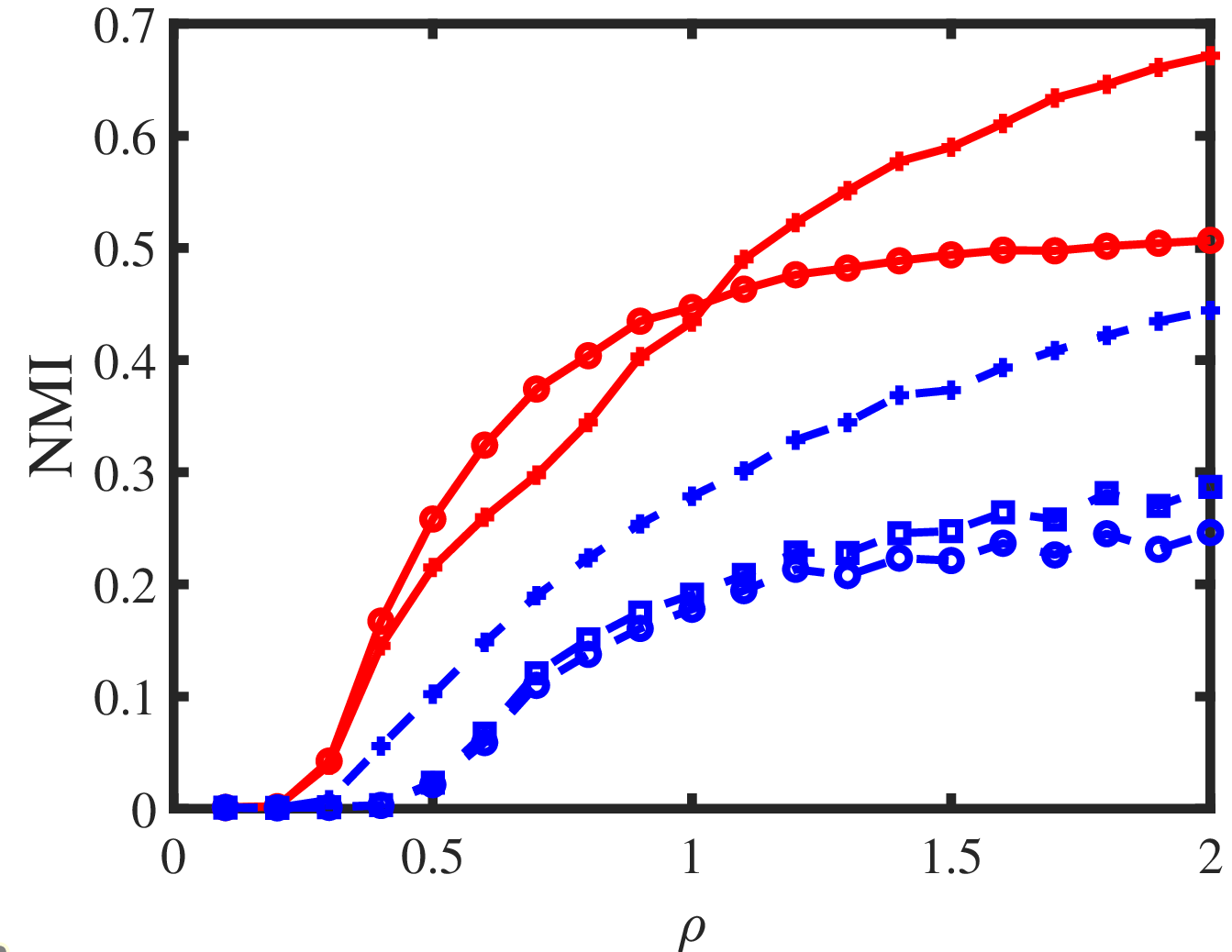}}
\subfigure[SIM 2(c)]{\includegraphics[width=0.24\textwidth]{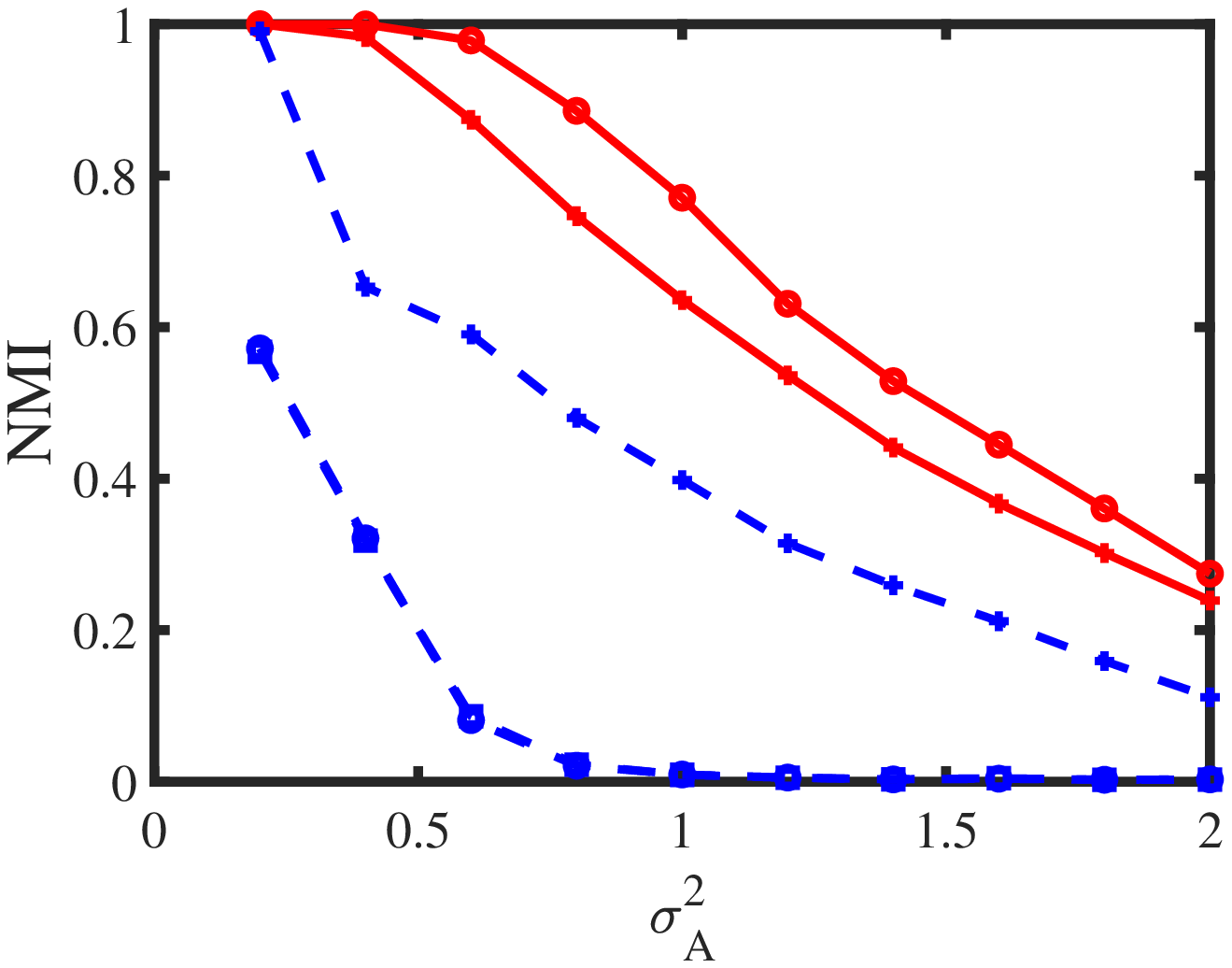}}
\subfigure[SIM 2(d)]{\includegraphics[width=0.24\textwidth]{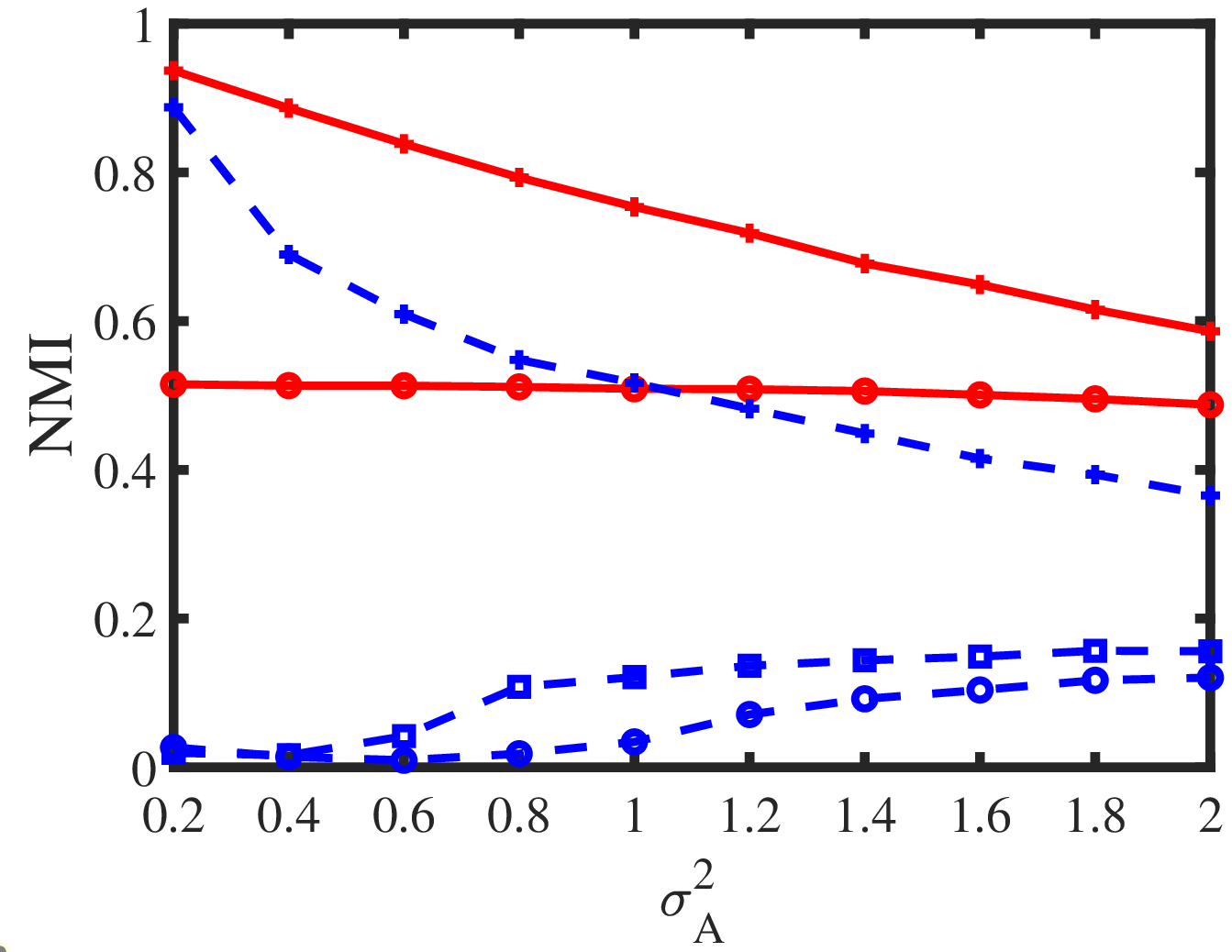}}
\subfigure[SIM 2(e)]{\includegraphics[width=0.24\textwidth]{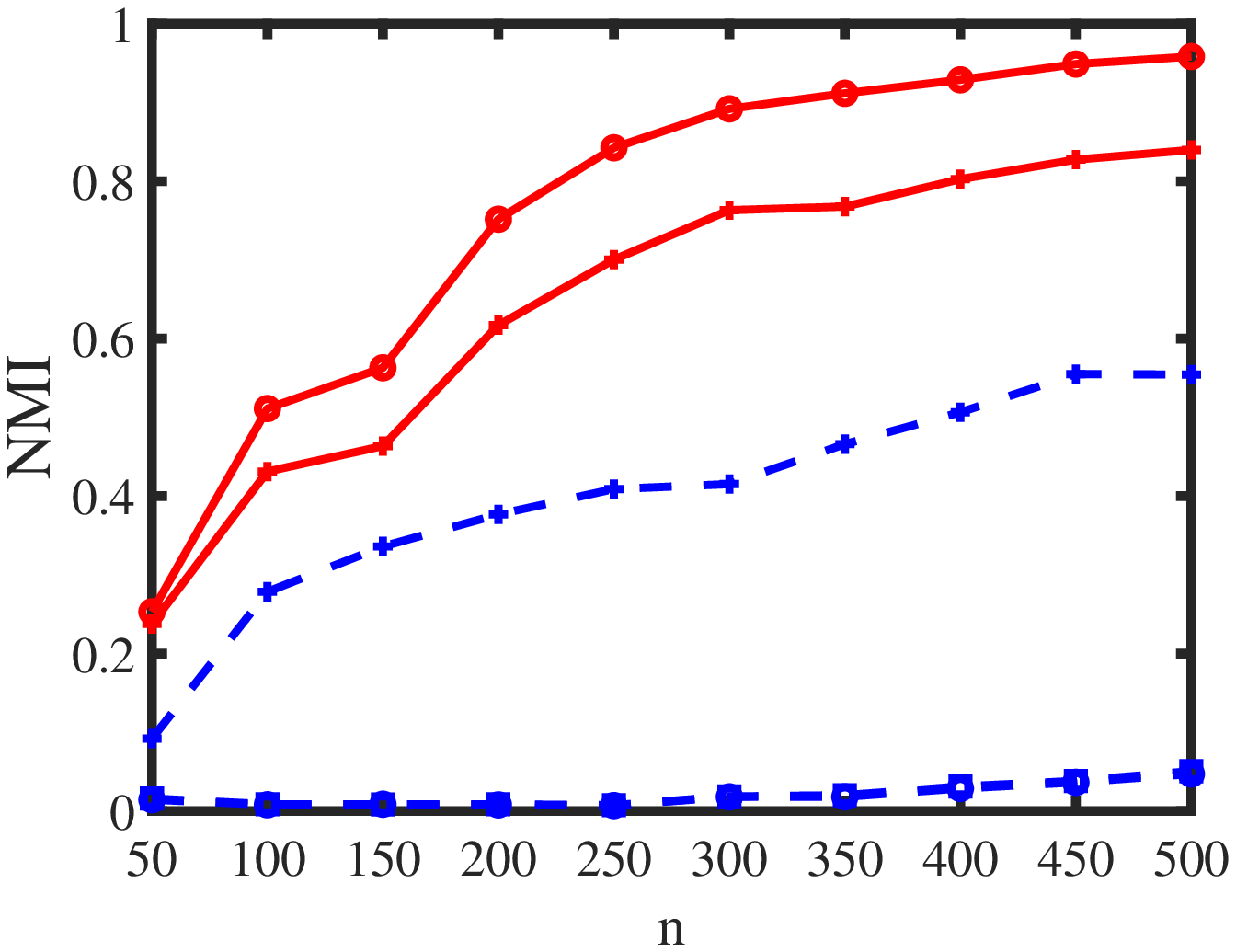}}
\subfigure[SIM 2(f)]{\includegraphics[width=0.24\textwidth]{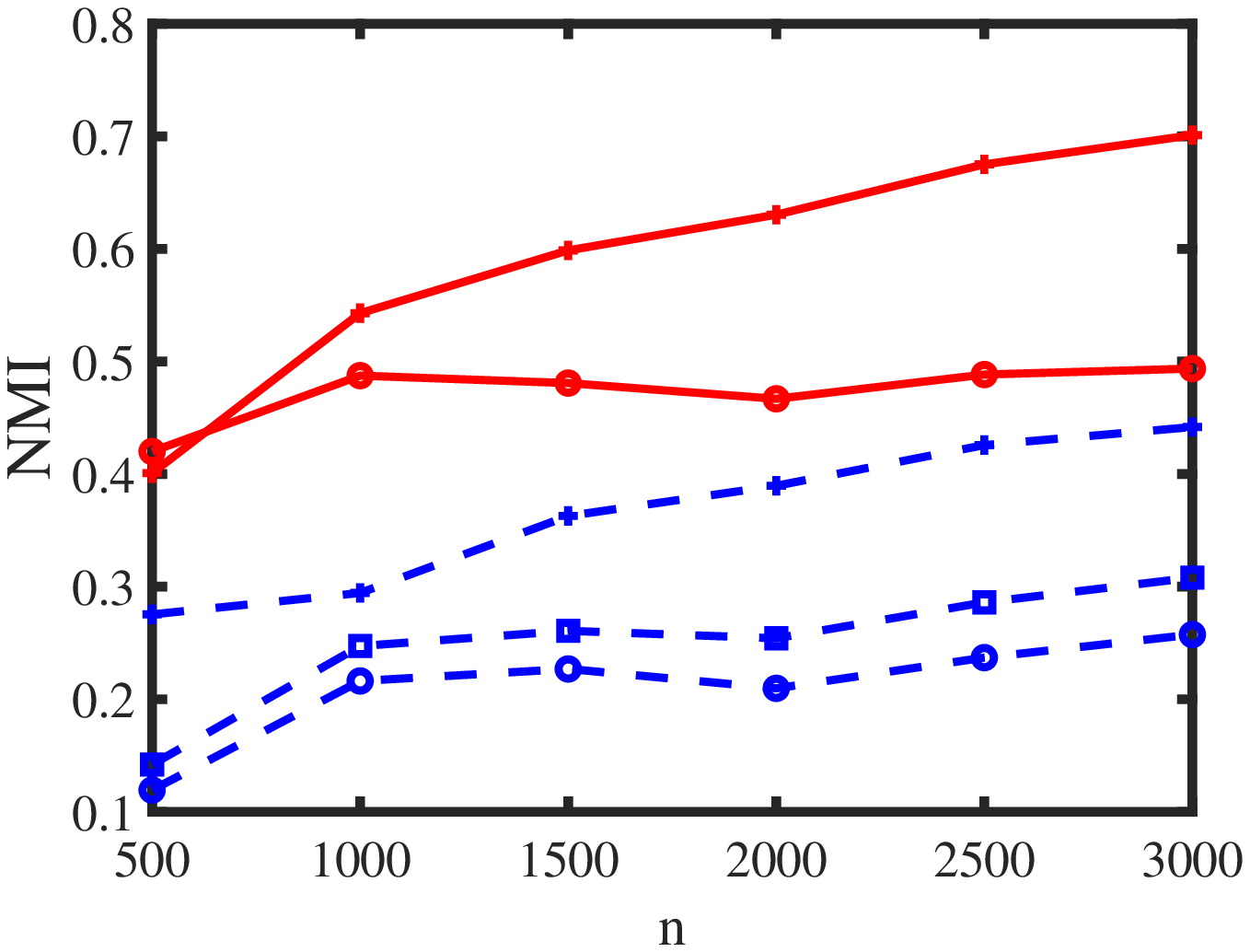}}
\subfigure[SIM 2(a)]{\includegraphics[width=0.24\textwidth]{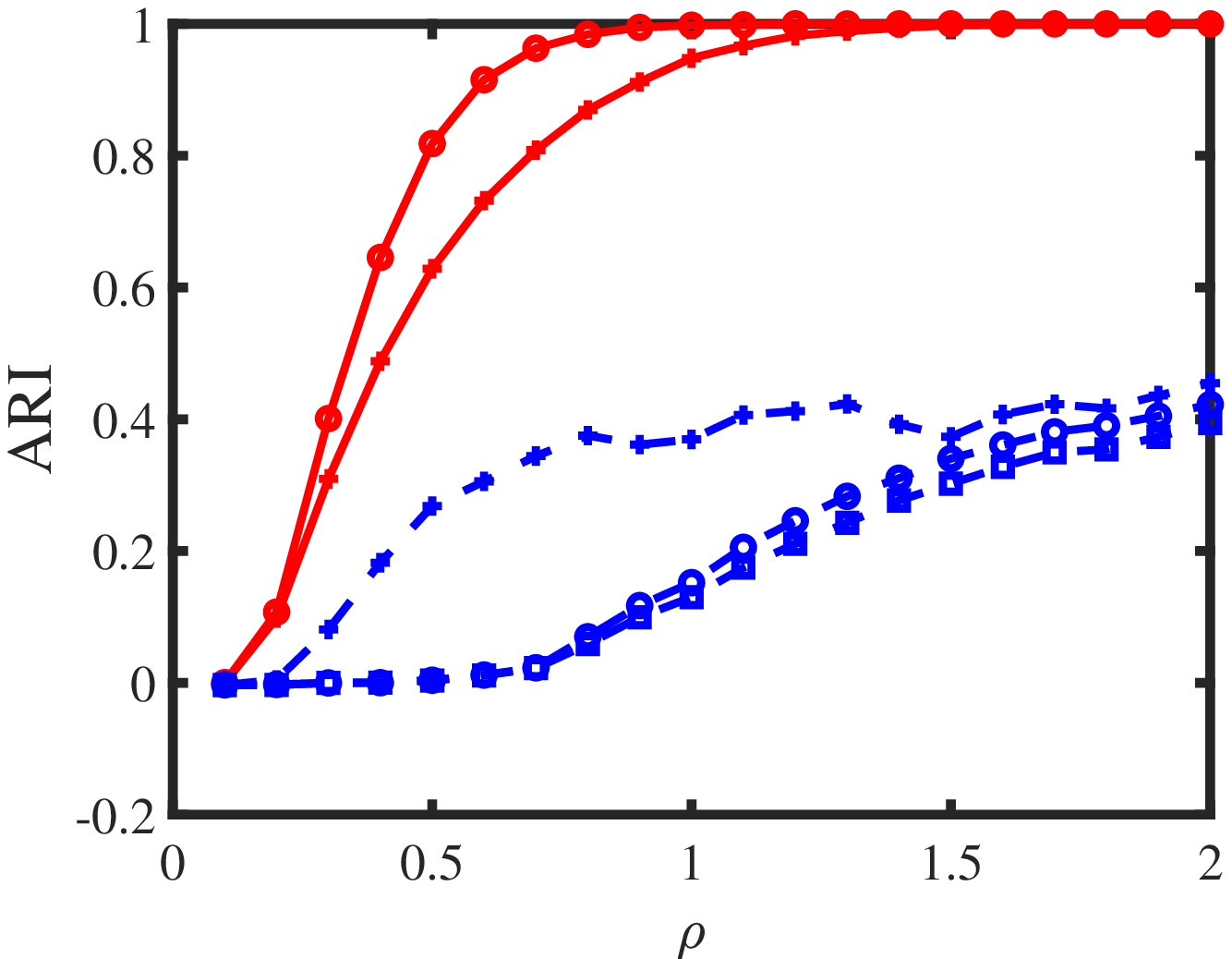}}
\subfigure[SIM 2(b)]{\includegraphics[width=0.24\textwidth]{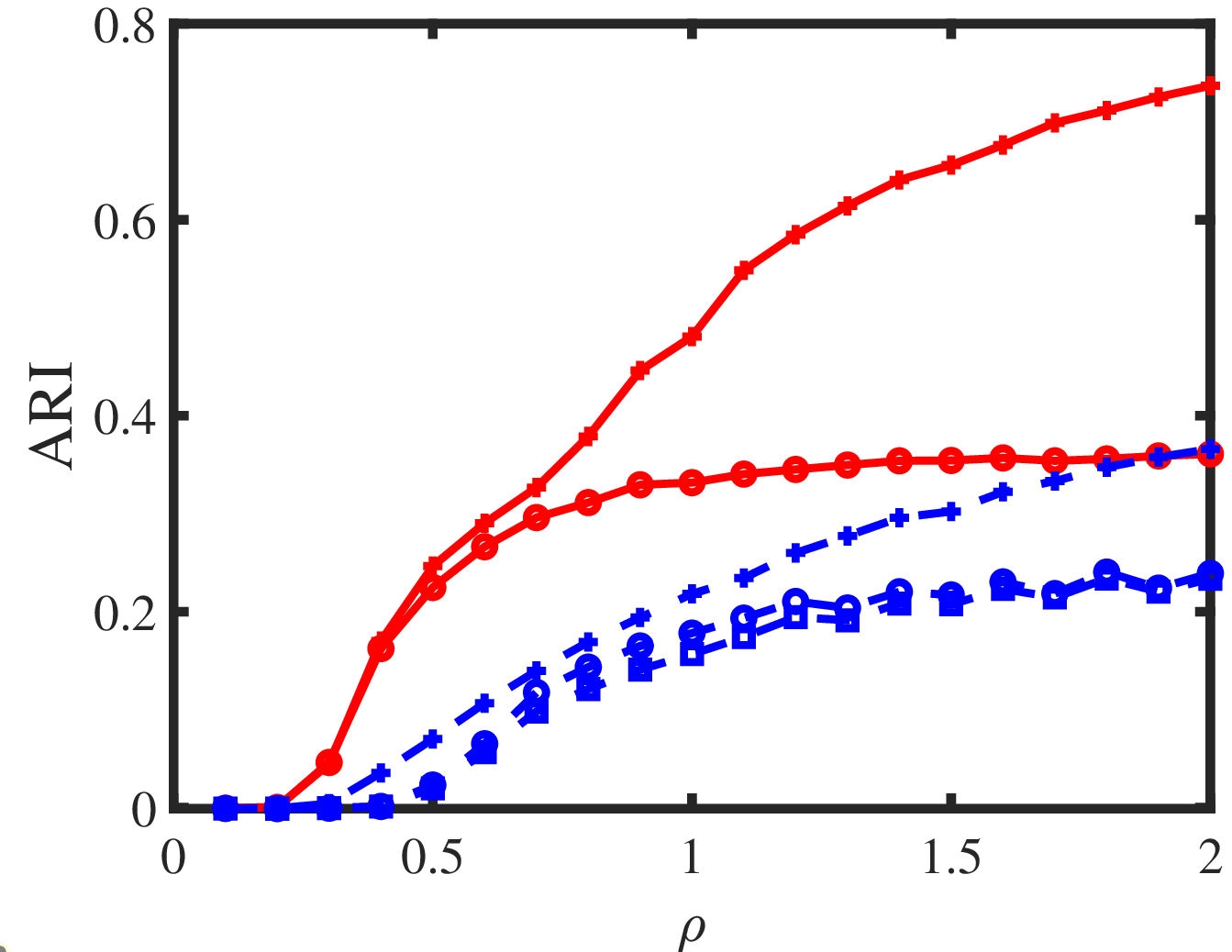}}
\subfigure[SIM 2(c)]{\includegraphics[width=0.24\textwidth]{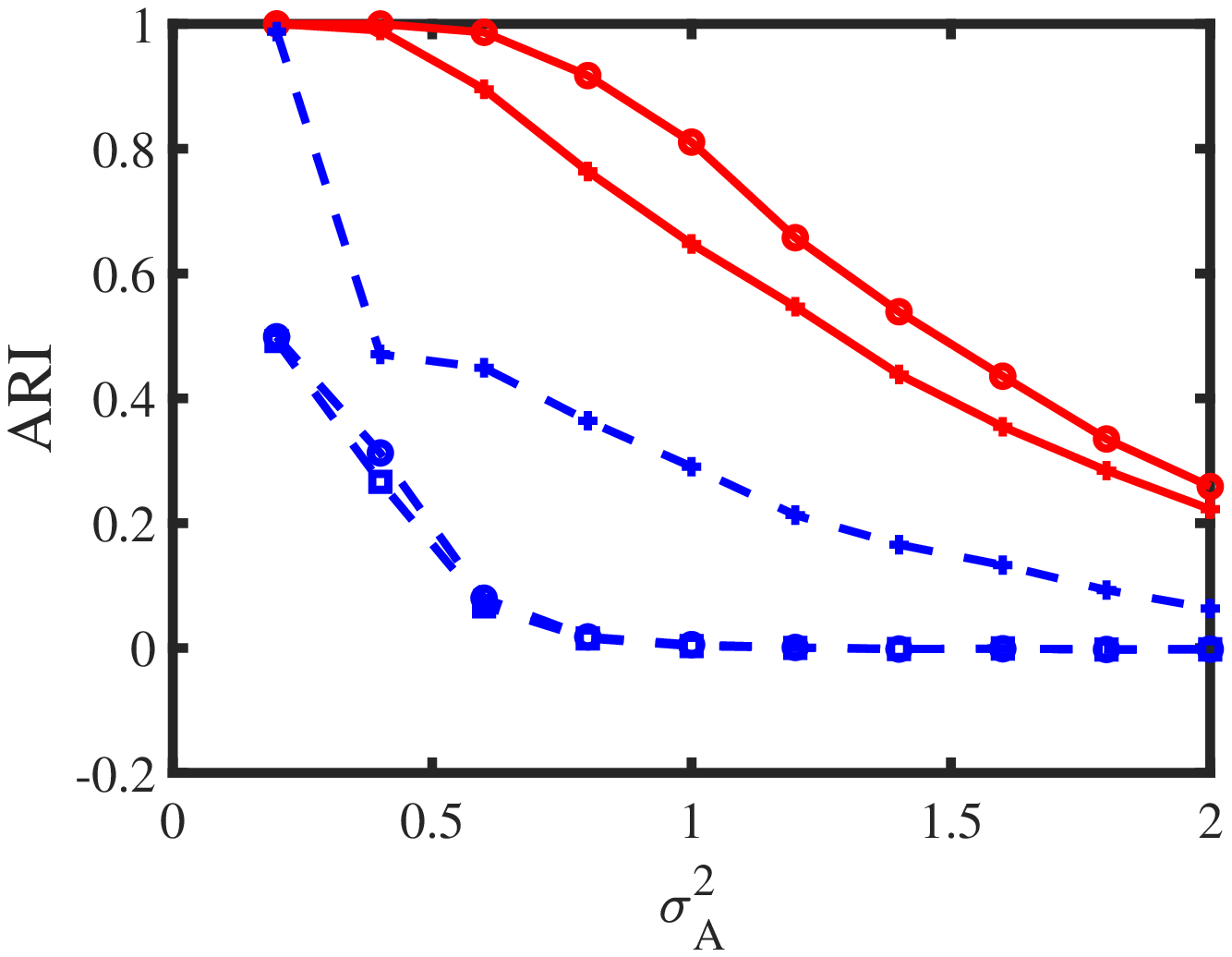}}
\subfigure[SIM 2(d)]{\includegraphics[width=0.24\textwidth]{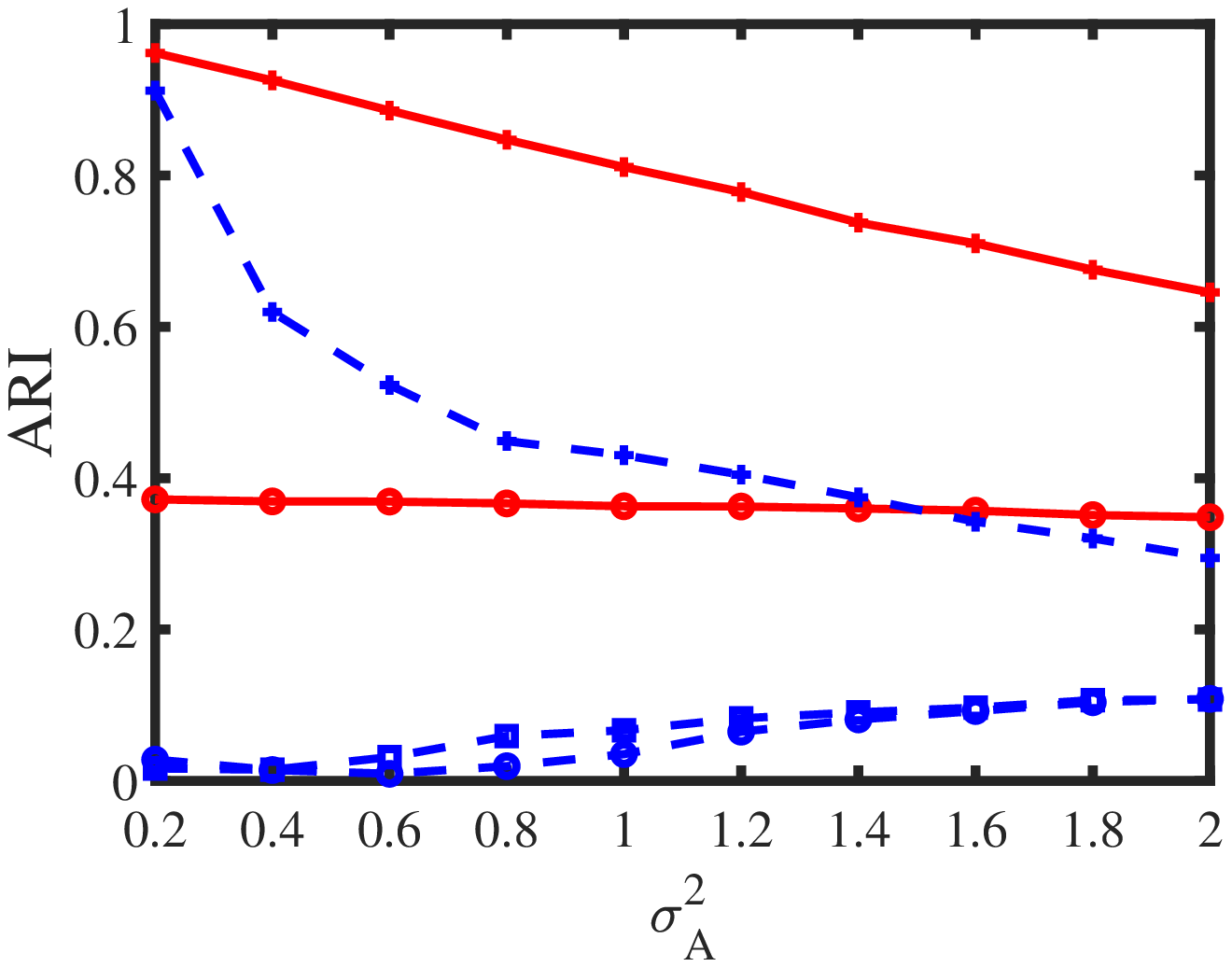}}
\subfigure[SIM 2(e)]{\includegraphics[width=0.24\textwidth]{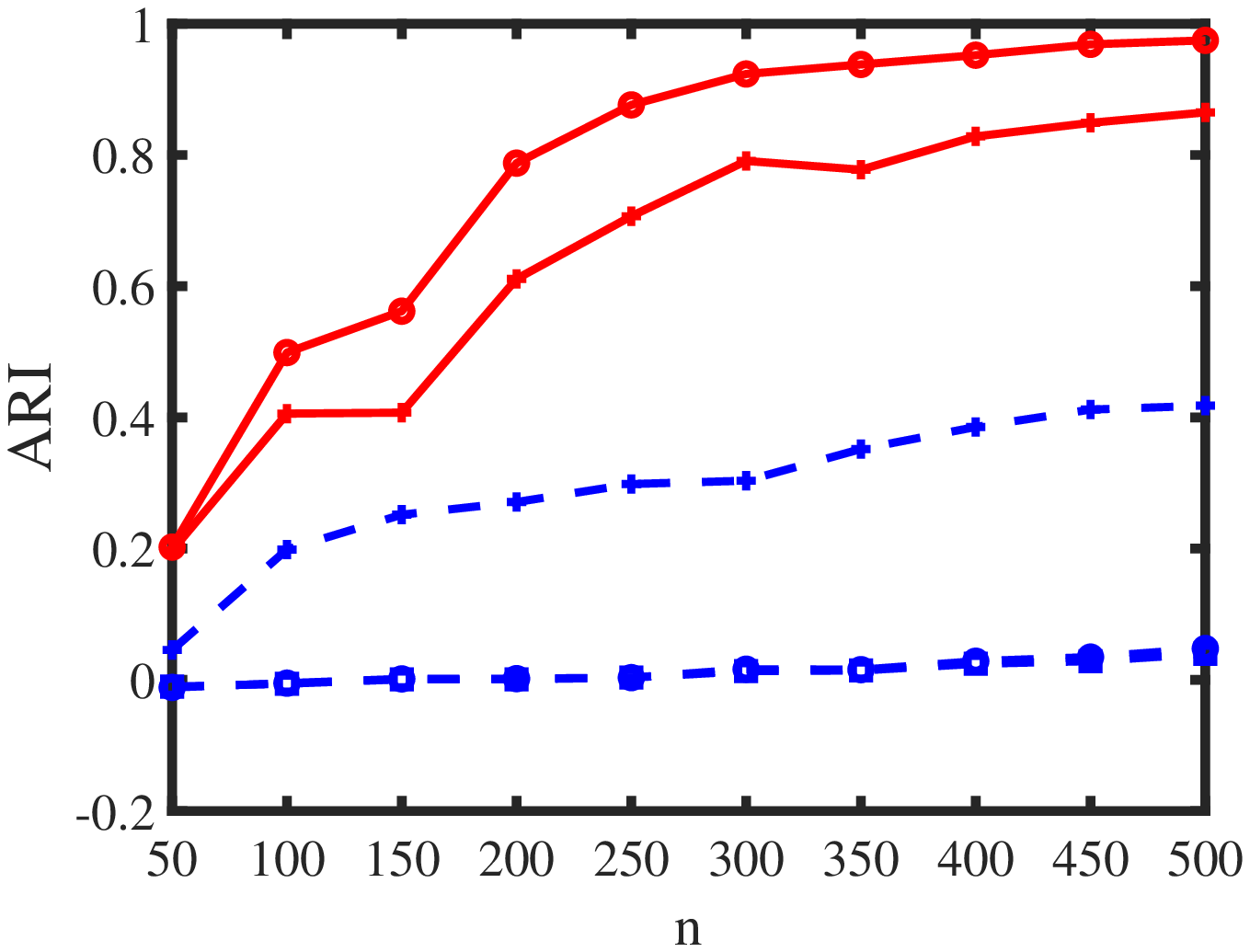}}
\subfigure[SIM 2(f)]{\includegraphics[width=0.24\textwidth]{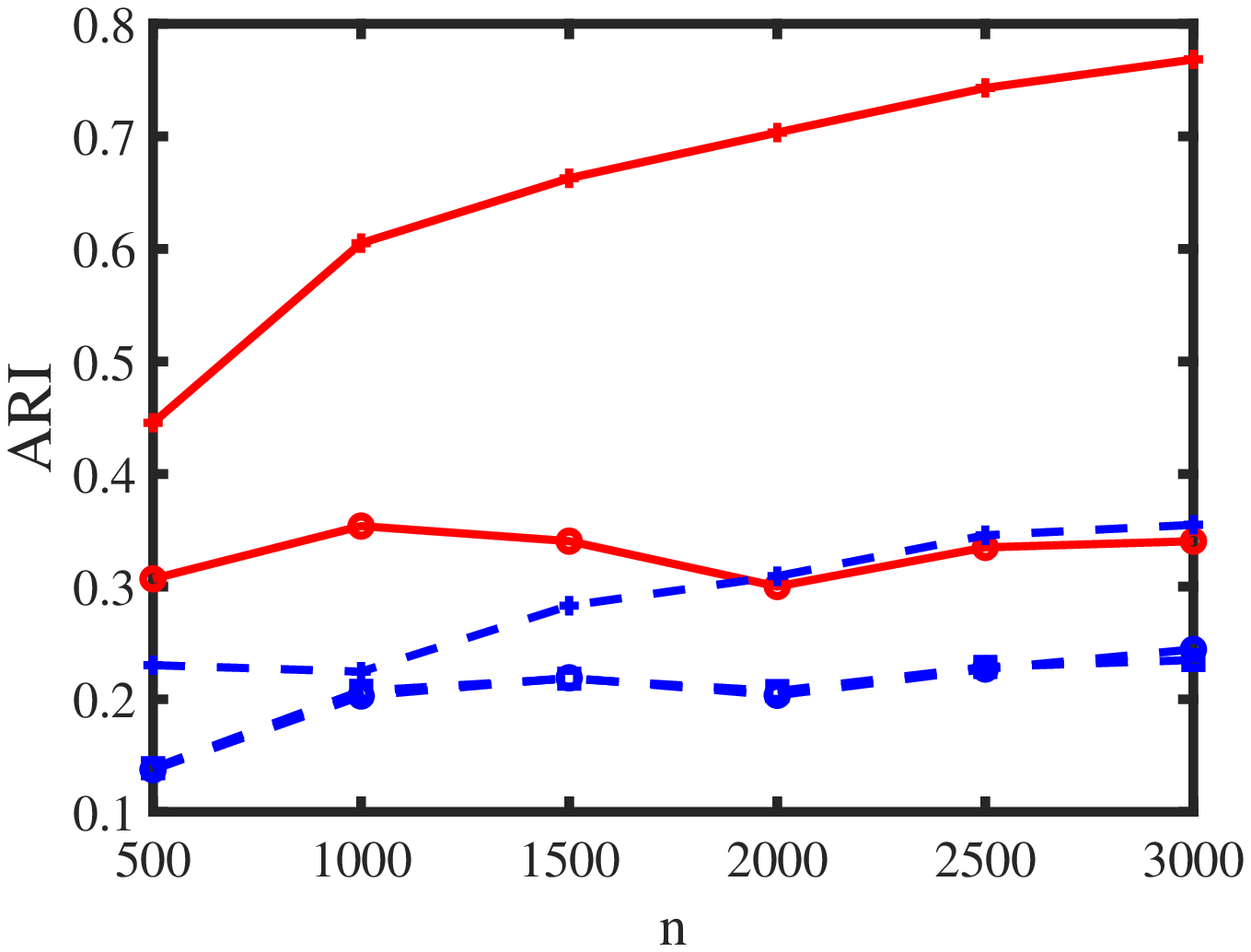}}
\caption{Numerical results of Simulation 2.}
\label{S2} %% label for entire figure
\end{figure}
\begin{rem}
Since the original DI-SIM and rD-SCORE are proposed based on a regularized Laplacian matrix, they fail to output when $A$ has negative elements for Normal distribution (and bipartite signed networks considered in Simulation 3). To make DI-SIM and rD-SCORE work,  we make all elements of $A$ positive by adding a sufficiently large constant for DI-SIM and rD-SCORE. For BiSC, nBiSC, and D-SCORE, they always work even when $A$ has negative entries because these three algorithms are designed based on the adjacency matrix instead of the regularized Laplacian matrix.
\end{rem}
The numerical results of Simulation 2 are reported in Figure \ref{S2}. For ErrorRates, the analysis for Simulations 2(a), 2(b), 2(e), and 2(f) is similar to that of Simulation 1. For Simulations 2(c) and 2(d), we see that BiSC and nBiSC perform poorer when increasing $\sigma^{2}_{A}$, and this is consistent with findings in Examples \ref{Normal} and \ref{NormalDC} since $\sigma^{2}_{A}$ is in the numerator position of theoretical bounds of error rates for BiSC and nBiSC. Meanwhile, we also see that our BiSC and nBiSC outperform DI-SIM, D-SCORE, and rD-SCORE for Simulation 2. The numerical results of criterions NMI and ARI shown in panels (g)-(f) of Figure \ref{S2} are consistent with that of criterion ErrorRate.
\subsubsection{Bipartite signed network}
For bipartite signed network when $\mathbb{P}(A(i_{r},j_{c})=1)=\frac{1+\Omega(i_{r},j_{c})}{2}$ and $\mathbb{P}(A(i_{r},j_{c})=-1)=\frac{1-\Omega(i_{r},j_{c})}{2}$, by Examples \ref{Signed} and \ref{SignedDC}, $P$ is set as $P_{2}$ and $\rho$ should be set no larger than 1.

\textbf{Simulation 3(a): changing $\rho$ under BiDFM}. Let $n_{r}=100,n_{c}=150$. Let $\rho$ range in $\{0.1,0.2,0.3,\ldots,1\}$.

\textbf{Simulation 3(b): changing $\rho$ under BiDCDFM}. Let $n_{r}=1000,n_{c}=1500$. Let $\rho$ range in $\{0.1,0.2,0.3,\ldots,1\}$.

\textbf{Simulation 3(c): changing $n$ under BiDFM}. Let $\rho=0.5$. Let $n$ range in $\{50,100,150,\ldots,500\}$.

\textbf{Simulation 3(d): changing $n$ under BiDCDFM}. Let $\rho=1$. Let $n$ range in $\{500,750,\ldots,3000\}$.
\begin{figure}[!h]
\centering
\subfigure[SIM 3(a)]{\includegraphics[width=0.24\textwidth]{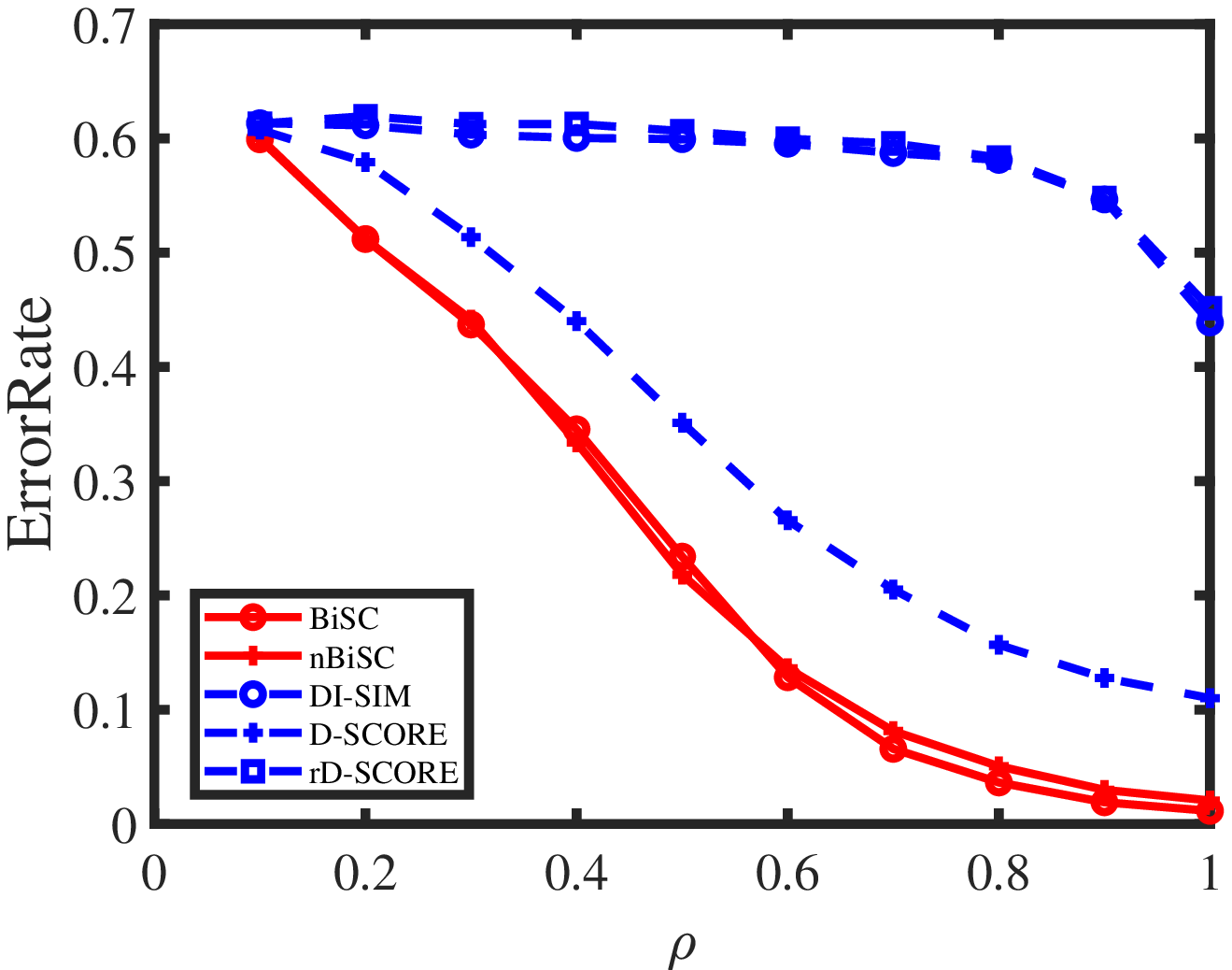}}
\subfigure[SIM 3(b)]{\includegraphics[width=0.24\textwidth]{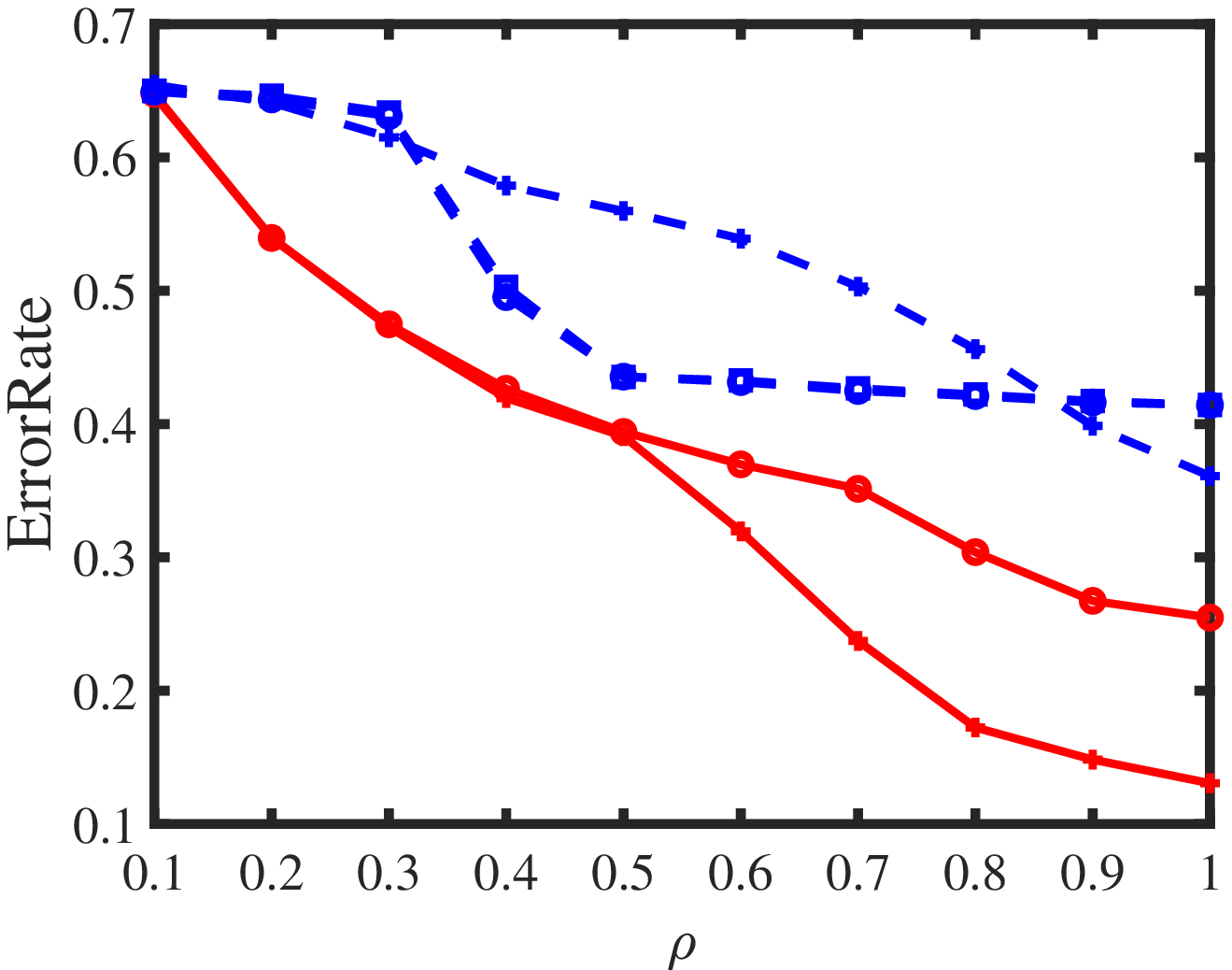}}
\subfigure[SIM 3(c)]{\includegraphics[width=0.24\textwidth]{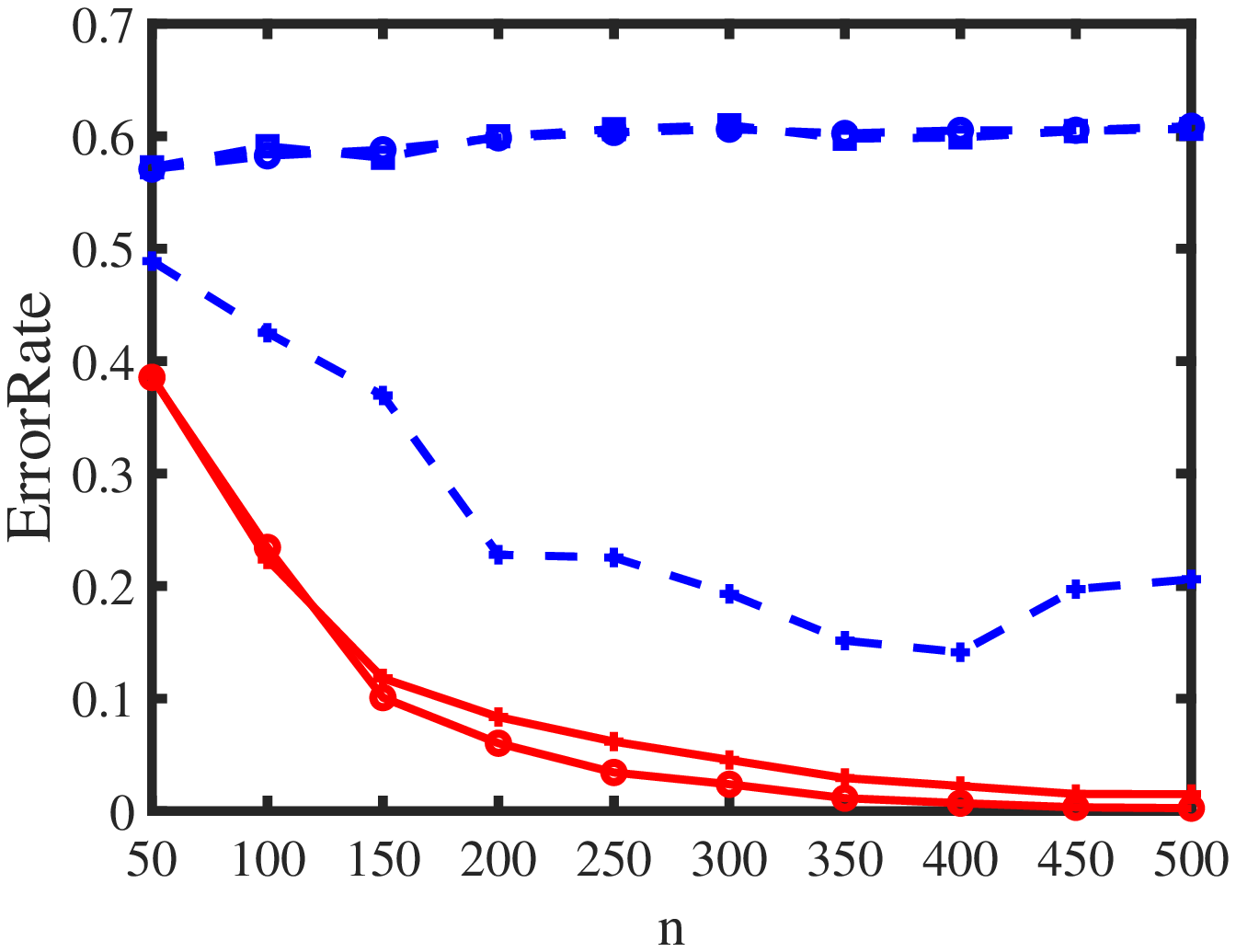}}
\subfigure[SIM 3(d)]{\includegraphics[width=0.24\textwidth]{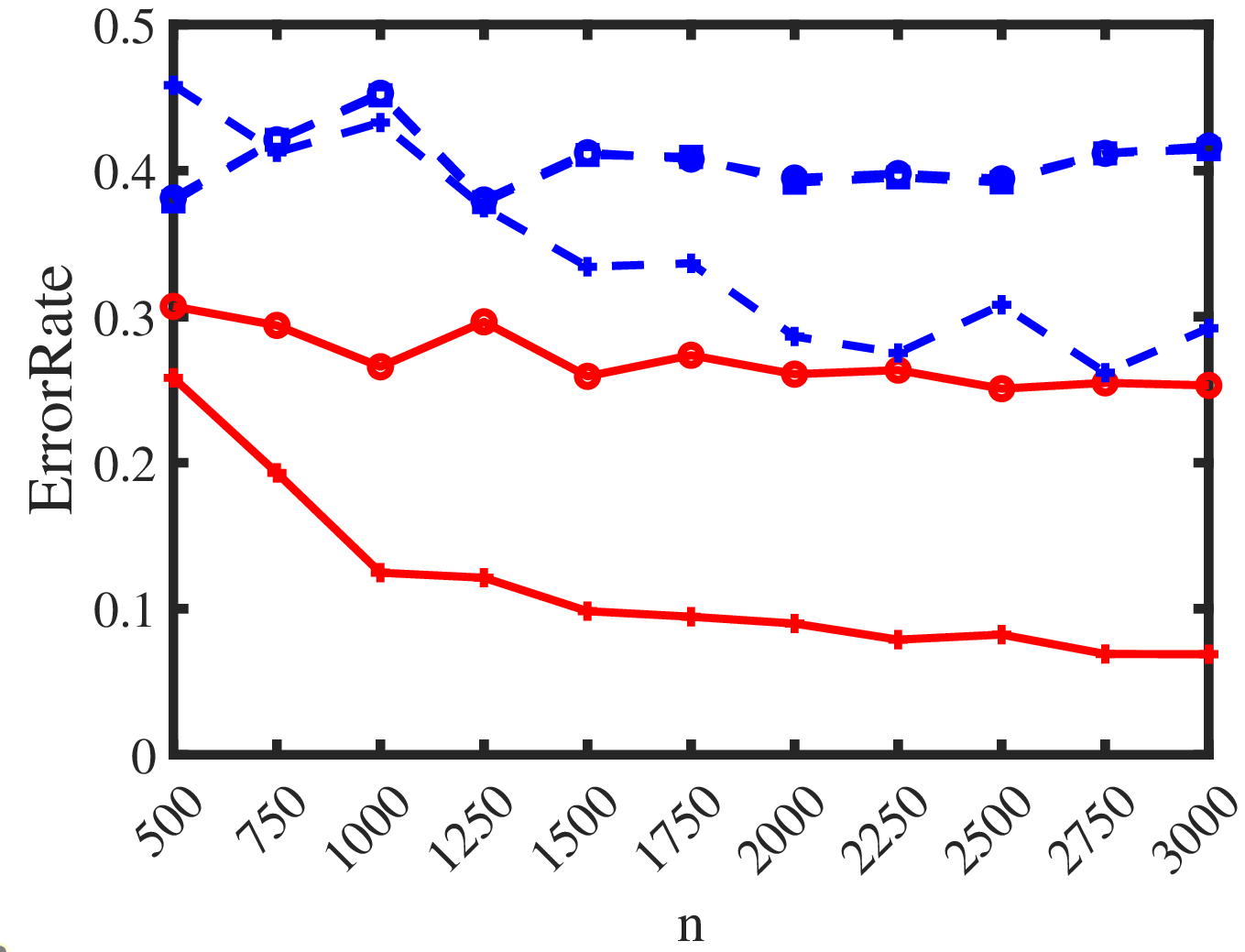}}
\subfigure[SIM 3(a)]{\includegraphics[width=0.24\textwidth]{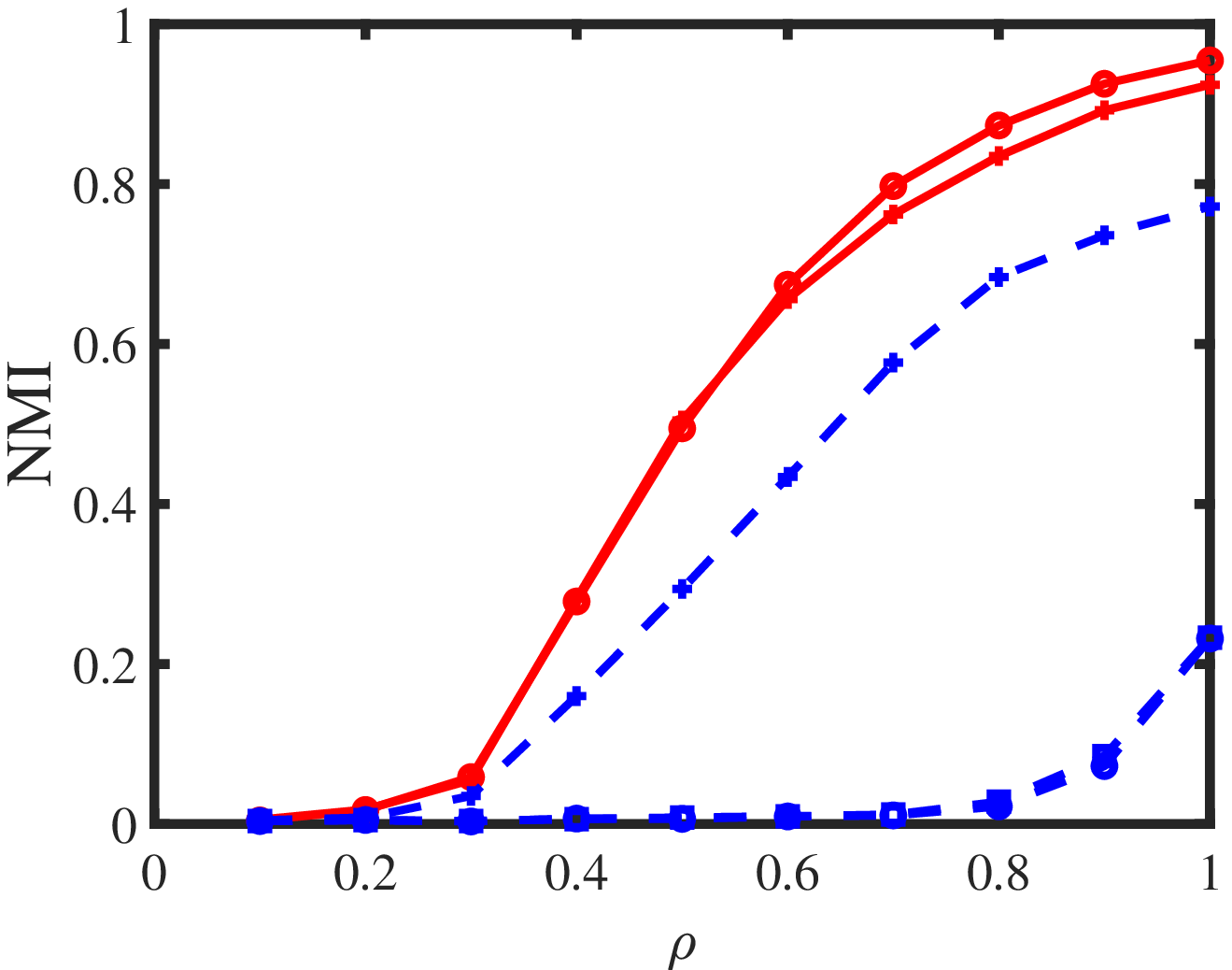}}
\subfigure[SIM 3(b)]{\includegraphics[width=0.24\textwidth]{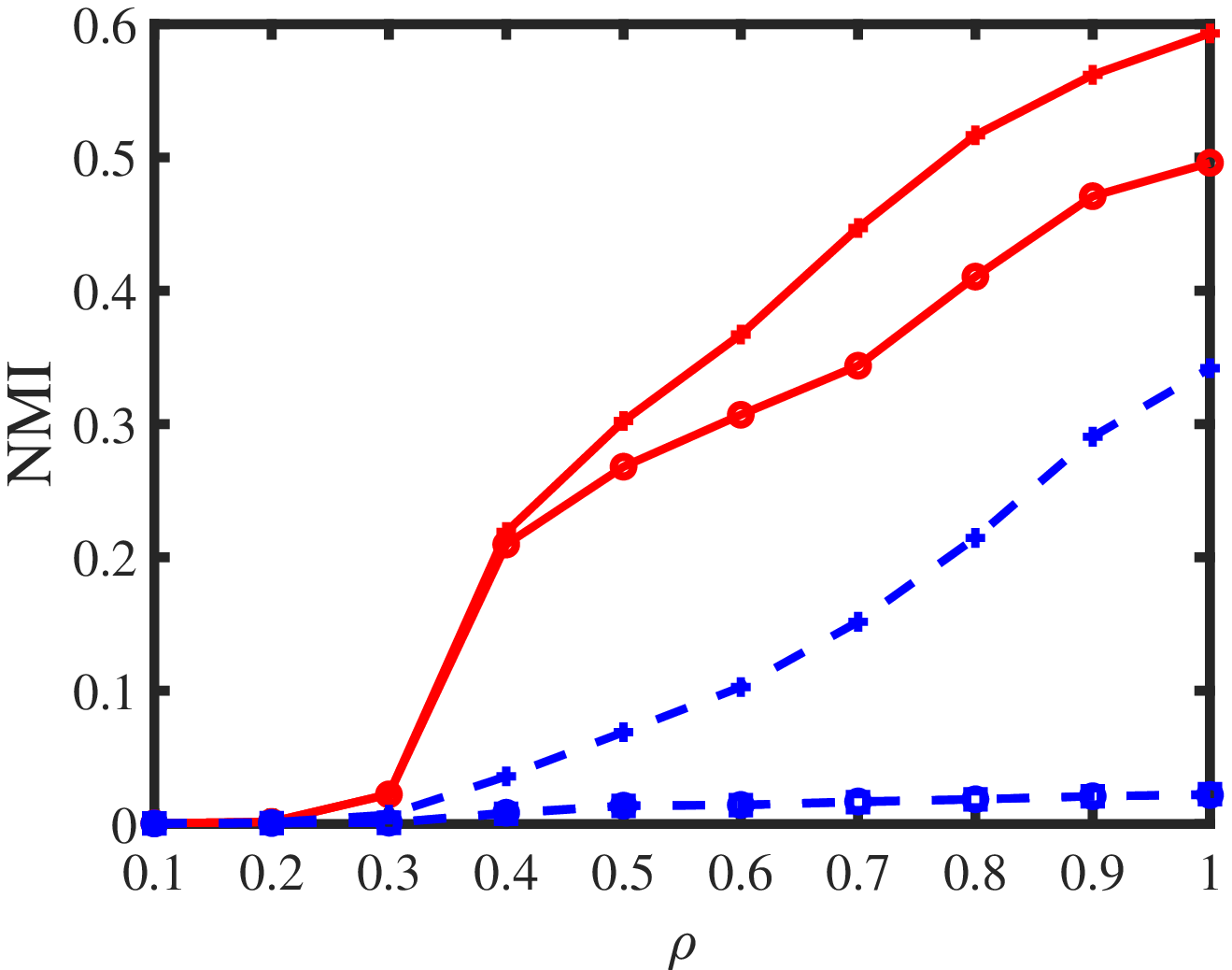}}
\subfigure[SIM 3(c)]{\includegraphics[width=0.24\textwidth]{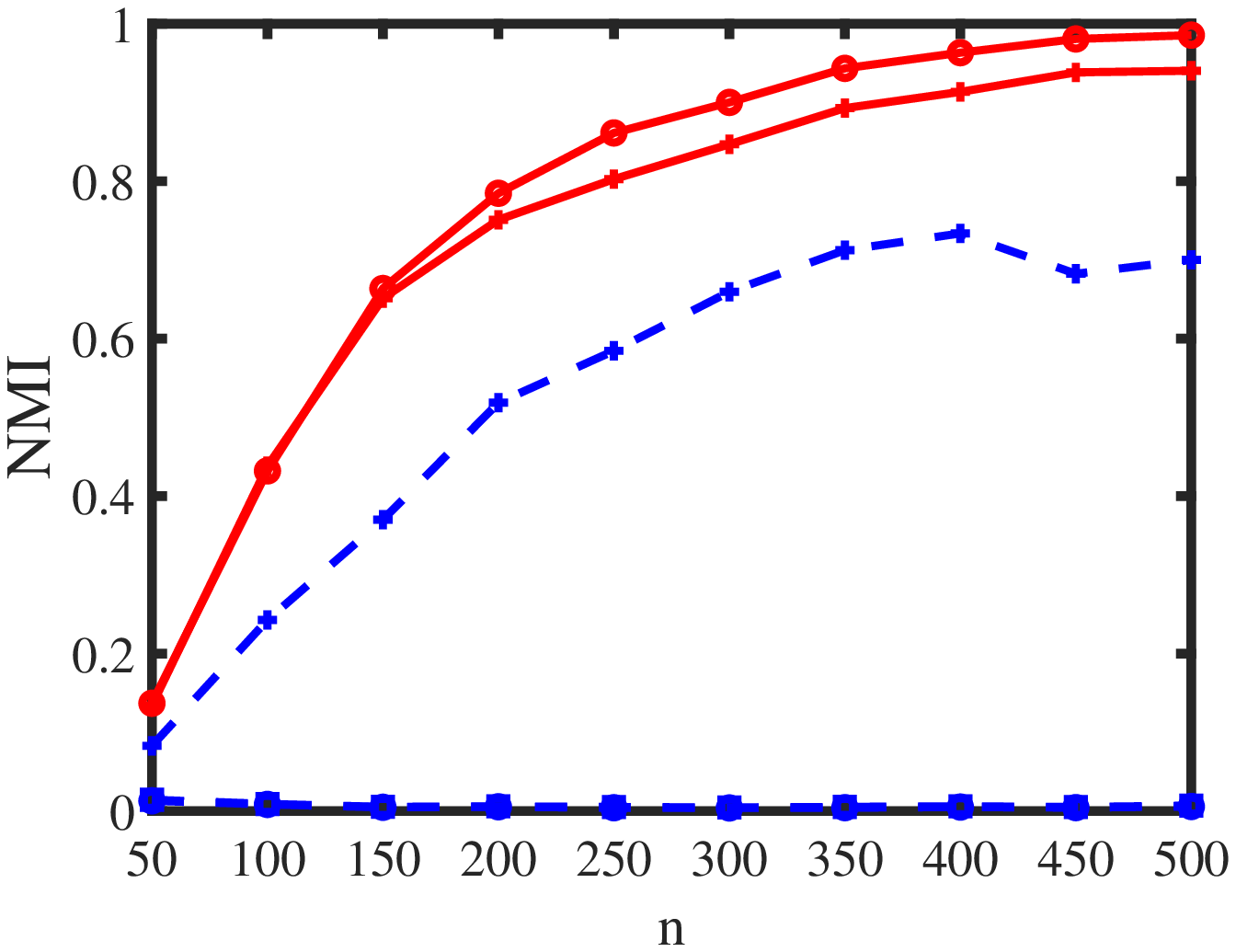}}
\subfigure[SIM 3(d)]{\includegraphics[width=0.24\textwidth]{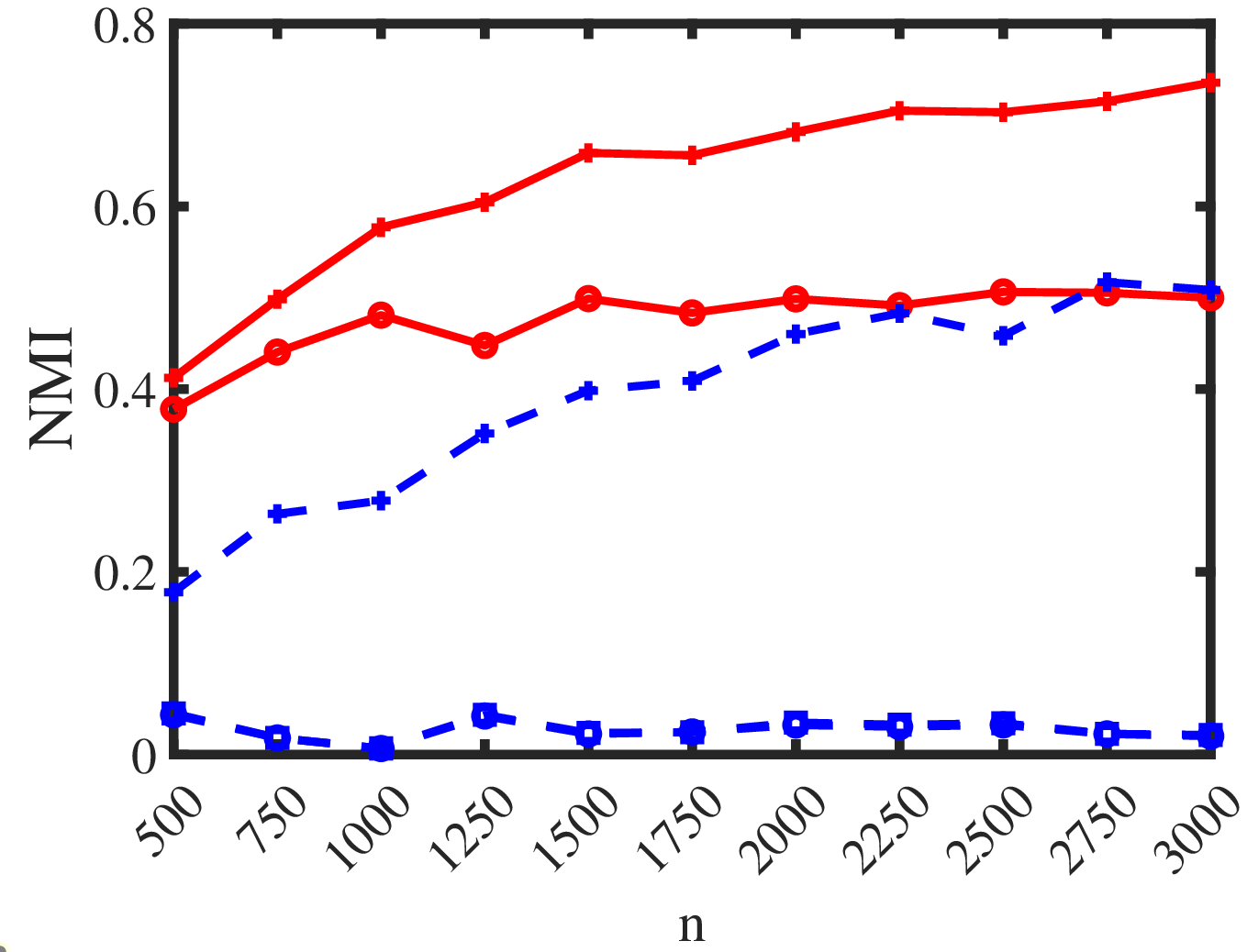}}
\subfigure[SIM 3(a)]{\includegraphics[width=0.24\textwidth]{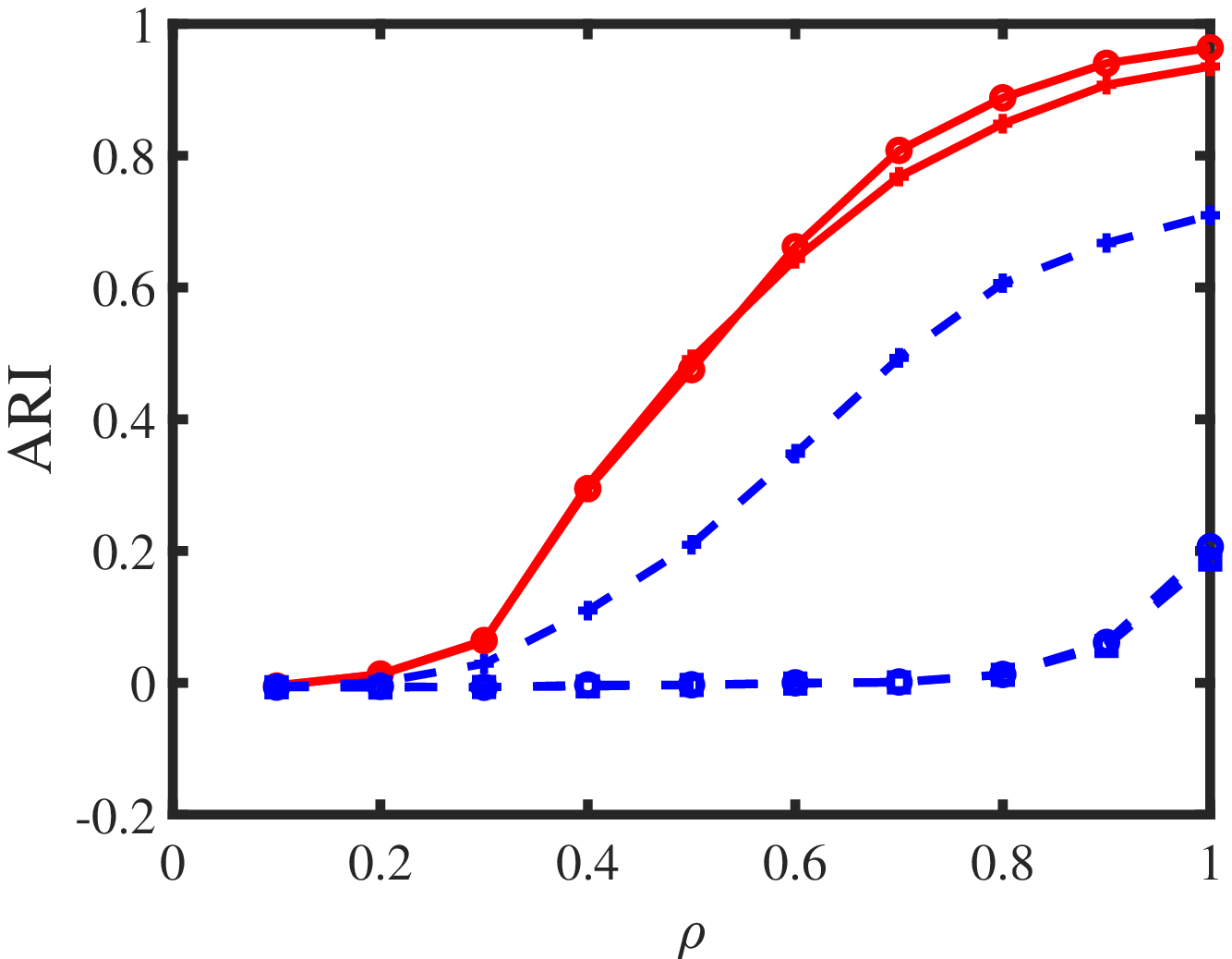}}
\subfigure[SIM 3(b)]{\includegraphics[width=0.24\textwidth]{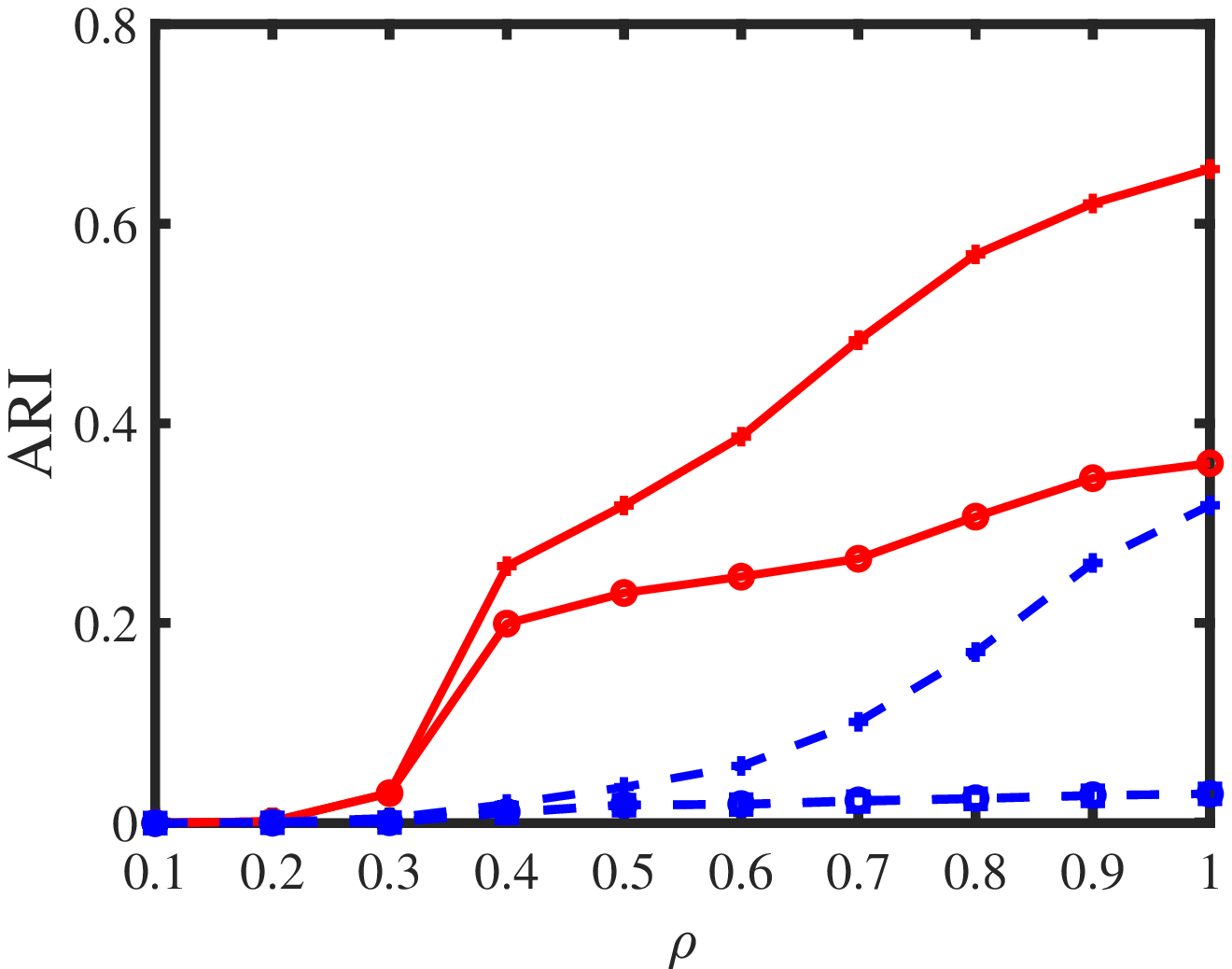}}
\subfigure[SIM 3(c)]{\includegraphics[width=0.24\textwidth]{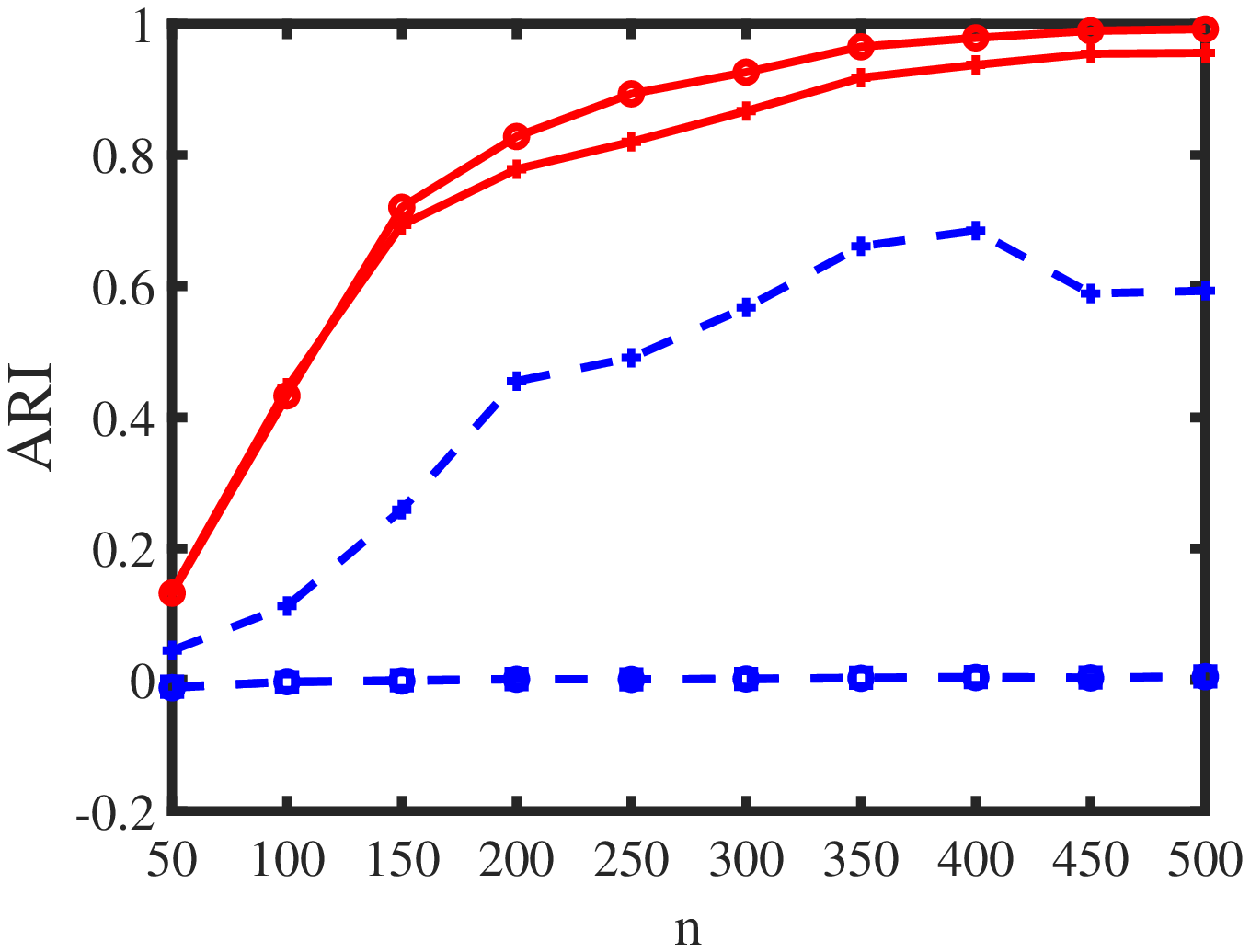}}
\subfigure[SIM 3(d)]{\includegraphics[width=0.24\textwidth]{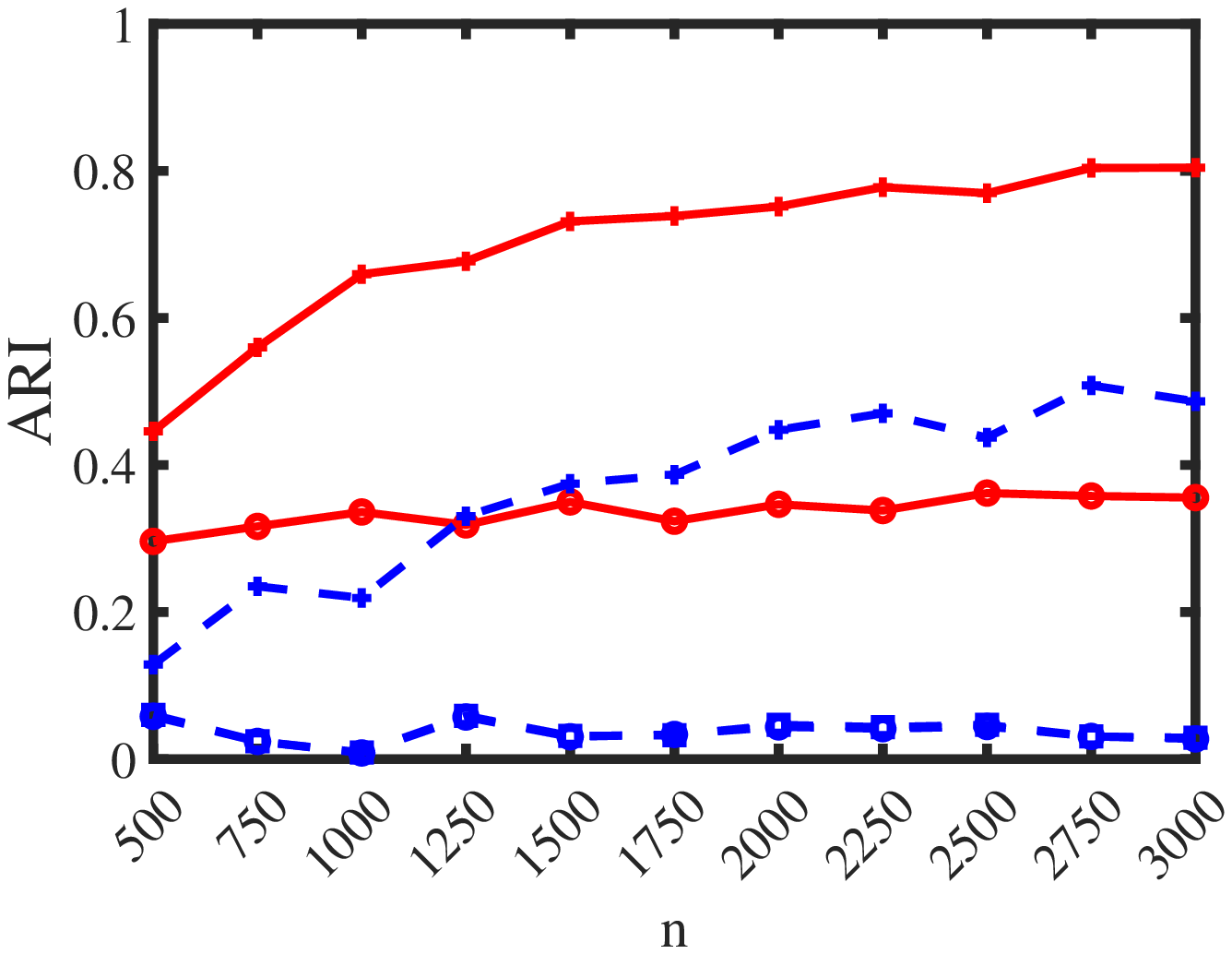}}
\caption{Numerical results of Simulation 3.}
\label{S3}
\end{figure}

The numerical results of Simulation 3 are reported in Figure \ref{S3}.  In terms of ErrorRate, from subfigures (a) and (b), we can see that the error rates of BiSC and nBiSC decrease rapidly and the error rates of the other three algorithms decrease slowly when $\rho $ increases. From subfigures (c), when $n$ increases, the error rates decrease for all methods under model BiDFM. While in subfigure (d), BiSC performs not as good as nBiSC under BiDCDFM and both BiSC and nBiSC perform better than their competitors. Thus for bipartite signed networks, BiSC and nBiSC can have similar performances under BiDFM, while when we consider the degree heterogeneity, BiSC may have large error rate. Numerical results of NMI and ARI displayed in Figure \ref{S3} (e)-(l) are consistent with that of ErrorRate. The numerical result supports analysis in Examples \ref{Signed} and \ref{SignedDC}.
\subsection{Real-world datasets}
In this section, we apply our proposed algorithms to several real-world directed weighted networks. Table \ref{realdata1} presents basic information about networks and Table \ref{realdata2} summaries the statistics of these datasets.
For all these networks, row nodes are the same as column nodes. Political blogs network can be downloaded from \url{http://www-personal.umich.edu/~mejn/netdata/}, Facebook-like Social Network can be downloaded from \url{https://toreopsahl.com/datasets/#online_social_network}, and the other three datasets are downloaded from \url{http://konect.cc/} (see also \cite{kunegis2013konect}). For visualization, we plot adjacency matrices of Crisis in a Cloister and Dutch college in Figure \ref{RealA}.

\begin{table}[h!]
\footnotesize
	\centering
	\caption{Basic information of networks}
	\label{realdata1}
	\resizebox{\columnwidth}{!}{
	\begin{tabular}{cccccccccc}
\hline\hline &Source&Directed?&Weighted?&Node meaning&Edge meaning\\
\hline
Political blogs&\cite{adamic2005political}&Yes&Yes&weblog&hyperlink\\ Crisis in a Cloister&\cite{breiger1975algorithm}&Yes&Yes&Monk&	Ratings\\
Dutch college&\cite{van1999friendship}&Yes&Yes&Student&Rating\\
Highschool&\cite{coleman1964introduction}&Yes&Yes&Boy&Friendship\\
Facebook-like Social Network&\cite{opsahl2009clustering}&Yes&Yes&User&Messages\\
\hline\hline
\end{tabular}}
\end{table}

\begin{table}[h!]
\footnotesize
	\centering
	\caption{Summarized statistics for network datasets}
	\label{realdata2}
	\resizebox{\columnwidth}{!}{
\begin{tabular}{cccccccccc}
\hline\hline
&$n$&$\mathrm{max}_{i,j}A(i,j)$&$\mathrm{min}_{i,j}A(i,j)$&$\sum_{i,j}^{n}|A(i,j)|$&\#Edges&\%Positive edges\\
\hline
Political blogs&1490&2&0&19090&19025&100\%\\
Crisis in a Cloister&18&1&-1&184&189&53.97\%\\
Dutch college&32&3&-1&162&3062&98.14\%\\
Highschool&70&2&0&506&366&100\%\\
Facebook-like Social Network&1899&98&0&59835&20296&100\%\\
\hline\hline
\end{tabular}}
\end{table}

\begin{figure}
\centering
\subfigure[Crisis in a Cloister]{\includegraphics[width=0.49\textwidth]{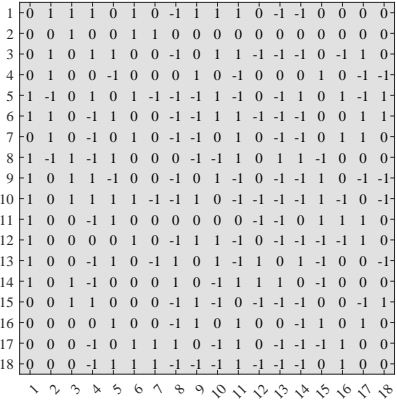}}
\subfigure[Dutch college]{\includegraphics[width=0.49\textwidth]{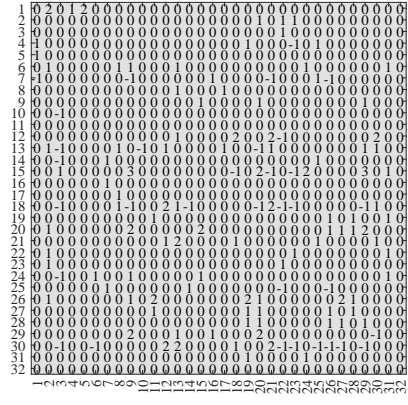}}
\caption{Adjacency matrices of Crisis in a Cloister and Dutch college.}
\label{RealA} %% label for entire figure
\end{figure}
Before applying our algorithms to these datasets, we introduce node degree and how to estimate the number of clusters. For adjacency matrix $A\in\mathbb{R}^{n_{r}\times n_{c}}$ of a bipartite weighted network, since $A$ may contain negative elements, we define node degrees as below: let $d_{r}(i_{r})=\sum_{j_{c}=1}^{n_{c}}|A(i_{r},j_{c})|$ be the degree of row node $i_{r}$, and let $d_{c}(j_{c})=\sum_{i_{r}=1}^{n_{r}}|A(i_{r},j_{c})|$ be the degree of column node $j_{c}$. For convenience, call $d_{r}(i_{r})$ as the out-degree of row node $i_{r}$ and $d_{c}(j_{c})$ as the in-degree of column node $j_{c}$.  $d_{r}$ and $d_{c}$ records node degree variety. Meanwhile, for real-world directed weighted networks, since we have no additional information to find the exact numbers of row and column clusters, we set $K_{r}=K_{c}=K$.

\subsubsection{Real networks with ground truth}
First, we consider the Political blogs network. Since the ground-truth labels for this network are suggested by the original authors or data creators, we can calculate ErrorRate, NMI, and ARI for the 5 aforementioned algorithms on this network. Because there are two political parties ``liberal'' and ``conservative'', we set $K_{r}=K_{c}=2$ for this dataset. For Political blogs, the original data has 1490 nodes, and we call its adjacency matrix $A_{1490\times 1490}$ for convenience. For $A_{1490\times 1490}$, as shown in Table \ref{realdata2}, it has 19025 edges, where the weight of 65 edges is 2 and the weight of the rest 18960 edges is 1, i.e., the Political blogs network is nearly unweighted because the ratio of non-unit weight is $\frac{65}{19025}\approx0.0034$, a number close to zero. Let $\mathcal{I}_{r,0}=
\{i\in\{1,2,\ldots,1490\}: d_{r}(i)=0\}$ be the set containing nodes with zero out-degree, $\mathcal{I}_{c,0}=
\{i\in\{1,2,\ldots,1490\}: d_{c}(i)=0\}$ be the set containing nodes with zero in-degree, $\mathcal{I}_{0}=
\{i\in\{1,2,\ldots,1490\}: d_{r}(i)=0\mathrm{~and~}d_{c}(i)=0\}$ be the set containing nodes with both zero out-degree and zero in-degree, and $\mathcal{I}=
\{i\in\{1,2,\ldots,1490\}: d_{r}(i)=0\mathrm{~or~}d_{c}(i)=0\}$ be the set containing nodes with zero out-degree or zero in-degree, i.e., $\mathcal{I}_{0}=\mathcal{I}_{r,0}\cap\mathcal{I}_{c,0}$ and $\mathcal{I}=\mathcal{I}_{r,0}\cup\mathcal{I}_{c,0}$. We find that $|\mathcal{I}_{r,0}|=500, |\mathcal{I}_{c,0}|=425,|\mathcal{I}_{0}|=266$ and $|\mathcal{I}|=659$, i.e.,  $A_{1490\times 1490}$ has 500 nodes with zero out-degree, 425 nodes with zero in-degree, and 266 nodes with zero in-degree and out-degree. Based on this finding and the fact that the ground-truth labels for all 1490 nodes are known, we construct five new political blogs networks as below:
\begin{itemize}
  \item[(1)] Let $A_{1224\times 1224}$ be the adjacency matrix obtained by removing rows and columns in  $A_{1490\times 1490}$ respective to nodes in $\mathcal{I}_{0}$.
  \item[(2)] Let $A_{831\times 831}$ be the adjacency matrix obtained by removing rows and columns in  $A_{1490\times 1490}$ respective to nodes in $\mathcal{I}$.
  \item[(3)] Let $A_{990\times 1490}$ be the adjacency matrix obtained by removing rows in $A_{1490\times 1490}$ respective to nodes in $\mathcal{I}_{r,0}$.
  \item[(4)] Let $A_{1490\times 1065}$ be the adjacency matrix obtained by removing columns in $A_{1490\times 1490}$ respective to nodes in $\mathcal{I}_{c,0}$.
  \item[(5)] Let $A_{990\times 1065}$ be the adjacency matrix obtained by removing rows (and columns) in $A_{1490\times 1490}$ respective to nodes in $\mathcal{I}_{r,0}$ (and $\mathcal{I}_{c,0}$).
\end{itemize}
Note that we construct three bipartite weighted networks $A_{990\times 1490}, A_{1490\times 1065}$, and $A_{990\times 1065}$ with true node labels from the Political blogs network. In Figure \ref{Pblogs}, panels (a)-(l) plot the distributions of degrees, and panels (m)-(r) present the leading 8 singular values of $A$ for the Political blogs network. We see that the distributions have long tails which suggest the heterogeneity of node degrees, and eigengap suggests $K=2$ which is consistent with the ground truth since there are two political parties. Tabels \ref{ErrorRateBlogs}, \ref{NMIBlogs}, and \ref{ARIBlogs} record ErrorRate, NMI, and ARI, respectively. We see that nBiSC performs much better than BiSC. A possible reason for this phenomenon is that nBiSC considers the degree heterogeneity. Meanwhile, for $ A_{831\times 831}, A_{1490\times1065},$ and $A_{990\times 1065}$, nBiSC enjoys competitive performances with DI-SIM, D-SCORE, and rD-SCORE; for $A_{1490\times 1490}$, nBiSC and DI-SIM perform better than the other three approaches; for $A_{1224\times 1224}$, nBiSC, DI-SIM, and D-SCORE outperform BiSC and rD-SCORE. For $A_{990\times1490}$, nBiSC and DI-SIM have similar performances that are slightly poorer than D-SCORE and rD-SCORE. Unlike the numerical results of Simulations 2 and 3 where nBiSC outperforms DI-SIM, D-SCORE, and rD-SCORE, nBiSC enjoys competitive performances with these three approaches on Political blogs. A possible reason for this phenomenon is, the Political blogs network is nearly unweighted as we analyzed early while DI-SIM, D-SCORE, and rD-SCORE are designed for unweighted networks.
\begin{figure}
	\centering
	\subfigure[$A_{1490\times1490}$: $d_{r}$]{\includegraphics[width=0.24\textwidth]{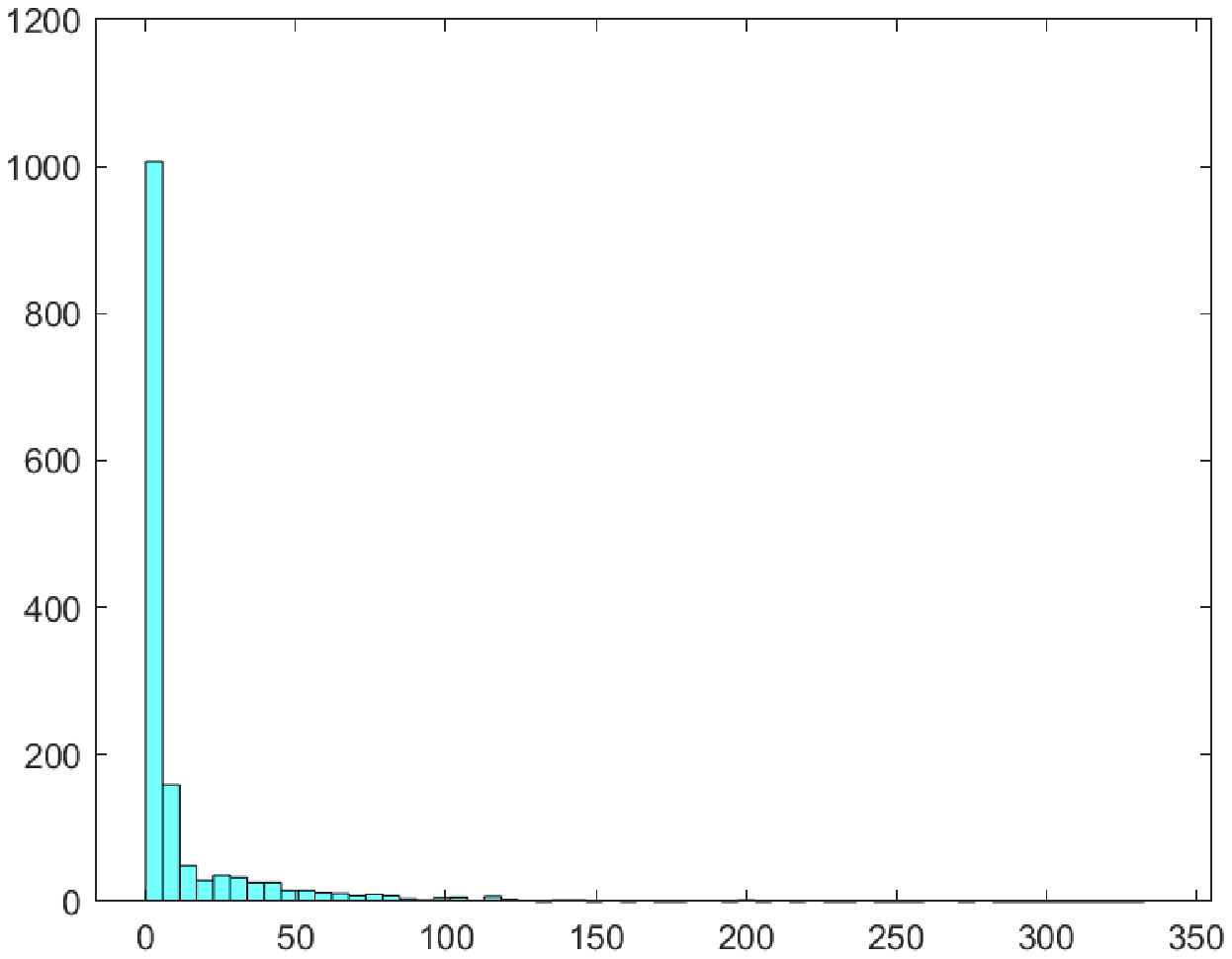}}
	\subfigure[$A_{1490\times1490}$: $d_{c}$]{\includegraphics[width=0.24\textwidth]{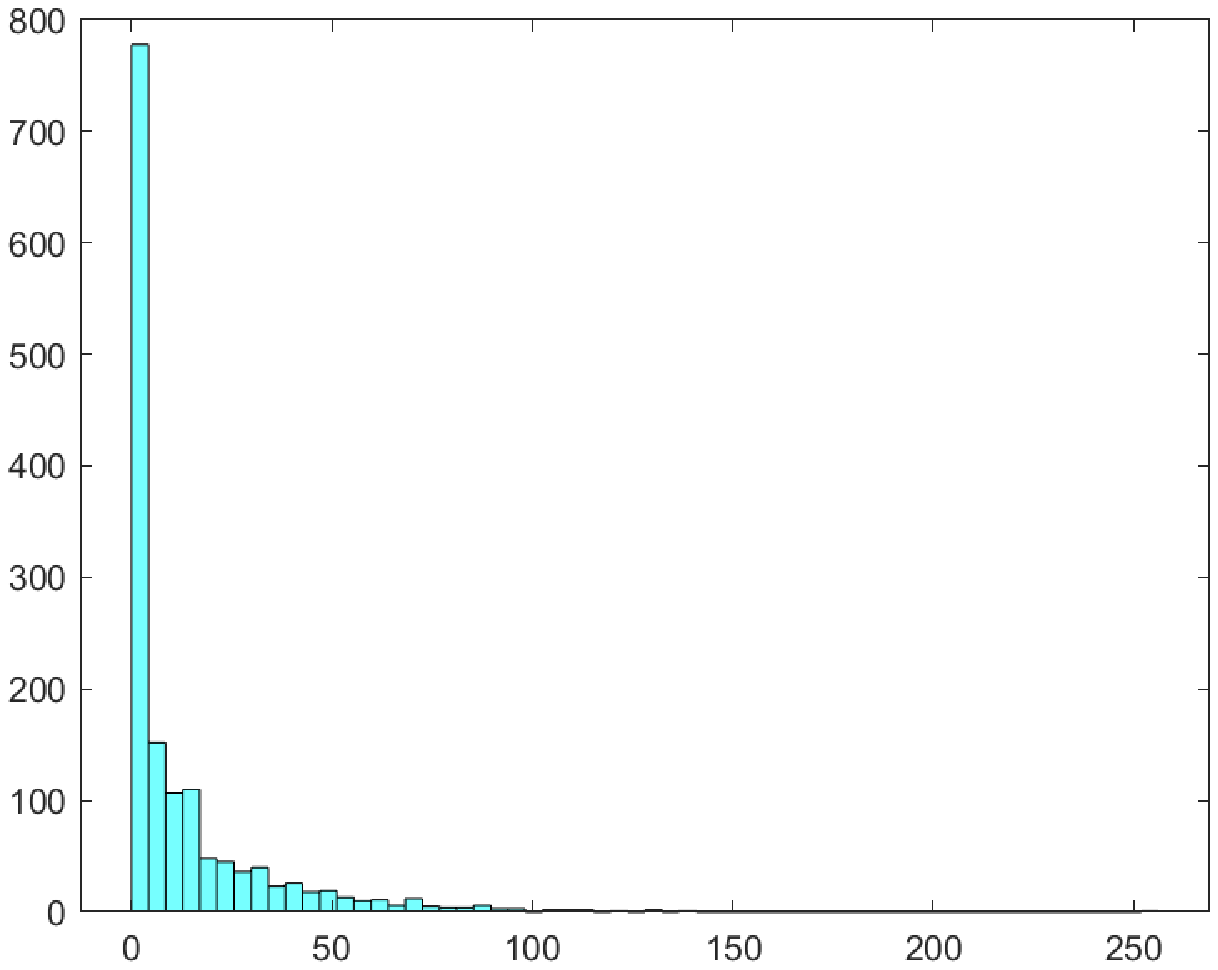}}
	\subfigure[$A_{1224\times1224}$: $d_{r}$]{\includegraphics[width=0.24\textwidth]{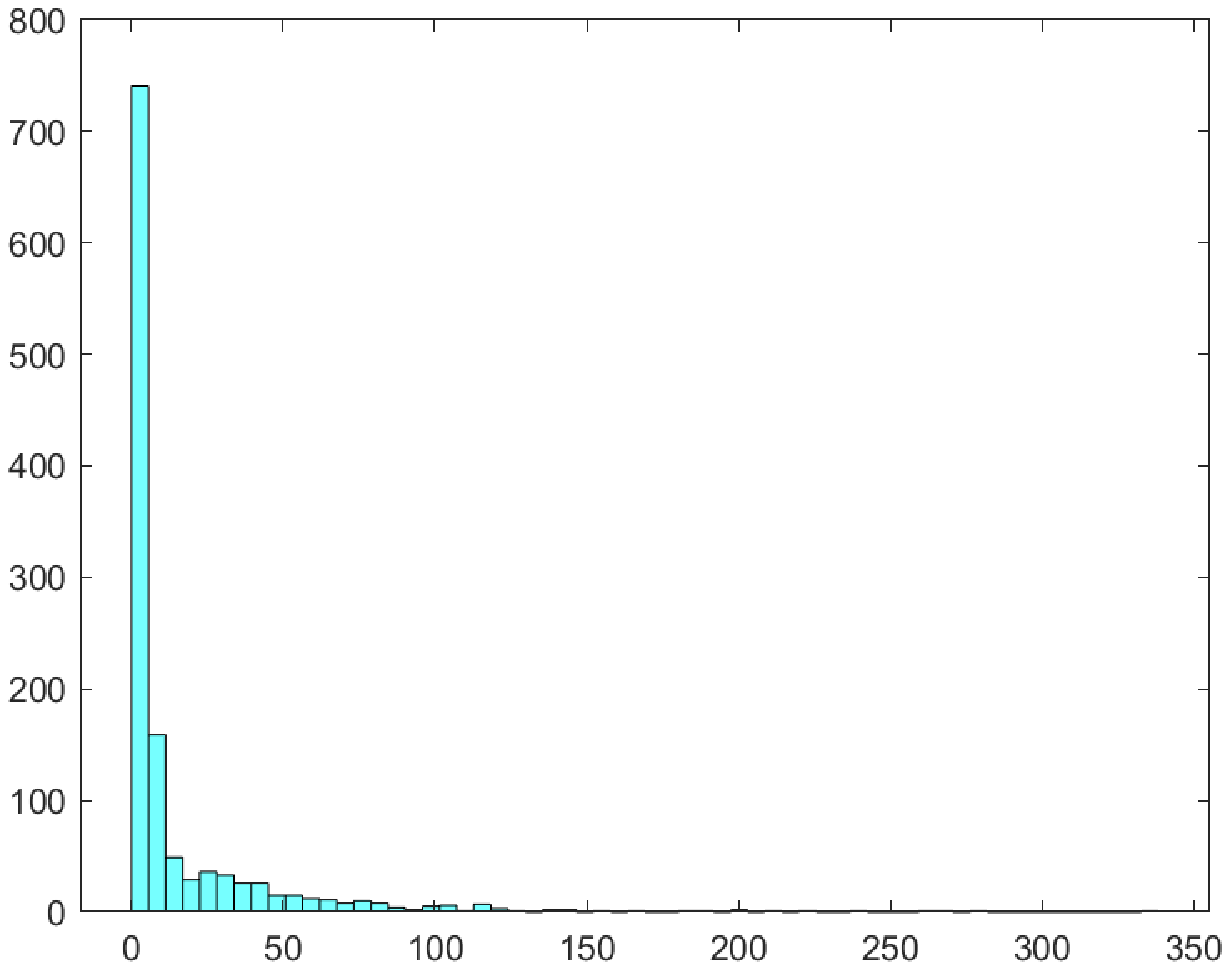}}
	\subfigure[$A_{1224\times1224}$: $d_{c}$]{\includegraphics[width=0.24\textwidth]{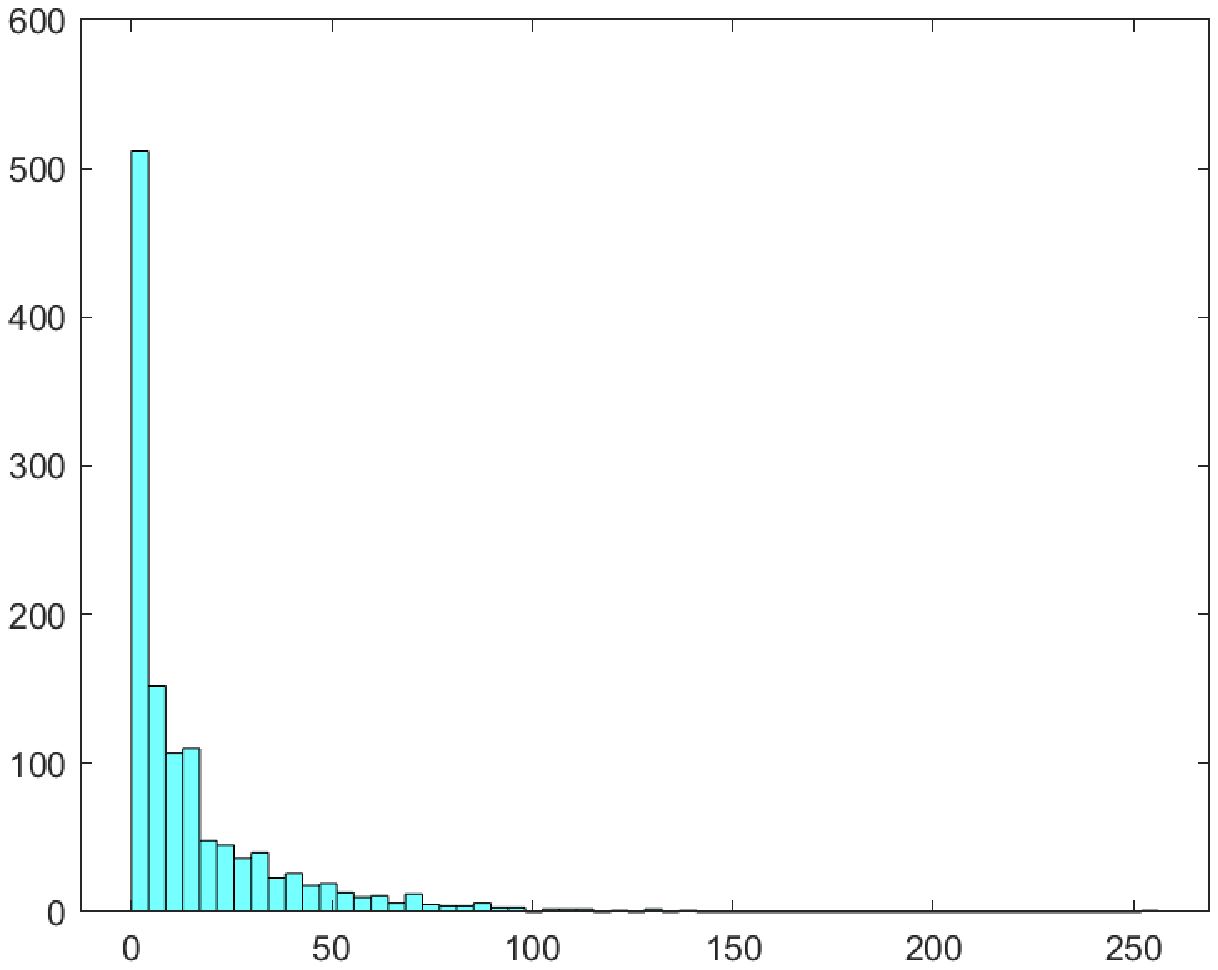}}
	\subfigure[$A_{831\times831}$: $d_{r}$]{\includegraphics[width=0.24\textwidth]{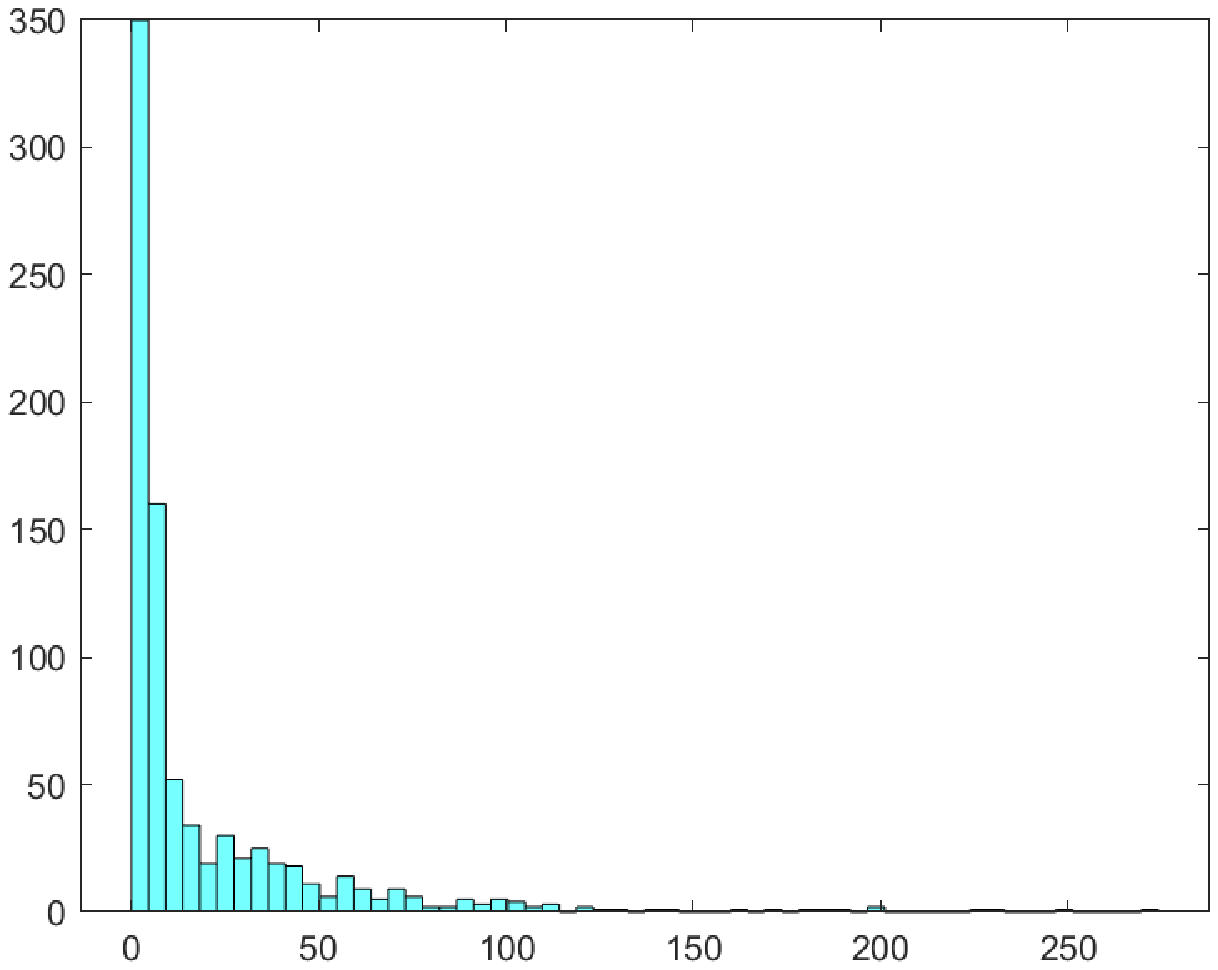}}
	\subfigure[$A_{831\times831}$: $d_{c}$]{\includegraphics[width=0.24\textwidth]{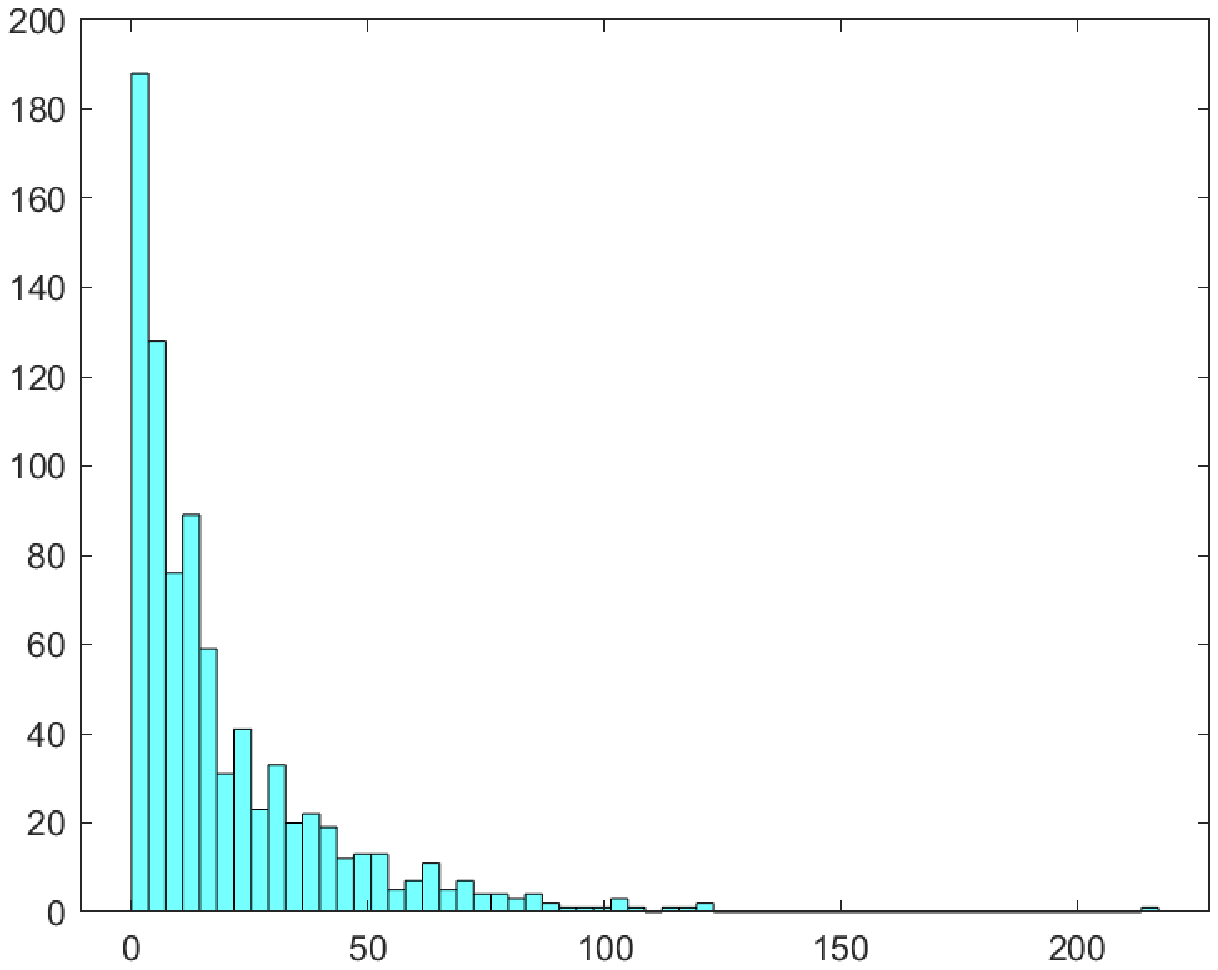}}
	\subfigure[$A_{990\times1490}$: $d_{r}$]{\includegraphics[width=0.24\textwidth]{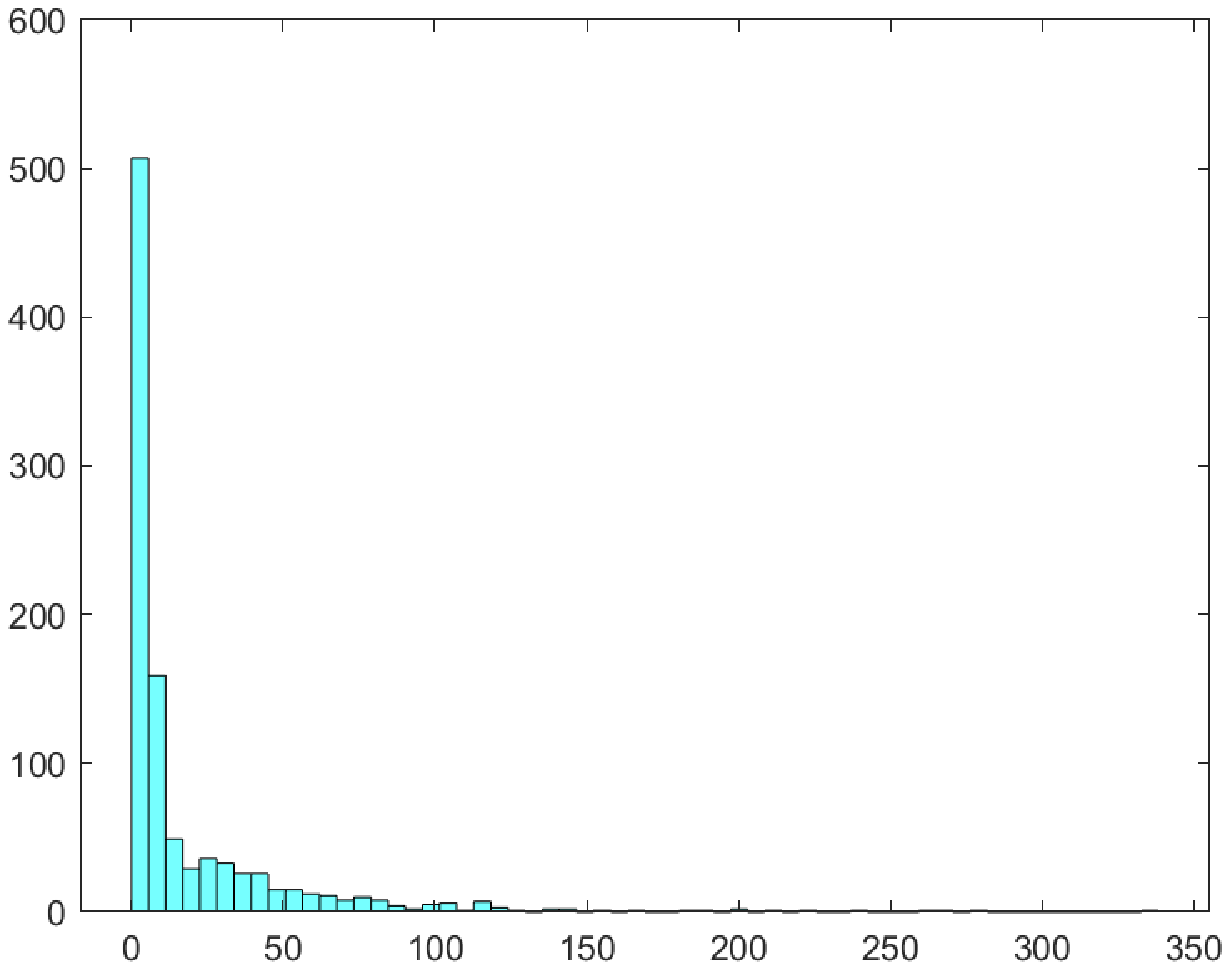}}
	\subfigure[$A_{990\times1490}$: $d_{c}$]{\includegraphics[width=0.24\textwidth]{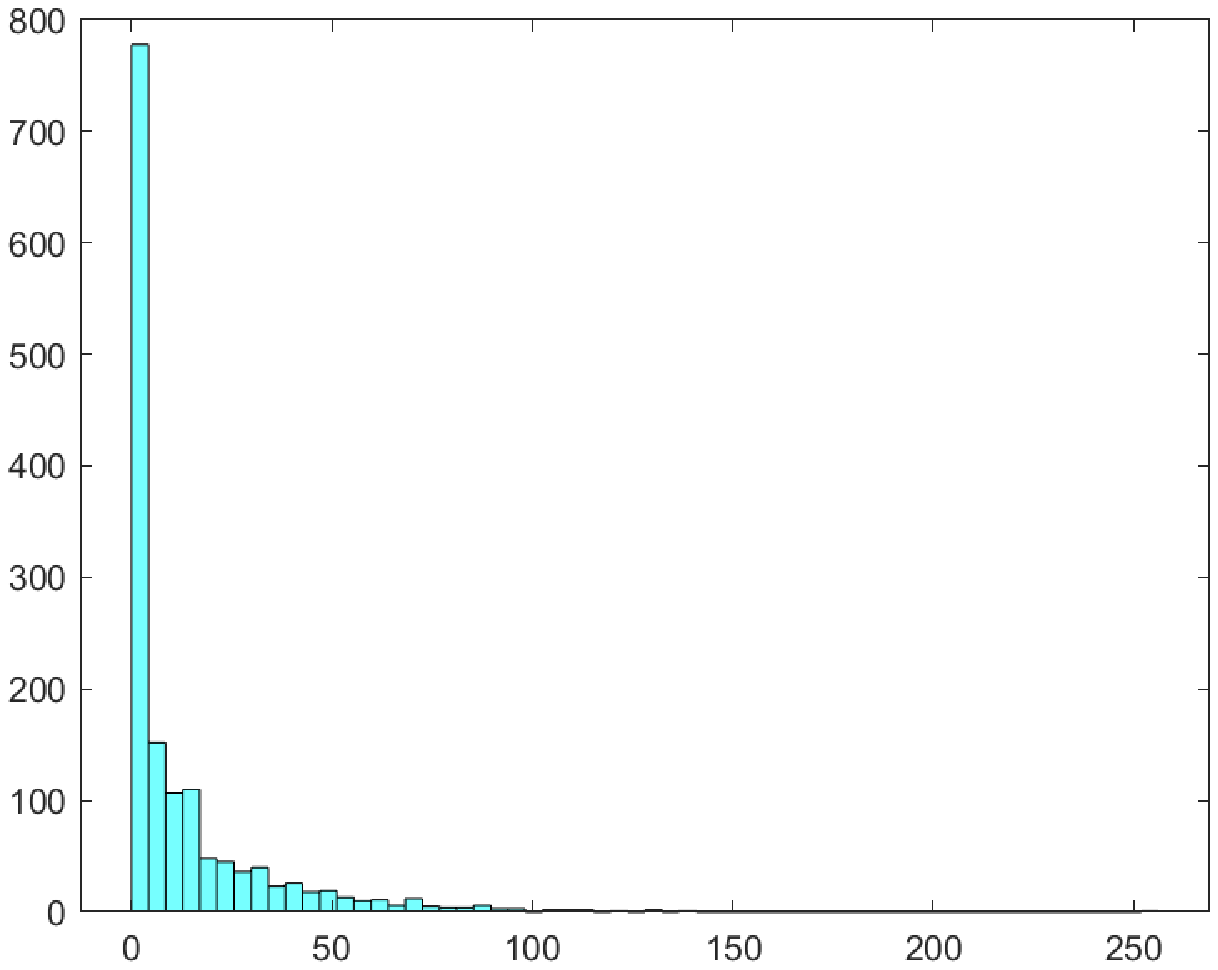}}
	\subfigure[$A_{1490\times1065}$: $d_{r}$]{\includegraphics[width=0.24\textwidth]{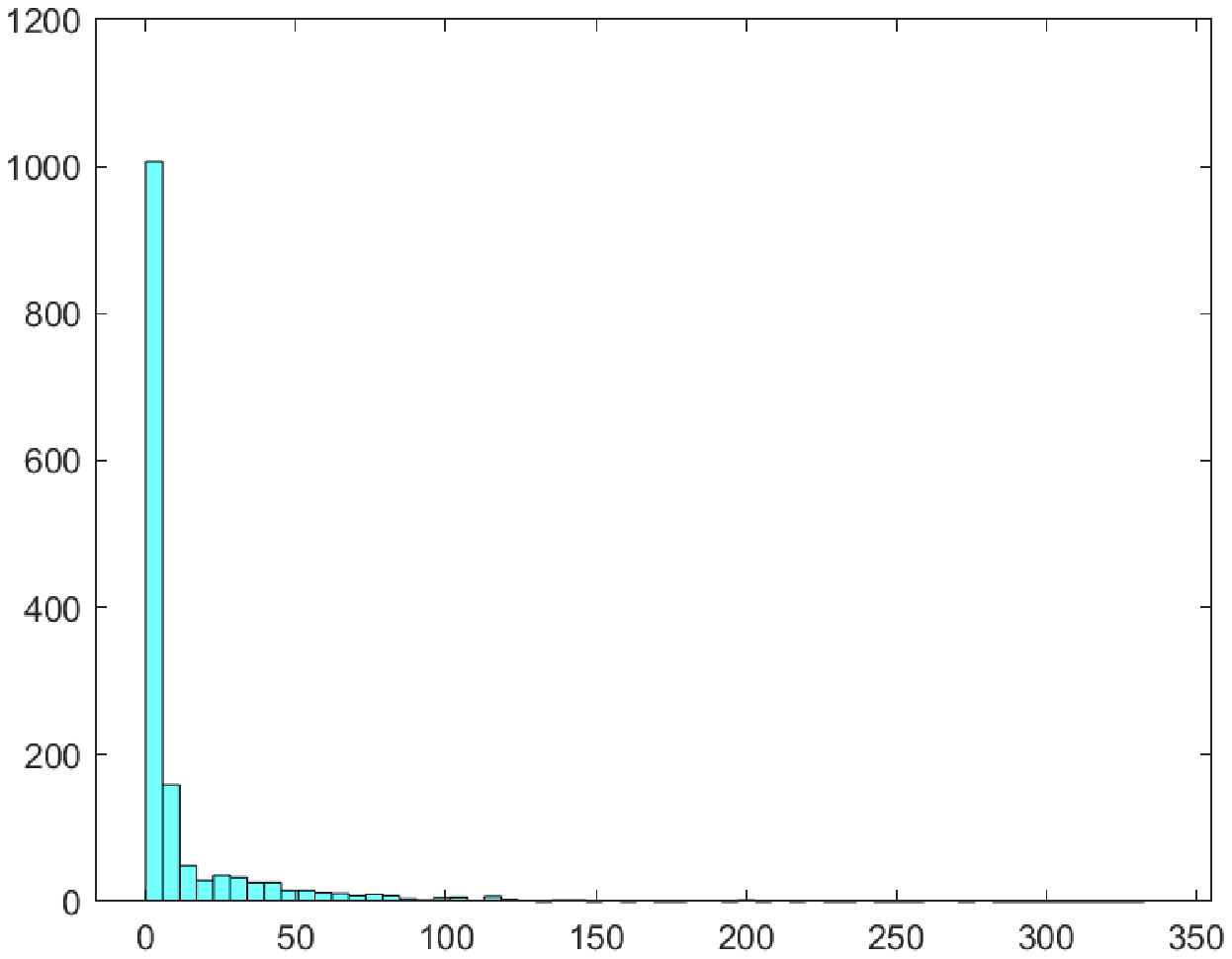}}
	\subfigure[$A_{1490\times1065}$: $d_{c}$]{\includegraphics[width=0.24\textwidth]{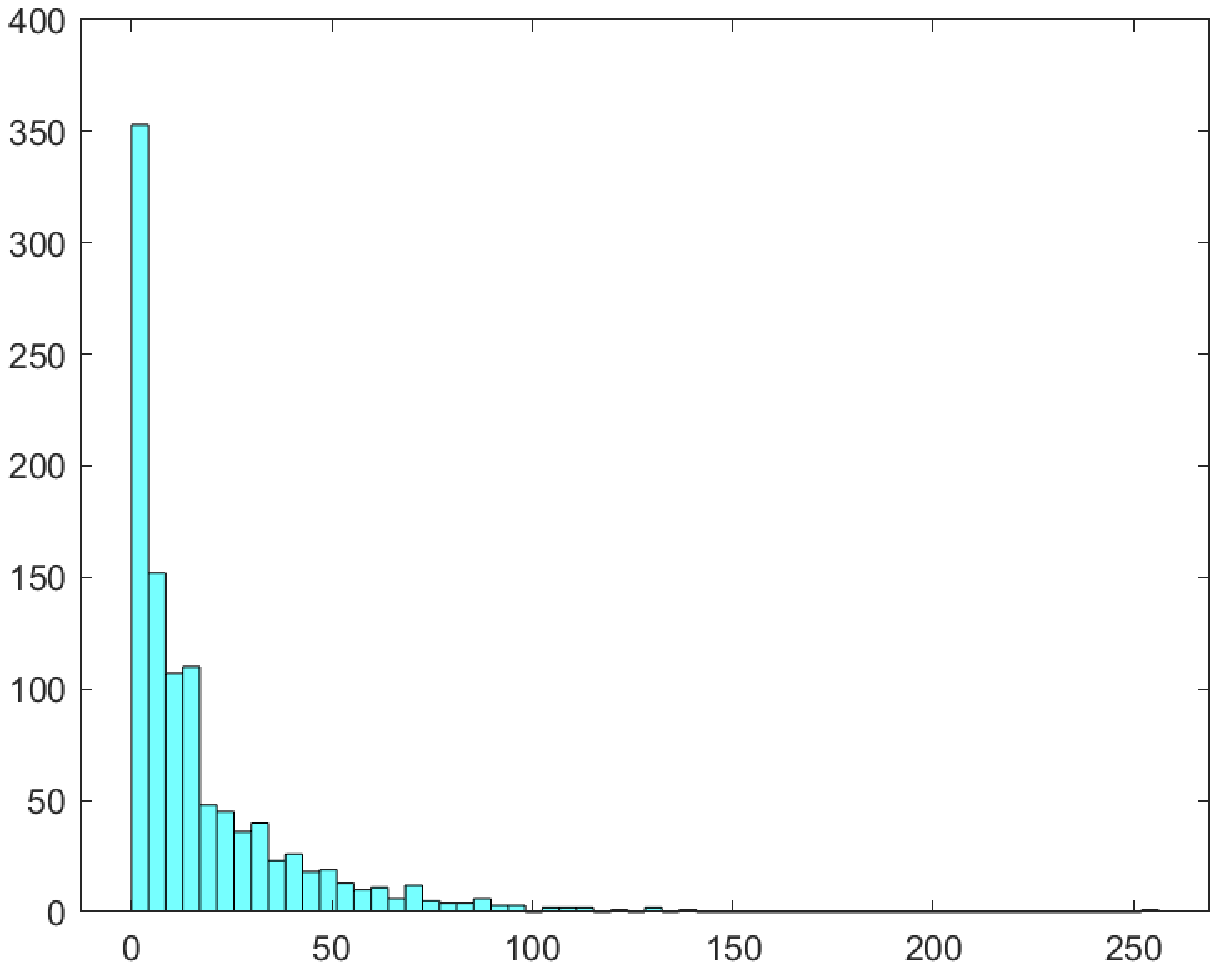}}
	\subfigure[$A_{990\times1065}$: $d_{r}$]{\includegraphics[width=0.24\textwidth]{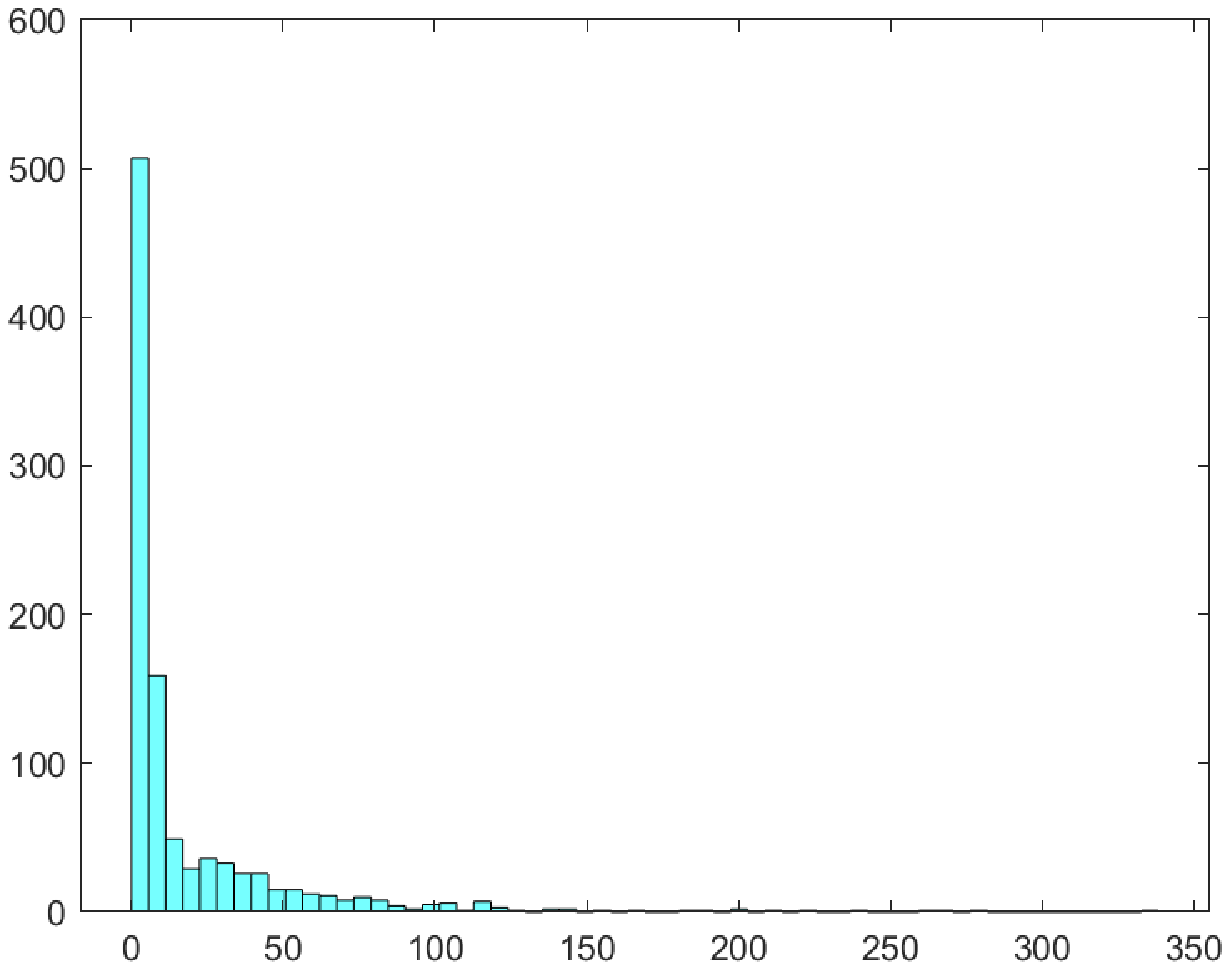}}
	\subfigure[$A_{990\times1065}$: $d_{c}$]{\includegraphics[width=0.24\textwidth]{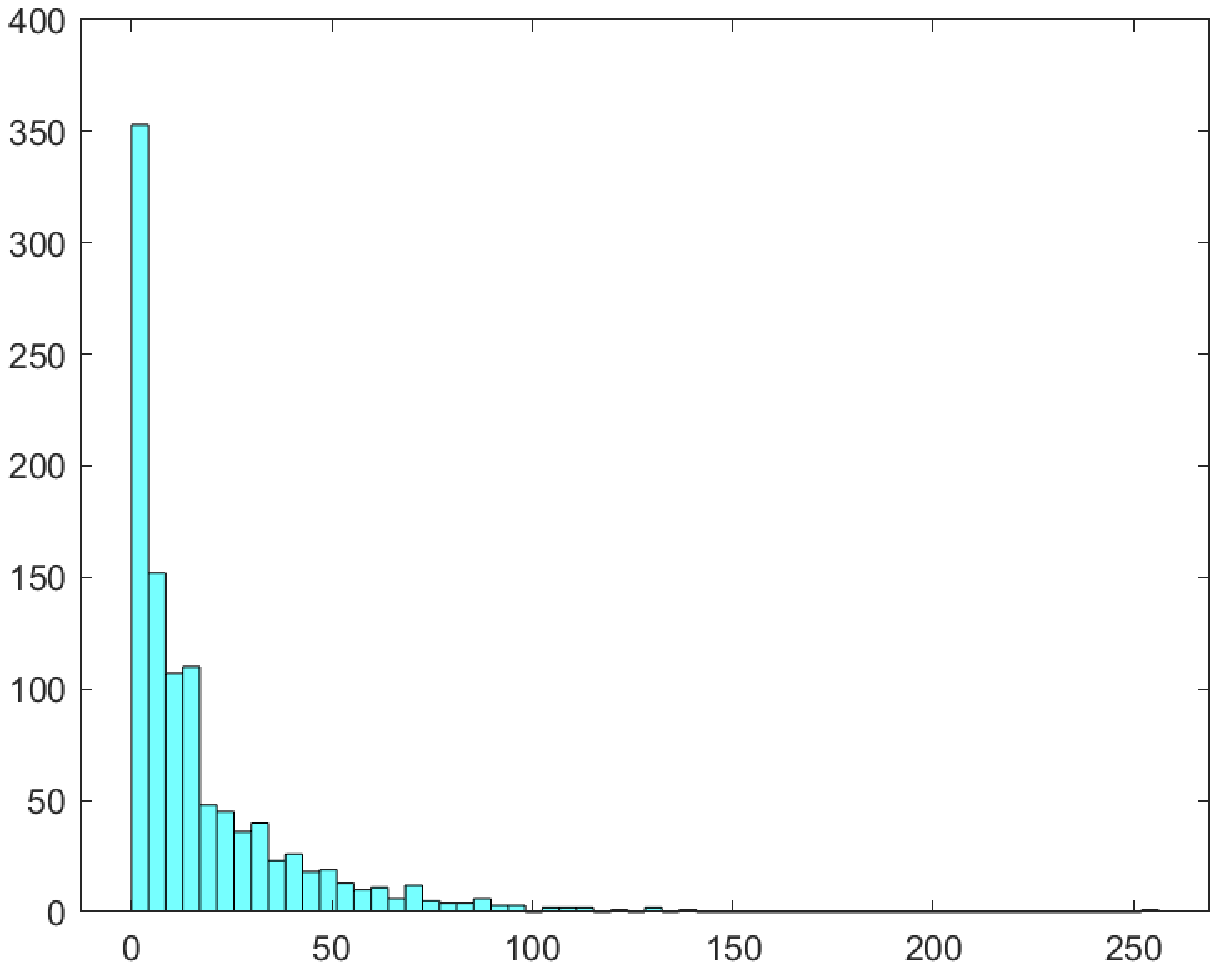}}
	\subfigure[$A_{1490\times1490}$]{\includegraphics[width=0.24\textwidth]{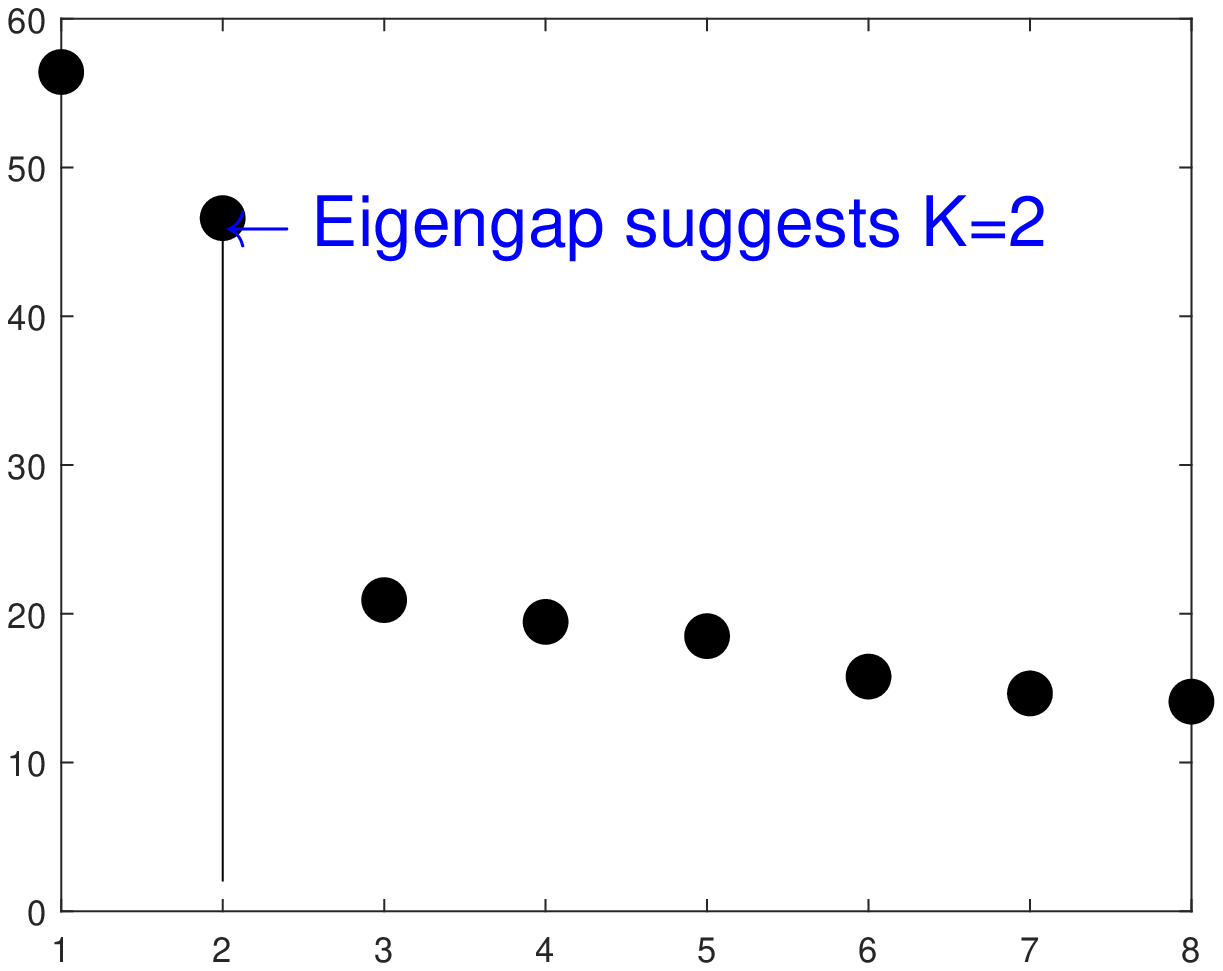}}
	\subfigure[$A_{1224\times1224}$]{\includegraphics[width=0.24\textwidth]{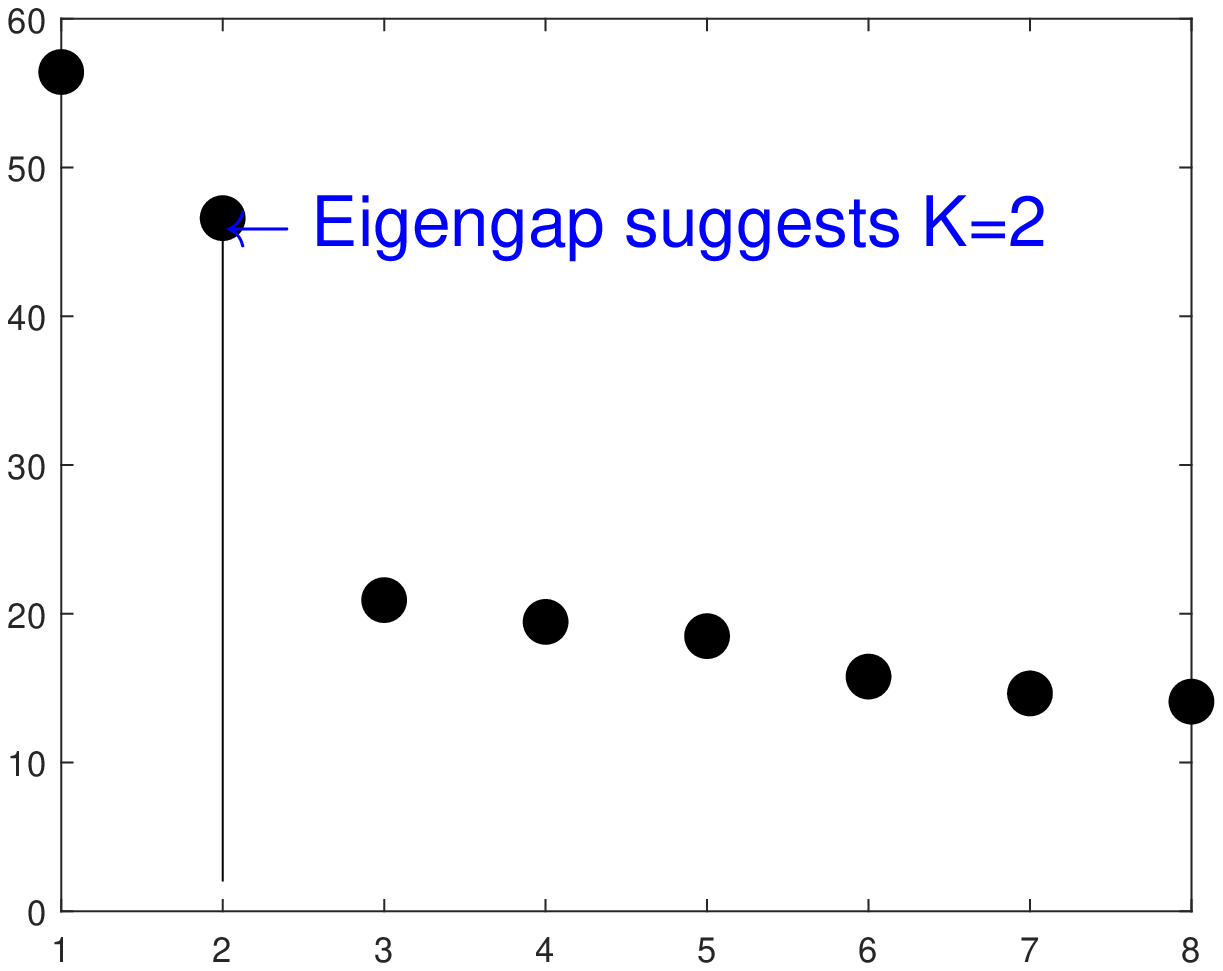}}
	\subfigure[$A_{831\times831}$]{\includegraphics[width=0.24\textwidth]{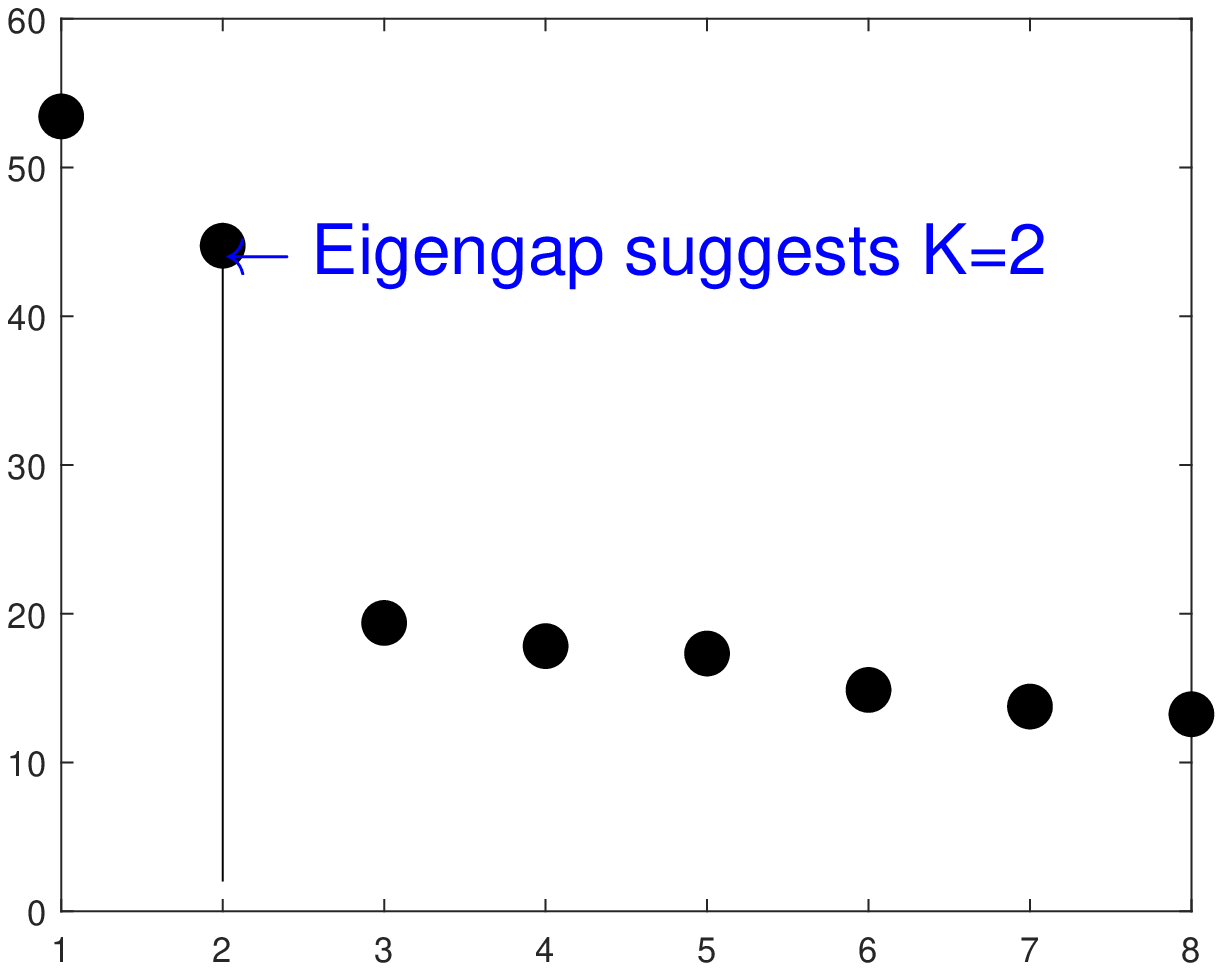}}
	\subfigure[$A_{990\times1490}$]{\includegraphics[width=0.24\textwidth]{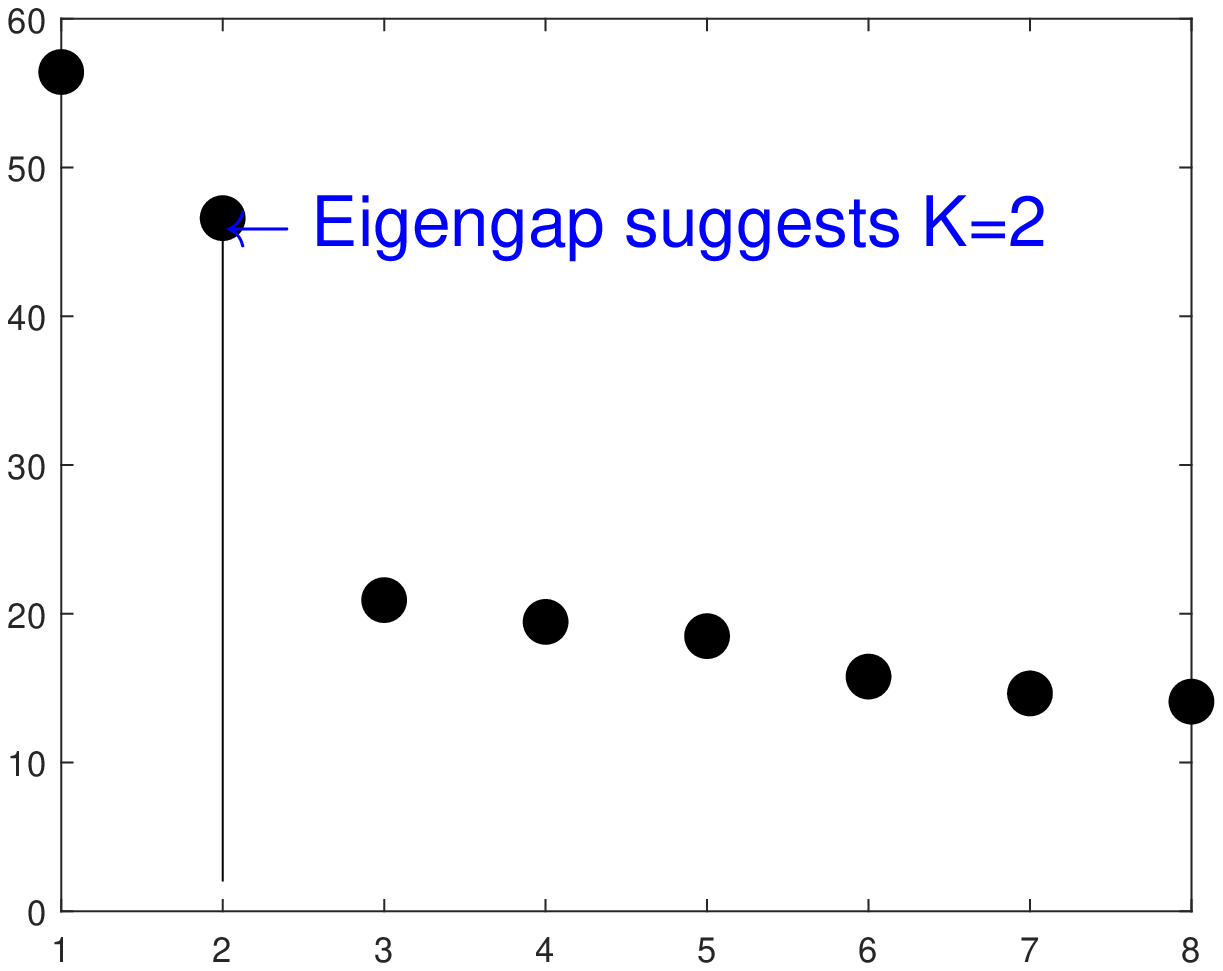}}
	\subfigure[$A_{1490\times1065}$]{\includegraphics[width=0.24\textwidth]{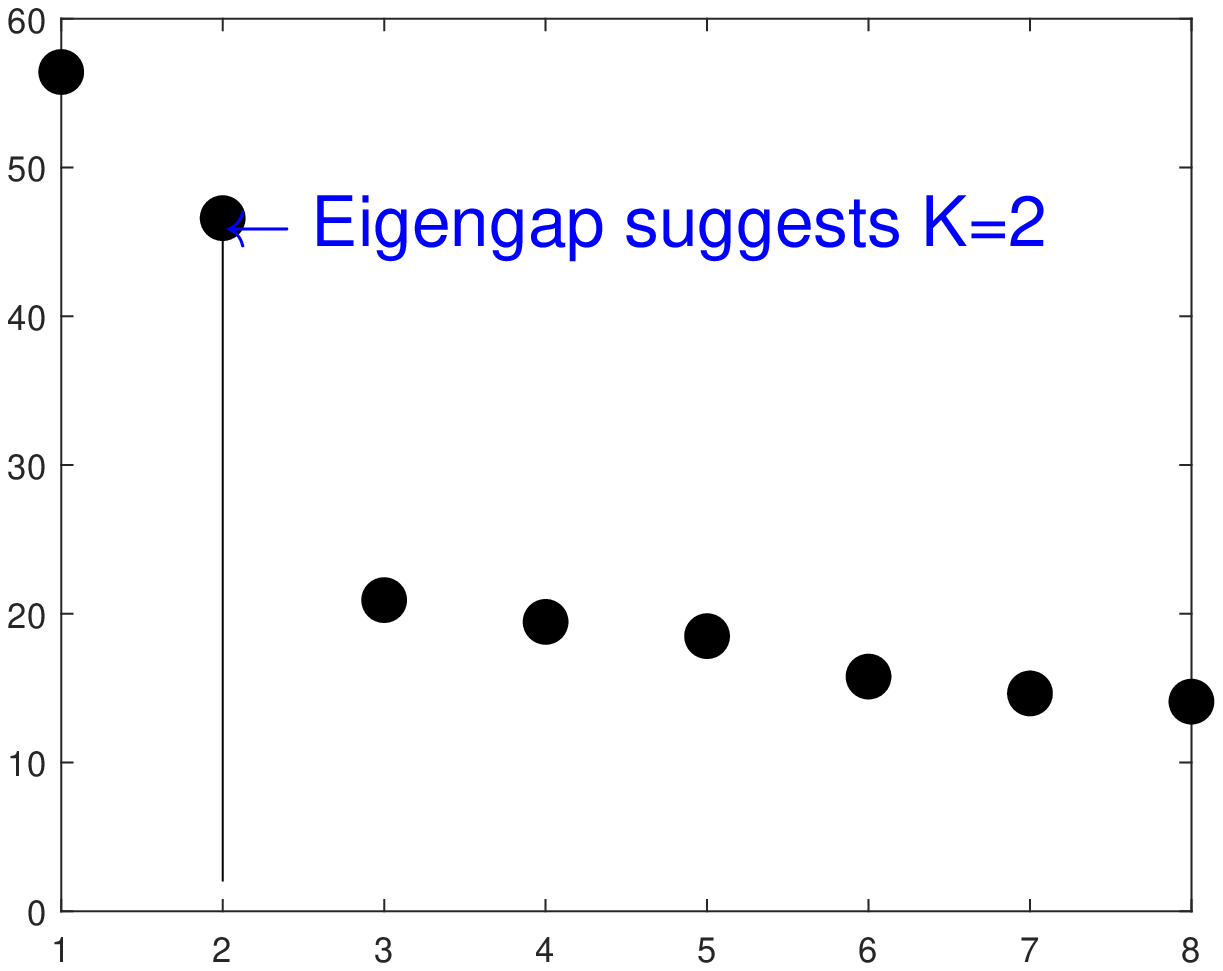}}
	\subfigure[$A_{990\times1065}$]{\includegraphics[width=0.24\textwidth]{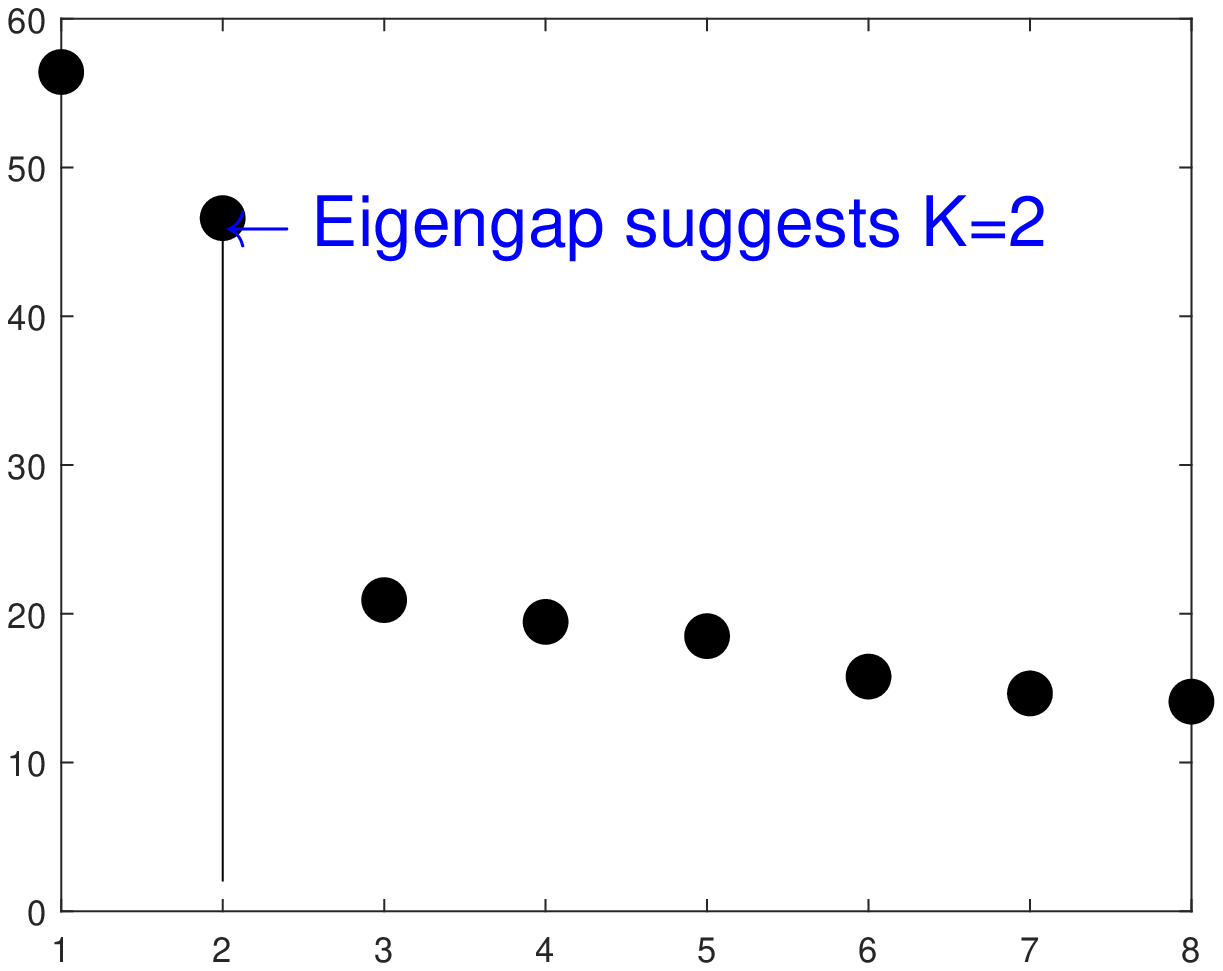}}
	\caption{Distribution of $d_{r}$ and $d_{c}$, and top 8 singular values of $A$ for Political blogs.}
	\label{Pblogs} %% label for entire figure
\end{figure}
\begin{table}[h!]
\footnotesize
	\centering
	\caption{ErrorRates for the five methods on Political blogs with $K=2$.}
	\label{ErrorRateBlogs}
	%\resizebox{\columnwidth}{!}{
	\begin{tabular}{cccccccccc}
		\hline\hline
&$A_{1490\times1490}$&$A_{1224\times1224}$& $A_{831\times831}$&$A_{990\times1490}$&$A_{1490\times1065}$&$A_{990\times1065}$\\
\hline
BiSC&0.4463&0.4044&0.3490&0.4221&0.4463&0.3606\\
nBiSC&0.2315&0.1168&0.0529&0.2114&0.1557&0.0601\\
DI-SIM&0.2302&0.1152&0.0493&0.2087&0.1557&0.0592\\
D-SCORE&0.4631&0.1111&0.0469&0.1631&0.1550&0.0582\\
rD-SCORE&0.4617&0.4837&0.0469&0.1570&0.1510&0.0516\\
\hline\hline
	\end{tabular}%}
\end{table}

\begin{table}[h!]
\footnotesize
	\centering
	\caption{NMI for the five methods on Political blogs with $K=2$.}
	\label{NMIBlogs}
	%\resizebox{\columnwidth}{!}{
	\begin{tabular}{cccccccccc}
		\hline\hline
&$A_{1490\times1490}$&$A_{1224\times1224}$& $A_{831\times831}$&$A_{990\times1490}$&$A_{1490\times1065}$&$A_{990\times1065}$\\
\hline
BiSC&0.0954&0.1230&0.1716&0.1201&0.0954&0.1598\\
nBiSC&0.3058&0.4832&0.7035&0.3602&0.4263&0.6880\\
DI-SIM&0.3051&0.4876&0.7186&0.3694&0.4292&0.6912\\
D-SCORE&0.0062&0.5260&0.7281&0.4061&0.4422&0.6799\\
rD-SCORE&0.0066&3.2996e-6&0.7251&0.4167&0.4560&0.7123\\
\hline\hline
	\end{tabular}%}
\end{table}

\begin{table}[h!]
\footnotesize
	\centering
	\caption{ARI for the five methods on Political blogs with $K=2$.}
	\label{ARIBlogs}
	%\resizebox{\columnwidth}{!}{
	\begin{tabular}{cccccccccc}
		\hline\hline
&$A_{1490\times1490}$&$A_{1224\times1224}$& $A_{831\times831}$&$A_{990\times1490}$&$A_{1490\times1065}$&$A_{990\times1065}$\\
\hline
BiSC&0.0111&0.0353&0.0847&0.0238&0.0111&0.0723\\
nBiSC&0.2878&0.5869&0.7991&0.3327&0.4738&0.7739\\
DI-SIM&0.2907&0.5920&0.8121&0.3390&0.4738&0.7772\\
D-SCORE&0.0048&0.6046&0.8209&0.4537&0.4757&0.7805\\
rD-SCORE&0.0053&-1.3176e-04&0.8209&0.4701&0.4869&0.8039\\
\hline\hline
	\end{tabular}%}
\end{table}

\begin{figure}
\centering
\subfigure[Crisis in a Cloister: $d_{r}$]{\includegraphics[width=0.24\textwidth]{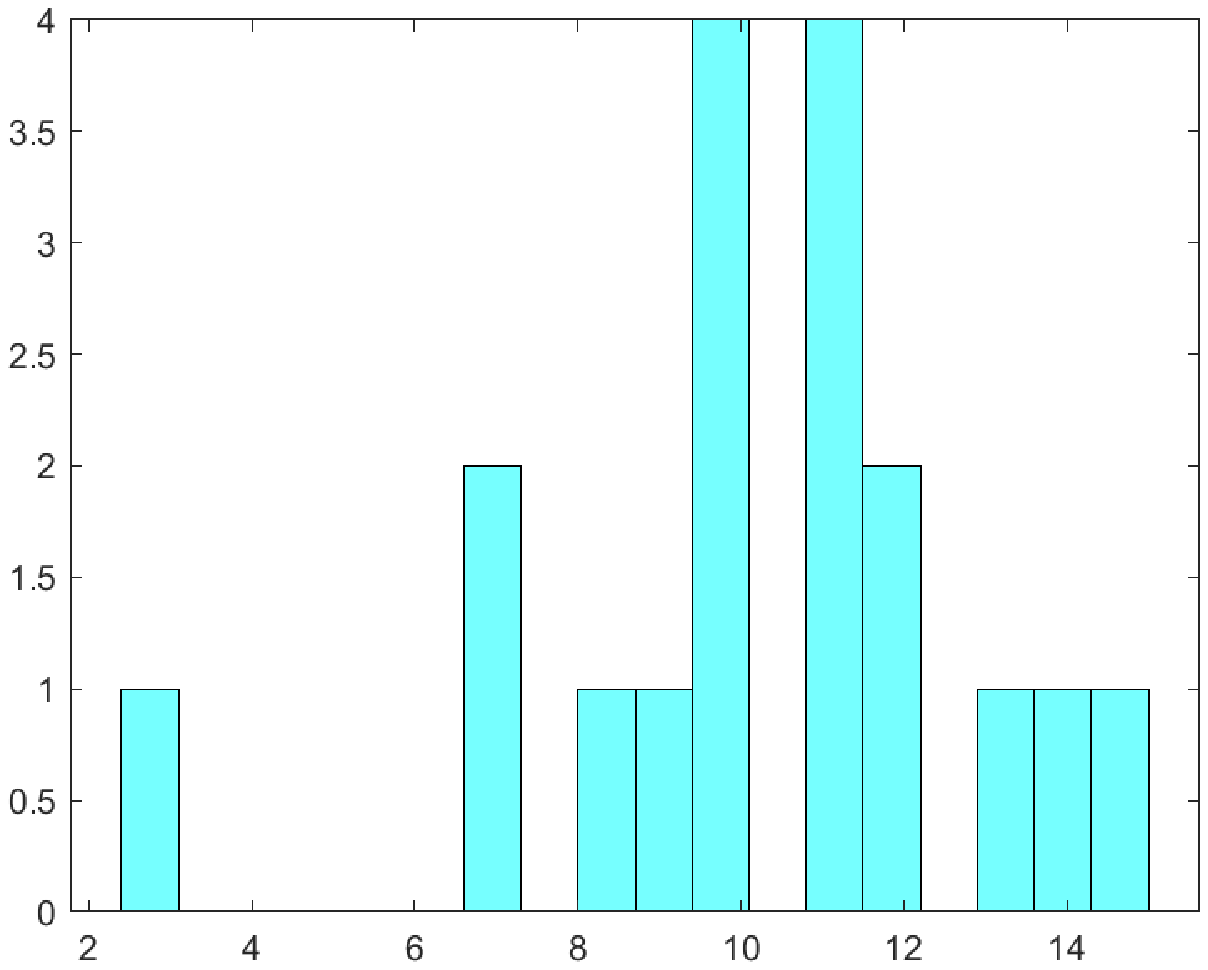}}
\subfigure[Crisis in a Cloister: $d_{c}$]{\includegraphics[width=0.24\textwidth]{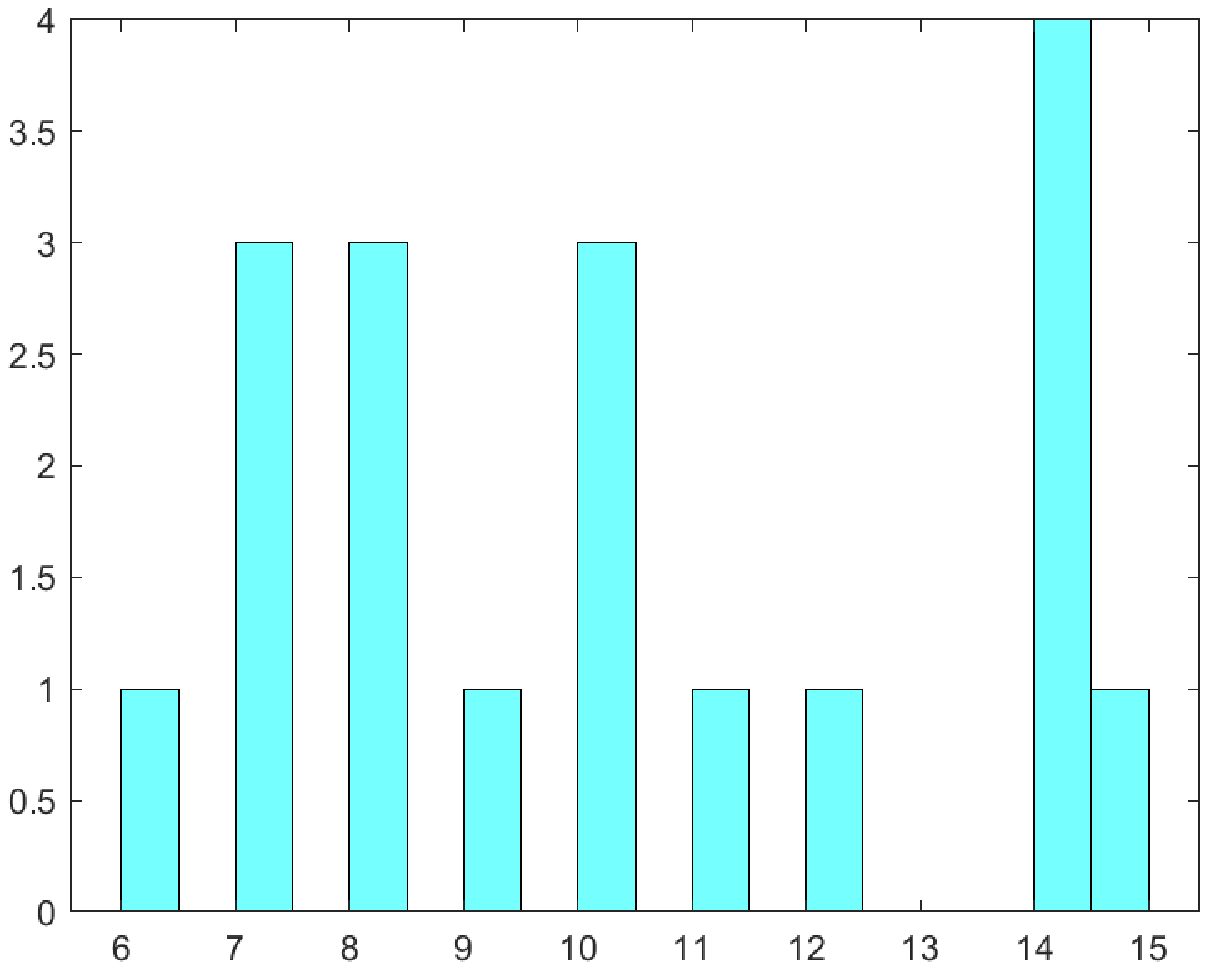}}
\subfigure[Dutch college: $d_{r}$]{\includegraphics[width=0.24\textwidth]{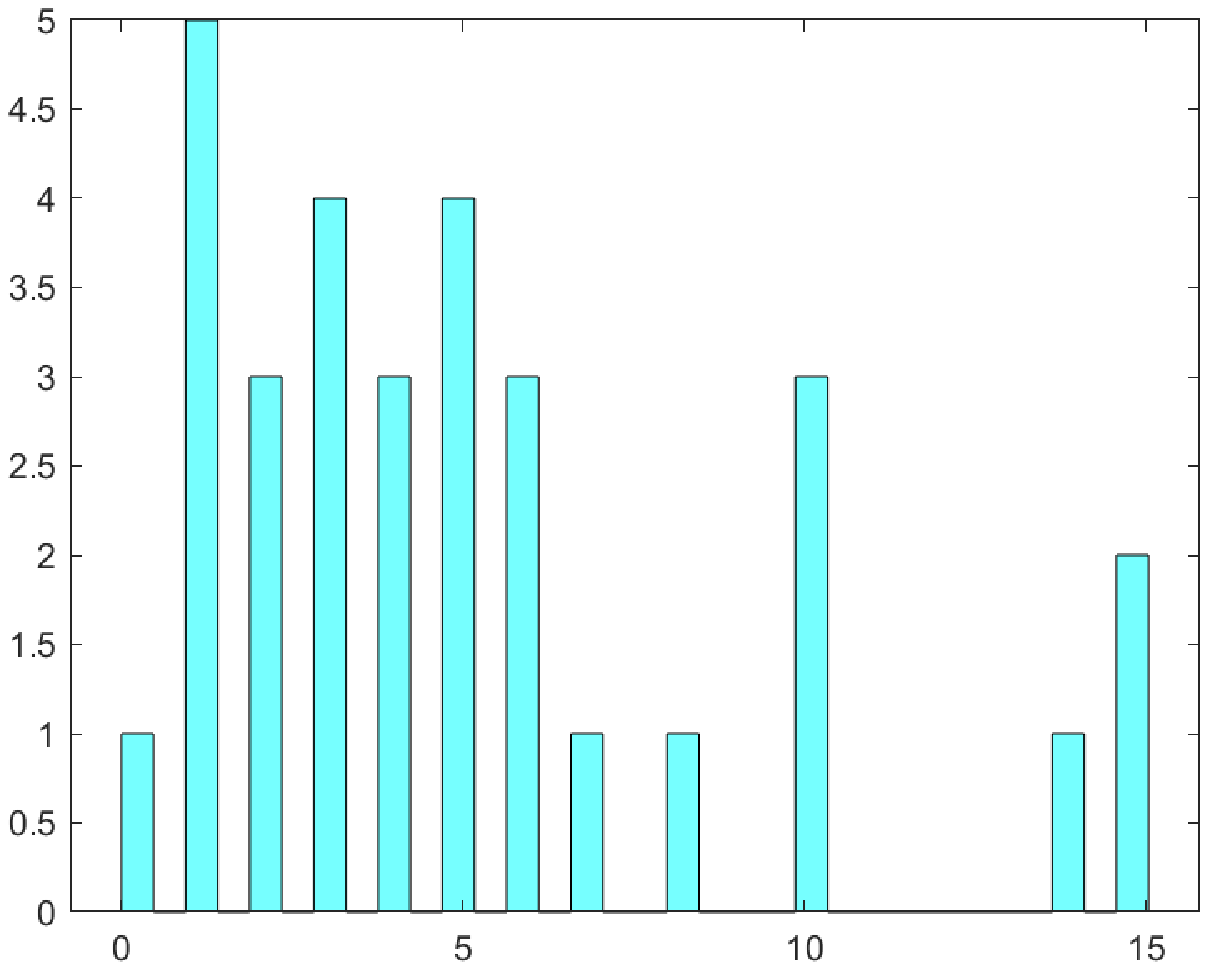}}
\subfigure[Dutch college: $d_{c}$]{\includegraphics[width=0.24\textwidth]{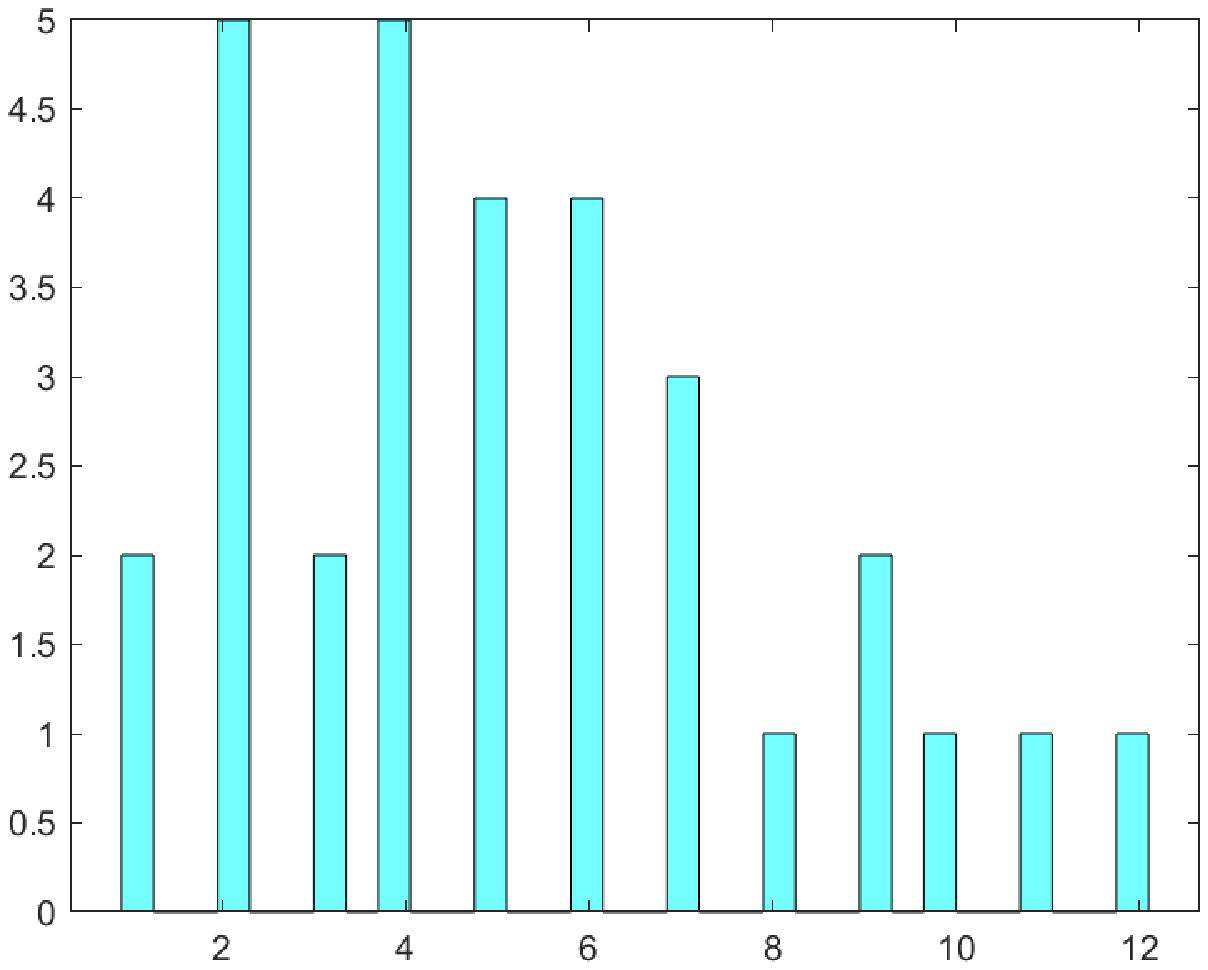}}
\subfigure[Highschool: $d_{r}$]{\includegraphics[width=0.2\textwidth]{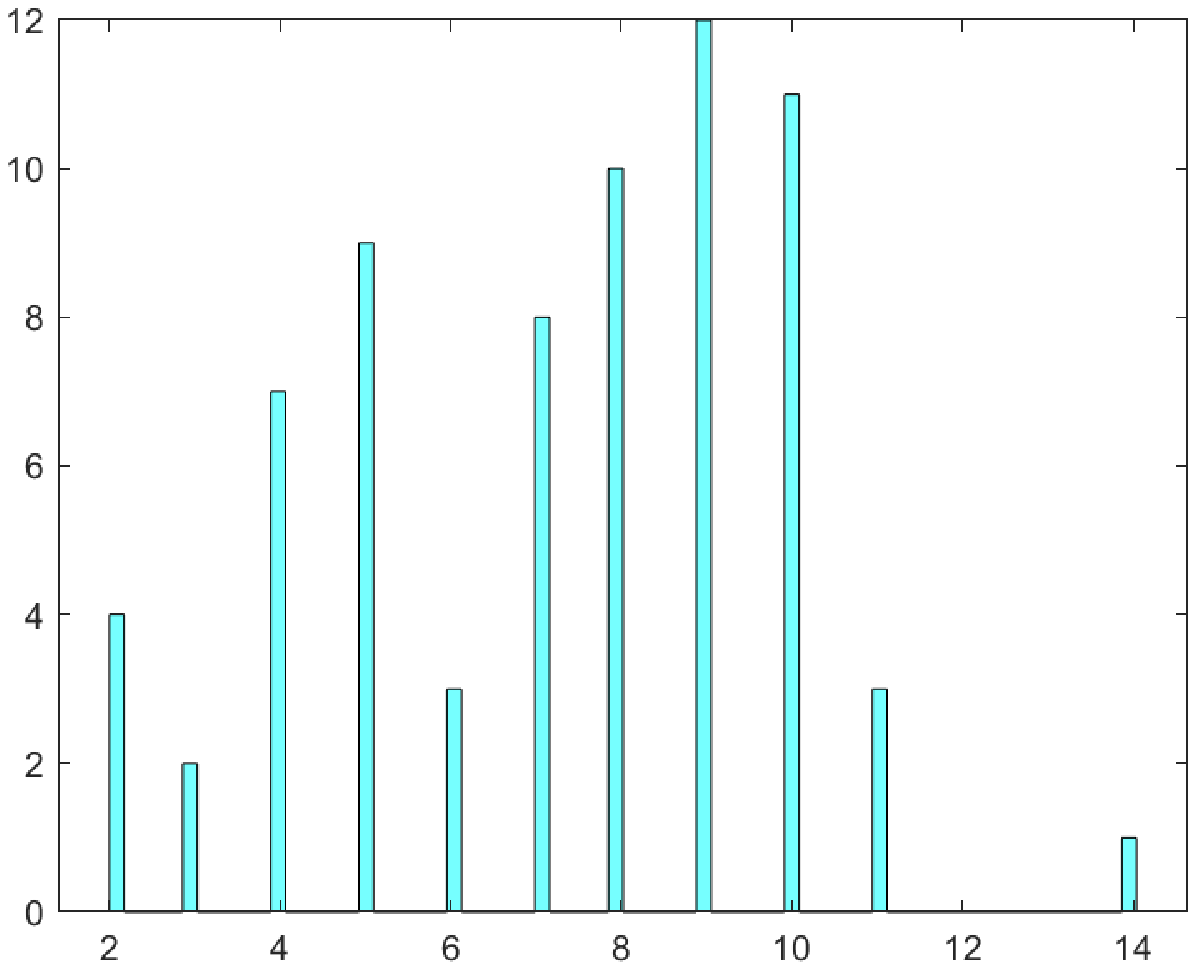}}
\subfigure[Highschool: $d_{c}$]{\includegraphics[width=0.24\textwidth]{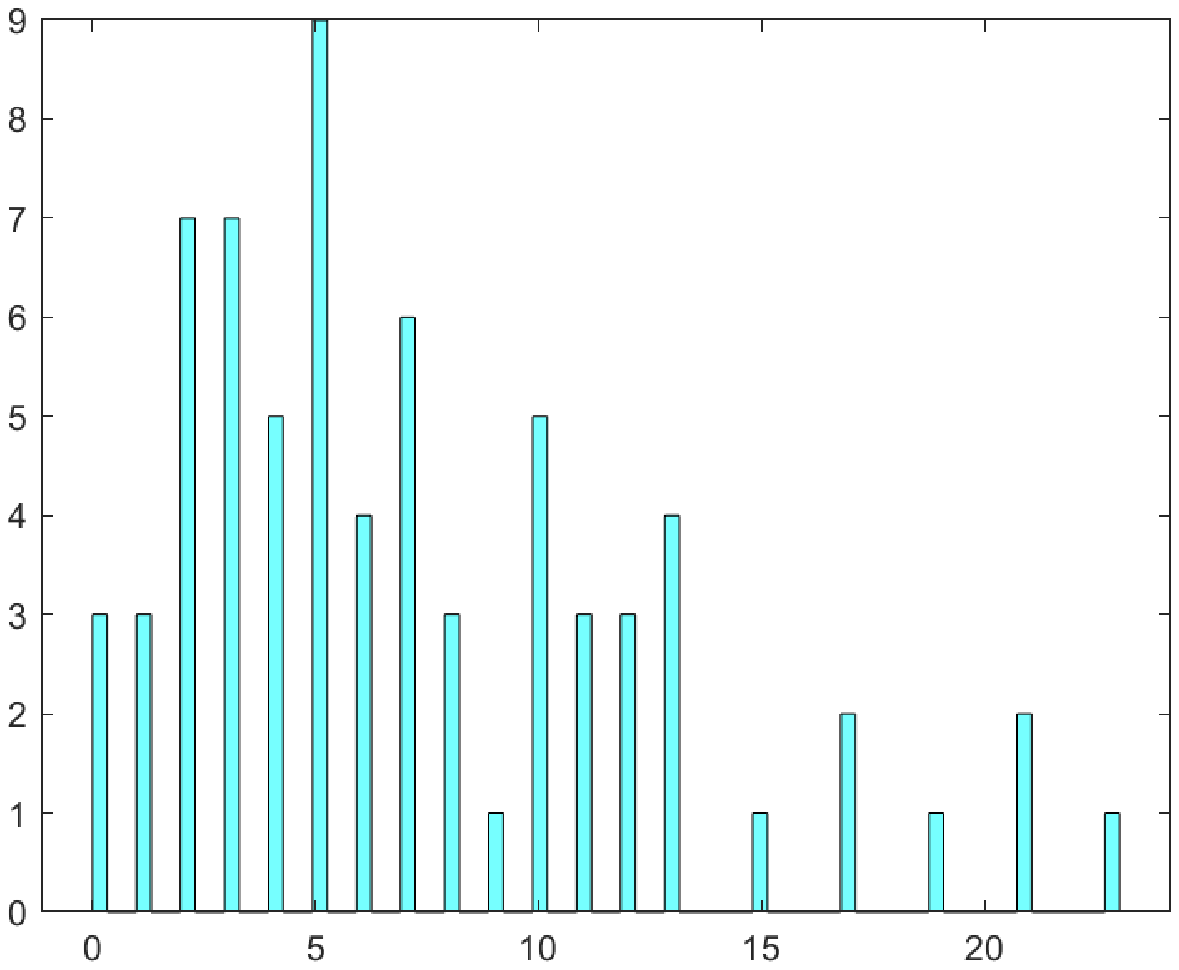}}
\subfigure[Facebook-like Social Network: $d_{r}$]{\includegraphics[width=0.24\textwidth]{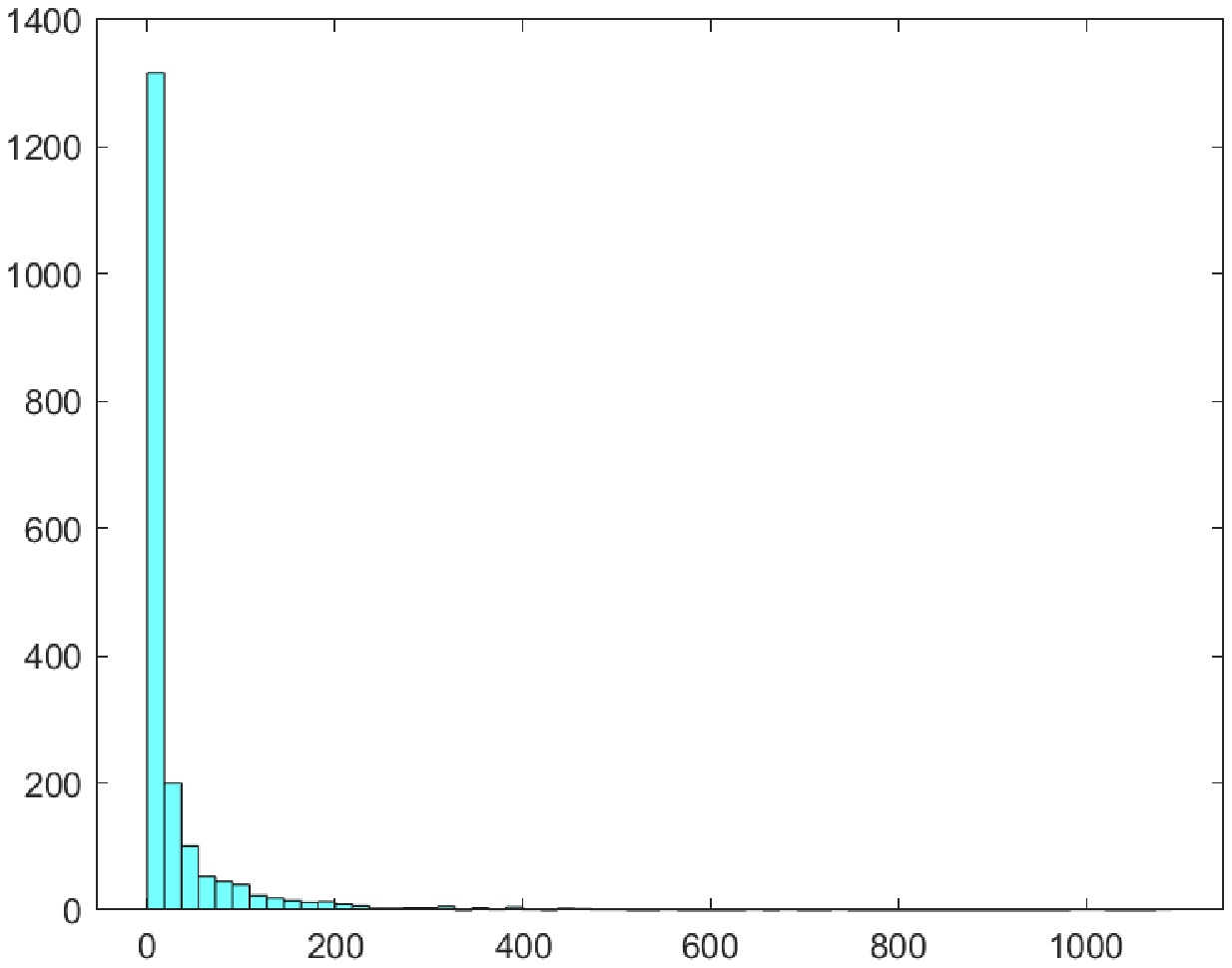}}
\subfigure[Facebook-like Social Network: $d_{c}$]{\includegraphics[width=0.24\textwidth]{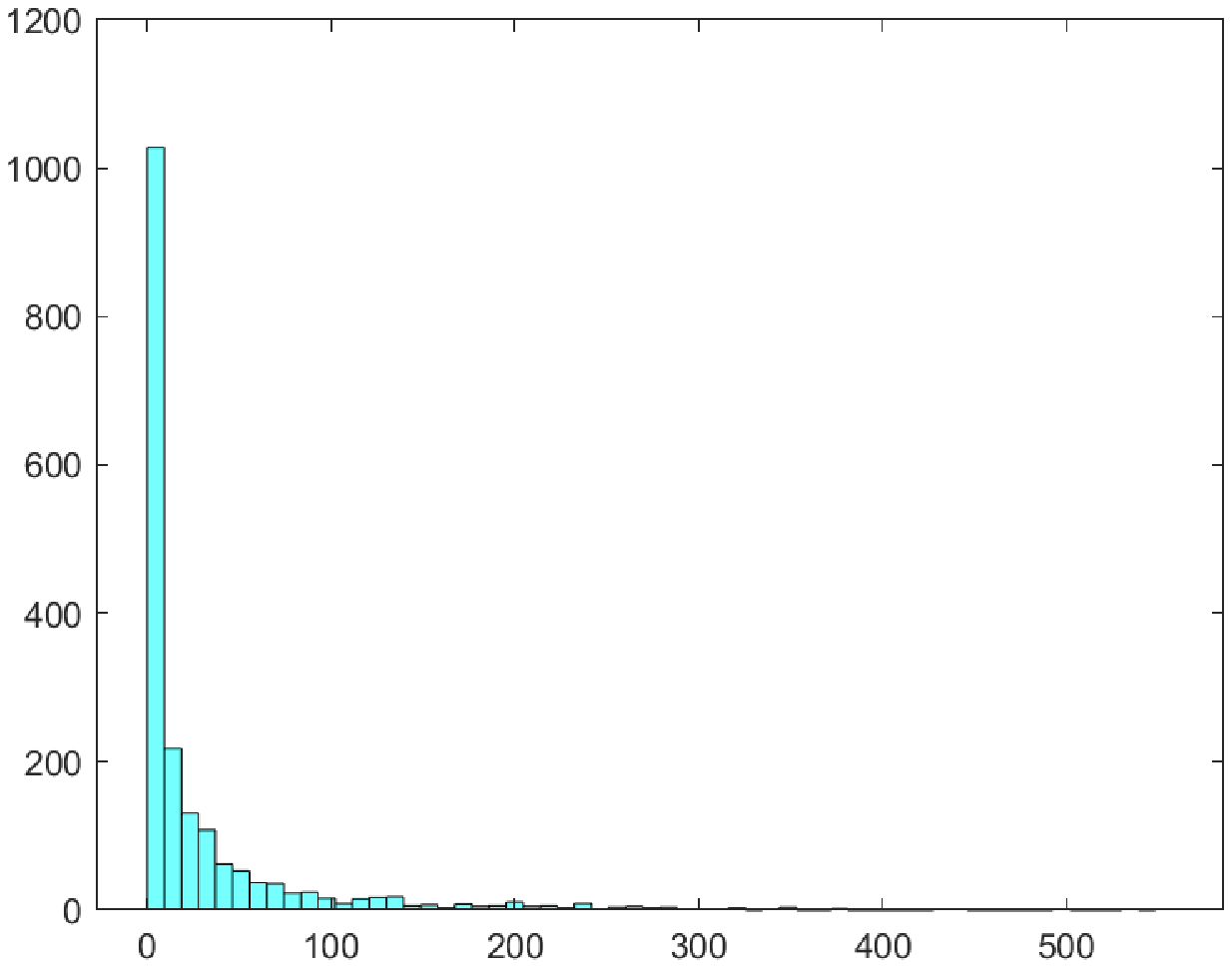}}
\subfigure[Crisis in a Cloister]{\includegraphics[width=0.24\textwidth]{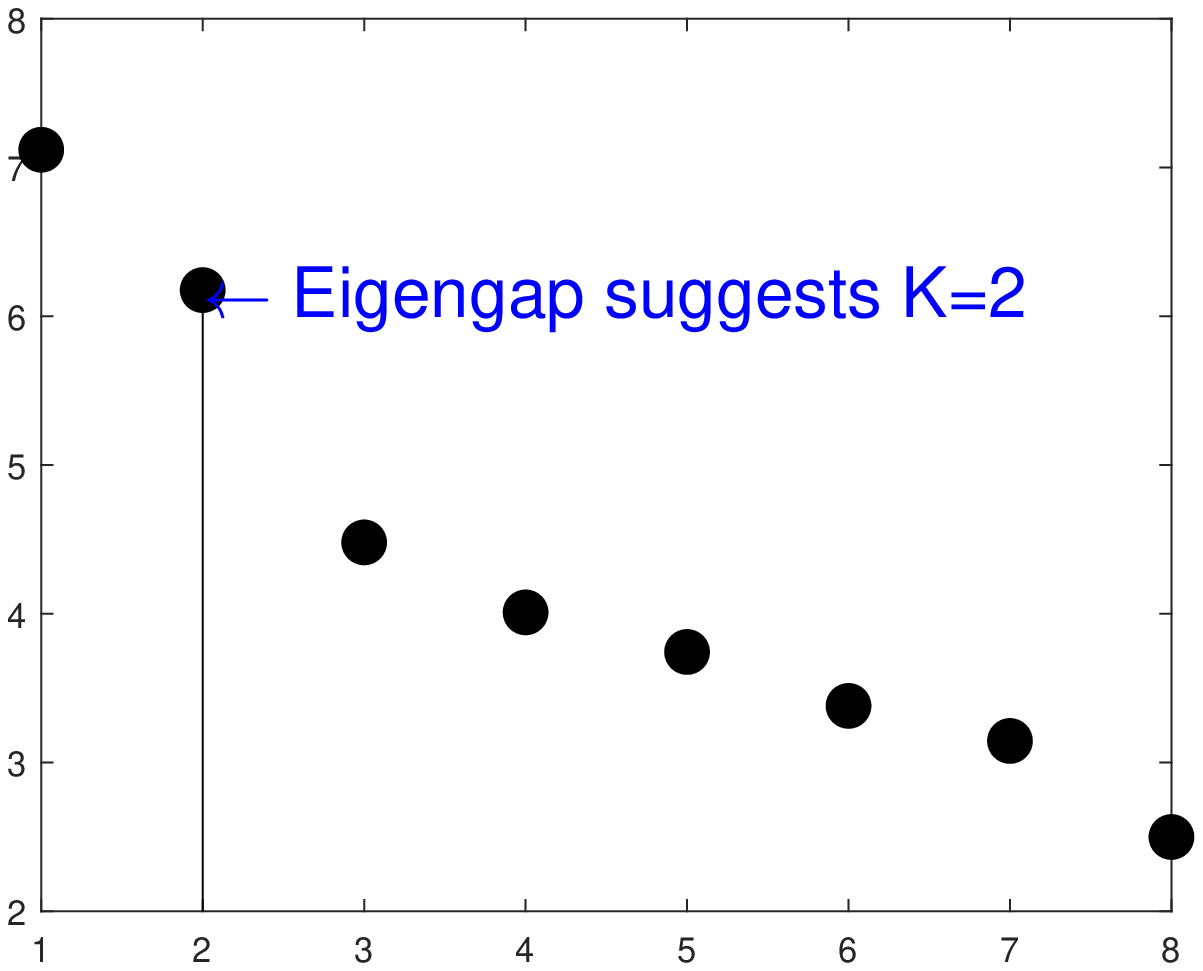}}
\subfigure[Dutch college]{\includegraphics[width=0.24\textwidth]{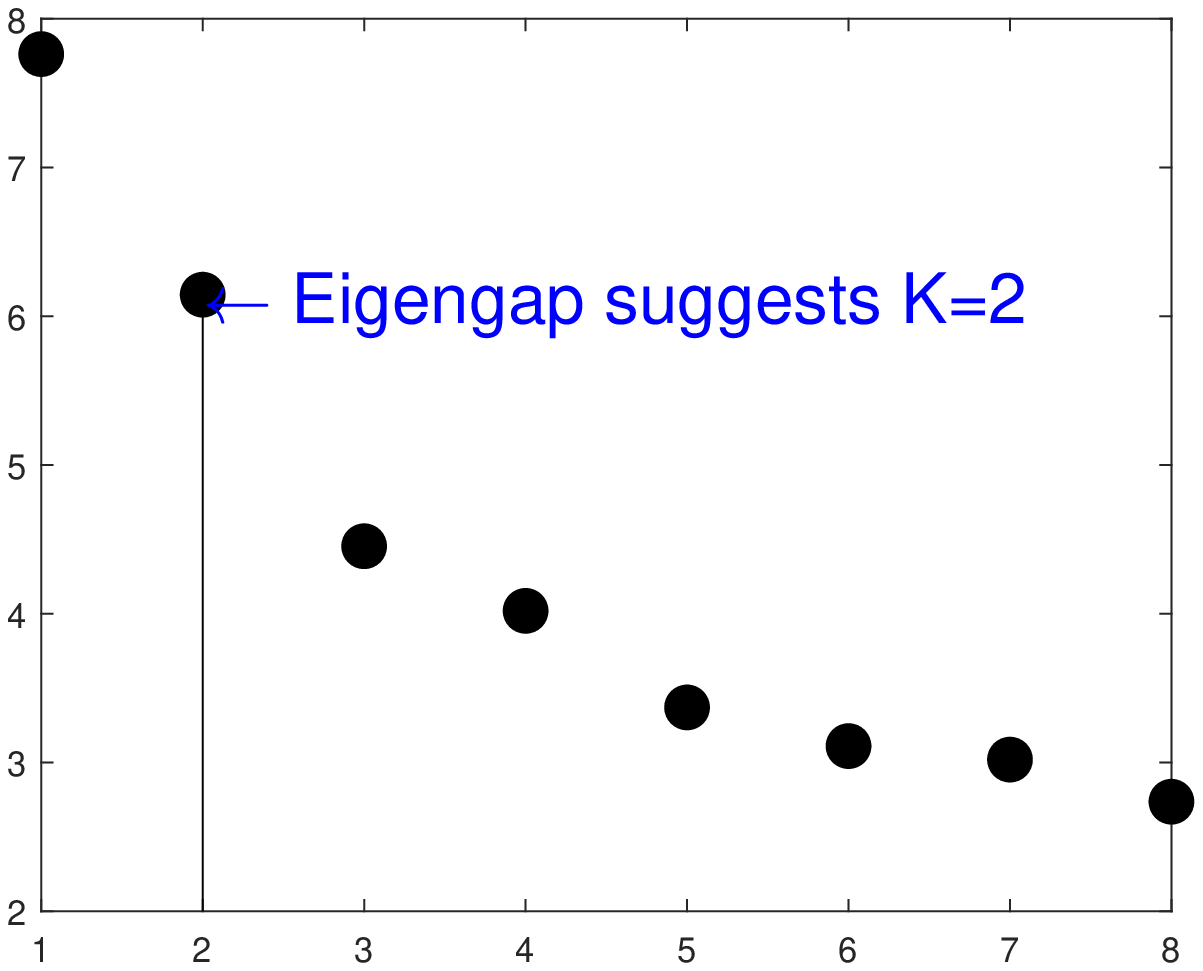}}
\subfigure[Highschool]{\includegraphics[width=0.24\textwidth]{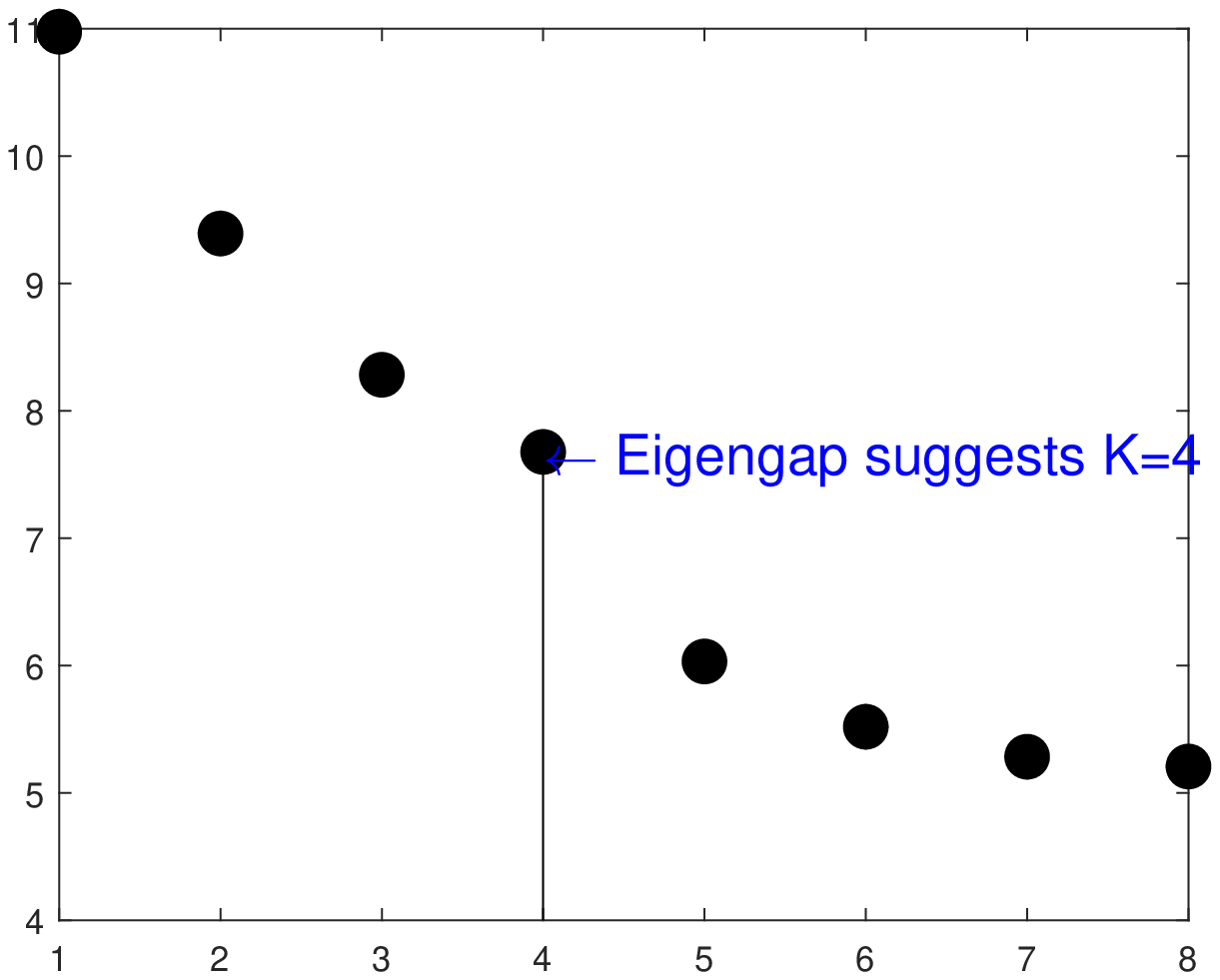}}
\subfigure[Facebook-like Social Network]{\includegraphics[width=0.24\textwidth]{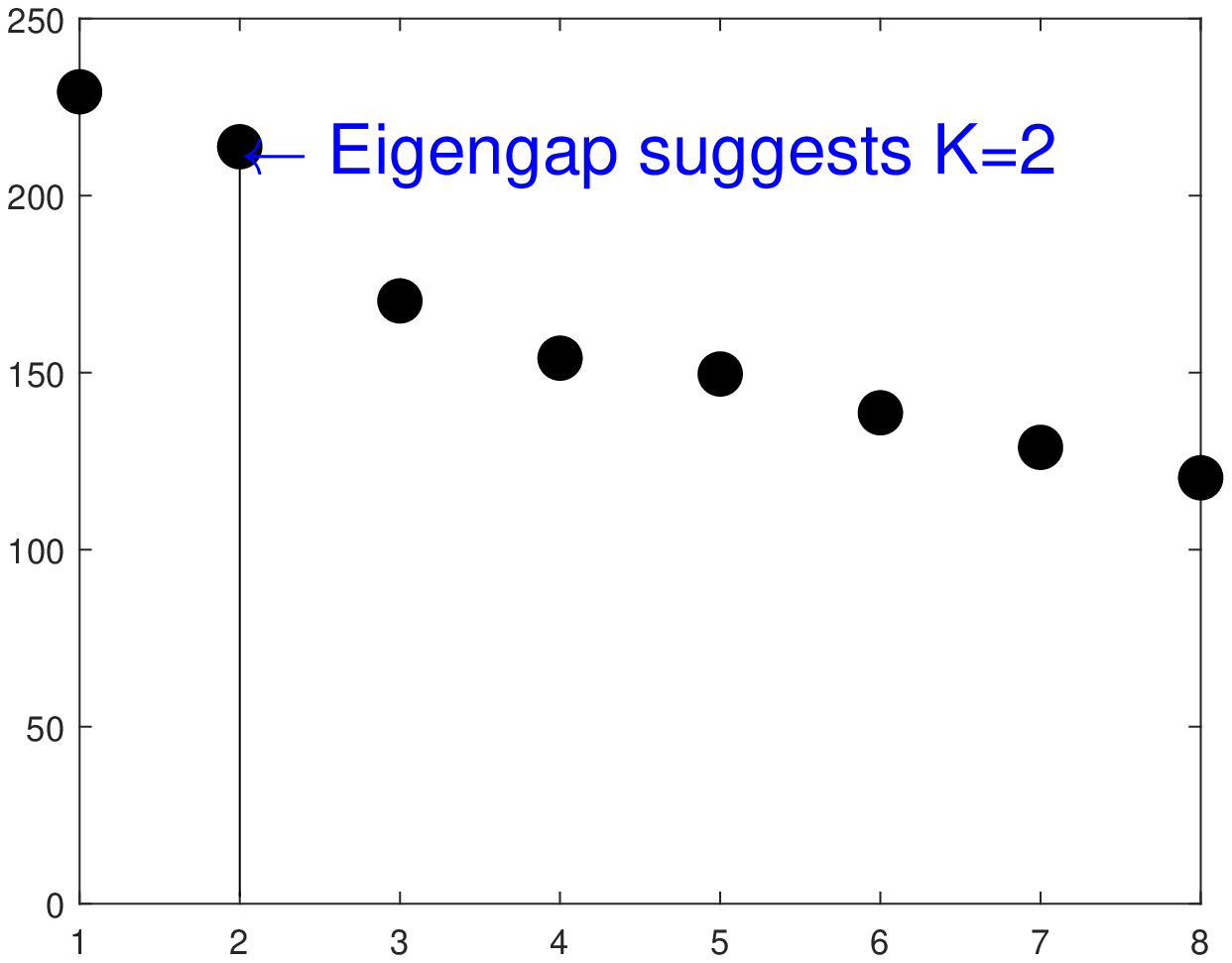}}
\caption{Panels (a)-(h): histogram of $d_{r}$ and $d_{c}$ for datasets. Panels(i)-(l): top 8 singular values of $A$ for datasets.}
\label{HistEig8} %% label for entire figure
\end{figure}
\begin{figure}
\centering
\subfigure[Crisis in a Cloister]{\includegraphics[width=0.24\textwidth]{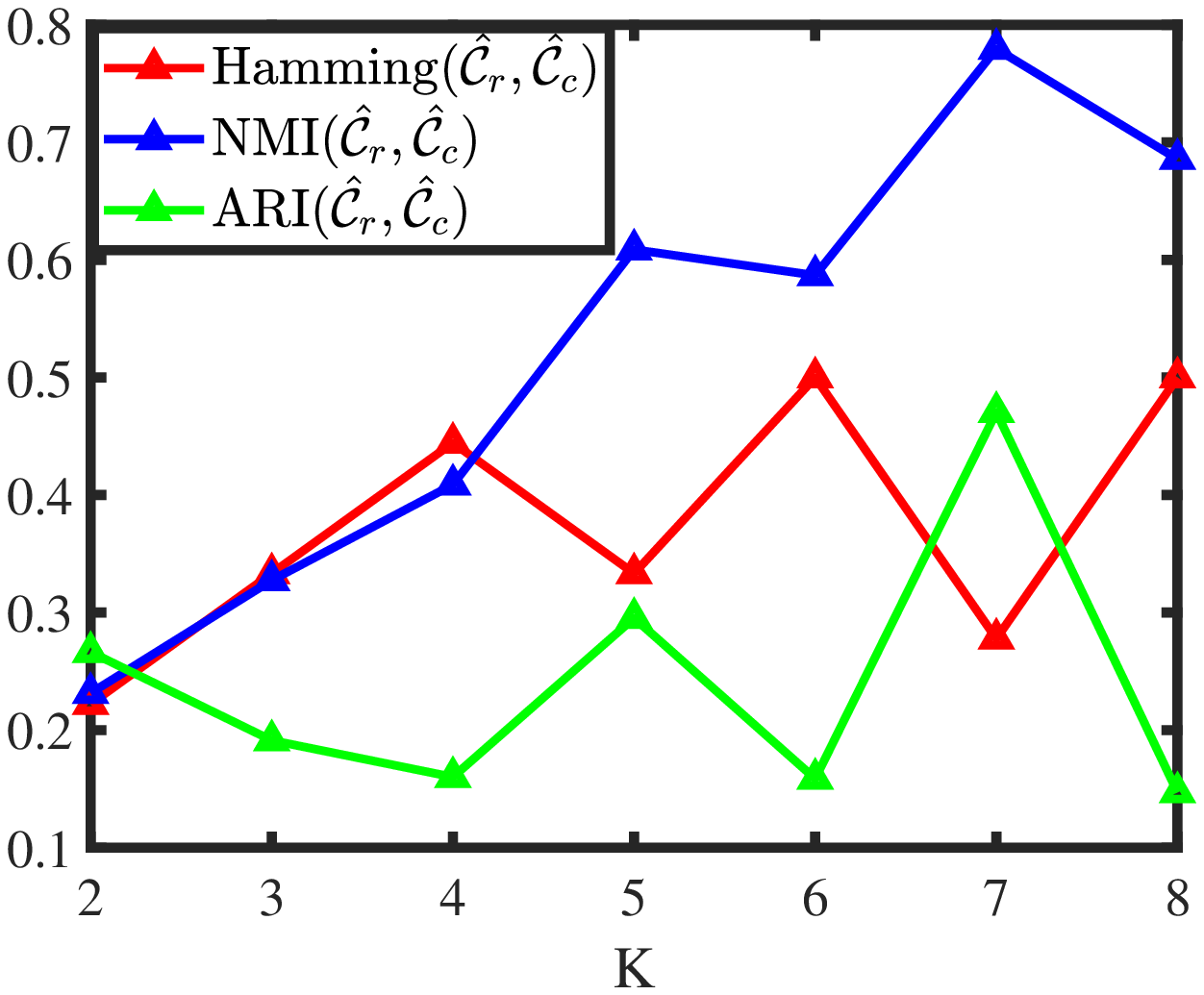}}
\subfigure[Dutch college]{\includegraphics[width=0.24\textwidth]{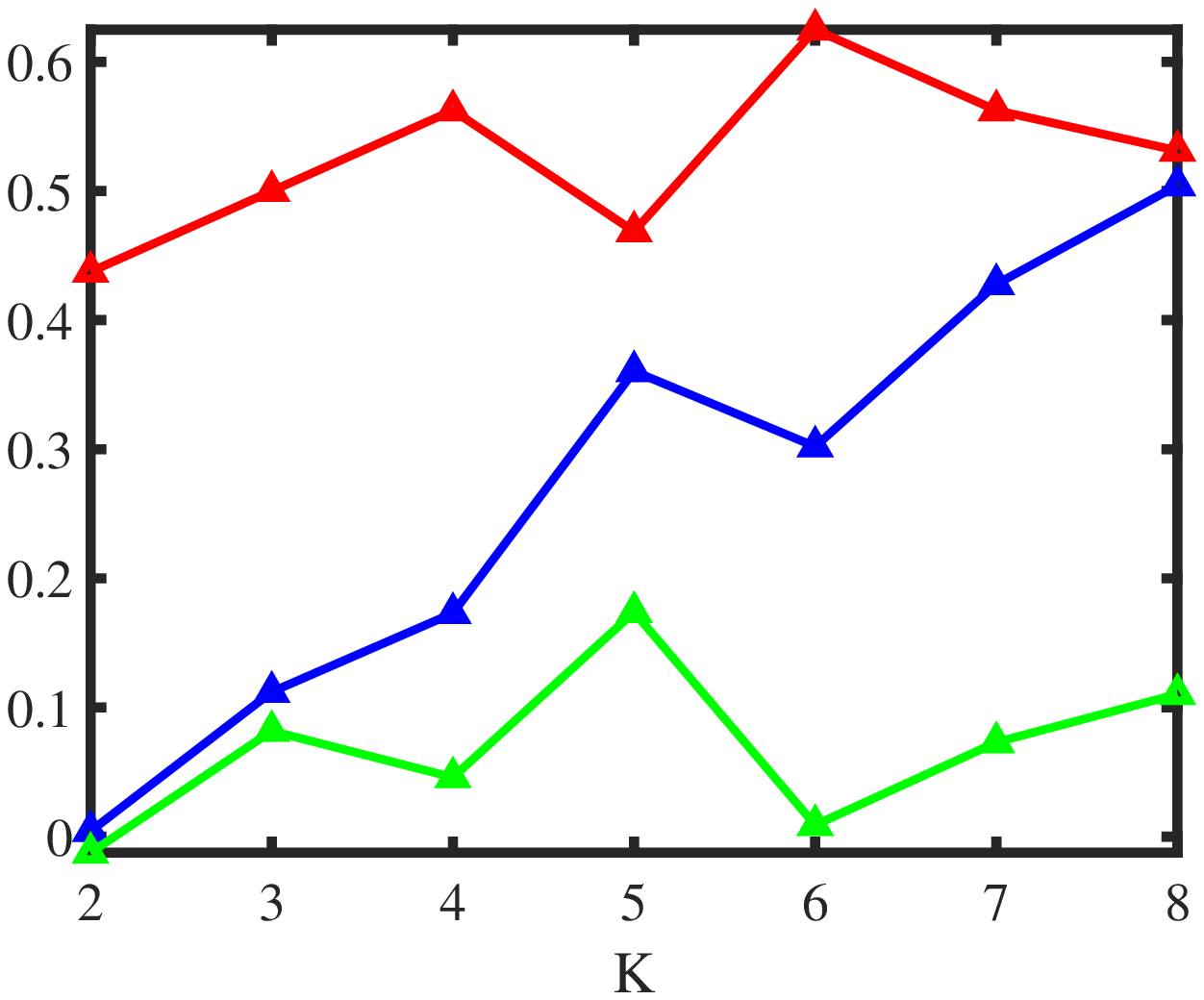}}
\subfigure[Highschool]{\includegraphics[width=0.24\textwidth]{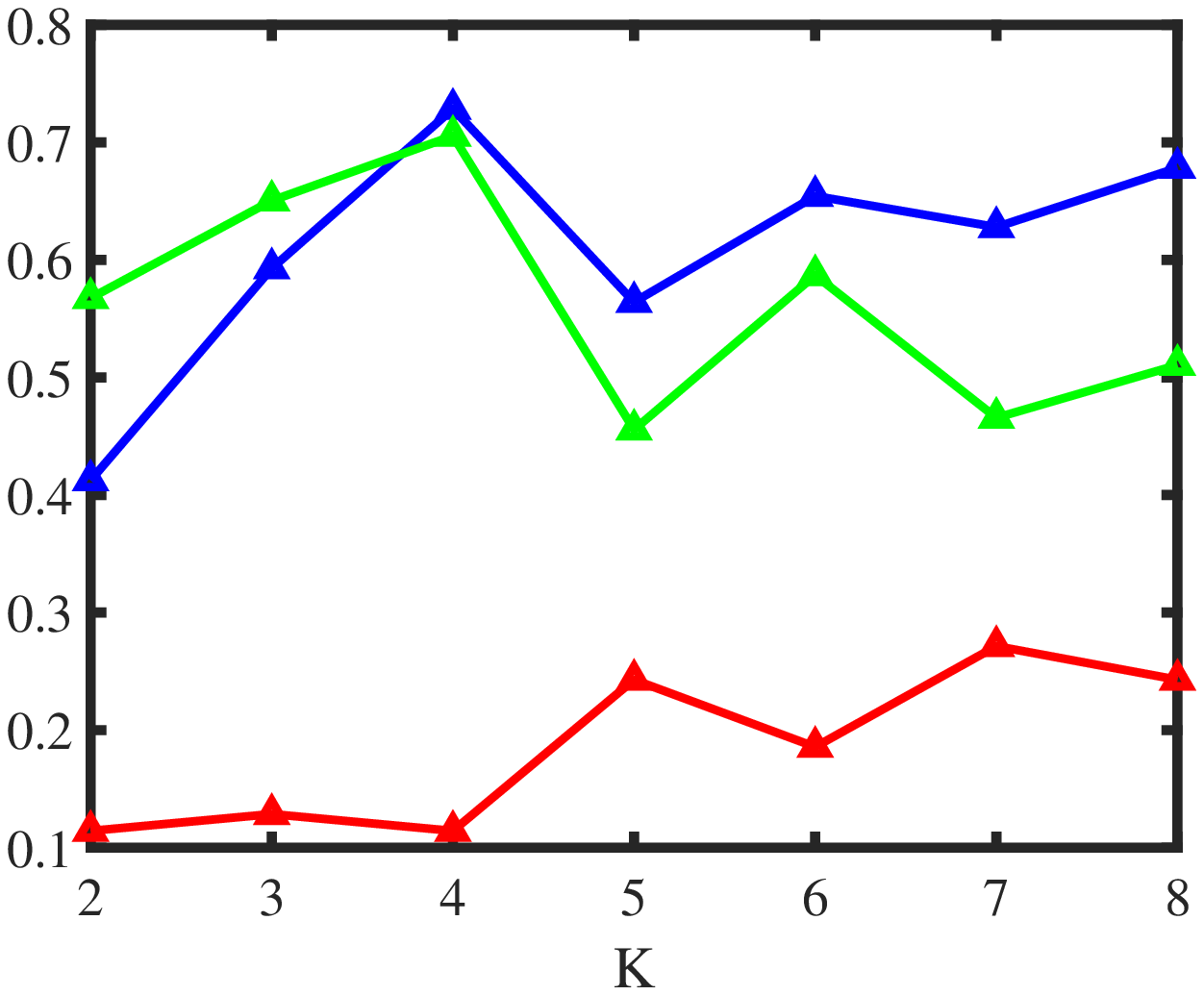}}
\subfigure[Facebook-like Social Network]{\includegraphics[width=0.24\textwidth]{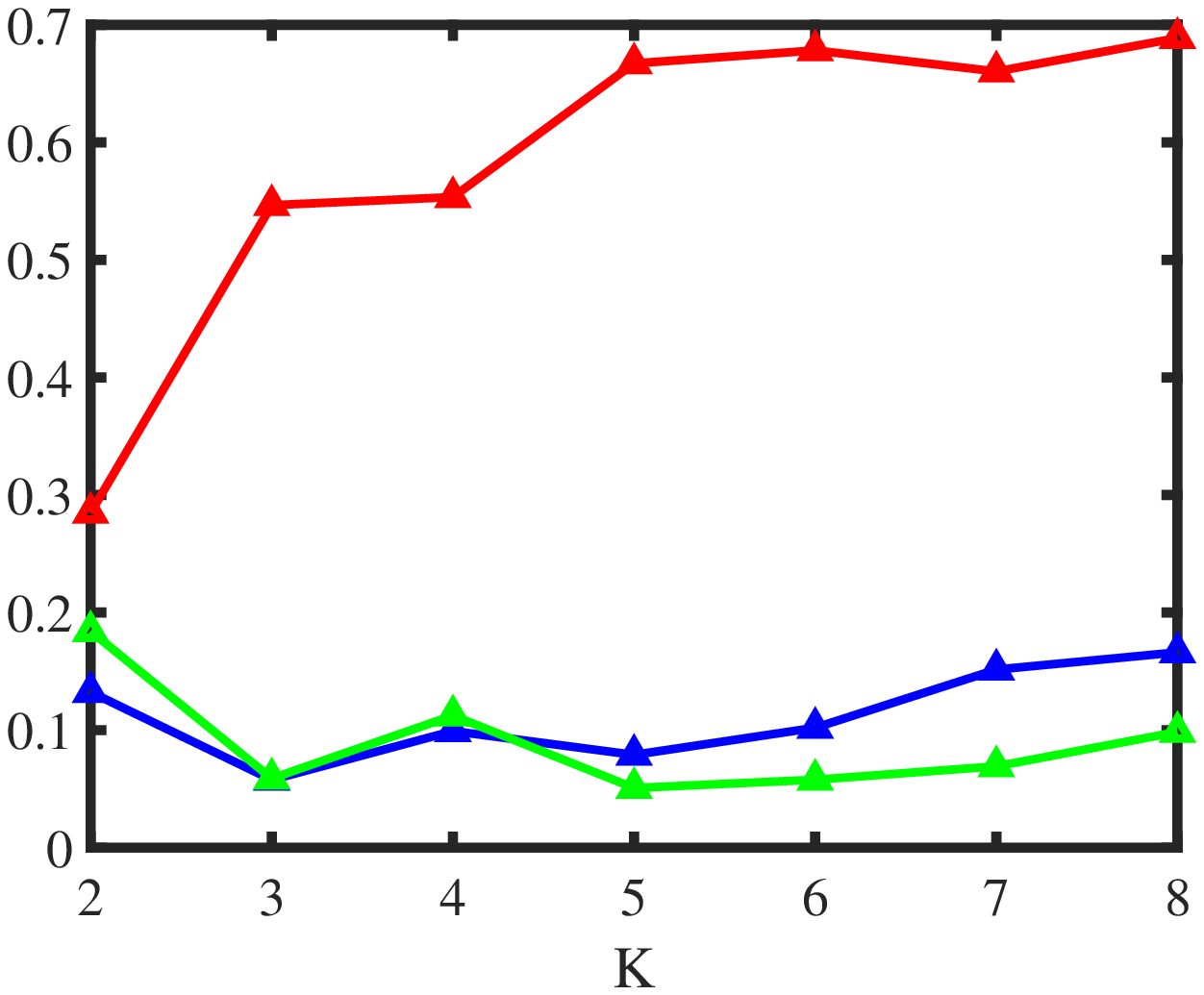}}
\caption{Three indices measuring difference between row clusters and column clusters against $K$ when applying nBiSC on datasets. }
\label{ErrorK} %% label for entire figure
\end{figure}
\begin{figure}
\centering
\subfigure[row clusters of Crisis in a Cloister]{\includegraphics[width=0.35\textwidth]{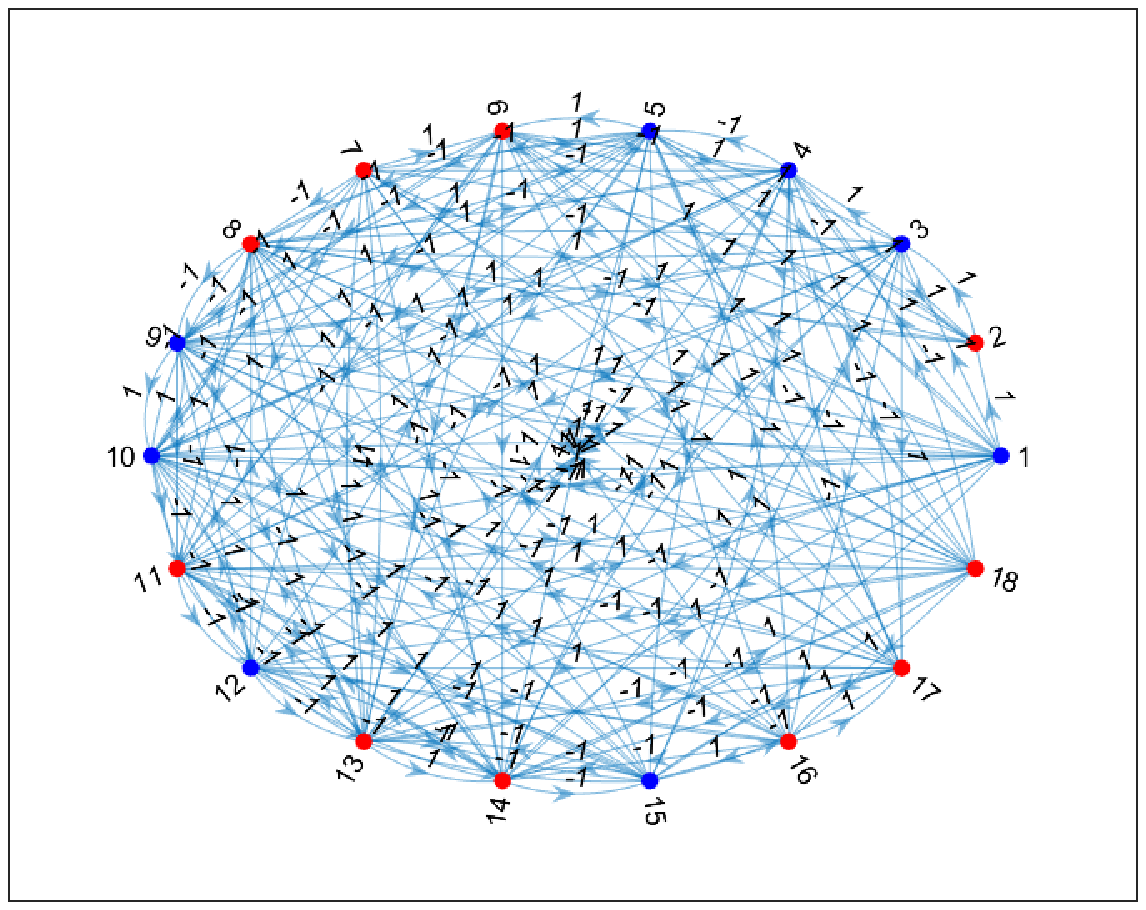}}
\subfigure[column clusters of Crisis in a Cloister]{\includegraphics[width=0.35\textwidth]{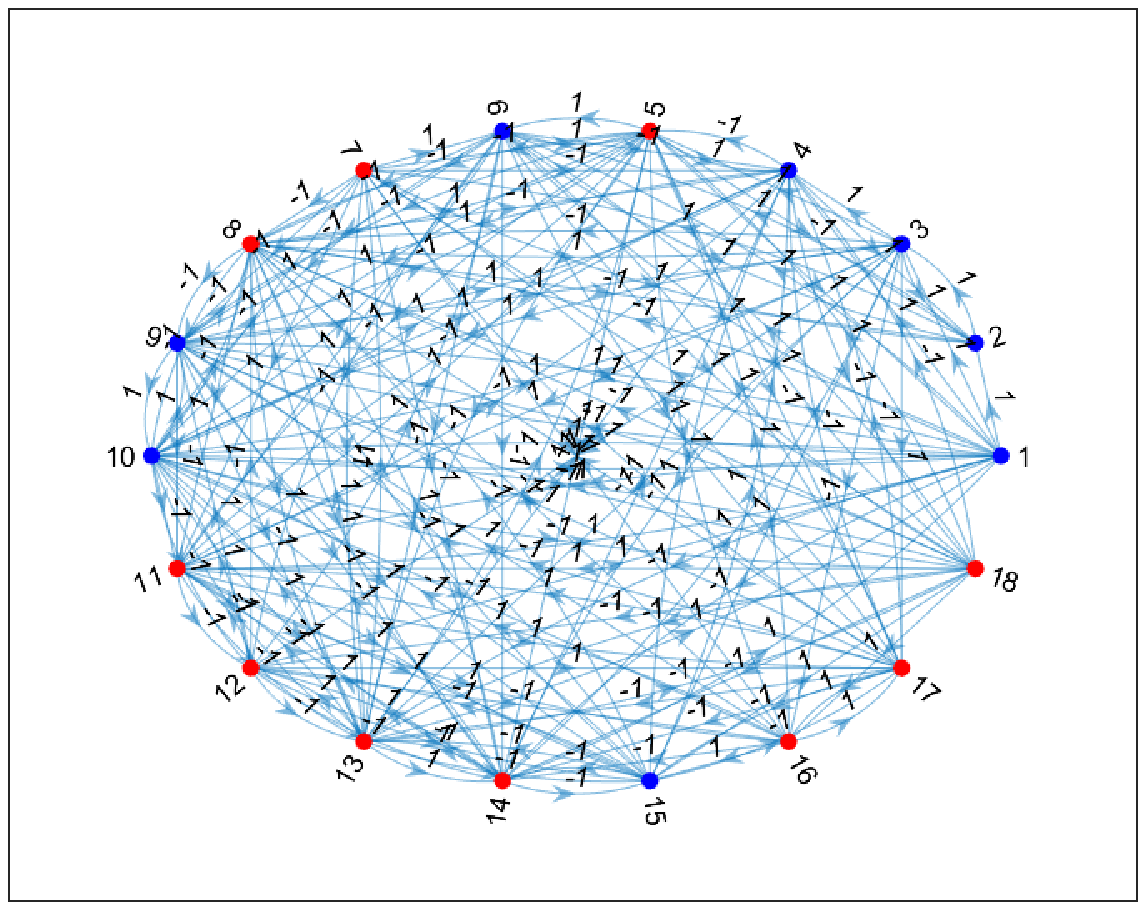}}
\subfigure[row clusters of Dutch college]{\includegraphics[width=0.35\textwidth]{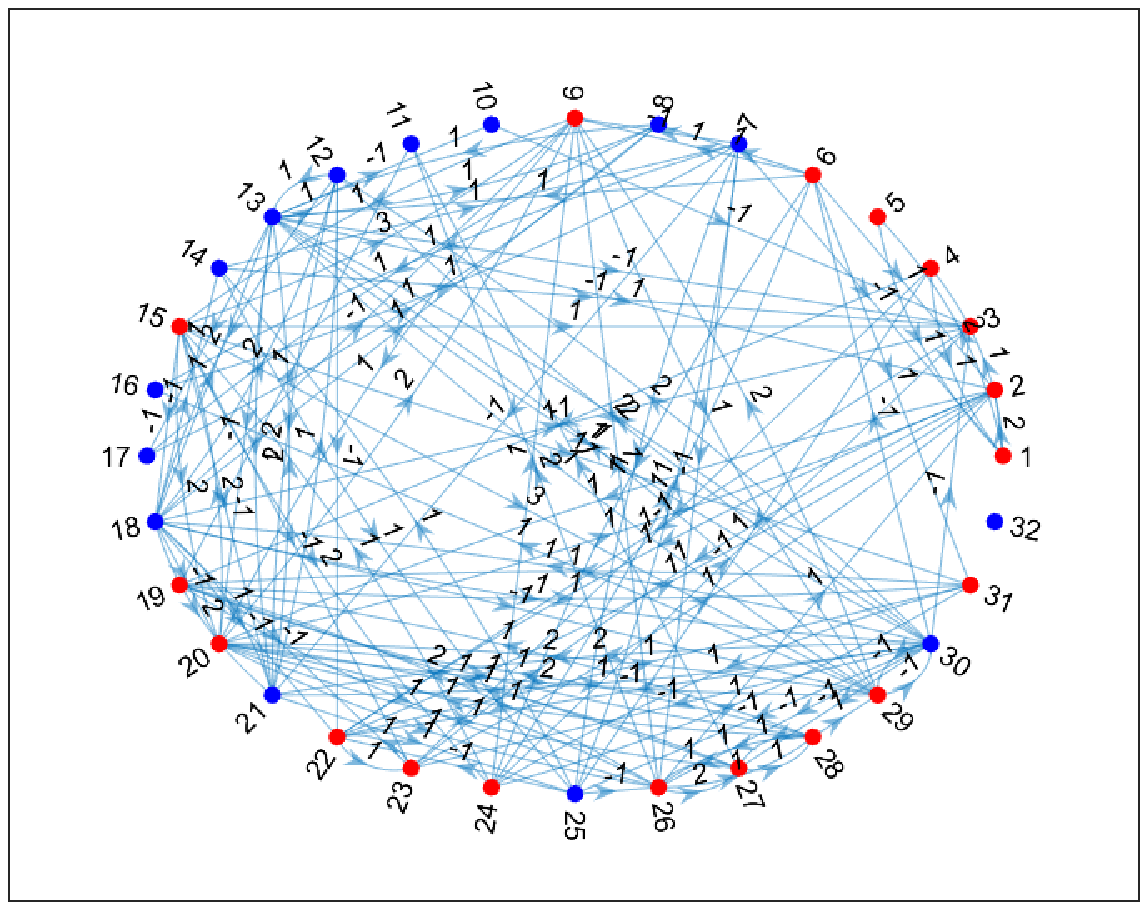}}
\subfigure[column clusters of Dutch college]{\includegraphics[width=0.35\textwidth]{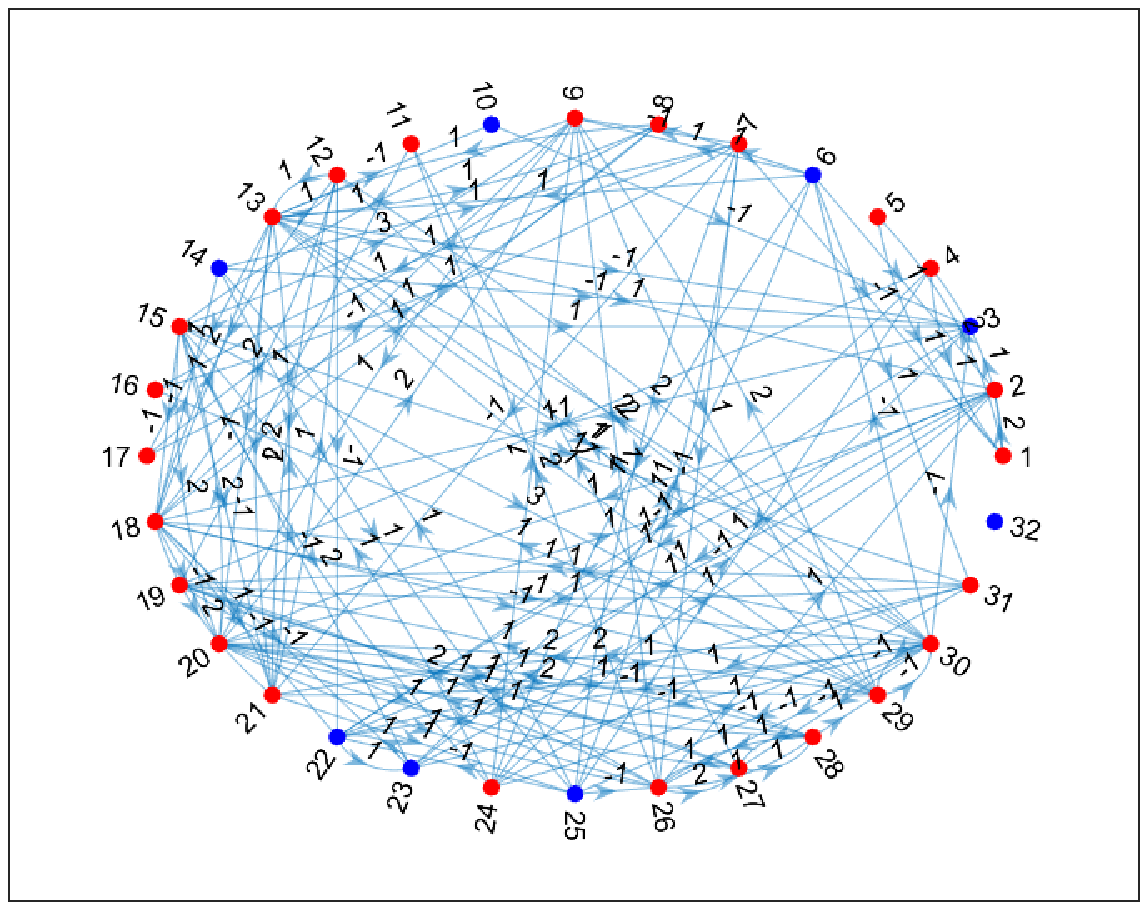}}
\subfigure[row clusters of Highschool]{\includegraphics[width=0.35\textwidth]{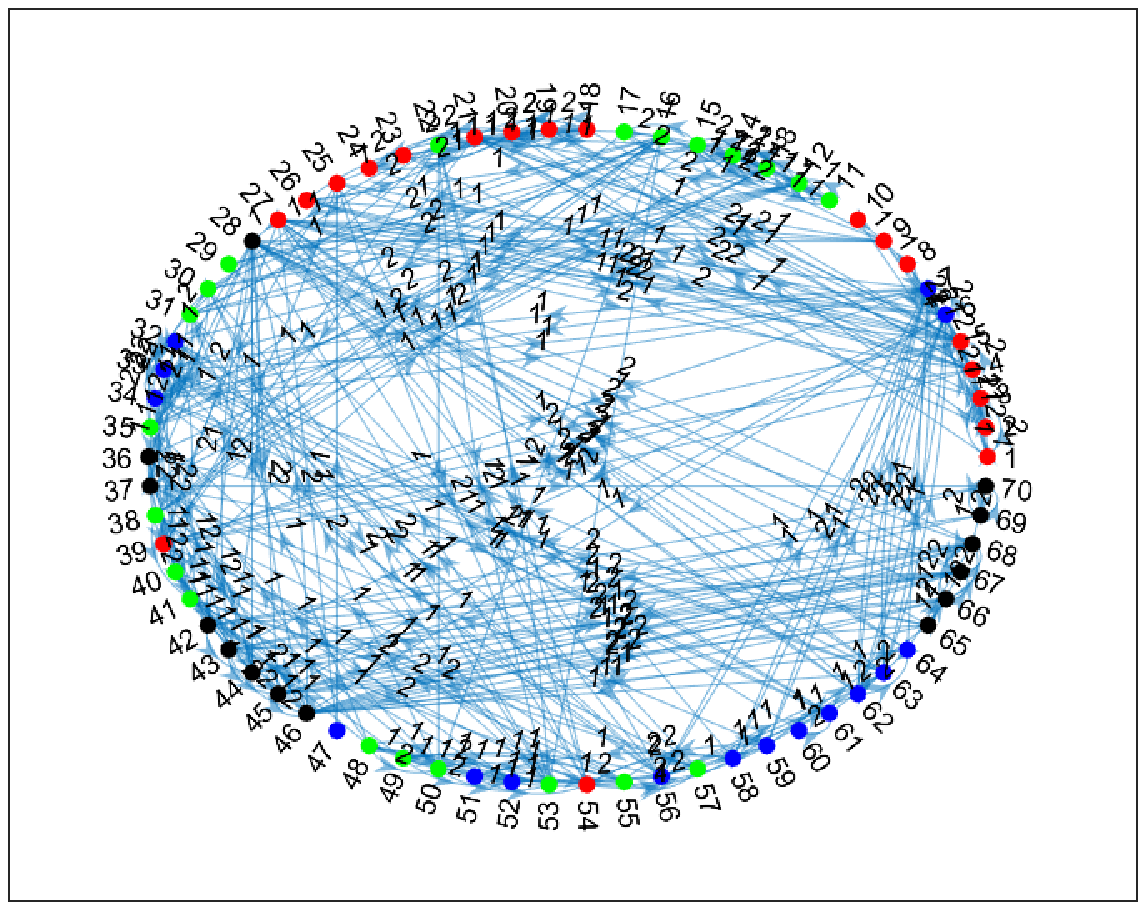}}
\subfigure[column clusters of Highschool]{\includegraphics[width=0.35\textwidth]{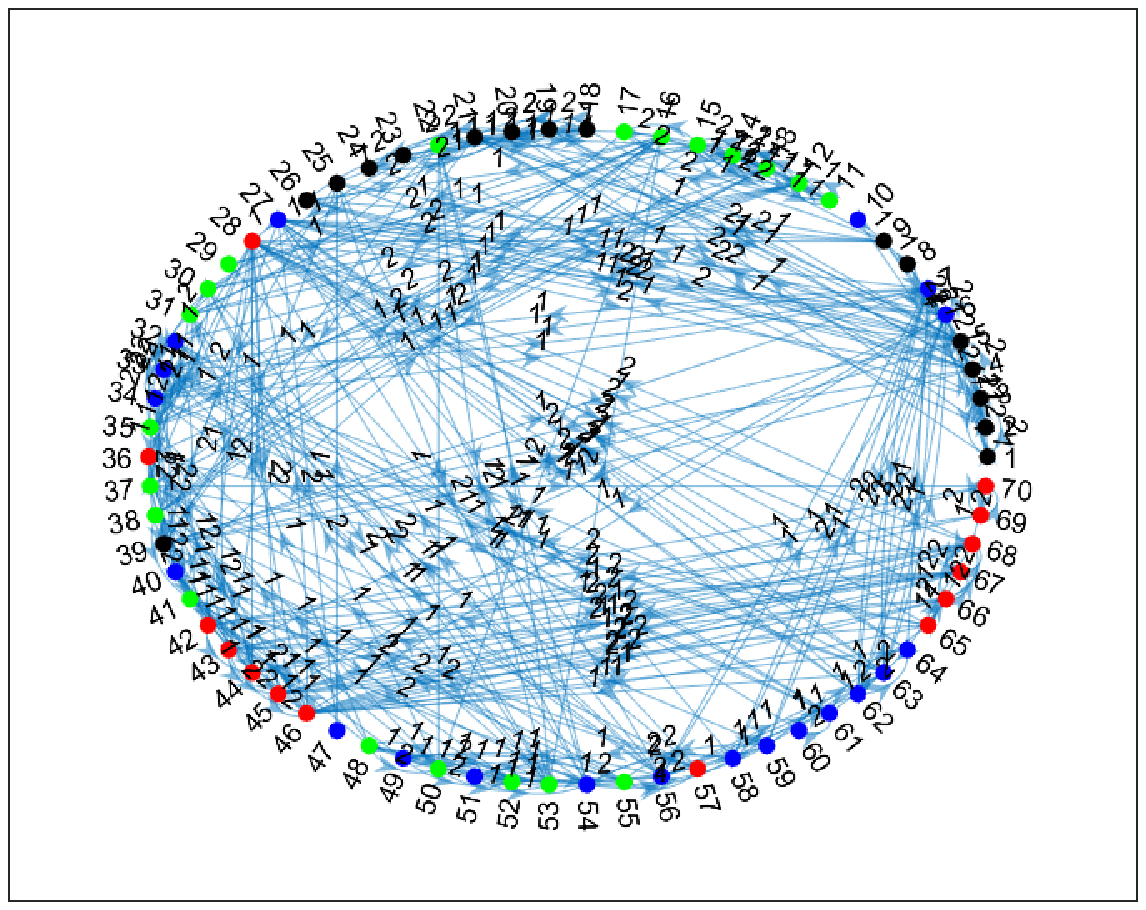}}
\caption{Row and column clusters obtained by nBiSC. For Crisis in a Cloister, Dutch college, and Highschool, as suggested by panels (q), (r), and (s) of Figure \ref{HistEig8}, we set $K$ as 2,2,4, respectively. Panel (a) records the row clusters of monks in Crisis in Cloister, where the value from monk $i$ to monk $j$ denotes the ratings of monk $i$ on monk $j$, and the 18 monks are color-coded by their row cluster returned by nBiSC. Panel (b) records the column clusters of monks in Crisis in Cloister returned by nBiSC. Similar illustration for other panels.}
\label{RandC} %% label for entire figure
\end{figure}
\subsubsection{Real networks without ground truth}
For networks without knowing ground-truth labels and the number of clusters, we plot the histogram of $d_{r}$ and $d_{c}$, and the top 8 singular values of $A$ for these datasets in Figure \ref{HistEig8}. The long tail of node degree suggests the variety of nodes degrees, and we only apply nBiSC to these networks since nBiSC performs better than BiSC when the degree varies and nBiSC usually enjoys at least competitive performance with DI-SIM, D-SCORE, and rD-SCORE. To choose the number of communities for Crisis in a Cloister, Dutch college, and Facebook-like Social network, the eigengaps reveal an ``elbow" at the second singular value, suggesting $K=2$, where \cite{rohe2016co} also used the idea of eigengap to estimate $K$ for real-world directed networks with an unknown number of clusters. For Highschool, the eigengap suggests $K=4$.

Because no ground-truth clusters are available for these networks, and row nodes are the same as column nodes, and $K_{r}=K_{c}=K$, we measure the similarity between row and column clusters by three indices $\mathrm{Hamming}(\hat{\mathcal{C}}_{r},\hat{\mathcal{C}}_{c})\equiv n^{-1}\mathrm{min}_{J\in\mathcal{P}_{K}}\|\hat{Z}_{r}J-\hat{Z}_{c}\|_{0}$, $\mathrm{NMI}(\hat{\mathcal{C}}_{r},\hat{\mathcal{C}}_{c})$ obtained by using $\hat{\mathcal{C}}_{c}$ to replace $\mathcal{C}_{r}$ in Eq (\ref{NMI}), and $\mathrm{ARI}(\hat{\mathcal{C}}_{r},\hat{\mathcal{C}}_{c})$ obtained by using $\hat{\mathcal{C}}_{c}$ to replace $\mathcal{C}_{r}$ in Eq (\ref{ARI}), where $\mathcal{P}_{K}$ is the set of all $K\times K$ permutation matrices and we have used the facts that $\hat{Z}_{r}$ and $\hat{Z}_{c}$ can be obtained from $\hat{\mathcal{C}}_{r}$ and $\hat{\mathcal{C}}_{c}$, respectively. For these three indices, a larger $\mathrm{Hamming}(\hat{\mathcal{C}}_{r},\hat{\mathcal{C}}_{c})$ (or a smaller $\mathrm{NMI}(\hat{\mathcal{C}}_{r},\hat{\mathcal{C}}_{c})$ and a smaller $\mathrm{ARI}(\hat{\mathcal{C}}_{r},\hat{\mathcal{C}}_{c})$) indicates heavier asymmetric structures between row and column communities. Our results are reported in Figure \ref{ErrorK} when changing $K$. Since $\mathrm{Hamming}(\hat{\mathcal{C}}_{r},\hat{\mathcal{C}}_{c})$ is large, $\mathrm{NMI}(\hat{\mathcal{C}}_{r},\hat{\mathcal{C}}_{c})$ and $\mathrm{ARI}(\hat{\mathcal{C}}_{r},\hat{\mathcal{C}}_{c})$ are small for Facebook-like Social Network, we see that this data has the heaviest asymmetric structure between row clusters and column clusters among these four datasets while Highschool has the slightest asymmetric structure. For better visibility on asymmetric structure for nodes, Figure \ref{RandC} depicts row and column clusters returned by nBiSC on three small networks Crisis in a Cloister, Dutch
college, and Highschool with $K$ suggested by eigengap provided by Figure \ref{HistEig8}. We compare node labels between panels (a) and (b) of Figure \ref{RandC} to explain the asymmetric structure. We noticed that although the row nodes and column nodes are the same, the row clusters and column clusters can be  different. This is mainly due to the directionality. For example,  for Crisis in Cloister network, a node represents a monk and an edge between two monks shows that the row monk rated the column monk. Figure \ref{RandC} panel (a) demonstrates the clusters of monks who rated, while panel (b) shows the clusters of monks who were rated. For example, node $1$ and $2$ are in different colors  in panel (a), but they are in the same color  in panel (b), which means that node $1$ and $2$ have different properties when they are rating, but they are similar when they were rated.  Similar arguments hold for Dutch college and Highschool networks, and these support the asymmetric structure in these datasets.

%%%--------------------------------------------------------------
\section{Discussion}\label{sec5}
This paper aims to detect community structures in the weighted bipartite network by extending both spectral clustering and the Stochastic co-Blockmodel to a distribution-free framework. We introduce the Bipartite Distribution-Free model, which is, to the best of our knowledge, the first null model for community detection on weighted bipartite networks. The distribution-free property of BiDFM allows the adjacency matrix $A$ can be generated from various distributions. Under BiDFM, we propose a spectral clustering method BiSC, and then we build a theoretical guarantee on the consistent estimation of the proposed method. We also extend BiDFM to BiDCDFM by considering the degree heterogeneity, and so are the algorithm and the theoretical results. Due to the distribution-free property, our error rates are general. If the distribution is specified, the error rates can also be calculated accordingly, and some examples are given in the paper. From simulation studies, numerical results support our theoretical results. On empirical weighted bipartite networks,  results suggest the dissimilarity between row clusters and column clusters. We expect that the proposed models will have applications beyond this paper, and can be widely used for detecting community structures of weighted bipartite networks in many areas, such as biology, sociology, physiology, computer science, transportation, economy, and so on.

Our idea can be extended in many ways. The hybrid-order stochastic block model \citep{wu2021hybrid} extends SBM to uniformly model the lower-order structure and higher-order structure of an undirected network. It is interesting to extend the hybrid-order idea of \cite{wu2021hybrid} to weighted bipartite networks. The binary tree stochastic block model introduced in \cite{li2022hierarchical} can model the hierarchical tree structure in an undirected unweighted network. Extending BiDFM and BiDCDFM to model the hierarchical tree structure in a weighted bipartite network is appealing. Similar to \cite{rohe2011spectral, RSC,joseph2016impact,rohe2016co}, it is of interest to design algorithms with a theoretical guarantee based on a regularized Laplacian matrix to fit BiDFM and BiDCDFM. BiSC and nBiSC can also be accelerated by the ideas of random-projection and random-sampling developed in \cite{zhang2022randomized} to handle large-scale networks. In this paper, we only focus on non-overlapping networks in which a node only belongs to a single community. Similar to \cite{MMSB,mao2018overlapping,mao2020estimating, OCCAM}, we can extend BiDFM and BiDCDFM to model overlapping weighted bipartite networks in which a node can belong to multiple communities. We leave studies of these problems for our future work.
\section*{CRediT authorship contribution statement}
\textbf{Huan Qing:} Conceptualization, Methodology, Investigation, Software, Formal analysis, Data curation, Writing-original draft, Writing-reviewing \& editing, Funding acquisition. \textbf{Jingli Wang:} Supervision, Methodology, Validation, Visualization,
Writing – review \& editing, Funding acquisition.
\section*{Declaration of competing interest}
The authors declare no competing interests.
\section*{Data availability}
Data and code will be made available on request.
\section*{Acknowledgements}
Qing's work was supported by the High level personal project of Jiangsu Province NO.JSSCBS20211218. Wang's work was supported by the Fundamental Research Funds for the Central Universities, Nankai Univerity, 63221044 and the National Natural Science Foundation of China (Grant 12001295).
\appendix
\section*{Appendix}
\section{Proof of theoretical results  for BiSC}\label{s1}
\subsection{Proof of Lemma \ref{boundABiDFM}}
\begin{proof}
We use Theorem 1.6 (Bernstein inequality for Rectangular case) in \cite{tropp2012user} to bound $\|A-\Omega\|$. This theorem is written below
\begin{thm}\label{Bern}
	Consider a finite sequence $\{X_{k}\}$ of independent, random matrices with dimensions $d_{1}\times d_{2}$. Assume that each random matrix satisfies
	\begin{align*}
		\mathbb{E}[X_{k}]=0, \mathrm{and~}\|X_{k}\|\leq R~\mathrm{almost~surely}.
	\end{align*}
	Then, for all $t\geq 0$,
	\begin{align*}
		\mathbb{P}(\|\sum_{k}X_{k}\|\geq t)\leq (d_{1}+d_{2})\cdot \mathrm{exp}(\frac{-t^{2}/2}{\sigma^{2}+Rt/3}),
	\end{align*}
	where $\sigma^{2}:=\mathrm{max}\{\|\sum_{k}\mathbb{E}(X_{k}X'_{k})\|,\|\sum_{k}\mathbb{E}(X'_{k}X_{k})\|\}$.
\end{thm}
Let $e_{i_{r}}$ be an $n_{r}\times 1$ vector with $e_{i_{r}}(i_{r})=1$ and $0$ elsewhere for row nodes $1\leq i_{r}\leq n_{r}$, and $e_{j_{c}}$ be an $n_{c}\times 1$ vector with $e_{j_{c}}(j_{c})=1$ and $0$ elsewhere for column nodes $1\leq j_{c}\leq n_{c}$.
Set $W=A-\Omega$, then $W=\sum_{i_{r}=1}^{n_{r}}\sum_{j_{c}=1}^{n_{c}}W(i_{r},j_{c})e_{i_{r}}e'_{j_{c}}$. Set $W^{(i_{r},j_{c})}=W(i_{r},j_{c})e_{i_{r}}e'_{j_{c}}$. Since $\mathbb{E}(W(i_{r},j_{c}))=\mathbb{E}(A(i_{r},j_{c})-\Omega(i_{r},j_{c}))=0$, we have $\mathbb{E}(W^{(i_{r},j_{c})})=0$ and
\begin{align*}	\|W^{(i_{r},j_{c})}\|=\|(A(i_{r},j_{c})-\Omega(i_{r},j_{c}))e_{i_{r}}e'_{j_{c}}\|=|A(i_{r},j_{c})-\Omega(i_{r},j_{c})|\|e_{i_{r}}e'_{j_{c}}\|\leq \tau,
\end{align*}
i.e.,  $R=\tau$.

Next, consider the variance parameter
\begin{align*}	\sigma^{2}=\mathrm{max}\{\|\sum_{i_{r}=1}^{n_{r}}\sum_{j_{c}=1}^{n_{c}}\mathbb{E}(W^{(i_{r},j_{c})}(W^{(i_{r},j_{c})})')\|,\|\sum_{i_{r}=1}^{n_{r}}\sum_{j_{c}=1}^{n_{c}}\mathbb{E}((W^{(i_{r},j_{c})})'W^{(i_{r},j_{c})})\|\}.
\end{align*}
Since $\mathbb{E}(W^{2}(i_{r},j_{c}))=\mathbb{E}((A(i_{r},j_{c})-\Omega(i_{r},j_{c}))^{2})=\mathrm{Var}(A(i_{r},j_{c}))\leq \gamma\rho$, we have
\begin{align*}	&\|\sum_{i_{r}=1}^{n_{r}}\sum_{j_{c}=1}^{n_{c}}\mathbb{E}(W^{(i_{r},j_{c})}(W^{(i_{r},j_{c})})')\|=\|\sum_{i_{r}=1}^{n_{r}}\sum_{j_{c}=1}^{n_{c}}\mathbb{E}(W^{2}(i_{r},j_{c}))e_{i_{r}}e'_{j_{c}}e_{j_{c}}e'_{i_{r}}\|\\ &=\|\sum_{i_{r}=1}^{n_{r}}\sum_{j_{c}=1}^{n_{c}}\mathbb{E}(W^{2}(i_{r},j_{c}))e_{i_{r}}e'_{i_{r}}\|\leq \gamma\rho n_{c}.
\end{align*}
Similarly, we have $\|\sum_{i_{r}=1}^{n_{r}}\sum_{j_{c}=1}^{n_{c}}\mathbb{E}((W^{(i_{r},j_{c})})'W^{(i_{r},j_{c})})\|\leq \gamma\rho n_{r}$, which gives that
\begin{align*}
	\sigma^{2}\leq \gamma\rho\mathrm{~max}(n_{r},n_{c}).
\end{align*}
Set $t=\frac{\alpha+1+\sqrt{\alpha^{2}+20\alpha+19}}{3}\sqrt{\gamma\rho \mathrm{~max}(n_{r},n_{c})\mathrm{log}(n_{r}+n_{c})}$. By Theorem \ref{Bern}, we have
\begin{align*}
	&\mathbb{P}(\|W\|\geq t)\leq (n_{r}+n_{c})\mathrm{exp}(-\frac{t^{2}/2}{\sigma^{2}+\frac{Rt}{3}})\leq (n_{r}+n_{c})\mathrm{exp}(-\frac{t^{2}/2}{\gamma\rho\mathrm{~max}(n_{r},n_{c})+Rt/3})\\	&=(n_{r}+n_{c})\mathrm{exp}(-(\alpha+1)\mathrm{log}(n_{r}+n_{c})\cdot \frac{1}{\frac{2(\alpha+1)\gamma\rho\mathrm{~max}(n_{r},n_{c})\mathrm{log}(n_{r}+n_{c})}{t^{2}}+\frac{2(\alpha+1)}{3}\frac{R\mathrm{log}(n_{r}+n_{c})}{t}})\\	&=(n_{r}+n_{c})\mathrm{exp}(-(\alpha+1)\mathrm{log}(n_{r}+n_{c})\cdot \frac{1}{\frac{18}{(\sqrt{\alpha+19}+\sqrt{\alpha+1})^{2}}+\frac{2\sqrt{\alpha+1}}{\sqrt{\alpha+19}+\sqrt{\alpha+1}}\sqrt{\frac{R^{2}\mathrm{log}(n_{r}+n_{c})}{\gamma\rho\mathrm{~max}(n_{r},n_{c})}}})\\
	&\leq (n_{r}+n_{c})\mathrm{exp}(-(\alpha+1)\mathrm{log}(n_{r}+n_{c}))=\frac{1}{(n_{r}+n_{c})^{\alpha}},
\end{align*}
where we have used Assumption \ref{asump} and the fact that \begin{align*}&\frac{18}{(\sqrt{\alpha+19}+\sqrt{\alpha+1})^{2}}+\frac{2\sqrt{\alpha+1}}{\sqrt{\alpha+19}+\sqrt{\alpha+1}}\sqrt{\frac{R^{2}\mathrm{log}(n_{r}+n_{c})}{\gamma\rho\mathrm{~max}(n_{r},n_{c})}}\\
	&\leq \frac{18}{(\sqrt{\alpha+19}+\sqrt{\alpha+1})^{2}}+\frac{2\sqrt{\alpha+1}}{\sqrt{\alpha+19}+\sqrt{\alpha+1}}=1\end{align*}
in the last inequality. Thus, the claim follows.
\end{proof}
\subsection{Proof of Theorem \ref{mainBiDFM}}
First, we present Lemma \ref{boundUrUcBiDFM} which bounds the difference between $\hat{U}_{r}$ and $U_{r}$  ($\hat{U}_{c}$ and $U_{c}$) up to an orthogonal matrix, and this lemma is directly related with the error rates of the BiSC algorithm. For convenience, denote $err_{s}=\|A-\Omega\|$ under $BiDFM(Z_{r}, Z_{c},P,\rho)$.
\begin{lem}\label{boundUrUcBiDFM}
Under $BiDFM(Z_{r}, Z_{c}, P,\rho)$, we have
\begin{align*}	\mathrm{max}(\|\hat{U}_{r}\hat{O}-U_{r}\|_{F},\|\hat{U}_{c}\hat{O}-U_{c}\|_{F})\leq\frac{2\sqrt{2K_{r}}err_{s}}{\sigma_{K_{r}}(P)\rho\sqrt{n_{r,\mathrm{min}}n_{c,\mathrm{min}}}},
\end{align*}
where $\hat{O}$ is a $K_{r}\times K_{r}$ orthogonal matrix.
\end{lem}
\begin{proof}
By the proof of Lemma 3 in \cite{zhou2019analysis} (the technique for the proof of \cite{zhou2019analysis}'s Lemma 3 is based on applications of Davis-Kahan theorem \citep{yu2015a} and the symmetric dilation operator, see their proof for detail), we know that there exists an orthogonal matrix $\hat{O}$ such that
\begin{align*}	\mathrm{max}(\|\hat{U}_{r}\hat{O}-U_{r}\|_{F},\|\hat{U}_{c}\hat{O}-U_{c}\|_{F})\leq\frac{\sqrt{2K_{r}}\|\tilde{A}-\Omega\|}{\sqrt{\lambda_{K_{r}}(\Omega\Omega')}}.
\end{align*}
Since $\tilde{A}$ is the $K_{r}$ dimensional reduced SVD of $A$ and $\mathrm{rank}(\Omega)=K_{r}$, we have $\|A-\tilde{A}\|\leq\|A-\Omega\|$, which gives that $\|\tilde{A}-\Omega\|=\|\tilde{A}-A+A-\Omega\|\leq 2\|A-\Omega\|$. Hence, there exists a $K_{r}\times K_{r}$ orthogonal matrix $\hat{O}$ such that
\begin{align}\label{BoungZhouZhixin} \mathrm{max}(\|\hat{U}_{r}\hat{O}-U_{r}\|_{F},\|\hat{U}_{c}\hat{O}-U_{c}\|_{F})\leq\frac{2\sqrt{2K_{r}}\|A-\Omega\|}{\sqrt{\lambda_{K_{r}}(\Omega\Omega')}}.
\end{align}
For $\lambda_{K_{r}}(\Omega\Omega')$, we have
\begin{align*}	\lambda_{K_{r}}(\Omega\Omega')&=\lambda_{K_{r}}(\rho Z_{r}PZ'_{c}\rho Z_{c}P'Z'_{r})=\rho^{2}\lambda_{K_{r}}(Z'_{r}Z_{r}PZ'_{c}Z_{c}P')\\
	&\geq\rho^{2}\lambda_{K_{r}}(Z'_{r}Z_{r})\lambda_{K_{r}}(PZ'_{c}Z_{c}P')=\rho^{2}\lambda_{K_{r}}(Z'_{r}Z_{r})\lambda_{K_{r}}(Z'_{c}Z_{c}P'P)\\ &\geq\rho^{2}\lambda_{K_{r}}(Z'_{r}Z_{r})\lambda_{K_{r}}(Z'_{c}Z_{c})\lambda_{K_{r}}(PP')\geq \rho^{2} n_{r,\mathrm{min}}n_{c,\mathrm{min}}\lambda_{K_{r}}(PP'),
\end{align*}
where we have used the facts that $\lambda_{K_{r}}(Z'_{r}Z_{r})=n_{r,\mathrm{min}}, \lambda_{K_{r}}(Z'_{c}Z_{c})\geq n_{c,\mathrm{min}}$, and for any matrices $X, Y$, the nonzero eigenvalues of $XY$ are the same as the nonzero eigenvalues of $YX$. The lower bound of $\lambda_{K_{r}}(\Omega\Omega')$ gives
\begin{align*} \mathrm{max}(\|\hat{U}_{r}\hat{O}-U_{r}\|_{F},\|\hat{U}_{c}\hat{O}-U_{c}\|_{F})\leq\frac{2\sqrt{2K_{r}}\|A-\Omega\|}{\sigma_{K_{r}}(P)\rho\sqrt{n_{r,\mathrm{min}}n_{c,\mathrm{min}}}}.
\end{align*}
\end{proof}
\begin{proof}
Now, we start the proof of Theorem \ref{mainBiDFM}. Let $\varsigma_{r}>0$ be a small quantity, by Lemma 2 in \cite{joseph2016impact} and Lemma 2.1, if
\begin{align}\label{holdrBiDFM}	\frac{\sqrt{K_{r}}}{\varsigma_{r}}\|U_{r}-\hat{U}_{r}\hat{O}\|_{F}(\frac{1}{\sqrt{n_{r,k}}}+\frac{1}{\sqrt{n_{r,l}}})\leq \sqrt{\frac{1}{n_{r,k}}+\frac{1}{n_{r,l}}}, \mathrm{~for~each~}1\leq k\neq l\leq K_{r},
\end{align}
then the clustering error $\hat{f}_{r}=O(\varsigma^{2}_{r})$. Setting $\varsigma_{r}=\sqrt{\frac{2K_{r}n_{r,\mathrm{max}}}{n_{r,\mathrm{min}}}}\|U_{r}-\hat{U}_{r}\hat{O}\|_{F}$ makes Eq (\ref{holdrBiDFM}) hold for all $1\leq k\neq l\leq K_{r}$. Then we have $\hat{f}_{r}=O(\varsigma^{2}_{r})=O(\frac{K_{r}n_{r,\mathrm{max}}\|U_{r}-\hat{U}_{r}\hat{O}\|^{2}_{F}}{n_{r,\mathrm{min}}})$ by Lemma 2 in \cite{joseph2016impact}. Combine with Lemma \ref{boundUrUcBiDFM}, we have
\begin{align*}
	\hat{f}_{r}=O(\frac{K^{2}_{r}n_{r,\mathrm{max}}}{n_{r,\mathrm{min}}}\frac{err^{2}_{s}}{\sigma^{2}_{K_{r}}(P)\rho^{2}n_{r,\mathrm{min}}n_{c,\mathrm{min}}}).
\end{align*}
Similarly, let $\varsigma_{c}>0$, if
\begin{align}\label{holdcBiDFM}	\frac{\sqrt{K_{c}}}{\varsigma_{c}}\|U_{c}-\hat{U}_{c}\hat{O}\|_{F}(\frac{1}{\sqrt{n_{c,k}}}+\frac{1}{\sqrt{n_{c,l}}})\leq \|X_{c}(k,:)-X_{c}(l,:)\|_{F}, \mathrm{~for~each~}1\leq k\neq l\leq K_{c},
\end{align}
then $\hat{f}_{c}=O(\varsigma^{2}_{c})$. Since we set $\delta_{c}=\mathrm{min}_{k\neq l}\|X_{c}(k,:)-X_{c}(l,:)\|_{F}$, setting $\varsigma_{c}=\frac{2}{\delta_{c}}\sqrt{\frac{K_{c}}{n_{c,\mathrm{min}}}}\|U_{c}-\hat{U}_{c}\hat{O}\|_{F}$ makes Eq (\ref{holdcBiDFM}) hold for all $1\leq k\neq l\leq K_{c}$. Then we have $\hat{f}_{c}=O(\varsigma^{2}_{c})=O(\frac{K_{c}\|U_{c}-\hat{U}_{c}\hat{O}\|^{2}_{F}}{\delta^{2}_{c}n_{c,\mathrm{min}}})$. By Lemma \ref{boundUrUcBiDFM}, we have
\begin{align*}
	\hat{f}_{c}=O(\frac{K_{r}K_{c}}{\delta^{2}_{c}n_{c,\mathrm{min}}}\frac{err^{2}_{s}}{\sigma^{2}_{K_{r}}(P)\rho^{2}n_{r,\mathrm{min}}n_{c,\mathrm{min}}}).
\end{align*}
By Lemma \ref{boundABiDFM}, the theorem holds.
\end{proof}
\vspace*{-10pt}
%\appendixtwo
%\section*{Appendix 2}
\section{Proof of theoretical results  for nBiSC}\label{s2}
\subsection{Proof of Lemma \ref{populationBiDCDFM}}
\begin{proof}
Before proving this lemma, we provide one lemma which presents SVD of $\Omega$ and is helpful for our proof.
Let $D_{r}\in\mathbb{R}^{K_{r}\times K_{r}}, D_{c}\in\mathbb{R}^{K_{c}\times K_{c}}$ be two diagonal matrices such that
\begin{align*}
	D_{r}(k,k)=\frac{\|\Theta_{r}Z_{r}(:,k)\|_{F}}{\|\theta_{r}\|_{F}}, D_{c}(l,l)=\frac{\|\Theta_{c}Z_{c}(:,l)\|_{F}}{\|\theta_{c}\|_{F}}, \qquad \mathrm{for~}1\leq k\leq K_{r}, 1\leq l\leq K_{c}.
\end{align*}
Let $\Gamma_{r}\in \mathbb{R}^{n_{r}\times K_{r}}, \Gamma_{c}\in \mathbb{R}^{n_{c}\times K_{c}}$ be two matrices such that
\begin{align*}
	\Gamma_{r}(:,k)=\frac{\Theta_{r}Z_{r}(:,k)}{\|\Theta_{r}Z_{r}(:,k)\|_{F}}, \Gamma_{c}(:,l)=\frac{\Theta_{c}Z_{c}(:,l)}{\|\Theta_{c}Z_{c}(:,l)\|_{F}}, \qquad \mathrm{for~}1\leq k\leq K_{r},1\leq l\leq K_{c}.
\end{align*}
Then we have $\Gamma'_{r}\Gamma_{r}=I_{K_{r}},$ $ \Gamma'_{c}\Gamma_{c}=I_{K_{c}}$, and $\Omega=\|\theta_{r}\|_{F}\|\theta_{c}\|_{F}\Gamma_{r}D_{r}PD_{c}\Gamma'_{c}$. The following lemma provides the singular value decomposition of $\Omega$.
\begin{lem}\label{eigenvectorsBiDCDFM}
	Under $BiDCDFM(Z_{r}, Z_{c}, P, \Theta_{r}, \Theta_{c})$, let  $\Omega=U_{r}\Lambda U'_{c}$ be the compact singular value decomposition of $\Omega$ (i.e., $U_{r}\in \mathbb{R}^{n_{r}\times K_{r}}, U_{c}\in \mathbb{R}^{n_{c}\times K_{r}}$ and $U'_{r}U_{r}=I_{K_{r}}, U'_{c}U_{c}=I_{K_{r}}$). Let $D_{r}PD_{c}=V_{r}\Sigma V'_{c}$ be the compact singular value decomposition of $D_{r}PD_{c}$ (i.e., $V_{r}\in\mathbb{R}^{K_{r}\times K_{r}}, V_{c}\in\mathbb{R}^{K_{c}\times K_{r}}$, and $V'_{r}V_{r}=I_{K_{r}}, V'_{c}V_{c}=I_{K_{r}}$), we have
	\begin{align*}
		\Lambda=\|\theta_{r}\|_{F}\|\theta_{c}\|_{F}\Sigma, U_{r}=\Gamma_{r}V_{r}, \mathrm{and~}U_{c}=\Gamma_{r}V_{c}.
	\end{align*}
\end{lem}
\begin{proof}
	Since $\Omega=\|\theta_{r}\|_{F}\|\theta_{c}\|_{F}\Gamma_{r}D_{r}PD_{c}\Gamma'_{c}$ and $D_{r}PD_{c}=V_{r}\Sigma V'_{c}$, we have
	\begin{align}\label{OmegaExp}
		\Omega=(\Gamma_{r}V_{r})(\|\theta_{r}\|_{F}\|\theta_{c}\|_{F}\Sigma) (\Gamma_{c}V_{c})'.
	\end{align}
	Since $(\Gamma_{r}V_{r})'\Gamma_{r}V_{r}=I_{K_{r}}, (\Gamma_{c}V_{c})'\Gamma_{c}V_{c}=I_{K_{r}}$, $\Gamma_{r}V_{r}$ and $\Gamma_{c}V_{c}$ have orthogonal columns. Thus, Eq (\ref{OmegaExp}) is the compact SVD of $\Omega$. Since $\Omega=U_{r}\Lambda U'_{c}$ denotes the compact SVD of $\Omega$, we have
	\begin{align*}
		\Lambda=\|\theta_{r}\|_{F}\|\theta_{c}\|_{F}\Sigma, U_{r}=\Gamma_{r}V_{r}, \mathrm{and~}U_{c}=\Gamma_{r}V_{c}.
	\end{align*}
\end{proof}
Now, we consider $U_{r,*}$ first. By Lemma \ref{eigenvectorsBiDCDFM}, since $U_{r}=\Gamma_{r}V_{r}$, combine it with the definition of $\Gamma_{r}$, for $1\leq i_{r}\leq n_{r}, 1\leq k\leq K_{r}$, we have
\begin{align*}	U_{r}(i_{r},k)&=e'_{i_{r}}U_{r}e_{k}=e'_{i_{r}}\Gamma_{r}V_{r}e_{k}=\Gamma_{r}(i_{r},:)V_{r}e_{k}\\	&=\theta_{r}(i_{r})[\frac{Z_{r}(i_{r},1)}{\|\Theta_{r}Z_{r}(:,1)\|_{F}}~\frac{Z_{r}(i_{r},2)}{\|\Theta_{r}Z_{r}(:,2)\|_{F}}~\ldots~\frac{Z_{r}(i_{r},K_{r})}{\|\Theta_{r}Z_{r}(:,K_{r})\|_{F}}]V_{r}e_{k}\\
	&=\frac{\theta_{r}(i_{r})}{\|\Theta_{r}Z_{r}(:,g_{i_{r}})\|_{F}}V_{r}(g_{i_{r}},k),
\end{align*}
which gives that
\begin{align*}	U_{r}(i_{r},:)=\frac{\theta_{r}(i_{r})}{\|\Theta_{r}Z_{r}(:,g_{i_{r}})\|_{F}}[V_{r}(g_{i_{r}},1)~V_{r}(g_{i_{r}},2~)\ldots~V_{r}(g_{i_{r}},K_{r})]=\frac{\theta_{r}(i_{r})}{\|\Theta_{r}Z_{r}(:,g_{i_{r}})\|_{F}}V_{r}(g_{i_{r}},:).
\end{align*}
For convenience, set $V_{r}(:,k)=a_{k}$ for $1\leq k\leq K_{r}$, then we have
\begin{align}\label{Uex}
	U_{r,*}(i_{r},:)=V_{r}(g_{i_{r}})=[a_{1}(g_{i_{r}}), a_{2}(g_{i_{r}}), \ldots, a_{K_{r}}(g_{i_{r}})],
\end{align}
where the last equality holds since $a_{1}, a_{2}, \ldots, a_{K_{r}}$ form an orthonormal base, i.e., $\sum_{j=1}^{K_{r}}a^{2}_{j}(k)=1, \sum_{l=1}^{K_{r}}a^{2}_{k}(l)=1$ and $\sum_{j=1}^{K_{r}}a_{j}(k)a_{j}(l)=0$ for $1\leq k,l\leq K_{r}$. Hence, when $g_{i_{r}}=g_{\bar{i}_{r}}$ and $i_{r}\neq \bar{i}_{r}$, we have $U_{r,*}(i_{r},:)=U_{r,*}(\bar{i}_{r},:)$ for $i_{r},\bar{i}_{r} =1,\dots, n_{r}$. When $g_{i_{r}}\neq g_{\bar{i}_{r}}$, we have
\begin{align*}	\|U_{r,*}(i_{r},:)-U_{r,*}(\bar{i}_{r},:)\|_{F}=\sqrt{\sum_{k=1}^{K_{r}}(a_{k}(g_{i_{r}})-a_{k}(g_{\bar{i}_{r}}))^{2}}=\sqrt{2}.
\end{align*}
Meanwhile, Eq (\ref{Uex}) also gives that $U_{r,*}=Z_{r}V_{r}$.

For $U_{c,*}$, follow similar proof as $U_{r,*}$, we have $U_{c,*}=Z_{c}V_{c}$. When $K_{r}<K_{c}$, the $K_{r}$ columns of $V_{c}$ do not form an orthonormal base. When $K_{r}=K_{c}$, the $K_{r}$ columns of $V_{c}$ form an orthonormal base, and this lemma holds by following similar proof as $U_{r,*}$.
\end{proof}
\subsection{Proof of Lemma \ref{boundABiDCDFM}}
\begin{proof}
Similar to the proof of Lemma \ref{boundABiDFM}, we bound $\|A-\Omega\|$ under $BiDCDFM(Z_{r},Z_{c},P,\Theta_{r},\Theta_{c})$ by Theorem \ref{Bern}. Set $W=A-\Omega$, and $W^{(i_{r},j_{c})}=W(i_{r},j_{c})e_{i_{r}}e'_{j_{c}}$. Since $\mathbb{E}(W(i_{r},j_{c}))=\mathbb{E}(A(i_{r},j_{c})-\Omega(i_{r},j_{c}))=0$, we have $\mathbb{E}(W^{(i_{r},j_{c})})=0$ and
$\|W^{(i_{r},j_{c})}\|\leq\tau$, i.e., $R=\tau$.
Next, consider the variance parameter
\begin{align*}	\sigma^{2}=\mathrm{max}\{\|\sum_{i_{r}=1}^{n_{r}}\sum_{j_{c}=1}^{n_{c}}\mathbb{E}(W^{(i_{r},j_{c})}(W^{(i_{r},j_{c})})')\|,\|\sum_{i_{r}=1}^{n_{r}}\sum_{j_{c}=1}^{n_{c}}\mathbb{E}((W^{(i_{r},j_{c})})'W^{(i_{r},j_{c})})\|\}.
\end{align*}
Since $\mathbb{E}(W^{2}(i_{r},j_{c}))=\mathbb{E}((A(i_{r},j_{c})-\Omega(i_{r},j_{c}))^{2})=\mathrm{Var}(A(i_{r},j_{c}))\leq\gamma_{*}\theta_{r}(i_{r})\theta_{c}(j_{c})$, we have
\begin{align*}	&\|\sum_{i_{r}=1}^{n_{r}}\sum_{j_{c}=1}^{n_{c}}\mathbb{E}(W^{(i_{r},j_{c})}(W^{(i_{r},j_{c})})')\|=\|\sum_{i_{r}=1}^{n_{r}}\sum_{j_{c}=1}^{n_{c}}\mathbb{E}(W^{2}(i_{r},j_{c}))e_{i_{r}}e'_{j_{c}}e_{j_{c}}e'_{i_{r}}\|\\	&=\|\sum_{i_{r}=1}^{n_{r}}\sum_{j_{c}=1}^{n_{c}}\mathbb{E}(W^{2}(i_{r},j_{c}))e_{i_{r}}e'_{i_{r}}\|\leq \gamma_{*}\theta_{r,\mathrm{max}}\|\theta_{c}\|_{1}.
\end{align*}
Similarly, we have $\|\sum_{i_{r}=1}^{n_{r}}\sum_{j_{c}=1}^{n_{c}}\mathbb{E}((W^{(i_{r},j_{c})})'W^{(i_{r},j_{c})})\|\leq\gamma_{*}\theta_{c,\mathrm{max}}\|\theta_{r}\|_{1}$, which gives that
\begin{align*}
	\sigma^{2}\leq\gamma_{*}\mathrm{~max}(\theta_{r,\mathrm{max}}\|\theta_{c}\|_{1},\theta_{c,\mathrm{max}}\|\theta_{r}\|_{1}).
\end{align*}
Set $t=\frac{\alpha+1+\sqrt{\alpha^{2}+20\alpha+19}}{3}\sqrt{\gamma_{*}\mathrm{~max}(\theta_{r,\mathrm{max}}\|\theta_{c}\|_{1},\theta_{c,\mathrm{max}}\|\theta_{r}\|_{1})\mathrm{log}(n_{r}+n_{c})}$. By Theorem \ref{Bern}, we have
\begin{align*}
	&\mathbb{P}(\|W\|\geq t)\leq (n_{r}+n_{c})\mathrm{exp}(-\frac{t^{2}/2}{\sigma^{2}+\frac{Rt}{3}})\leq (n_{r}+n_{c})\mathrm{exp}(-\frac{t^{2}/2}{\gamma_{*}\mathrm{max}(\theta_{r,\mathrm{max}}\|\theta_{c}\|_{1},\theta_{c,\mathrm{max}}\|\theta_{r}\|_{1})+Rt/3})\\	&=(n_{r}+n_{c})\mathrm{exp}(-(\alpha+1)\mathrm{log}(n_{r}+n_{c})\cdot \frac{1}{\frac{2(\alpha+1)\gamma_{*}\mathrm{max}(\theta_{r,\mathrm{max}}\|\theta_{c}\|_{1},\theta_{c,\mathrm{max}}\|\theta_{r}\|_{1})\mathrm{log}(n_{r}+n_{c})}{t^{2}}+\frac{2(\alpha+1)}{3}\frac{R\mathrm{log}(n_{r}+n_{c})}{t}})\\	&=(n_{r}+n_{c})\mathrm{exp}(-(\alpha+1)\mathrm{log}(n_{r}+n_{c})\cdot \frac{1}{\frac{18}{(\sqrt{\alpha+19}+\sqrt{\alpha+1})^{2}}+\frac{2\sqrt{\alpha+1}}{\sqrt{\alpha+19}+\sqrt{\alpha+1}}\sqrt{\frac{R^{2}\mathrm{log}(n_{r}+n_{c})}{\gamma_{*}\mathrm{max}(\theta_{r,\mathrm{max}}\|\theta_{c}\|_{1},\theta_{c,\mathrm{max}}\|\theta_{r}\|_{1})}}})\\
	&\leq (n_{r}+n_{c})\mathrm{exp}(-(\alpha+1)\mathrm{log}(n_{r}+n_{c}))=\frac{1}{(n_{r}+n_{c})^{\alpha}},
\end{align*}
where we have used Assumption \ref{assump2} that $\gamma_{*}\mathrm{max}(\theta_{r,\mathrm{max}}\|\theta_{c}\|_{1},\theta_{c,\mathrm{max}}\|\theta_{r}\|_{1})\geq\tau^{2}\mathrm{log}(n_{r}+n_{c})$ in the last inequality. Thus, the claim follows.
\end{proof}
\subsection{Proof of Theorem \ref{mainBiDCDFM}}
For convenience, set $err_{d}=\|A-\Omega\|$ under $BiDCDFM(Z_{r}, Z_{c}, P, \Theta_{r}, \Theta_{c})$. Next lemma is built under the BiDCDFM model and works similarly as Lemma \ref{boundUrUcBiDFM}.
\begin{lem}\label{boundUrUcBiDCDFM}
Under $BiDCDFM(Z_{r},Z_{c},P,\Theta_{r},\Theta_{c})$, we have
\begin{align*}
	&\|\hat{U}_{r,*}\hat{O}_{*}-U_{r,*}\|_{F}\leq \frac{4\theta_{r,\mathrm{max}}\sqrt{2K_{r}n_{r,\mathrm{max}}}err_{d}}{\theta^{2}_{r,\mathrm{min}}\theta_{c,\mathrm{min}}\sigma_{K_{r}}(P)\sqrt{n_{r,\mathrm{min}}n_{c,\mathrm{min}}}},\\
	&\|\hat{U}_{c,*}\hat{O}_{*}-U_{c,*}\|_{F}\leq \frac{4\theta_{c,\mathrm{max}}\sqrt{2K_{r}n_{c,\mathrm{max}}}err_{d}}{\theta_{r,\mathrm{min}}\theta^{2}_{c,\mathrm{min}}\sigma_{K_{r}}(P)m_{V_{c}}\sqrt{n_{r,\mathrm{min}}n_{c,\mathrm{min}}}}.
\end{align*}
where $\hat{O}_{*}$ is a $K_{r}\times K_{r}$ orthogonal matrix and $m_{V_{c}}=\mathrm{min}_{1\leq k\leq K_{c}}\|V_{c}(k,:)\|_{F}$.
\end{lem}
\begin{proof}
Similar as the proof of Lemma \ref{boundUrUcBiDFM}, there exists a $K_{r}\times K_{r}$ orthogonal matrix $\hat{O}_{*}$ such that
\begin{align*}	\mathrm{max}(\|\hat{U}_{r}\hat{O}_{*}-U_{r}\|_{F},\|\hat{U}_{c}\hat{O}_{*}-U_{c}\|_{F})\leq\frac{2\sqrt{2K_{r}}\|A-\Omega\|}{\sqrt{\lambda_{K_{r}}(\Omega\Omega')}}.
\end{align*}
To obtain a lower bound of $\lambda_{K}(\Omega\Omega')$ under $BiDCDFM(Z_{r}, Z{c}, P, \Theta_{r}, \Theta_{c})$, we have
\begin{align*}	\lambda_{K_{r}}(\Omega\Omega')&=\lambda_{K_{r}}(\Theta_{r}Z_{r}PZ'_{c}\Theta^{2}_{c}Z_{c}P'Z'_{r}\Theta_{r})=\lambda_{K_{r}}(\Theta^{2}_{r}Z_{r}PZ'_{c}\Theta^{2}_{c}Z_{c}P'Z'_{r})\\
	&\geq\lambda_{K_{r}}(\Theta^{2}_{r})\lambda_{K_{r}}(Z_{r}PZ'_{c}\Theta^{2}_{c}Z_{c}P'Z'_{r})=\lambda_{K_{r}}(\Theta^{2}_{r})\lambda_{K_{r}}(Z'_{r}Z_{r}PZ'_{c}\Theta^{2}_{c}Z_{c}P')\\ &\geq\lambda_{K_{r}}(\Theta^{2}_{r})\lambda_{K_{r}}(Z'_{r}Z_{r})\lambda_{K_{r}}(PZ'_{c}\Theta^{2}_{c}Z_{c}P')=\lambda_{K_{r}}(\Theta^{2}_{r})\lambda_{K_{r}}(Z'_{r}Z_{r})\lambda_{K_{r}}(P'PZ'_{c}\Theta^{2}_{c}Z_{c})\\	&\geq\lambda_{K_{r}}(\Theta^{2}_{r})\lambda_{K_{r}}(Z'_{r}Z_{r})\lambda_{K_{r}}(P'P)\lambda_{K_{r}}(Z'_{c}\Theta^{2}_{c}Z_{c})=\lambda_{K_{r}}(\Theta^{2}_{r})\lambda_{K_{r}}(Z'_{r}Z_{r})\lambda_{K_{r}}(P'P)\lambda_{K_{r}}(Z_{c}Z'_{c}\Theta^{2}_{c})\\ &\geq\lambda_{K_{r}}(\Theta^{2}_{r})\lambda_{K_{r}}(Z'_{r}Z_{r})\lambda_{K_{r}}(P'P)\lambda_{K_{r}}(Z_{c}Z'_{c})\lambda_{K_{r}}(\Theta^{2}_{c})\geq\theta^{2}_{r,\mathrm{min}}\theta^{2}_{c,\mathrm{min}}\sigma^{2}_{K_{r}}(P)n_{r,\mathrm{min}}n_{c,\mathrm{min}},
\end{align*}
which gives that
\begin{align*}	\mathrm{max}(\|\hat{U}_{r}\hat{O}_{*}-U_{r}\|_{F},\|\hat{U}_{c}\hat{O}_{*}-U_{c}\|_{F})\leq\frac{2\sqrt{2K_{r}}err_{d}}{\theta_{r,\mathrm{min}}\theta_{c,\mathrm{min}}\sigma_{K_{r}}(P)\sqrt{n_{r,\mathrm{min}}n_{c,\mathrm{min}}}}.
\end{align*}
By Lemma F.2 in \cite{mao2018overlapping}, for $1\leq i_{r}\leq n_{r}, 1\leq j_{c}\leq n_{c}$, we have
\begin{align*}
	&\|\hat{U}_{r,*}(i_{r},:)\hat{O}_{*}-U_{r,*}(i_{r},:)\|_{F}\leq \frac{2\|\hat{U}_{r}(i_{r},:)\hat{O}_{*}-U_{r}(i_{r},:)\|_{F}}{\|U_{r}(i_{r},:)\|_{F}},\\
	&\|\hat{U}_{c,*}(j_{c},:)\hat{O}_{*}-U_{c,*}(j_{c},:)\|_{F}\leq \frac{2\|\hat{U}_{c}(j_{c},:)\hat{O}_{*}-U_{c}(j_{c},:)\|_{F}}{\|U_{c}(j_{c},:)\|_{F}}.
\end{align*}
Set $m_{r}=\mathrm{min}_{1\leq i_{r}\leq n_{r}}\|U_{r}(i_{r},:)\|_{F}$ and $m_{c}=\mathrm{min}_{1\leq j_{c}\leq n_{c}}\|U_{c}(j_{c},:)\|_{F}$, we have
\begin{align*}	&\|\hat{U}_{r,*}\hat{O}_{*}-U_{r,*}\|_{F}=\sqrt{\sum_{i_{r}=1}^{n_{r}}\|\hat{U}_{r,*}(i_{r},:)\hat{O}_{*}-U_{r,*}(i_{r},:)\|^{2}_{F}}\leq \frac{2\|\hat{U}_{r}\hat{O}_{*}-U_{r}\|_{F}}{m_{r}},\\	&\|\hat{U}_{c,*}\hat{O}_{*}-U_{c,*}\|_{F}=\sqrt{\sum_{j_{c}=1}^{n_{c}}\|\hat{U}_{c,*}(j_{c},:)\hat{O}_{*}-U_{c,*}(j_{c},:)\|^{2}_{F}}\leq \frac{2\|\hat{U}_{c}\hat{O}_{*}-U_{c}\|_{F}}{m_{c}}.
\end{align*}
Next, we provide lower bounds of $m_{r}$ and $m_{c}$ by below analysis. By the proof of Lemma \ref{populationBiDCDFM}, we have
\begin{align*}	U_{r}(i_{r},:)=\frac{\theta_{r}(i_{r})}{\|\Theta_{r}Z_{r}(:,g_{i_{r}})\|_{F}}V_{r}(g_{i_{r}},:),U_{c}(j_{c},:)=\frac{\theta_{c}(j_{c})}{\|\Theta_{c}Z_{c}(:,g_{j_{c}})\|_{F}}V_{c}(g_{j_{c}},:),
\end{align*}
which gives that
\begin{align*}	\|U_{r}(i_{r},:)\|_{F}&=\|\frac{\theta_{r}(i_{r})}{\|\Theta_{r}Z_{r}(:,g_{i_{r}})\|_{F}}V_{r}(g_{i_{r}},:)\|_{F}=\frac{\theta_{r}(i_{r})}{\|\Theta_{r}Z_{r}(:,g_{i_{r}})\|_{F}}\|V_{r}(g_{i_{r}},:)\|_{F}\\ &=\frac{\theta_{r}(i_{r})}{\|\Theta_{r}Z_{r}(:,g_{i_{r}})\|_{F}}\geq\frac{\theta_{r,\mathrm{min}}}{\theta_{r,\mathrm{max}}\sqrt{n_{r,\mathrm{max}}}},
\end{align*}
and
\begin{align*}	\|U_{c}(j_{c},:)\|_{F}&=\|\frac{\theta_{c}(j_{c})}{\|\Theta_{c}Z_{c}(:,g_{j_{c}})\|_{F}}V_{c}(g_{j_{c}},:)\|_{F}=\frac{\theta_{c}(j_{c})}{\|\Theta_{c}Z_{c}(:,g_{j_{c}})\|_{F}}\|V_{c}(g_{j_{c}},:)\|_{F}\\ &\geq\frac{\theta_{c}(j_{c})}{\|\Theta_{c}Z_{c}(:,g_{j_{c}})\|_{F}}m_{V_{c}}\geq\frac{\theta_{c,\mathrm{min}}}{\theta_{c,\mathrm{max}}\sqrt{n_{c,\mathrm{max}}}}m_{V_{c}},
\end{align*}
where we set $m_{V_{c}}=\mathrm{min}_{1\leq k\leq K_{c}}\|V_{c}(k,:)\|_{F}$ (Note that $m_{V_{c}}=1$ when $K_{r}=K_{c}$ by Lemma \ref{populationBiDCDFM}). Hence, we have $\frac{1}{m_{r}}\leq \frac{\theta_{r,\mathrm{max}}\sqrt{n_{r,\mathrm{max}}}}{\theta_{r,\mathrm{min}}}$ and $\frac{1}{m_{c}}\leq \frac{\theta_{c,\mathrm{max}}\sqrt{n_{c,\mathrm{max}}}}{\theta_{c,\mathrm{min}}m_{V_{c}}}$. Then, we have
\begin{align*}
	&\|\hat{U}_{r,*}\hat{O}_{*}-U_{r,*}\|_{F}\leq \frac{2\|\hat{U}_{r}\hat{O}_{*}-U_{r}\|_{F}}{m_{r}}\leq \frac{4\theta_{r,\mathrm{max}}\sqrt{2K_{r}n_{r,\mathrm{max}}}err_{d}}{\theta^{2}_{r,\mathrm{min}}\theta_{c,\mathrm{min}}\sigma_{K_{r}}(P)\sqrt{n_{r,\mathrm{min}}n_{c,\mathrm{min}}}},\\
	&\|\hat{U}_{c,*}\hat{O}_{*}-U_{c,*}\|_{F}\leq \frac{2\|\hat{U}_{c}\hat{O}_{*}-U_{c}\|_{F}}{m_{c}}\leq \frac{4\theta_{c,\mathrm{max}}\sqrt{2K_{r}n_{c,\mathrm{max}}}err_{d}}{\theta_{r,\mathrm{min}}\theta^{2}_{c,\mathrm{min}}\sigma_{K_{r}}(P)m_{V_{c}}\sqrt{n_{r,\mathrm{min}}n_{c,\mathrm{min}}}}.
\end{align*}
\end{proof}
\begin{proof}
Now, we start the proof of the Theorem \ref{mainBiDCDFM}. For $\hat{f}_{r}$, let $\varsigma_{r}>0$ be a small quantity, by Lemma 2 in \cite{joseph2016impact} and Lemma \ref{populationBiDCDFM}, if
\begin{align}\label{hold}	\frac{\sqrt{K_{r}}}{\varsigma_{r}}\|U_{r,*}-\hat{U}_{r,*}\hat{O}_{*}\|_{F}(\frac{1}{\sqrt{n_{r,k}}}+\frac{1}{\sqrt{n_{r,l}}})\leq \sqrt{2}, \mathrm{~for~each~}1\leq k\neq l\leq K_{r},
\end{align}
we have  $\hat{f}_{r}=O(\varsigma^{2}_{r})$. Set $\varsigma_{r}=\sqrt{\frac{2K_{r}}{n_{r,\mathrm{min}}}}\|U_{r,*}-\hat{U}_{r,*}\hat{O}_{*}\|_{F}$, we have Eq (\ref{hold}) holds for all $1\leq k\neq l\leq K_{r}$. Hence, $\hat{f}_{r}=O(\varsigma^{2}_{r})=O(\frac{K_{r}\|U_{r,*}-\hat{U}_{r,*}\hat{O}_{*}\|^{2}_{F}}{n_{r,\mathrm{min}}})$,  combine it with Lemma \ref{boundUrUcBiDCDFM}, we have
\begin{align*} \hat{f}_{r}=O(\frac{\theta^{2}_{r,\mathrm{max}}K^{2}_{r}n_{r,\mathrm{max}}err^{2}_{d}}{\theta^{4}_{r,\mathrm{min}}\theta^{2}_{c,\mathrm{min}}\sigma^{2}_{K_{r}}(P)n^{2}_{r,\mathrm{min}}n_{c,\mathrm{min}}}).
\end{align*}
Similarly, for $\hat{f}_{c}$, let $\varsigma_{c}>0$, by Lemma 2 in \cite{joseph2016impact} and Lemma \ref{populationBiDCDFM}, if
\begin{align*}	\frac{\sqrt{K_{c}}}{\varsigma_{c}}\|U_{c,*}-\hat{U}_{c,*}\hat{O}_{*}\|_{F}(\frac{1}{\sqrt{n_{c,k}}}+\frac{1}{\sqrt{n_{c,l}}})\leq \|V_{c}(k,:)-V_{c}(l,:)\|_{F}, \mathrm{~for~each~}1\leq k\neq l\leq K_{c},
\end{align*}
we have  $\hat{f}_{c}=O(\varsigma^{2}_{c})$. Set $\varsigma_{c}=\frac{2}{\delta_{c,*}}\sqrt{\frac{K_{c}}{n_{c,\mathrm{min}}}}\|U_{c,*}-\hat{U}_{c,*}\hat{O}_{*}\|_{F}$, then we have $\hat{f}_{c}=O(\varsigma^{2}_{c})=O(\frac{K_{c}\|U_{c,*}-\hat{U}_{c,*}\hat{O}_{*}\|^{2}_{F}}{\delta^{2}_{c,*}n_{c,\mathrm{min}}})$,  combine it with Lemma \ref{boundUrUcBiDCDFM}, we have
\begin{align*}
	\hat{f}_{c}=O(\frac{\theta^{2}_{c,\mathrm{max}}K_{r}K_{c}n_{c,\mathrm{max}}err^{2}_{d}}{\theta^{2}_{r,\mathrm{min}}\theta^{4}_{c,\mathrm{min}}\sigma^{2}_{K}(P)\delta^{2}_{c,*}m^{2}_{V_{c}}n_{r,\mathrm{min}}n^{2}_{c,\mathrm{min}}}).
\end{align*}
By Lemma \ref{boundABiDCDFM}, when Assumption \ref{assump2} holds, for any $\alpha>0$, with probability at least $1-o((n_{r}+n_{c})^{-\alpha})$, we have
\begin{align*}
	&\hat{f}_{r}=O(\gamma_{*}\frac{\theta^{2}_{r,\mathrm{max}}K^{2}_{r}n_{r,\mathrm{max}}\mathrm{~max}(\theta_{r,\mathrm{max}}\|\theta_{c}\|_{1},\theta_{c,\mathrm{max}}\|\theta_{r}\|_{1})\mathrm{log}(n_{r}+n_{c})}{\theta^{4}_{r,\mathrm{min}}\theta^{2}_{c,\mathrm{min}}\sigma^{2}_{K_{r}}(P)n^{2}_{r,\mathrm{min}}n_{c,\mathrm{min}}}),\\
	&\hat{f}_{c}=O(\gamma_{*}\frac{\theta^{2}_{c,\mathrm{max}}K_{r}K_{c}n_{c,\mathrm{max}}\mathrm{~max}(\theta_{r,\mathrm{max}}\|\theta_{c}\|_{1},\theta_{c,\mathrm{max}}\|\theta_{r}\|_{1})\mathrm{log}(n_{r}+n_{c})}{\theta^{2}_{r,\mathrm{min}}\theta^{4}_{c,\mathrm{min}}\sigma^{2}_{K_{r}}(P)\delta^{2}_{c,*}m^{2}_{V_{c}}n_{r,\mathrm{min}}n^{2}_{c,\mathrm{min}}}).
\end{align*}
\end{proof}
\bibliographystyle{agsm}
\bibliography{refBiDFMs}
\end{document}